\documentclass[10pt,a4paper]{article}

\usepackage{times}
\usepackage{soul}
\usepackage{url}
\usepackage[hidelinks]{hyperref}
\usepackage[utf8]{inputenc}
\usepackage[small]{caption}
\usepackage{graphicx}
\usepackage{amsmath}
\usepackage{amsthm}
\usepackage{booktabs}
\usepackage{multirow}
\usepackage{todonotes}

\usepackage[ruled,linesnumbered,vlined,scleft]{algorithm2e}
\usepackage{amssymb}
\usepackage{bbm} 
\usepackage{comment}
\usepackage{mathrsfs} 
\usepackage{stmaryrd}
\usepackage{tikz} 
\usetikzlibrary{shapes.misc,automata,positioning,arrows}
\usepgflibrary{shapes.arrows} 

\theoremstyle{definition}

\newcommand{\Llang}{\ensuremath{\mathcal{L}}}

\newcommand{\N}{\ensuremath{\mathcal{N}}}
\newcommand{\Nanon}{\ensuremath{\N_\mathtt{anon}}}

\newcommand{\nd}{\noindent}

\newcommand{\BM}{\mathcal{M}}

\newcommand{\assigned}{\mathrel{\mathop:}=}
\newcommand{\grammar}{\mathrel{\mathop:}\assigned}
\newcommand{\defined}{\:\raisebox{1ex}{\scalebox{0.5}{\ensuremath{\mathrm{def}}}}\hskip-1.65ex\raisebox{-0.1ex}{\ensuremath{=}}\:}

\newcommand{\dllite}{\ensuremath{\text{\emph{DL-Lite}}}}
\newcommand{\dllitehorn}{\ensuremath{\dllite_{horn}}}
\newcommand{\dllitecore}{\ensuremath{\dllite_{core}}}
\newcommand{\dllitehornH}{\ensuremath{\dllite_{horn}^\mathcal{H}}}
\newcommand{\dllitecoreH}{\ensuremath{\dllite_{core}^\mathcal{H}}}

\newcommand{\dlAnd}{\sqcap}
\newcommand{\dlOr}{\sqcup}

\newcommand{\ELbot}{\ensuremath{\mathcal{E}\mathcal{L}_\bot}}
\newcommand{\ELObot}{\ensuremath{\mathcal{E}\mathcal{L}\mathcal{O}_\bot}}

\newcommand{\rationalent}{\entails_R}
\newcommand{\minent}{\entails_R^{\leq}}

\newcommand{\entabox}{\minent}

\renewcommand{\vec}[1]{\mathbf{#1}}

\newcommand{\ie}{\mbox{i.e.}}
\newcommand{\eg}{\mbox{e.g.}}
\newcommand{\cf}{\mbox{cf.}}
\newcommand{\viz}{\mbox{viz.}}
\newcommand{\wrt}{\mbox{w.r.t.}}
\newcommand{\sst}{\mbox{s.t.}}

\newcommand{\tuple}[1]{\langle #1 \rangle}

\newtheorem{definition}{Definition}
\newtheorem{lemma}{Lemma}
\newtheorem{theorem}{Theorem}
\newtheorem{proposition}{Proposition}
\newtheorem{corollary}{Corollary}
\newtheorem{example}{Example}
\newtheorem{remark}{Remark}

\newcommand{\T}{\ensuremath{\mathcal{T}}} 
\newcommand{\KB}{\ensuremath{\mathcal{K}}} 
\newcommand{\D}{\ensuremath{\mathcal{D}}}
\newcommand{\calS}{\ensuremath{\mathcal{S}}}

\newcommand{\sat}{\Vdash}

\newcommand{\entails}{\models}
\newcommand{\nentails}{\not\entails}

\newcommand{\M}{\mathscr{M}}

\newcommand{\dland}{\sqcap}
\newcommand{\dlor}{\sqcup}
\newcommand{\subs}{\sqsubseteq}

\newcommand{\I}{\ensuremath{\mathcal{I}}}
\newcommand{\J}{\ensuremath{\mathcal{J}}}
\newcommand{\Dom}{\Delta}
\newcommand{\ALC}{\ensuremath{\mathcal{ALC}}}
\newcommand{\EL}{\ensuremath{\mathcal{EL}}}

\newcommand{\Moddelta}[1]{{\ensuremath{\text{\it Mod}_{\Delta}(#1)}}}

\newcommand{\twiddle}{\hspace{-0.04cm}\mathrel|\joinrel\sim\hspace{-0.05cm}}

\newcommand{\usually}{\:\raisebox{0.45ex}{\ensuremath{\sqsubset}}\hskip-1.7ex\raisebox{-0.6ex}{\scalebox{0.9}{\ensuremath{\sim}}}\:}
\newcommand{\dsubs}{\:\raisebox{0.45ex}{\ensuremath{\sqsubset}}\hskip-1.7ex\raisebox{-0.6ex}{\scalebox{0.9}{\ensuremath{\sim}}}\:}

\newcommand{\pref}{\prec} 

\newcommand{\rk}{\ensuremath{\text{\it rk}}} 

\newcommand{\calG}{\mathcal{G}}

\newcommand{\KBi}[1]{\ensuremath{\KB^{<\infty}}}
\newcommand{\KBiAB}[1]{\ensuremath{\KB^{A\twiddle\lnot B\mid <\infty }}}

\newcommand{\A}{\ensuremath{\mathcal{A}}}

\newcommand{\E}{\ensuremath{\mathcal{E}}}
\newcommand{\R}{\ensuremath{\mathcal{R}}}
\newcommand{\RR}{\ensuremath{\mathcal{D}_\mathtt{ranked}}}

\newcommand{\rank}[2]{\textit{rank}_{#2}(#1)}

\newcommand{\cass}[2]{\mbox{$#1$:$#2$}}
\newcommand{\rass}[3]{\mbox{$(#1,#2)$:$#3$}}
\newcommand{\dass}[2]{\mbox{$#1$:$#2$}}

\newcommand{\andc}{\sqcap}

\newcommand{\some}{\exists}
\newcommand{\notc}{\neg}

\newcommand{\csome}{\exists}

\newcommand{\impc}{\sqsubseteq}

\newcommand{\spec}{{\mbox{$\tau$}}}
\newcommand{\defspec}{{\mbox{$\tau$}}}

\newcommand{\ii}[1]{\mbox{$(#1)$}}
\newcommand{\entail}{\models}

\newcommand{\RI}{\ensuremath{\mathcal{R}}} 
\newcommand{\OI}{\ensuremath{\mathcal{O}}} 

\newcommand{\df}[1]{\textit{#1}}
\renewcommand{\rank}{\ensuremath{\mathsf{rank}}} 

\newcommand{\norm}[1]{\mathsf{N}#1}

\newcommand{\clncH}{\ensuremath{{_{core}^\mathcal{H}} cln}}
\newcommand{\clnhH}{\ensuremath{{_{horn}^\mathcal{H}} cln}}
\newcommand{\clnd}{\ensuremath{{_{horn}^\mathcal{H}} cln}}
\newcommand{\clnn}{\ensuremath{\norm{cln}}}

\newcommand{\chasecH}{\ensuremath{{_{core}^\mathcal{H}}chase}}
\newcommand{\chasehH}{\ensuremath{{_{horn}^\mathcal{H}}chase}}

\newcommand{\dchasecH}{\ensuremath{{_{core}^\mathcal{H}}dchase}}
\newcommand{\dchasehH}{\ensuremath{{_{horn}^\mathcal{H}}dchase}}

\newcommand{\cancH}{\ensuremath{{_{core}^\mathcal{H}}can}}
\newcommand{\canhH}{\ensuremath{{_{horn}^\mathcal{H}}can}}

\newcommand{\CcH}{\ensuremath{{_{core}^\mathcal{H}} C}}
\newcommand{\ChH}{\ensuremath{{_{horn}^\mathcal{H}} C}}
\newcommand{\cH}[1]{\ensuremath{{_{core}^\mathcal{H}} {#1}}}
\newcommand{\hH}[1]{\ensuremath{{_{horn}^\mathcal{H}} {#1}}}

\newcommand{\cHRef}{\ensuremath{\cH{Ref}}}
\newcommand{\cHContr}{\ensuremath{\cH{Contr}}}
\newcommand{\cHTrans}{\ensuremath{\cH{Trans}}}
\newcommand{\RcH}{\ensuremath{{_{core}^\mathcal{H}} R}}
\newcommand{\nSup}{\ensuremath{\mathbf{\norm{Sup}}}}

\newcommand{\hHRef}{\ensuremath{\hH{Ref}}}
\newcommand{\hHContr}{\ensuremath{\hH{Contr}}}
\newcommand{\hHTrans}{\ensuremath{\hH{Trans}}}
\newcommand{\RhH}{\ensuremath{{_{horn}^\mathcal{H}} R}}

\usepackage{authblk}

\title{\Large\bfseries Tractable Reasoning and Conjunctive Query Answering for Defeasible \dllite~under Rational Closure}

\author[1,2]{Giovanni Casini}
\author[1]{Umberto Straccia}

\affil[1]{Istituto di Scienza e Tecnologie dell'Informazione, CNR - ISTI, Pisa, Italy}
\affil[2]{University of Cape Town and CAIR, Cape Town, South Africa}

\date{}

\begin{document}
\maketitle

\begin{abstract}
In Description Logics (DLs), reasoning under Rational Closure (RC) is a well-known 
and widely accepted non-monotonic formalism to handle defeasible knowledge. 
In this paper, we study the application of RC to the core and horn variants of 
the \dllite{} family of lightweight description logics. We analyze both entitlement 
(instance checking) and Conjunctive Query (CQ) answering under RC. Our main 
contribution is providing a plug-in architecture that builds upon existing standard 
classical reasoners, establishing that reasoning and CQ answering under RC for 
\dllite{} can be done efficiently with minimal computational overhead.
\end{abstract}

\section{Introduction}

\nd \emph{Description Logics}~(DLs) under \emph{Rational Closure}~(RC) \cite{BritzEtAl2015a,CasiniEtAl2014,CasiniEtAl2013b,CasiniEtAl2015,CasiniStraccia2010,CasiniStraccia13,CasiniStraccia14,Casini19,Giordano16,Giordano16a,Giordano16b,GiordanoEtAl2012a,Giordano14a,Giordano14,Giordano15,GiordanoEtAl2012b,GiordanoEtAl2013a} are a framework for non-monotonic reasoning that satisfies well-known and widely accepted reasoning patterns of conditional reasoning~\cite{KrausEtAl1990}. RC has some interesting properties: the conclusions are intuitive, and the decision procedure can be reduced to a series of classical decision problems, sometimes preserving the computational complexity of the underlying classical decision problem.


On the other hand, \dllite~languages (see, \eg~\cite{ArtaleEtAl2009,Calvanese05,Calvanese06a,Calvanese13}), of which $\dllite_\mathcal{R}$ (\viz~\dllitecoreH) is the basis of the OWL 2 QL profile~\cite{Corona09},\footnote{\url{https://www.w3.org/TR/owl2-profiles/\#OWL_2_QL}} have been specifically tailored to capture  main features necessary to express conceptual models such as UML class diagrams, while keeping low the complexity of reasoning tasks such as answering \emph{Conjunctive Queries} (CQs) over a large set of instances maintained in a relational database. 

In this paper, we enrich \dllitecoreH~and \dllitehornH, which are prominent language representatives of the \dllite~family, with the ability to express defeasible class inclusions \eg~of the form $B\dsubs A$, read as ``usually, an object in $B$ is also in $A$'' under the well-known RC semantics. 

A major reasoning task we focus on is that of answering CQs over a set of defeasible and non-defeasible inclusion axioms, and a set of instances maintained in a relational database. 
In fact, we will show that \dllitecoreH~and \dllitehornH, even in the presence of defeasible class inclusion axioms, inherit the low computational complexity results of their non-defeasible analogues. Specifically, we show that CQ answering remains in $\textsf{AC}^0$ \wrt~data complexity. To do so, we show that the reasoning (\viz~satisfiability testing) and CQ answering  algorithms are essentially the same as those for the classical case, with some appropriate adjustments.
In particular, the query answering procedure is based on a query rewriting strategy allowing one to rewrite a given CQ into a set of CQs, which  are then submitted via SQL to a database. Therefore, our results pave the way for a low-complexity CQ answering algorithm for defeasible OWL~QL under RC and its implementation in current \emph{Ontology Based Data Access} (OBDA) systems such as, \eg~\textsf{Ontop}\footnote{\url{https://ontop-vkg.org}} and  \textsf{Mastro}.\footnote{\url{https://obdm.obdasystems.com/mastro/}}

While various works, as cited above, address defeasible DLs under RC, none of them, to the best of our knowledge, address \dllitecoreH~and \dllitehornH~ and also provide a CQ answering procedure via a query reformulation (FOL rewritable) whose data complexity is $\textsf{AC}^0$.

In summary, our contribution is as follows:
\begin{enumerate}
    \item we present defeasible \dllitecoreH~and \dllitehornH~with defeasible class inclusions of the form $\{B_1, \ldots, B_n\}\dsubs A$ ($n=1$ for \dllitecoreH) read as ``usually, an object in the conjunction of the  $B_i$'s is also in $A$'' under RC semantics;

    \item we show that both defeasible \dllitecoreH~and \dllitehornH, unlike other defeasible DLs under RC, have the unique extension property;

    \item we present reasoning algorithms for defeasible \dllitecoreH~and \dllitehornH~under RC addressing the \emph{Knowledge Base} (KB) satisfiability,  the subsumption and the instance checking problems. We anticipate that defeasible \dllitehornH~will also help us in proving the correctness of our reasoning algorithms for \dllitecoreH;

    \item we present,  for defeasible \dllitecoreH~and \dllitehornH~under RC, query reformulation procedures that are similar as for their non-defeasible variants; and 

    \item we show that defeasible \dllitecoreH~and \dllitehornH~retain the low computational complexity results of their non-defeasible analogues.
\end{enumerate}


\nd In the following, we proceed as follows. In the next section, for completeness, we introduce the necessary definitions and algorithms related to \dllitecoreH~and \dllitehornH~used in this work. In Section~\ref{defdllite} we show how to extend these languages with defeasible inclusion axioms under the RC semantics and illustrate some basic reasoning algorithms. In Section~\ref{sect_QRW} we show how to compute the answers to a CQ under RC and in Section~\ref{compplex} we analyse the computational complexity of all reasoning procedures. Section~\ref{related} addresses related work, while in Section~\ref{concl} we succinctly summarise our contribution and point to some topics for future research. The appendixes provide proofs of main propositions and some procedures used within the paper.

\section{Background} \label{backgr}

\noindent\fbox{\begin{minipage}{0.975\textwidth}
\textsc{Overview:} This section recalls the basic notions of the classical formal framework to prepare the defeasible  \dllite~in Section~\ref{defdllite}.
 \begin{itemize}
     \item Sections \ref{sect_back_lang} and \ref{reassect}: it introduces the  \dllite~languages \dllitecoreH~and \dllitehornH~and classical reasoning problems in them. Proposition \ref{nominalABox} is a novel contribution.
     \item Section \ref{sect_chase_niclos}:  it introduces chases and NI-closures for \dllitecoreH~and \dllitehornH. The \dllitehornH~version is a novel contribution. 
 \end{itemize}
\end{minipage}}

\subsection{\dllite~languages}\label{sect_back_lang}

\nd The DL-Lite languages we consider here are defined as follows~\cite{ArtaleEtAl2009,Calvanese05,Calvanese06a,Calvanese07,Calvanese13}. 
We assume a finite set~
$\N$ of \emph{individuals}, a finite set~$\mathcal{C}$ of \emph{concept names}, and a finite set~$\mathcal{R}$ of \emph{role names}. Individuals are denoted by~$a$, concept names by~$A$, and role names by~$P$,  possibly with super- and sub-scripts.

The language $\dllitecore$ allows for \emph{complex roles}, denoted by~$R$, and \emph{complex concepts}, denoted by~$C$ (informally, a concept denotes a set of objects, while a role denotes a set of pairs of objects).
These are recursively defined according to the grammar:
\begin{eqnarray}
    R & \grammar & P \mid P^- \nonumber \\ 
    B & \grammar & A \mid \exists R \\ \label{dllite}
    C & \grammar & B \mid \lnot B \ ,\nonumber 
\end{eqnarray}
where $A\in\mathcal{C}$, $P\in\mathcal{R}$, and $P^-$ denotes the \emph{inverse} of~$P$. Concepts of the type~$B$ are called \emph{basic}. Notice that the existential restriction in the definition of~$B$ is \emph{unqualified}. By convention, $R^-\defined P$ if $R=P^-$, and $R^-\defined P^-$ if $R=P$.

A $\dllitecore$ \emph{TBox}, denoted~$\mathcal{T}$, is a finite set of \emph{concept inclusion axioms} of the form:
\begin{equation} \label{inclaxiom}
B \subs C \ .
\end{equation}

\nd Informally, $B \subs C$ states  that an instance of $B$ is also an instance of $C$.
An inclusion axiom of the form $B_1\subs B_2$ is called a \emph{positive inclusion}~(PI), whereas an inclusion axiom of the form $B_1\subs\lnot B_2$ is called a \emph{negative inclusion}~(NI).

A $\dllitecore$ \emph{ABox}, denoted~$\mathcal{A}$, is a finite set of \emph{assertion axioms} of the form:
\begin{equation} \label{assaxiom}
\cass{a}{A} \quad \text{ and } \quad \rass{a_1}{a_2}{P} \ .
\end{equation}

\nd The former indicates that `$a$ is an instance of $A$', while the latter indicates that `$a_1$ is related to $a_2$ via role $P$'. A $\dllitecore$ \emph{knowledge base} (KB) is a tuple $\KB\defined\tuple{\T,\A}$, where~$\T$ is a TBox and~$\A$ is an ABox. Each of the two components may be omitted, with the understanding that the missing component is empty. With $\T_{PI}$ (resp.~$\T_{NI}$) we denote the set of~PIs (resp.~NIs) in~$\T$ and with $\N_\KB$, or also with $\N_\A$, we denote the set of individuals that occur in $\KB$, \viz, occur in $\A$.

As for the semantics, an \emph{interpretation} is a pair $\I\defined\tuple{\Delta^\I,\cdot^\I}$, where~$\Delta^\I\neq\emptyset$ is the \emph{interpretation domain} and $\cdot^\I$ is an \emph{interpretation function} mapping (\emph{i})~every concept name~$A$ into a set~$A^\I \subseteq  \Delta^\I$; (\emph{ii})~every role name~$P$ into a set~$P^\I \subseteq  \Delta^\I \times \Delta^\I$, and (\emph{iii})~every individual $a$ into an element $a^\I \in \Delta^\I$, in such a way that $a_{1}^\I \neq a_{2}^\I$ whenever $a_{1}\neq a_{2}$, \ie, we adopt the \emph{unique name assumption} (UNA).

Given an interpretation $\I=\tuple{\Delta^\I,\cdot^\I}$, the interpretation function~$\cdot^\I$ is extended to interpret complex role and concept expressions in the following way:
\begin{eqnarray}
(P^-)^\I & \defined & \{(y,x) \mid (x,y)\in P^\I\} \nonumber\\
(\exists R)^\I & \defined & \{x \mid R^\I(x)\neq\emptyset\}\label{dllitsem}\\
(\neg B)^\I & \defined & \Delta^\I\setminus B^\I  \ . \nonumber
\end{eqnarray}

\nd We say an interpretation~$\I=\tuple{\Delta^\I,\cdot^\I}$ \emph{satisfies} a concept $C$ if $C^\I\neq \emptyset$. $\I$ \emph{satisfies} (\emph{is a model of}) the concept inclusion axiom $B\subs C$, denoted $\I\sat B\subs C$, if $B^\I \subseteq C^\I$. $\I$ \emph{satisfies} (\emph{is a model of}) the concept (resp.~role) assertion $\cass{a}{A}$ (resp.~$\rass{a_1}{a_2}{P}$), denoted $\I\sat\cass{a}{A}$ (resp.~$\I\sat\rass{a_1}{a_2}{P}$), if $a^\I\in A^\I$ (resp.~$(a_1^\I,a_2^\I)\in P^I$). $\I$ \emph{satisfies} (\emph{is a model of}) a set of axioms if it satisfies each axiom in the set. $\I$ \emph{satisfies} (\emph{is a model of}) a knowledge base $\KB=\tuple{\T,\A}$ if it is a model of both~$\T$ and~$\A$. We say~$\KB$ \emph{entails} an axiom~$\alpha$, denoted~$\KB\entails\alpha$, if every model of~$\KB$ is a model of~$\alpha$. We say two sets of axioms are \emph{equivalent} if they have exactly the same models.

It is widely known that $\dllitecore$ allows for the expression of several essential constructs of both~ER and~UML class diagrams. Among these are the \emph{is-a} relationship between classes ($A_1\subs A_2$), \emph{class disjointness} ($A_1\subs\lnot A_2$), the \emph{domain} and \emph{range} of a relation ($\exists P\subs A$ and $\exists P^-\subs A$, respectively), and also \emph{mandatory participation} ($A\subs\exists P$ and $A\subs\exists P^-$). Note that $\dllitecoreH$ is the basis of the OWL~2~QL profile~\cite{Corona09}.\footnote{\url{http://www.w3.org/TR/owl2-profiles}}


The language of $\dllitehorn$ extends that of~$\dllitecore$ by  allowing also concept inclusion axioms of the following form, where $n\geq 1$:
\begin{equation}\label{hornincl}
\{B_1, \ldots, B_n\} \subs C \ ,
\end{equation}

\nd where $\{B_1, \ldots, B_n\}$ is interpreted as the conjunction of the $B_i$. We may also write
$B_1 \andc \ldots \andc B_n \subs C$ in place of $\{B_1, \ldots, B_n\} \subs C$.
That is, $\dllitehorn$ allows for conjunctions of basic concepts of arbitrary (finite) length on the~LHS of inclusion axioms. Semantically, given an interpretation~$\I=\tuple{\Delta^\I,\cdot^\I}$, $\{B_1, \ldots, B_n\}^\I \defined (B_1 \andc \ldots \andc B_n)^\I\defined\bigcap_{1\leq k\leq n}B_{k}^\I$. Satisfaction of $\dllitehorn$ inclusion axioms carries over from~$\dllitecore$. 

For $\mathcal{L}\in\{\dllitecore,\dllitehorn\}$, with~$\mathcal{L}^\mathcal{H}$ we denote~$\mathcal{L}$ extended with \emph{role inclusion axioms} of the form
\begin{equation} \label{roleincl}
R_1\subs R_2 
\end{equation}
with obvious semantics: given~$\I=\tuple{\Delta^\I,\cdot^\I}$, we say~$\I$ \emph{satisfies} (\emph{is a model of}) $R_1\subs R_2$, denoted $\I\sat R_1\subs R_2$, if $R_1^\I\subseteq R_2^\I$. 

An \emph{inclusion axiom} (IN) is either a concept or role inclusion axiom.


\begin{remark}\label{roleinc}
    For convenience, we will assume that whenever a role inclusion 
    $R_1 \subs R_2$ occurs in a KB, it also contains the INs 
    $\exists R_1 \subs \exists R_2$ and  $\exists R_1^- \subs \exists R_2^-$.
\end{remark}

\begin{remark} \label{remexpdllitecore}
We recall that the languages \dllitecoreH, \dllitehorn, and \dllitehornH~are also known, respectively, as~$\dllite_\mathcal{R}$, $\dllite_\dlAnd$, and $\dllite_{\mathcal{R},\dlAnd}$. Please note that we may  allow (\emph{i})~concept disjunctions of the form $B_1 \dlOr B_2$, with semantics $(B_1 \dlOr B_2)^\I\defined B_1^\I \cup B_2^\I$, on the LHS of concept inclusion axioms, and (\emph{ii})~concept conjunctions of the form $C_1 \dlAnd C_2$ on the RHS of concept inclusion axioms. We note that this does not increase the expressiveness of the language as $\{B_1\dlOr B_2\subs C\}$ is equivalent to $\{B_1\subs C, B_2\subs C\}$, and $\{B\subs C_1\dlAnd C_2\}$ is equivalent to $\{B\subs C_1, B\subs C_2\}$. Analogously, we may allow for the concept~$\bot$, whose semantics is given by~$\bot^\I\defined\emptyset$, on the RHS of a concept inclusion axiom, as
\begin{itemize}
    \item an inclusion $B\subs\bot$ can be expressed as the \dllitecoreH~inclusion $B\subs\lnot B$;
    \item an inclusion  $\{B_1,\ldots, B_n\}\subs \bot$ can be expressed as the \dllitehornH~inclusion $\{B_1,\ldots, B_{n-1}\}\subs\lnot B_n$. 
\end{itemize}

\nd Additionally,  we can also allow qualified existential restrictions of the form $\exists R.C$, with semantics given by $(\exists R.C)^\I\defined\{o \mid R^\I(o)\cap C^{\I}\neq\emptyset\}$, on the RHS of concept inclusions as $\{B\subs\exists R.C\}$ is equivalent to $\{B\subs\exists P, P\subs R,\exists P^-\subs C\}$, where $P$ is a new role name. Eventually, a symmetric role can be enforced via the role inclusion axiom $R \subs R^-$.

However, for ease of presentation, we disallow the negation of roles in the RHS of role inclusions, as it has been done in~\cite{Calvanese06a} for $\dllite_{\mathcal{R},\dlAnd}$. Negative role inclusions have the form $R_1 \subs \neg R_2$, where, for an interpretation~$\I=\tuple{\Delta^\I,\cdot^\I}$, $(\neg R)^\I = (\Delta^\I \times \Delta^\I) \setminus R^\I$ and allow one to declare the disjointness of roles, an empty role via $R \subs \neg R$ and an asymmetric role via $R \subs \neg R^-$. 
How to deal with them and the additional constructs allowed by OWL QL, such as \emph{reflexive} and \emph{irreflexive} roles, not presented here, is illustrated in~\cite{CalvaneseEtAl2007,Calvanese13,Corona09}.

\end{remark} 


\begin{remark}[Finite Model Property] \label{finitedllitehornH}
Note that, according to~\cite{Calvanese07,Rosati06,Rosati08}, the languages \dllitecoreH~and, \dllitehornH~enjoy the finite model property.
\end{remark}

\subsection{Reasoning Problems} \label{reassect}

\nd We now summarise the standard reasoning tasks within the~$\dllite$ family of languages presented thus far. To start with, for a $\dllite$ language $\mathcal{L}$, we call $\mathcal{L}$-\emph{concept inclusion} any concept inclusion allowed by~$\mathcal{L}$. The notions of $\mathcal{L}$-\emph{TBox} and $\mathcal{L}$-\emph{KB} follow immediately. A concept occurring on the RHS of an $\mathcal{L}$-concept inclusion or a conjunction of such concepts will be called an~$\mathcal{L}$-\emph{concept}.
Note that \dllitecoreH~is a sublanguage of \dllitehornH, so properties proved for \dllitehornH~hold  for \dllitecoreH~as well.


\subsubsection{Subsumption} \label{subssect}

\nd Given an~$\Llang$-TBox~$\T$ and an~$\Llang$-inclusion~$C_1\subs C_2$, the \emph{subsumption problem} consists in checking whether $\T\entails C_1\subs C_2$. Subsumption can be reduced to~KB satisfiability as follows. Let~$A$ be a new concept name and let~$a$ be a new individual name. Furthermore, consider $\KB=\tuple{\T',\A}$ with
\begin{equation}\label{subsreduction}
\begin{array}{lcl}
\T' & = & \T \cup \{A \subs C_1, A \subs \neg C_2 \} \\
\A & = & \{\cass{a}{A} \} \ .
\end{array}
\end{equation}
It can easily be verified that $\T\entails C_1\subs C_2$ iff~$\KB$ is \emph{unsatisfiable}. For core and Horn~KBs, if $C_2=D_1\dlAnd D_2$, where each~$D_i$ is a (possibly negated) basic concept, checking the satisfiability of~$\KB$ amounts to checking the satisfiability of each of the KBs~$\KB_i=\tuple{\T_i,\A}$, where $\T_i=\T\cup\{A\subs C_1, A\subs\lnot D_i\}$. 



\subsubsection{Concept satisfiability} \label{csubssect}

\nd The \emph{concept satisfiability problem}, given a $\Llang$-TBox $\T$ and an $\Llang$-concept $C$, consists in checking  whether $C^\I \neq \emptyset$ in a model $\I$ of $\T$. This problem is also reducible to the KB satisfiability problem. Indeed, consider a new concept name $A$ and new individual $a$, and set $\KB= \tuple{\T',\A}$, where
\begin{equation} \label{csatreduction}
\begin{array}{lcl}
\T' & = & \T \cup \{A \subs C\} \\
\A & = & \{\cass{a}{A} \} \ .
\end{array}
\end{equation}

\nd Then $C$ is satisfiable \wrt~$\T$ iff $\KB$ is satisfiable.

\subsubsection{Instance checking} \label{icssect}

\nd The \emph{instance checking problem} is to decide, given an individual  $a$,
an $\Llang$-concept $C$ and an $\Llang$-KB $\KB = \tuple{\T,\A}$, whether $\KB \models \cass{a}{C}$. 
Instance checking is also reducible to KB (un)satisfiability: in fact, consider a new concept name $A$ and set 
$\KB' = \tuple{\T', \A'}$, where
\begin{equation} \label{instreduction}
\begin{array}{lcl}
\T' & := & \T \cup \{A \subs \notc C\} \\
\A' & := & \A \cup \{\cass{a}{A} \} \ .
\end{array}
\end{equation}

\nd Then, it is easy to see that $\KB \models \cass{a}{C}$ iff $\KB'$ is unsatisfiable.

\begin{remark}\label{qansinstcheck}
    Note that the instance checking problem can also be seen as a special case of \emph{query answering} (see also \eg~\cite{Calvanese07} or 
    Section~\ref{sect_QRW}). That is,  an individual $a$ is an instance of concept $B$ \wrt~a KB $\KB$ iﬀ $a$ is in the \emph{answer set} of the  \emph{Conjunctive Query} (CQ) $q(x) \leftarrow A(x)$ \wrt~ $\KB'$, 
    where $\KB' = \tuple{\T',\A}$ and $\T' = \T \cup \{B \subs A\}$, with $A$ a new concept name.
    On the other hand, an individual $a$ is an instance of concept $\neg B$ \wrt~a KB $\KB$ iﬀ $a$ is in the answer set
    of the union of CQs  $\vec{q} = \{q(x) \leftarrow B'(x) \mid B' \subs \neg B \text{ in the NI-closure of } \T \}$ \wrt~ $\KB$ (the notion of NI-closure is defined later on in Section~\ref{sect_chase_niclos}). In a similar way,  $\KB \models \rass{a}{b}{P}$ iff $\tuple{a,b}$ is in the answer set of the query $q(x,y) \leftarrow P(x,y)$ \wrt~ $\KB$.
\end{remark}

\subsubsection{KB satisfiability} \label{kbsatsect}

\nd Eventually, it remains to show how to check whether a KB is \emph{satisfiable} or not. Specifically,
given an $\Llang$-KB~$\KB$, the \emph{KB satisfiability problem} consists in checking whether there is a model of~$\KB$. 
Please note that a $\Llang$-TBox~$\T$ is always satisfiable.
\begin{proposition} \label{tsat}
    Any \dllitehornH~TBox~$\T$ is satisfiable.
\end{proposition}
\begin{proof}
    Just consider an interpretation mapping concept names and role names to the empty set.
\end{proof}


\nd In order to determine whether a \dllitecoreH~(resp.~\dllitehornH) KB $\KB$  is satisfiable or not, various methods have been proposed, which we illustrate next for completeness (see also~\cite{Calvanese05,Calvanese06a,Calvanese07}).
Note however that these methods refer to the notion of $NI$-closure, that will be introduced in the coming Section \ref{sect_chase_niclos}.

So, let $\KB=\tuple{\T,\A}$ be a KB. In~\cite{Calvanese05}, at first, $\KB$ is expanded in the following way: 
%
\begin{equation} \label{pabtrans}
\parbox{0.85\textwidth}{
If $ \rass{a}{b}{P} \in \A$ then for new concept names $A_a$ and $A_b$, we add both $\cass{a}{A_a}$ and $\cass{b}{A_b}$ to $\A$ and add both $A_a \subs \some P$ and $A_b \subs \some P^-$ to $\T$. 
}
\end{equation}
%
%

%
\nd Then
\begin{proposition}[\cite{Calvanese05}] \label{propCalvanese05}
    A \dllitecoreH~(resp.~\dllitehornH) KB $\KB$, expanded according to~(\ref{pabtrans}),  is unsatisfiable iff there is an inclusion axiom $A_1 \subs \neg A_2$ (resp., $\{A_1,\ldots, A_n\} \subs \neg A_{n+1}$) in the \emph{Negative Inclusions} closure ($NI$-closure) of $\KB$ 
    (see Section~\ref{nicor} and~\ref{nihorn}, resp.) and there is an individual $a$ such that all $\cass{a}{A_i}$ are in $\A$. Note that it can be the case that $A_1= A_2$ (resp., $A_{n+1}=A_i$, with $1\leq i\leq n$).
\end{proposition}

\nd Note that the above method adds to the TBox new concept names and inclusion axioms for all individuals occurring in role assertions. This may not be computationally attractive in the presence of many role assertions. The method proposed in~\cite{Calvanese06a,Calvanese07} avoids this problem. Specifically, let us define the function $ra(R,a,b)$ as
\[
ra(R,a,b) = 
\begin{cases}
\rass{a}{b}{P} & \text{if } R=P\\
\rass{b}{a}{P} & \text{if } R=P^- \ .
\end{cases}
\]

\nd As next, we define the notion of \emph{occurs in} (see~\cite{Calvanese07}).

\begin{definition}[occurs in]\label{occin}
For an arbitrary finite set $\calS$ of assertions, we say that 
\begin{enumerate}
    \item $\cass{a}{A}$ \emph{occurs} in $\calS$ 
        if $\cass{a}{A} \in \calS$;
    \item $\rass{a}{b}{R}$ \emph{occurs} in $\calS$ 
        if $ra(R,a,b) \in \calS$;        
    \item $\cass{a}{\csome R}$ \emph{occurs} in $\calS$ 
        if $\rass{a}{b}{R}$ occurs in $\calS$, for some $b$.
\end{enumerate}
\end{definition}

\nd Then
\begin{proposition}[\cite{Calvanese06a,Calvanese07}] \label{propCalvanese0607}
    A \dllitecoreH~(resp.~\dllitehornH) KB $\KB$  is unsatisfiable iff there is an inclusion axiom $B_1 \subs \neg B_2$ (resp. $\{B_1,\ldots, B_n\} \subs \neg B_{n+1}$) in the \emph{Negative Inclusions} closure ($NI$-closure) of $\KB$ 
    (see Section~\ref{nicor} and~\ref{nihorn}, resp.) and there is individual $a$ such that all $\cass{a}{B_i}$ occur in $\A$.
\end{proposition}


\nd    Note that the occurrence check, as required in Proposition~\ref{propCalvanese0607} (and in fact also in  Proposition~\ref{propCalvanese05}) can easily be done by submitting an SQL query to the ABox stored as a relational database (see~\cite{Calvanese06a,Calvanese07}).  Specifically, let $\mathtt{DB}(\A)$ be the database storing the ABox $\A$ (see Definition~\ref{def_DB(A)}). Then, \wrt~\dllitehornH~and inclusion axiom $\tau$ of the form $\{B_1,\ldots, B_n\} \subs \neg B_{n+1}$,\footnote{The case \dllitecoreH~is obtained by imposing $n=1$.} we have that all $\cass{a}{B_i}$ occur in $\A$ iff $a$ is an \emph{answer} to the \emph{Conjunctive Query} (CQ) $q_\tau$ evaluated over $\mathtt{DB}(\A)$ (see also Section~\ref{sect_query_rewite}), where $q_\tau$ is
    \begin{eqnarray} \label{qtau}
    q_\tau(x) & \leftarrow & \gamma(B_1)(x), \ldots, \gamma(B_{n+1})(x)    
    \end{eqnarray}

    \nd with 
    \begin{itemize}
    \item $\gamma(B_i)(x) = A_i(x)$ if $B_i = A_i$;
    \item $\gamma(B_i)(x) = \exists y_i. P_i(x, y_i)$ if $B_i = \csome P_i$;
    \item $\gamma(B_i)(x) = \exists y_i. P_i(y_i, x)$ if $B_i = \csome P^-_i$.    
    \end{itemize}
    
\nd Therefore, 
    
\begin{proposition}[\cite{Calvanese06a,Calvanese07}] \label{propCalvanese0607B}    
A \dllitecoreH~(resp.~\dllitehornH) KB $\KB$ is unsatisfiable iff there is an inclusion axiom $B_1 \subs \neg B_2$ (resp. $\{B_1,\ldots, B_n\} \subs \neg B_{n+1}$) in the \emph{Negative Inclusions} closure ($NI$-closure) of $\KB$ (see Section~\ref{nicor} and~\ref{nihorn}, resp.) 
such that the \emph{answer set} $ans(\A, q_\tau)$ of $q_\tau$ is non-empty.
 \end{proposition}

\nd To conclude this section, we provide here yet an additional method to decide the instance checking and the satisfiability problems. 
It will be useful in Section \ref{sect_RC_ABox}.  
The method borrows the idea proposed in~\cite{Kazakov14} for $\mathcal{ELO^+_\bot}$, and used in~\cite{Casini19} as well,  that consists in encoding assertion axioms as standard inclusion axioms, and thus allows one to remove the ABox. Essentially, let us consider the \emph{nominal} concept expression operator, denoted by the letter $\mathcal{O}$ in the DL literature, which, given an individual $a$, allows one to build  the so-called  \emph{nominal concept} $\{a\}$, whose semantics is 
$\{a\}^\I = \{a^\I\}$. 
Clearly, an interpretation $\I$ is a model of a concept assertion $\cass{a}{A}$ (resp.~role assertion $\rass{a}{b}{P}$) iff $\I$ is a model of the concept inclusion axiom $\{a\} \impc A$ 
(resp.~$\{a\} \impc \csome P.\{b\}$, or, alternatively $\{b\} \impc \csome P^-.\{a\}$).
Now, we replace a nominal concept $\{a\}$ with a new concept name $A_a$. Specifically,  consider any KB $\KB=\tuple{\T,\A}$ and let a new TBox $\T_\A$ defined as 

\begin{equation} \label{ta}
  \begin{array}{lcl}
    \T_\A & = & \T \cup \{A_a \impc A' \mid \cass{a}{A'} \in \A \} \\
    & \cup & \{A_a \impc \csome P.A_b \mid \rass{a}{b}{P} \in \A \} \\
    & \cup & \{A_b \impc \csome P^-.A_a \mid \rass{a}{b}{P} \in \A \} \ . 
\end{array}  
\end{equation}


\nd We allow qualified existential restrictions on the RHS of the concept inclusions in $\T_\A$, since they are compatible with the \dllite~expressivity, as described in Remark~\ref{remexpdllitecore}. Please note that the newly introduced concept names $A_a$ are not nominals, but standard concept names, that can be interpreted into any subset of a domain.





The following proposition allows one to reduce part of the ABox reasoning to TBox reasoning, and it is similar to Theorem 4 in~\cite{Kazakov14}, that has been proved in the context of $\EL^+_{\bot}$.

\begin{proposition}\label{nominalABox}
Consider any \dllitehornH~KB $\KB=\tuple{\T,\A}$ and any concept assertion $\cass{a}{A}$. 
Then
\begin{enumerate}
    \item $\KB$ is unsatisfiable iff $\T_\A \models A_b \impc \neg A_b$ for some individual $b$ occurring in $\KB$;
    \item If $\KB$ is satisfiable, $\KB \models \cass{a}{A}$ iff $\T_\A \models A_a \impc A$.
    
\end{enumerate} 
\end{proposition}

\nd By Propositions~\ref{entailmentcor} and \ref{entailmenthornH} (see later on), we get immediately that

\begin{proposition}\label{nominalABoxBis}
Consider any \dllitecoreH~or \dllitehornH~KB $\KB=\tuple{\T,\A}$. 
Then $\KB$ is unsatisfiable iff  for any individual $b$ occurring in $\KB$, $A_b \impc \neg A_b$ is in the NI-closure of $\T_\A$.
\end{proposition}

\subsection{Chases and NI-closures}\label{sect_chase_niclos}
\nd This section still presents some known technical preliminaries for the main content of the paper, but some of the following results are  novel contributions.
Specifically, we are going to present two operations, the \emph{NI-closure} and the \emph{chase}, for \dllite~KBs. The operations for \dllitecoreH~were already presented in \cite{Calvanese07}. We additionally present and require later on the analogue reformulations for \dllitehornH. 


For the purpose of constructing the chase, we shall hereafter assume the existence of a new infinite and countable alphabet $\Nanon$ of \emph{anonymous individuals}, disjoint from $\N$. 

\subsubsection{Chase for \dllitecoreH} \label{chasecore}

\nd First we introduce the \emph{chase} procedure for \dllitecoreH. 

\begin{definition}[Interpretation $\mathtt{Int}(\calS)$, Database $\mathtt{DB}(\A)$~\cite{Calvanese07}]\label{def_DB(A)}
Consider an ABox $\A$ and a set of assertion axioms $\calS$, possibly involving anonymous individuals with $\A \subseteq \calS$. Let $\N_\calS$ be the set of individuals occurring in $\calS$. 

The \emph{interpretation} $\mathtt{Int}(\calS) = (\Delta^{\mathtt{Int}(\calS)},\cdot^{\mathtt{Int}(\calS)})$ is s.t.:
\begin{itemize}
    \item $\Delta^{\mathtt{Int}(\calS)} = \N_\calS$;
    \item $a^{\mathtt{Int}(\calS)}=a$, for each individual $a \in \N_\calS$;
    \item $A^{\mathtt{Int}(\calS)}=\{a\mid \cass{a}{A}\in \calS\}$, for each concept name $A$;
    \item $P^{\mathtt{Int}(\calS)}=\{(a,b)\mid \rass{a}{b}{P}\in \calS \}$, for each role name $P$.
\end{itemize}

\nd The \emph{database} $\mathtt{DB}(\A)$ is defined as follows:
    \begin{enumerate}
        \item     for each  concept name $A$ occurring in $\A$, we define a relational table $tab_{A}$ of arity 1, such that $\tuple{a} \in tab_{A}$ iff $\cass{a}{A} \in \A$;
    
        \item for each role $P$ occurring in $\A$, we define a relational table $tab_{P}$ of arity 2, such that $\tuple{a,b} \in tab_{P}$ iff $\rass{a}{b}{P} \in \A$.
    \end{enumerate}

\end{definition}

    



\nd It is easily verified that the interpretation $\mathtt{Int}(\calS)$ is in fact a model of $\calS$~\cite{CalvaneseEtAl2007}. {\color{black}In particular, $\mathtt{DB}(\A)$ can be seen as corresponding to the interpretation $\mathtt{Int}(\A)$, that is the smallest model of $\A$.}
\footnote{Please note that sometimes we will use $\mathtt{DB}(\A)$ as interpretation, while in others as database. The meaning should be clear depending from the context in which we will use it.}






The \emph{chase}  for \dllitecoreH~is a well-known procedure to build a model of a \dllitecoreH \ KB $\KB=\tuple{\T,\A}$, if it exists ~\cite{Calvanese05,Calvanese07}. 
Given 
$\KB =(\T,\A)$, the \emph{chase} of $\KB$ (denoted $\chasecH(\KB)$) is the (possibly infinite) ABox obtained starting from $\A$, that is, setting 
\[
\chasecH_0(\KB)=\A
\]
\nd and proceeding with the following \emph{chase rules}:

    
    

{
\begin{description}
    \item[($\mathbf{\CcH_1}$)] if $B \subs A \in \T$, $\cass{a}{B}$ occurs in $\chasecH_i(\KB)$,  and $\cass{a}{A}$ does not occur in $\chasecH_i(\KB)$ (condition $f$), then let $\chasecH_{i+1}(\KB) = \chasecH_i(\KB)\cup\{\cass{a}{A}\}$ (condition $f_{new}$); \\
    
    \item[($\mathbf{\CcH_2}$)] if $B \subs \exists R \in \T$, $\cass{a}{B}$ occurs in $\chasecH_i(\KB)$ (condition $f$),  
    and $\cass{a}{\csome R}$ does not occur in $\chasecH_i(\KB)$, 
    then let $\chasecH_{i+1}(\KB)=\chasecH_i(\KB)\cup\{ra(R,a,b)\}$, where $b$ is a new individual taken from $\Nanon$ (condition $f_{new}$);\\
    
    \item[($\mathbf{\CcH_3}$)] if $R_1\subs R_2\in\T$, $\rass{a}{b}{R_1}$ occurs in $\chasecH_i(\KB)$ (condition $f$),  and $\rass{a}{b}{R_2}$ does not occur in $\chasecH_i(\KB)$, then let $\chasecH_{i+1}(\KB) = $ \\ $ \chasecH_i(\KB) \cup 
    \{ra(R_2,a,b)\}$ (condition $f_{new}$). 
\end{description}
}

    
    

\nd At every step, one of the applicable rules is enforced. Eventually, 
\begin{equation} \label{chasec}
\chasecH(\KB)\defined\bigcup_{i} \chasecH_i(\KB) \ .
\end{equation}

\nd From every $\chasecH_i(\KB)$ (resp., the final $\chasecH(\KB)$) an interpretation 
$\cancH_i(\KB)$ (resp., $\cancH(\KB)$) can be defined as per Definition~\ref{def_DB(A)}, \ie
\begin{equation} \label{canmod}
    \begin{array}{lcl}
    \cancH_i(\KB) & = & \mathtt{Int}(\chasecH_i(\KB)) \\
    \cancH(\KB) & = & \mathtt{Int}(\chasecH(\KB)) \ .
\end{array}
\end{equation}

\nd In particular, note that $\cancH_0(\KB)=\mathtt{DB}(\A)$, and if $\A = \emptyset$ then  $\chasecH(\KB)=\emptyset$.


\subsubsection{NI-Closure for \dllitecoreH} \label{nicor}

\nd As next, for the sake of completeness, we present a procedure to compute the \emph{Negative Inclusions} closure ($NI$-closure) for 
\dllitecoreH~\cite[Sect. 3.1.2]{Calvanese07}. We recall that, given a TBox $\T$, we denote with $\T_{PI}$ (resp.~$\T_{NI}$) the set of positive inclusions $B_1\subs B_2$ and $R_1\subs R_2$ (resp.~negative inclusions $B_1\subs \neg B_2$) in $\T$. 

The \emph{NI-closure} of a \dllitecoreH~TBox $\T$, denoted $\clncH(\T)$, is the TBox that is defined inductively as follows: 


{
\begin{description}
    \item[($\mathbf{\cHRef}$)] all NI's in $\T$ are in $\clncH(\T)$, that is, $\T_{NI}\subseteq\clncH(\T)$;\\

    \item[($\mathbf{\cHContr}$)] {\bf IF}  $B_{1} \impc \notc B_2\in \clncH(\T)$ {\bf THEN}  $B_{2} \impc \notc B_{1} \in \clncH(\T)$; \\
    
    \item[($\mathbf{\cHTrans}$)] {\bf IF}  $B_{1} \impc  B_{2} \in \T$ {\bf AND}  $B_{2} \impc \notc B_{3} \in \clncH(\T)$ {\bf THEN}  $B_{1} \impc \notc B_{3} \in \clncH(\T)$; \\



    \item[($\mathbf{\RcH_1}$)] {\bf IF}  $R_1\subs R_2\in\T$ {\bf AND } $\exists R_2\subs\neg B\in \clncH(\T)$, {\bf THEN}  $\exists R_1\subs\neg B\in \clncH(\T)$; \\

    \item[($\mathbf{\RcH_2}$)] {\bf IF}  $R_1\subs R_2\in\T$ {\bf AND } $\exists R_2^-\subs\neg B\in \clncH(\T)$, {\bf THEN}  $\exists R_1^-\subs\neg B\in \clncH(\T)$; \\


    \item[($\mathbf{\RcH_3}$)] {\bf IF}   $\exists R \impc  \neg \exists R \in \clncH(\T)$ {\bf THEN} 
    $\exists R^- \impc  \neg \exists R^-$ is in $\clncH(\T)$.

\end{description}
}


\begin{remark}
    Please note that, \wrt~the version of NI-closure presented in \cite[Sect. 3.1.2]{Calvanese07}, we have added the rule of contraposition ($\mathbf{\cHContr}$). Such a rule is relevant for what follows, and its addition is not problematic since it is sound \wrt~classical \dllite~entailment. Also, we have modified the other rules accordingly, due to the presence of this new rule. Eventually, we have eliminated the rules denoted in \cite{Calvanese07} as (2), (6), and (7a), since they are associated to functional assertion and negated roles, that we do not consider here. Such changes do not affect the validity of the propositions that follow, that are recalled from \cite{Calvanese07}.
\end{remark}



\begin{proposition}[\cite{Calvanese07}, Lemma 12] \label{chasecorA}
    Let $\KB=(\T,\A)$ be a \dllitecoreH~KB. Then, $\cancH(\KB)$ is a model of $\KB$ if and only if the interpretation $\mathtt{DB}(\A)$ is a model of $(\clncH(\T),\A)$.
\end{proposition}

\begin{proposition}[\cite{Calvanese07}, Lemma 10 + Corollary 13]\label{prop_calvanese07_lemma10_coroll13}
    Let $\T$ be a \dllitecoreH~TBox, and let $\alpha$ be a NI. Then $\clncH(\T)\entails \alpha$ iff $\T\entails\alpha$.
\end{proposition}

\nd Also, combining the chasing and the NI-closure, we obtain the following.

%
\begin{proposition}[\cite{Calvanese07}] \label{chasecor}
  Given a \dllitecoreH~$\KB=(\T,\A)$.  Then
  \begin{enumerate}
    \item $\cancH(\KB)$ is a model of $(\T_{PI},\A)$;


    \item $\KB$ is satisfiable iff $\cancH(\KB)$ is a model of $\KB$.
      

      
  \end{enumerate}
\end{proposition}




\nd Eventually, we can strengthen Proposition \ref{prop_calvanese07_lemma10_coroll13} as follows.

\begin{proposition} \label{entailmentcor}
  Consider a \dllitecoreH~TBox $\T$ and any pair of basic concepts $B_1,B_2$. Then $\T \models B_{1} \impc \notc B_{2}$ iff one of the following holds:
  \begin{enumerate}
        \item $B_{1} \impc \notc B_{2} \in \clncH(\T)$;
        \item $B_{1} \impc \notc B_{1} \in \clncH(\T)$;
        \item $B_{2} \impc \notc B_{2} \in \clncH(\T)$.
  \end{enumerate}
\end{proposition}

\subsubsection{Chase for \dllitehornH} \label{chasehorn}

\nd We extend the previous results  to \dllitehornH. The \emph{chasing} procedure is almost identical to the \dllitecoreH~version (it corresponds to Definition 3.5 in \cite{Calvanese13}, constrained and reformulated for \dllitehornH). 

We start by setting 
\[
\chasehH_0(\KB)=\A \ ,
\]

\nd and we proceed with the following $(\mathbf{\ChH_i})$ rules that extend the $(\mathbf{\CcH_i})$ rules to \dllitehornH:

{
\begin{description}
    \item[$(\mathbf{\ChH_1})$] if $\{B_1,\ldots,B_n\} \subs A \in \T$, $\cass{a}{B_k}$ occurs in $\chasehH_i(\KB)$ for every $k$ ($1\leq k\leq n$) (condition $f$), and $\cass{a}{A}$ does not occur in $\chasehH_i(\KB)$, then let $\chasehH_{i+1}(\KB) = \chasehH_i(\KB)\cup\{\cass{a}{A}\}$ (condition $f_{new}$);\\
    
    \item[$(\mathbf{\ChH_2})$] if $\{B_1,\ldots,B_n\} \subs \exists R \in \T$, $\cass{a}{B_k}$ occurs in $\chasehH_i(\KB)$ for every $k$ ($1\leq k\leq n$) (condition $f$),  
    and  $\cass{a}{\csome R}$ does not occur in $\chasehH_i(\KB)$,
    then let $\ chasehH_{i+1}(\KB) = \chasehH_i(\KB)\cup\{ra(R,a,b)\}$, where $b$ is a new individual taken from $\Nanon$ (condition $f_{new}$);\\
    

     \item[$(\mathbf{\ChH_3})$] if $R_1\subs R_2\in\T$, $\rass{a}{b}{R_1}$ occurs in $\chasehH_i(\KB)$ (condition $f$),  and $\rass{a}{b}{R_2}$ does not occur in $\chasehH_i(\KB)$, then let $\chasehH_{i+1}(\KB) = $ \\ $\chasehH_i(\KB)\cup\{\rass{a}{b}{R_2}\}$ (condition $f_{new}$). 
\end{description}
}

\nd At every step, one of the applicable rules is enforced. 
Eventually, 
\[
\chasehH(\KB)\defined\bigcup_{i}\chasehH_i(\KB) \ .
\]

\nd As for \dllitecoreH, we can define the models corresponding to every step in the chasing, denoted as $\canhH_i(\KB)$ and $\canhH(\KB)$.
That is,  as per Definition~\ref{def_DB(A)},
\begin{equation} \label{canmodhorn}
\begin{array}{lcl}
    \canhH_i(\KB) & = & \mathtt{Int}(\chasehH_i(\KB)) \\
    \canhH(\KB) & = & \mathtt{Int}(\chasehH(\KB)) \ .
\end{array}
\end{equation}

\subsubsection{NI-Closure for \dllitehornH} \label{nihorn}
\nd We next move on to the definition of the \emph{NI-closure} for \dllitehornH. 
As far as we know, the NI-closure has not been previously defined for \dllitehornH~and, as we will need it later on, we add it here. For a \dllitehornH~TBox $\T$ we propose the following modified version of the \emph{NI}-closure, that we denote as $\clnhH(\T)$. In the rules below, $CL$ indicates a finite set of basic concepts.

\begin{description}

    \item[($\mathbf{\hHRef}$)] all NI's in $\T$ are in $\clnhH(\T)$;\\

     \item[($\mathbf{\hHContr}$)] {\bf IF} $CL \cup\{B_1\}  \subs \neg B_2 \in \clnhH(\T)$ 
     {\bf THEN}  $CL \cup\{B_2\}  \subs \neg B_1 \in \clnhH(\T)$;\\


    \item[($\mathbf{\hHTrans}$)] {\bf IF} $CL \subs B_1 \in \T$ {\bf AND} 
    $CL' \cup\{B_1\}  \subs \neg B_2 \in \clnhH(\T)$ 
    {\bf THEN}  $CL \cup CL' \subs \neg B_2 \in \clnhH(\T)$;\\



    \item[($\mathbf{\RhH_1}$)]  
    {\bf IF} $R_2 \subs R_1 \in \T$ {\bf AND} 
    $CL \cup \{\exists R_1\}  \subs \neg B \in \clnhH(\T)$ 
    {\bf THEN}  $CL \cup \{\exists R_2 \}  \subs \neg B  \in \clnhH(\T)$;\\


    \item[($\mathbf{\RhH_2}$)]  {\bf IF} $R_2 \subs R_1 \in \T$ 
    {\bf AND}  
    $CL \cup \{\exists R_1^-\}  \subs \neg B \in \clnhH(\T)$ 
    {\bf THEN}  $CL \cup \{\exists R_2^- \}  \subs \neg B \in \clnhH(\T)$;\\


    \item[($\mathbf{\RhH_3}$)]  {\bf IF} 
     $\exists R \subs \neg \exists R \in \clnhH(\T)$ 
    {\bf THEN}  $\exists R^- \subs \neg \exists R^-\in \clnhH(\T)$.





\end{description}

\nd Now, the following proposition can be shown.

\begin{proposition}[\dllitehornH~analogue of Theorem 15 in~\cite{Calvanese07}]\label{Theo15calv07}
     Let $\KB=(\T,\A)$ be a \dllitehornH~KB. Then $\KB$ is satisfiable iff the interpretation $\mathtt{DB}(\A)$ is a model of $(\clnhH(\T),\A)$.
\end{proposition}


\nd Eventually, we can prove the \dllitehornH~analogue of Proposition~\ref{entailmentcor}.

\begin{proposition}[\dllitehornH~analogue of Proposition~\ref{entailmentcor}] \label{entailmenthornH}
  Consider a \dllitehornH~TBox $\T$. Then $\T \models CL \impc \notc B$ iff one of the following holds:
  \begin{enumerate}
        \item $CL' \impc \notc B \in \clnhH(\T)$ with $CL' \subseteq CL$;

    \item $CL' \impc \neg B' \in \clnhH(\T)$ with $CL'\cup\{B'\} \subseteq CL$;
  
     \item $B \impc \notc B \in \clnhH(\T)$.

  \end{enumerate}
\end{proposition}

\section{Defeasible \dllite} \label{defdllite}

\noindent\fbox{\begin{minipage}{0.975\textwidth}
\textsc{Overview:} 
 This section introduces defeasible reasoning in \dllite. In particular, we present two versions of rational closure (RC) that are optimised for \dllitecoreH~and \dllitehornH, respectively.
  \begin{itemize}
      \item Section~\ref{sect_RCNOABox} considers the case without ABox.
        \begin{itemize}
            \item Section \ref{sect_RCsemantics} recalls from \cite{BritzEtAl2021} the semantic construction.
            \item Section \ref{sect_RC_horn} introduces the procedures for deciding RC in \dllitehornH. The procedures are novel, but inspired by the ones for defeasible $\ELbot$ from \cite{Casini19}.
            \item Section \ref{sect_RC_core} introduces the procedures for deciding RC in \dllitecoreH. The procedures are novel and ad-hoc for defeasible \dllitecoreH.
        \end{itemize}

        \item Section~\ref{sect_RC_ABox} addresses the case with ABox. It shows how to solve the satisfiability problem. A major results is that  \dllitehornH~and, thus, \dllitecoreH, is not affected by the multiple-extension problem, which is typical for defeasible DLs. This results is fundamental for addressing then the conjunctive query answering problem  via query rewriting, which will be addressed later on in Section~\ref{sect_QRW}.
%
  \end{itemize}
\end{minipage}}
\vspace*{5mm}

\nd In this section we introduce defeasible inclusion axioms in the context of~$\dllite$. Our extension follows the line of  
well-established approaches to defeasible DLs such as~\cite{BritzEtAl2021,CasiniStraccia2010,CasiniStraccia2013,Casini19}.
Specifically, a \emph{Defeasible Inclusion Axiom} (DI) in \dllitehornH~has the form
\begin{equation} \label{defeasibleinclH}
\{B_1,\ldots, B_n\}\dsubs A \ ,
\end{equation}
where each $B_i$ is a basic concept and $A$ is a concept name, and may be read as ``usually, an instance of the conjunction of the concepts $B_1,\ldots, B_n$ is an instance of $A$''. For defeasible \dllitecoreH, we restrict the form of the DIs to a single basic concept on the LHS, that is, 
\begin{equation} \label{defeasibleinclC}
B\dsubs A \ ,
\end{equation}

\nd in line with the expressivity of classical \dllitecoreH. 
A \emph{DBox}~$\D$ is a finite set of such DIs. Hence, from now on, a  \emph{defeasible KB} is a tuple $\KB\defined\tuple{\T,\D,\A}$, where one or more elements (but not all) may be omitted, in which case the omitted sets are considered empty.

\begin{remark} \label{negdef}
    Note that, in the following, a DI inclusion such as $\{B_1,\ldots,B_n\}\dsubs C$, with $C$ being a complex concept, has to be considered instantiated via the DI $\{B_1,\ldots,B_n\}\dsubs A$ in the DBox and the IN 
    $A\subs C$ in the TBox, where $A$ is a new concept name. In the following, whenever we write 
    $\{B_1,\ldots,B_n\}\dsubs C$ (or $B\dsubs C$, in case of \dllitecoreH), we assume that such a transformation has been made.
\end{remark}

\nd Concerning the reasoning patterns in the presence of a DBox, there are various entailment relations that we may apply, such as~\cite{BonattiEtAl2015,BonattiEtAl2011,GiordanoEtAl2013}. In this work, we consider \emph{Rational Closure} (RC), a well-known and popular entailment relation  in modelling conditional reasoning~\cite{LehmannMagidor1992,Pearl1990}. 
We chose RC since it is a representative non-monotonic entailment relation. In particular:
\begin{itemize}
    \item its application in various DLs has already been investigated 
    (see, \eg~\cite{Bonatti2019,BritzEtAl2021,CasiniStraccia2010,Casini19,GiordanoEtAl2015});  
    
    \item various interesting defeasible entailment relations have been built as extensions of RC \cite{CasiniHaldimannMeyer2025,CasiniEtAl2014,Casini19b,CasiniStraccia2013,CasiniStraccia2013b,Giordano21,HaldimannBierle2024,Lehmann1995}; and


    \item it can be implemented on top of classical reasoners for DL (see, \eg~\cite{CasiniStraccia2010,CasiniEtAl2015,Casini19}).

\end{itemize}




\nd As next, we are going to introduce the semantic construction, that is essentially the same for all RC-based DLs, and then we present decision procedures that are tailored to \dllitecoreH~and \dllitehornH, respectively.

\subsection{Rational Closure without ABox}\label{sect_RCNOABox}

\nd In the following, we are going to consider only defeasible KBs without the ABox (\ie~the ABox is empty). That is, all KBs will have the form $\KB=\tuple{\T,\D}$. We will address the ABox later on, in Section~\ref{sect_RC_ABox}.

\subsubsection{Rational Closure - Semantics}\label{sect_RCsemantics}

\nd There are various equivalent semantic characterisation of RC for 
DLs (see, \eg~\cite{Bonatti2019,BritzEtAl2021,GiordanoEtAl2015,Casini19}). Here we will refer to the one presented in \cite{BritzEtAl2015a,BritzEtAl2021} and adapt it to \dllitehornH~(for \dllitecoreH~is sufficient to constrain the expressivity accordingly).\footnote{The semantic construction we present in this section has been originally formulated for the DL $\mathcal{ALC}$, but it is appropriate for all the DLs that satisfy the \emph{disjoint union model property}. A logic satisfies the disjoint model union property if, given any KB $\KB$, the disjoint union of two  models of $\KB$ is still a model of $\KB$. About DLs, this property is usually broken by constructs such as nominals, ABoxes, and the universal role. Since here we consider only   \dllitecoreH~and \dllitehornH~KB~ with an empty ABox,  the disjoint model union property is satisfied. We will see in Section~\ref{sect_RC_ABox} how to start from such a construction  to draw conclusions also in the presence of an ABox. For a semantic characterisation of RC that covers also the DLs not satisfying the disjoint model union property, see~\cite{Bonatti2019}.
} 

This semantic characterisation of RC is based on 
\emph{ranked interpretations}, which are defined next.

\begin{definition}[Ranked Order]\label{Def:Ranked_order}
Given a set $X$, the binary relation $\pref\ \subseteq X\times X$ is a \df{ranked order} if there is a mapping $h_{\RI}:X\longrightarrow\mathbb{N}$ satisfying the following convexity property:
\begin{itemize}
\item for every $i\in \mathbb{N}$, if for some $x\in X$ $h_{\RI}(x)=i$, then, for every $j$ such that $0\leq j<i$, there is a $y\in X$ for which $h_{\RI}(y)=j$,
\end{itemize}
and \st\ for every $x,y\in X$, $x\pref y$ iff $h_{\RI}(x)<h_{\RI}(y)$.
\end{definition}

\begin{definition}[Ranked Interpretation]\label{Def:Ranked_interpretation}
A \df{ranked interpretation} is an interpretation $\RI=\tuple{\Dom^{\RI},\cdot^{\RI},\pref^{\RI}}$ s.t. $\Delta^\RI$ and $\cdot^{\RI}$ are defined as for a classical DL interpretation, and $\pref^{\RI}:\Dom^{\RI}\times \Dom^{\RI}$ is a ranked order.

Given an axiom $\alpha$, we say that $\RI$ \df{satisfies} (\emph{is a model of}) $\alpha$ ($\RI\sat\alpha$) iff 

\begin{itemize}
    \item for $\alpha=\{B_1,\ldots,B_n\}\subs C$, $(B_1^{\RI}\cap\ldots\cap B_n^{\RI})\subseteq C^{\RI}$;
    \item for $\alpha=R_1\subs R_2$, $R_1^{\RI}\subseteq R_2^{\RI}$;
    \item for $\alpha=\cass{a}{A}$, $a^{\RI}\in A^{\RI}$;
    \item for $\alpha=\rass{a}{b}{P}$, $(a^{\RI},b^{\RI})\in P^{\RI}$;
    \item for $\alpha=\{B_1,\ldots,B_n\}\dsubs A$, $\min_{\pref^{\RI}}(B_1^{\RI}\cap\ldots\cap B_n^{\RI})\subseteq A^{\RI}$, where $\min_{\pref^{\RI}}(B_1^{\RI}\cap\ldots\cap B_n^{\RI})=\{x\in (B_1^{\RI}\cap\ldots\cap B_n^{\RI})\mid\not\exists y\in (B_1^{\RI}\cap\ldots\cap B_n^{\RI})\text{ s.t. } y\pref^{\RI} x\}$.
\end{itemize}

\nd A ranked interpretation $\R$ \emph{satisfies} (\emph{is a model of}) a KB $\KB=\tuple{\T,\D}$ iff $\R\sat\alpha$ for every $\alpha\in\T\cup\D$. 
We indicate with $\mathfrak{R}_{\KB}$ the set of all the ranked models of a KB $\KB$.


\end{definition}
 
\nd Informally, a ranked interpretation is a classical DL interpretation on top of which a ranked order $\pref^{\RI}$ is defined. Such an order is interpreted as representing the level of typicality of every object in the domain: $x\pref^{\RI} y$ is read as `$x$ behaves in a more typical way than $y$'.  The condition  $\min_{\pref^{\RI}}(B_1^{\RI}\cap\ldots\cap B_n^{\RI})\subseteq A^{\RI}$ for a DI $\{B_1,\ldots,B_n\}\dsubs A$ indicates that all the most typical instances of all the $B_i$s are also instances of $A$.


\begin{example}[Adapted from~\cite{CasiniStraccia23}]\label{ex_guiding_1}
    Consider the following sentences:
    \begin{itemize}
        \item Drug users make use of some drug.
        \item Terminal patients have been diagnosed with a terminal disease.
        \item People with controlled drug use (\ie, controlled drug users) are drug users.
        \item Young people are usually happy.
        \item  Drug users are usually not happy.
        \item  Drug users are usually young.
        \item  Controlled drug users are usually happy.
        \item  Controlled drug users with a terminal illness are usually not happy.
    \end{itemize}
    
\nd They can be formalised as the following  KB $\KB=\tuple{\T,\D}$:

    \begin{itemize}
        \item $\T=\{Du\subs \exists uses.D, Tp\subs \exists diagn.Ti, CDu\subs Du\}$,
        \item $\D=\{Y\dsubs H, Du\dsubs\neg H,Du\dsubs Y, CDu\dsubs H, \{CDu,  Tp\}\dsubs\neg H\}$,
    \end{itemize}

   \nd where $Du$ is read as `drug user', $D$ as `drug', $Tp$ as `terminal patient', $Ti$ as `terminal illness', $CDu$ as `controlled drug user', $Y$ as `young', $H$ as `happy', $Ti$ as `terminal illness',   the role ${uses}$ as `makes use of', and  the role ${diagn}$ as `diagnosed with'.\footnote{This formalisation is actually a simplification, for the sake of exposition. For example, $Du\dsubs\neg H$ is not a legitimate DI, since we should have only atomic concepts on the right. See Remarks \ref{remexpdllitecore} and \ref{negdef}.} 

    Fig. \ref{figrkmod} shows a ranked model of  $\KB$.
\qed 
\end{example}

\nd Given a ranked order $\pref^{\RI}$, there is only one mapping $h_{\RI}$ that defines it as in Definition \ref{Def:Ranked_order} \cite[Proposition 1]{BritzEtAl2021}, and we call this function $h_{\RI}$ the \emph{characteristic ranking function} of an interpretation $\RI=\tuple{\Dom^{\RI},\cdot^{\RI},\pref^{\RI}}$.
Based on that we can define the notions of \emph{height} and \emph{layer}.

\begin{definition}[Height \& Layers]\label{Def:height}
Given a ranked interpretation $\RI=\tuple{\Dom^{\RI},\cdot^{\RI},\pref^{\RI}~}$, its characteristic ranking function $h_{\RI}(\cdot)$, and an object $x\in\Dom^{\RI}$, $h_{\RI}(x)$ is called the \df{height} of~$x$ in~$\RI$. The \emph{height}  of an individual $a$ \wrt~$\RI$ is the height of $a^{\RI}$, \ie~$h_{\RI}(a)=h_{\RI}(a^{\RI})$. 
For every ranked interpretation $\RI=\tuple{\Dom^{\RI},\cdot^{\RI},\pref^{\RI}}$, we can partition the domain~$\Dom^{\RI}$ into \emph{layers} $L_{0},\ldots,L_{n},\ldots$, where, for every object $x\in\Dom^{\RI}$, we have $x\in L_{i}$ iff $h_{\RI}(x)= i$.
The height of a concept $C$ corresponds to the height of its minimal elements. That is, $h_{\R}(C)=i$ iff $\min_{\prec^\R}(C^\R)\subseteq L_{i}$. If $C^\R=\emptyset$, then $h_{\R}(C)=\infty$.
\end{definition}

\begin{figure}
    \centering
    \includegraphics[width=1\linewidth]{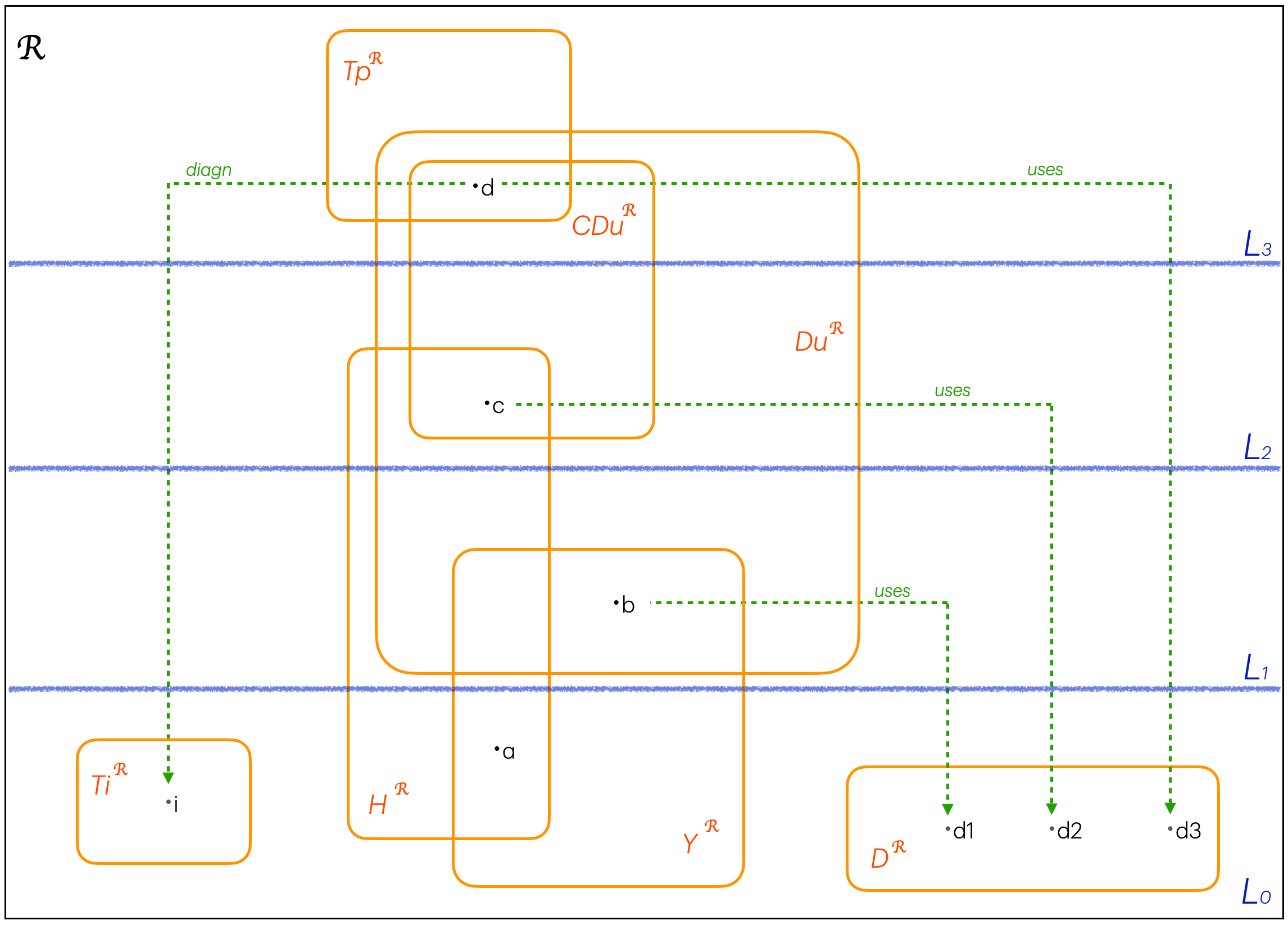}
    \caption{ a ranked model satisfying the KB in Example~\ref{ex_guiding_1}.}
    \label{figrkmod}
\end{figure}

\nd Please note  also that for a defeasible KB $\KB=\tuple{\T,\D}$,  the analogue of Proposition~\ref{tsat} holds.

\begin{proposition} \label{dsat}
    Any defeasible \dllitehornH~KB~$\KB=\tuple{\T,\D}$ is satisfiable.
\end{proposition}

\nd Given a KB $\KB=\tuple{\T,\D}$, the basic form of reasoning that we can define is the classical Taskian one, that is, an entailment relation that is based on \emph{all} the ranked model of $\KB$.

\begin{definition}[Ranked Entailment (RE); \cite{Casini19}, Definition 3]\label{def_rankedentail}
    Given a KB $\KB = \tuple{\T, \D}$ and an inclusion $\{B_1,\ldots,B_n\} \ \star \ C$ (where $\star\in\{\subs,\dsubs\}$), $\KB$ \emph{rationally entails}  $\{B_1,\ldots,B_n\} \ \star \ C$ (denoted $\KB\rationalent \{B_1,\ldots,B_n\} \ \star \ C$) iff $\forall \R\in\mathfrak{R}_{\KB}$, $\R\sat \{B_1,\ldots,B_n\} \ \star \ C$.
\end{definition}

\nd It is well-known that RE  has some limitations. In particular it is too weak from the inferential point of view \cite[Section 4]{BritzEtAl2021}. RC is an entailment relation that extends RE, allowing us to overcome such limitations, and it is based on a notion of \emph{exceptionality}, which we define next. 

\begin{definition}[Exceptionality; \cite{Casini19}, Definition~4]\label{Def:Exceptionality}
Let $\KB$ be a defeasible KB and $C$ a complex concept. We say~$C$ is \df{exceptional} \wrt~$\KB$ if  for every ranked model $\R$ of $\KB$,
\[
C^\R\cap\min_{\R}(\Delta^\R)=\emptyset \ .
\]

\nd That is, $C^\R\cap L_0= \emptyset $. A DI $B\dsubs A$ is \df{exceptional} \wrt~\KB\ if  $B$ is exceptional in~$\KB$. 


Analogously, any finite set of  concepts~$\{C_1,\ldots,C_n\}$ is \df{exceptional} \wrt~$\KB$ if 
\[
C_1^\R\cap\ldots\cap C_n^\R\cap\min_{\R}(\Delta^\R)=\emptyset \ .
\]

\nd A DI $\{B_1,\ldots,B_n\}\dsubs A$ is \df{exceptional} \wrt~\KB\ if  $\{B_1,\ldots,B_n\}$ is exceptional in~$\KB$. 

\end{definition}

\nd That is, a concept $C$ is \emph{exceptional} \wrt~$\KB$ if there is no model of $\KB$ in which a typical object (that is, an object in $L_0$) is in $C^{\RI}$.  The same notion of exceptionality is extended to a set of  concepts $\{C_1,\ldots,C_n\}$: the set $\{C_1,\ldots,C_n\}$ is exceptional \wrt~$\KB$ if there is no model of $\KB$ in which a typical object (that is, an object in $L_0$) is in $(C_1^\R\cap\ldots\cap C_n^\R)$.

By applying the notion of exceptionality iteratively, we can associate to every set $\{C_1,\ldots,C_n\}$ a \emph{rank} w.r.t.~$\KB$, which we denote by $\rank_{\KB}(\{C_1,\ldots,C_n\})$. That is, 

\begin{enumerate}
\item $\rank_{\KB}(\{C_1,\ldots,C_n\})=0$ if $\{C_1,\ldots,C_n\}$ is not exceptional w.r.t.~$\KB$. 
Moreover, $\rank_{\KB}(\{B_1,\ldots,B_n\}\dsubs A)=0$ if $\rank_{\KB}(\{B_1,\ldots,B_n\})=0$. 
The set of DIs in~$\D$ with rank~$0$ is denoted as $\D^\rank_{0}$;

\item let $\KB^1=\tuple{\T,\D\setminus\D^\rank_{0}}$; that is, $\KB^1$ is composed of $\T$~and the exceptional part of~$\D$. Then
$\rank_{\KB}(\{C_1,\ldots,C_n\})=1$ if $\{C_1,\ldots,C_n\}$ does not have rank~$0$ and $\{C_1,\ldots,C_n\}$ is not exceptional in the knowledge base $\KB^1$. 
Moreover, $\rank_{\KB}(\{B_1,\ldots,B_n\}\dsubs A)=1$ if $\rank_{\KB}(\{B_1,\ldots,B_n\})=1$.
The set of DIs  in~$\D$ with rank $1$ is denoted $\D^\rank_{1}$;

\item in general, for $i>0$, let $\KB^{i}=\tuple{\T,\D\setminus\bigcup_{j=0}^{i-1}\D^{\rank}_{j}}$. Then
$\rank_{\KB}(\{C_1,\ldots,C_n\})=i$ if $\{C_1,\ldots,C_n\}$ does not have a rank strictly lower than $i$ and is not exceptional in~$\KB^{i}$. The set of DIs  in~$\D$ with rank $i$ is denoted $\D^\rank_{i}$;

\item by iterating the previous step, we eventually reach a (possibly empty) subset $\E\subseteq\D$ such that all the DIs in~$\E$ are exceptional (since~$\D$ is finite, we must reach such a point). 
We define the rank of the DIs in $\E$ (if there are any) as~$\infty$, and the set~$\E$ is denoted $\D^\rank_\infty$. 
Moreover, we set $\rank_{\KB}(\{B_1,\ldots,B_n\})=\infty$ for every~$\{B_1,\ldots,B_n\}$ in the LHS (Left-Hand-Side) of some DI  in~$\D^\rank_\infty$. The concepts ranked $\infty$ are the concepts that cannot be satisfied in any model of the KB.
\end{enumerate}

\begin{example}[Example \ref{ex_guiding_1} cont'd]\label{ex_guiding_2}
    We take under consideration the antecedents of the DIs in $\D$: $Y, Du, CDu, \{CDu, Tp\}$. In Figure \ref{figrkmod} we have an example of a model $\R$ satisfying the KB $\KB$, where $Y^\R \cap L_0\neq\emptyset$. Hence the concept $Y$ is not exceptional w.r.t. $\KB$ and must have rank $0$, since there is at least one model in which a typical object is in the interpretation of $Y$. 
    
    Consider the concept $Du$. In $\R$ $Du^\R \cap L_0=\emptyset$, but  can there be a model $\R'$ of $\KB$ s.t. $Du^{\R'} \cap L_0\neq\emptyset$? Let's assume such an $\R'$ exists, with some object $f\in Du^{\R'} \cap L_0$. $f\in Du^{\R'} \cap L_0$ implies that $f\in \min_{\prec^{\R'}}(Du^{\R'})$. Since $Du\dsubs \neg H\in\KB$ and $Du\dsubs Y\in\KB$, we are forced to conclude that $f\notin H^{\R'}$ and $f\in Y^{\R'}$. The latter, together with $f\in L_0$, implies $f\in \min_{\prec^{\R'}}(Y^{\R'})$, that, together with $Y\dsubs H\in \KB$, implies $f\in H^{\R'}$. Therefore, we must conclude that the object $f$ is and is not in $H^{\R'}$ at the same time, that is an absurdity. Hence it cannot be that there is a model of $\KB$ s.t. $Du^{\R'} \cap L_0\neq\emptyset$, and $Du$ is an exceptional concept.

    Applying the same reasoning, it can be verified that we end up with the following partition of $\D$: $\D^\rank_{0}=\{Y\dsubs H\}$; $\D^\rank_{1}=\{Du\dsubs \neg H, Du\dsubs Y\}$; $\D^\rank_{2}=\{CDu\dsubs  H\}$; $\D^\rank_{3}=\{\{CDu, Tp\}\dsubs \neg H\}$. We anticipate that in the following sections we will introduce appropriate procedures to decide the ranking of any concept.
\qed 
\end{example}

\nd The rank of a concept represents a  simple intuition: the higher the rank of the concept, the more exceptional/less expected an object satisfying it is considered. For instance, the ranking values of the elements of $\D$ in Example \ref{ex_guiding_2} tell us that, according to the information in $\KB$, encountering a young person that is happy is more expected/typical than encountering a young person that is a drug user and is not happy.

RC is an entailment relation that is based on this kind of `measure' of expectations: given a KB $\KB=\tuple{\T,\D}$, we derive $B\dsubs A$ if it is more expected/typical to meet any object satisfying $B$ and $A$ than any object satisfying $B$ and $\neg A$. This principle is encoded in the following definition.

\begin{definition}[Rational Closure; \cite{BritzEtAl2021}, Definition 24]\label{def_rankedentail}
    Given a KB $\KB = \tuple{\T, \D}$, a DI $\{B_1,\ldots,B_n\}\dsubs A$ \emph{is in the RC of} $\KB$ (denoted $\KB\minent \{B_1,\ldots,B_n\}\dsubs A$ ) iff \\
\nd either
    \[
    \begin{array}{lcl}
         \rank_{\KB}(\{B_1,\ldots,B_n,A\}) & < & 
         \rank_{\KB}(\{B_1,\ldots,B_n,\neg A\})  \\
    \end{array}
    \]
    \nd or
\[
    \begin{array}{lcl}
         \rank_{\KB}(\{B_1,\ldots,B_n\}) & = & \infty \ .
    \end{array}
\]
    
    \nd On the other hand, an inclusion $\{B_1,\ldots,B_n\}\subs C$ \emph{is in the RC of} $\KB$ 
    (denoted $\KB\minent \{B_1,\ldots,B_n\}\subs C$) iff 
\[
\begin{array}{lcl}
\rank_{\KB}(\{\{B_1,\ldots,B_n,\neg C\}) & = & \infty  \ ,
\end{array}
\]
\nd where, obviously, $\neg \neg B$ is $B$.
\end{definition}

\nd From the semantics point of view, the intuition behind RC corresponds to taking under consideration only the models of the KB in which all the individuals are interpreted at the lowest possible height, that is, we treat all the individuals as behaving in the most expected/typical way (assuming the KB is satisfied). In such a way we implement the \emph{Presumption of Typicality} \cite{Lehmann1995}: that is, we assume that, compatibly with the available information, everything behaves as typically as possible.


Let us note that we can define a single model that represents such a presumption. First, we need to introduce the notion of \emph{ranked union} and the set $\Moddelta{\KB}$. 


\begin{definition}[Ranked Union; \cite{BritzEtAl2021}, Definition 20] \label{rankedunion}
Given a countable set of ranked interpretations $\mathfrak{R} = \{\RI_1$, $\RI_2,\ldots\}$, a ranked interpretation $\RI^{\mathfrak{R}}\defined\tuple{\Dom^{\mathfrak{R}},\cdot^{\mathfrak{R}},\pref^{\mathfrak{R}}}$ is the \df{ranked union} of $\mathfrak{R}$ if the following holds:
\begin{itemize}
\item $\Dom^{\mathfrak{R}}\defined\coprod_{\RI\in\mathfrak{R}}\Delta^{\RI}$, \ie, the disjoint union of the domains from $\mathfrak{R}$, where each $\RI\in\mathfrak{R}$ has the elements $x,y,\ldots$ of its domain renamed as $x_{\RI}$, $y_{\RI}$, \ldots\ so that they are all distinct in $\Dom^{\mathfrak{R}}$;

\item $x_{\RI}\in A^{\mathfrak{R}}$ iff $x\in A^{\RI}$, for every concept name $A$;

\item $(x_{\RI},y_{\RI'})\in P^{\mathfrak{R}}$ iff $\RI=\RI'$ and $(x,y)\in P^{\RI}$, for every role name $P$;

\item $h_{\mathfrak{R}}(x_{\RI})=h_{\RI}(x)$, for every $x_{\RI}\in \Dom^{\mathfrak{R}}$.
\end{itemize}

\nd The latter condition corresponds to imposing that $x_{\RI}\pref^{\mathfrak{R}}y_{\RI'}$ iff $h_{\RI}(x)<h_{\RI'}(y)$.
\end{definition}

\nd Informally, the ranked union of a set of ranked interpretations is the result of merging all their layers of height $i$ into a single layer of height $i$, for all $i$.


\begin{lemma}\label{Lemma:ClosureDisjointUnionMod}(\cite{BritzEtAl2021}, Lemma 8)
Given a set of ranked models of a defeasible KB $\KB$, their ranked union is itself a ranked model of $\KB$.
\end{lemma}

\nd Now, let $\KB$ be a defeasible KB and let $\Delta$ be any fixed countably infinite set. We define $\Moddelta{\KB}$ as follows:
\begin{equation}
\Moddelta{\KB} \defined \{\RI=\tuple{\Dom^{\RI},\cdot^{\RI},\pref^{\RI}} \mid \RI\sat\KB, \RI\text{ is ranked and } \Dom^{\RI}=\Delta\} \ .    
\end{equation}

\nd That is, $\Moddelta{\KB}$ contains all the ranked models of $\KB$ that have $\Delta$ as domain.





We  introduce now the definition of the \emph{Big Ranked Model} of a defeasible KB $\KB=\tuple{\T,\D}$.

\begin{definition}[Big Ranked Model; \cite{BritzEtAl2021}, Definition 21]\label{Def:Bigrankedmodel}
Let $\KB=\tuple{\T,\D}$ be a defeasible KB and $\Delta$ be any countable infinite set.  The \df{big ranked model} of \KB~is the ranked model $\OI\defined\tuple{\Dom^{\OI},\cdot^{\OI},\pref^{\OI}}$ that is the ranked union of the models in $\Moddelta{\KB}$.
\end{definition}

\nd $\OI$ is the model characterising the RC of $\KB$, that is, an inclusion $\alpha$ is in the RC of a defeasible KB $\KB=\tuple{\T,\D}$ iff 
\[
\OI\sat \alpha \ ,
\]

\nd where $\OI$ is the \emph{Big Ranked Model} of $\KB$. This correspondence between the big ranked model of a defeasible KB $\KB$ and the classical definition of RC has been proved in Section 5.1 in \cite{BritzEtAl2021}. 
That is,
\begin{proposition}[\cite{BritzEtAl2021}, Theorem 4]\label{prop_minent}
    Let $\KB=\tuple{\T,\D}$ be a defeasible KB and let $\OI$ be its \emph{big ranked model}. A DI $C\dsubs D$ is in the RC of $\KB$ (\ie~$\KB\minent C\dsubs D$) iff $\OI\sat C\dsubs D$.
\end{proposition}


\nd We refer to \cite{BritzEtAl2017,BritzEtAl2021,BritzEtAl2019} for a more in depth  presentation, explaining in detail why the \emph{big ranked model} is a model characterising the RC of $\KB$.

\begin{remark}
    We briefly outline here the path that motivated the use of the big ranked model to characterise RC for DLs with low expressivity like $\mathcal{EL}_\bot$ \cite{Casini19} and the \dllite\ family, that we consider here. RC is an entailment relation that was developed inside the framework of the KLM preferential approach \cite{KrausEtAl1990,LehmannMagidor1992}, that is defined for propositional logic. The KLM foundational results are representation theorems linking \emph{preferential semantics} to a set of \emph{structural properties}. To express such structural properties we require all the classical propositional connectives, and consequently they can be reformulated at the level of DLs  using a language that is at least as expressive as $\mathcal{ALC}$ \cite{BaaderEtAl2007}. The reader can find in \cite{BritzEtAl2021} the reformulation of the KLM structural properties in $\mathcal{ALC}$ and the corresponding representation theorems characterising their satisfaction using  \emph{preferential}  and \emph{modular} interpretations, redefined for the DL framework \cite[Theorems 1 and 2]{BritzEtAl2021}. There, it is also proved that the \emph{big ranked model} is a semantic construction characterising RC for $\mathcal{ALC}$ \cite[Theorem 4]{BritzEtAl2021}, and such a proof applies also to all the languages that extend $\mathcal{ALC}$ and satisfy the \emph{disjoint union model property}, including $\mathcal{ALCHI}$, which extends $\mathcal{ALC}$ with inverted roles and role inclusions. 
    Consequently, we are  fully justified in using such a construction also for the languages of the \dllite~family that are sublanguages of $\mathcal{ALCHI}$, like \dllitecoreH and \dllitehornH, the ones used here.
\end{remark}


\subsubsection{Reasoning Procedures for \dllitehornH~under RC}\label{sect_RC_horn}

\nd We now present the decision procedures for RC tailored for \dllitehornH. The procedures have been obtained by slightly modifying the corresponding ones defined for $\mathcal{EL}_\bot$~\cite{Casini19}. 





We start by showing how to check whether $\KB\minent \alpha$, where $\alpha$ is a DI.
So, consider a $\dllitehornH$ defeasible KB $\KB=\tuple{\T,\D}$, where $\T$ is a $\dllitehornH$ TBox $\T$ and a DBox $\D$ containing DIs of form 
$\{B_1,\ldots,B_n\}\dsubs A$. Let $\delta$ be a new concept name, called \emph{$\delta$-atom}, not occurring in $\KB$. 
The \emph{$\delta$-transformation} of $\KB$ is defined as $\KB^\delta=\T\cup\D^\delta$, where 
\begin{equation}
\D^\delta\defined\{\{B_1,\ldots,B_n,\delta\}\subs A\mid \{B_1,\ldots,B_n\}\dsubs A\in \D\} \ .    
\end{equation}

\nd Whether a DI is in the RC of a defeasible $\dllitehornH$ KB $\KB=\tuple{\T,\D}$ is decided using the procedure $\mathtt{RationalClosure.Horn}$, that in turn is based on the procedures $\mathtt{ComputeRanking.Horn}$ and $\mathtt{Exceptional.Horn}$.

\begin{procedure}[h]
\caption{Exceptional.Horn($\T,\E$)}\label{algexH}
 \KwIn{\dllitehornH~TBox $\T\text{ and set of DIs } \E\subseteq{\D}$}
 \KwOut{$\E'\subseteq\E$ such that $\E'$ is a set  of exceptional axioms \wrt~$(\T,\E)$}
$\E':=\emptyset$\;
$\E^{\delta}=\{\{B_1,\ldots,B_n,\delta\}\subs A\mid \{B_1,\ldots,B_n\}\dsubs A\in \E\}$, where  $\delta$ is a new concept name\;
 \ForEach{$\{B_1,\ldots,B_n,\}\dsubs A\in \E$}{	
   	\If{$\T\cup\E^{\delta}\entails \{B_1,\ldots,B_n,\delta\}\subs \bot$}{$\mathcal{E'}$ := $\mathcal{E'}\cup\{\{B_1,\ldots,B_n\}\dsubs A\}$\;}
}
\Return{$\mathcal{E'}$}
\end{procedure}

 \begin{procedure}[h]
\caption{ComputeRanking.Horn($\KB$)\label{algrankH}}
\KwIn{Defeasible \dllitehornH~KB $\KB=\tuple{\T,\D}$}
\KwOut{Defeasible \dllitehornH~KB $\tuple{\T^*,\D^*}$, partitioning (ranking) $\RR =\{\D_0,\ldots,\D_n\}$ of $\D^*$}
$\T^*$:=$\T$\;
$\D^*$:=$\D$\;
$\RR$:=$\emptyset$\;
\Repeat {$\D_\infty=\emptyset$}
{$i$ := $0$\;
$\mathcal{E}_{0}$ := $\D^*$\;
$\mathcal{E}_{1}$ := $\mathtt{Exceptional.Horn}$($\T^*,\mathcal{E}_{0}$)\;
\While{$\E_{i+1}\neq\E_{i}$}{
$i$ := $i$ + 1\;
$\E_{i+1}$ := $\mathtt{Exceptional.Horn}$($\T^*,\E_{i}$)\;
}
$\D_\infty$ := $\E_{i}$\;
$\D^*$ := $\D^*\setminus\D_\infty$\;
$\T^*$ := $\T^*\cup\{\{B_1,\ldots B_n\} \subs \bot\mid \{B_1,\ldots B_n\}\dsubs A \in \D_\infty\}$\;}

\For{$j$ = $1$ to $i$}{
$\D_{j-1}$ := $\E_{j-1}\setminus\E_{j}$\;
$\RR$ := $\RR\cup\{\D_{j-1}\}$\;
}
\Return{$\tuple{\tuple{\T^*,\D^*},\RR}$}
\end{procedure}

\begin{procedure}[h]
\caption{RationalClosure.Horn($\KB,\alpha$)} \label{RCalgH}
{
\KwIn{Defeasible \dllitehornH~KB $\KB=\tuple{\T,\D}$ and DI $\alpha$ of the form  $\{B_1,\ldots, B_n\} \dsubs A$}
\KwOut{$\mathtt{true}$ iff $\{B_1,\ldots, B_n\} \dsubs A$ is in the Rational Closure of $\KB$}
$CL : = \T \entails \{B_1,\ldots, B_n\} \subs A$ //Check if $\alpha$ holds classically\;
  \If{CL}{\Return{$CL$}}
  $\tuple{\tuple{\T^*,\D^*},\{\D_0,\ldots,\D_n\}}$ := $\mathtt{ComputeRanking.Horn}$($\KB$)\;
  $CL : = \T^* \entails \{B_1,\ldots, B_n\} \subs A$ //Check if $\alpha$ holds classically, after finding strict knowledge in $\D$\;
  \If{CL}{\Return{$CL$}}
  //Compute $\{B_1,\ldots, B_n\}$'s rank $i$\;
  $i$ :=  $0$; $\D_\R$ :=  $\D^*$\;
  $\D_{\delta_0}:=\{\{B'_1,\ldots, B'_m,\delta_0\}\subs A'\mid \{B'_1,\ldots, B'_m\}\dsubs A'\in\D_\R\}$, where  $\delta_0$ is a new concept name\;
  \While{$\T^*\cup\D_{\delta_i}\entails \{\{B_1,\ldots, B_n,\delta_i\}\subs \bot$ {\bf and} $\D_\R\neq\emptyset$}{
    $\D_\R$ := $\D_\R\backslash\D_{i}$; $i$ := $i + 1$\;
    $\D_{\delta_i}:=\{\{B'_1,\ldots, B'_m,\delta_i\}\subs A'\mid \{B'_1,\ldots, B'_m\}\dsubs A'\in\D_\R\}$, where  $\delta_i$ is a new concept name\;
  }
  // Check now if $\alpha$ holds under RC\;
  \eIf{$\D_\R\neq\emptyset$}{\Return{$\T^*\cup\D_{\delta_i}\entails 
  \{B_1,\ldots B_n,\delta_i\}\subs A$}}{\Return{$CL$}}
  }
\end{procedure}

\nd As shown in \cite{BritzEtAl2021,CasiniStraccia2010}, the decision  whether a DI $C\dsubs D$ is in the RC of a defeasible KB $\KB=\tuple{\T,\D}$ can be reduced to procedures that use classical DL reasoners. The limit of such procedures is that they use the full expressivity of the DL \ALC, that is, they also use the negation and the disjunction operators. Note that the latter is absent in \dllite, while the use of the former is heavily constrained, and if we rely on decision procedures using them, the computational costs rise significantly with respect of the classical decision problems for low-complexity DLs like \dllitehornH. Hence, what we present here are procedures that avoid the use of disjunction and do not use negation beyond the \dllite~contraints, still obtaining exactly the same derivations of the procedures in \cite{BritzEtAl2021}, but preserving the low computational costs. The correctness of such procedures is proved by the following theorem.

\begin{theorem}\label{prop_dllitehornRC}
Given a defeasible \dllitehornH~KB  $\KB=\tuple{\T,\D}$, a DI $\{B_1,\ldots,B_n\}\dsubs A$ is in the RC of $\KB$ (that is, $\KB\minent \{B_1,\ldots,B_n\}\dsubs A$) iff 
\[
\mathtt{RationalClosure.Horn}(\KB, \{B_1,\ldots,B_n\}\dsubs A) \text{ returns } \mathtt{true} \ .
\]
\end{theorem}





\nd We refer to  \ref{defdlliteproofs} for a detailed proof of Theorem \ref{prop_dllitehornRC}. The intuition behind the procedures $\mathtt{Exceptional.Horn}(\T,\E)$, $\mathtt{Computeranking.Horn}(\KB)$ and \\
$\mathtt{RationalClosure.Horn}(\KB,\alpha)$ is that we can use the new concept name $\delta$ to represent the most typical elements of the domain. Hence, an object in $B^\R \cap\delta^\R$ should represent a typical element of the concept $B$, and a subsumption like $\{B,\delta\}\subs A$ wants to represent that a typical instance of $B$ is also an instance of $A$. This allows us then to rely on classical \dllitehornH~decision procedures to draw conclusions about DIs.

The following examples illustrate how these procedures are applied to our guiding example.

\begin{example}[$\mathtt{Exceptional.Horn}$, Example \ref{ex_guiding_2} cont'd]\label{ex_guiding_3}
    We go through the application of the procedure $\mathtt{Exceptional.Horn}(\T,\D)$ to the KB $\KB=\tuple{\T,\D}$ from Example \ref{ex_guiding_2}.

The output set, $\E'$, will contain the exceptional axioms in $\D$ \wrt~$\KB$. In line 1 $\E'$ is initialised as the empty set. In line 2, the set $\D^\delta$ is defined:
\begin{itemize}
    \item $\D^\delta=\{\{Y,\delta\}\subs H, \{Du,\delta\}\subs\neg H$,
    $\{Du,\delta\}\subs Y$, $\{CDu,\delta\}\subs H$, \\ $\{CDu,Tp, \delta\}\subs \neg H\}$.
\end{itemize}

    We execute lines 3-5 for each axiom in $\D^\delta$. It can be verified that
    \begin{eqnarray*}
    \T\cup\D^\delta & \entails & \{Du,\delta\}\subs\bot \\
    \T\cup\D^\delta & \entails & \{CDu,\delta\}\subs\bot \\
    \T\cup\D^\delta &\entails & \{CDu,Tp,\delta\}\subs\bot \\
    \T\cup\D^\delta & \not\entails & \{Y,\delta\}\subs\bot \ .
    \end{eqnarray*}
   
   \nd Hence the procedure returns $\E'=\{Du\dsubs\neg H,Du\dsubs Y, CDu\dsubs H, \{CDu,  Tp\}\dsubs\neg H\}$, that are exceptional DIs \wrt~$\KB$.\footnote{We recall from Remark \ref{negdef} that the inclusions $\{Du,\delta\}\subs\bot$, $\{CDu,\delta\}\subs\bot$, $\{CDu,Tp,\delta\}\subs\bot$, and 
    $\{Y,\delta\}\subs\bot$, despite not involving proper \dllitehornH~inclusions, are equivalent to \dllitehornH~inclusions. For example, they can be expressed as $Du \subs\neg \delta$ and $CDu\subs\neg \delta$, $\{Tp,CDu\}\subs\neg \delta$, and $,Y\subs\neg \delta$, respectively.}
\qed 
\end{example}

\nd The procedure $\mathtt{Exceptional.Horn}(\T,\D)$ allows the identification of exceptional DIs in $\D$ relying on the $\delta$-transformation of the KB and a series of classical \dllitehornH~subsumption checks. Procedure $\mathtt{ComputeRanking.Horn}(\KB)$, going through the iterated application of procedure $\mathtt{Exceptional.Horn}$, gives us two results:

\begin{enumerate}
    \item The identification of the DIs in $\D$ which actually convey information that is equivalent to classical \dllitehornH~inclusions, and that consequently should be moved in the TBox.
    \item A ranking of the DIs in $\D$, organising the DIs in $\D$ according to the potential conflicts among them, moving from the most generic (\eg, \emph{Birds usually fly} at rank $0$), up to the most specific ones (\eg, \emph{Penguins usually don't fly} at a higher rank.)
\end{enumerate}

\nd Regarding the first point, there is the possibility that some defeasible information in $\D$ actually corresponds to strict information that should be in $\T$. A simple example is a KB that contains two axioms  $A\dsubs B$ and $A\dsubs \neg B$: typical $A$'s are $B$'s
 and typical $A$'s are not $B$'s. These two DIs, despite conveying defeasible information, are  in direct conflict with each other, and the only  ranked interpretations which can satisfy both of them are those in which $A$ is empty. That is, only the models satisfying $A\subs \bot$. Procedure $\mathtt{ComputeRanking.Horn}$ takes care of identifying such defeasible axioms (they are the ones which are eventually assigned rank $\infty$), and translate them into corresponding strict ones by transforming each KB $\KB=\tuple{\T,\D}$ into an \emph{equivalent}\footnote{We say that two defeasible KBs are equivalent if they have exactly the same ranked models.} KB  $\KB^*=\tuple{\T^*,\D^*}$ \cite[Lemma 12]{BritzEtAl2021}. We refer the reader to Example 4 in \cite{BritzEtAl2021} and Example 4 in \cite{Casini19} for more details.

\begin{example}[$\mathtt{ComputeRanking.Horn}$, Example \ref{ex_guiding_3} cont'd]\label{ex_guiding_4}
We go through the procedure $\mathtt{ComputeRanking.Horn}(\KB)$.
%
The outputs of the procedure will be two: 
\begin{enumerate}
    \item A KB $\KB^*=\tuple{\T^*,\D^*}$, that is equivalent to $\KB$, but in which all the DIs with infinite rank have been transformed into equivalent classical inclusions. 
    \item A ranking $\R =\{\D_0,\ldots,\D_n\}$ of $\D^*$ partitioning all the DIs according to their rank.
\end{enumerate}

\nd $\T^*$ and $\D^*$ are initialised as $\T$ and $\D$, respectively. The algorithm calls iteratively the procedure $\mathtt{Exceptional.Horn}$, determining a sequence of DIs $\E_0,\ldots,\E_i$: namely,

\begin{itemize}
    \item $\E_0=\D$.
    \item $\E_1$ is the set containing all the DIs that are exceptional \wrt~$\tuple{\T^*,\D^*}$, that is, $\tuple{\T,\E_0}$. By Example~\ref{ex_guiding_3}, $\E_1=\{Du\dsubs\neg H,Du\dsubs Y, CDu\dsubs H, \{CDu, Tp\}\dsubs\neg H\}$.
    \item $\E_2$ is the set containing all the DIs that are exceptional \wrt~$\tuple{\T,\E_1}$. It can be verified that $\mathtt{Exceptional.Horn}(\T,\E_1)$ returns $\E_2=\{CDu\dsubs H, \{CDu, Tp\}\dsubs\neg H\}$.
     \item $\E_3$ is the set containing all the DIs that are exceptional \wrt~$\tuple{\T,\E_2}$. It can be verified that $\mathtt{Exceptional.Horn}(\T,\E_2)$ returns $\E_3=\{\{CDu, Tp\}\dsubs\neg H\}$.
    \item $\E_4$ is the set containing all the DIs that are exceptional \wrt~$\tuple{\T,\E_3}$. It can be verified that $\E_4=\emptyset$. The latter implies that $\E_5=\emptyset$ too.
    \item Since $\E_4=\E_5=\emptyset$, we exit the {\bf while} loop in lines 8-10, and we have:
    \begin{itemize}
        \item $\D_\infty=\E_4=\emptyset$;
        \item Since $\D_\infty=\emptyset$, $\D^*$ does not change, that is, $\D^*=\D$;
        \item The same for $\T^*$, that remains equal to $\T$.      
    \end{itemize}
\end{itemize}

\nd Since $\D_\infty=\emptyset$, we exit the {\bf repeat-until} loop in lines 4-14, and we end up with the ranking $\RR=\{\D_0,\D_1,\D_2, \D_3\}$, where
\begin{itemize}
    \item $\D_0=\E_0\setminus\E_1=\{Y\dsubs H\}$;
    \item $\D_1=\E_1\setminus\E_2=\{Du\dsubs \neg H, Du\dsubs Y\}$;
    \item $\D_2=\E_2\setminus\E_3=\{CDu\dsubs  H\}$.
    \item $\D_3=\E_3\setminus\E_4=\{\{CDu, Tp\}\dsubs\neg H\}$.
\end{itemize}

\nd Eventually, the procedure outputs the KB $\tuple{\T,\D}$ and the ranking $\RR=\{\D_0,\D_1,$ $\D_2,\D_3\}$.
\qed 
\end{example}

{\color{black}

The procedure $\mathtt{ComputeRanking.Horn}$ relies on a series of iterations of the lines 4-14, each iteration moving strict information that is `hidden' in $\D$ to the TBox. Due to the presence of roles, such a procedure must be iterated to be sure that we move all the strict information from $\D$. While Example \ref{ex_guiding_4} above does not need multiple iterations, since we do not have DIs with infinite rank, here below we present a KB that needs more iterations of the lines 4-14.

\begin{example}\label{ex_infty}

Consider a KB $\KB=\tuple{\T,\D}$ with $\T=\{A\subs B, \exists R^{-}\subs A, D\subs\neg C\}$ and $\D=\{B\dsubs C, A\dsubs D, E\dsubs \exists R\}$, and run the algorithm $\mathtt{ComputeRanking.Horn}(\KB)$. In the first iteration of the lines 4-14,  $\T\cup\D^\delta$ entails $\{A,\delta\}\subs\bot$, but neither $\{B,\delta\}\subs\bot$ nor $\{E,\delta\}\subs\bot$. Hence we obtain the following sequence $\E_0,\E_1,\ldots$:

\begin{itemize}
    \item $\E_0=\{A\dsubs D, B\dsubs C, E\dsubs \exists R\}$;
    \item $\E_1=\{A\dsubs D\}$;
    \item $\E_2=\{A\dsubs D\}$.
\end{itemize}

\nd That is, $\D_\infty=\{A\dsubs D\}$, $\D^*=\{B\dsubs C, E\dsubs \exists R\}$, and $\T^*=\T\cup\{A\subs \bot\}$. Since $\D_\infty\neq\emptyset$, we need to iterate again lines 4-14. Now $\T^*\cup{\D^*}^\delta$ entails $\exists R^-\subs\bot$, and consequently, since $\{E,\delta\}\subs\exists R\in{\D^*}^\delta$, $\{E,\delta\}\subs\bot$. That is,

\begin{itemize}
    \item $\E_0=\{B\dsubs C, E\dsubs \exists R\}$;
    \item $\E_1=\{E\dsubs \exists R\}$;
    \item $\E_2=\{E\dsubs \exists R\}$.
\end{itemize}

\nd We end up with $\D_\infty=\{E\dsubs \exists R\}$, $\D^*=\{B\dsubs C\}$, and $\T^*=\T\cup\{ A\subs \bot, E\subs\bot\}$. Since $\D_\infty\neq\emptyset$, we need to iterate again lines 4-14. The third iteration does not point any exceptionality out, and we end up with a KB $\KB^*=\tuple{\T^*,\D^*}$ with the only defeasible inclusion $B\dsubs C$.

\qed 
\end{example}
}

\nd Once we have moved all the DIs with infinite rank to the TBox (that is, obtained $\KB^*$ from $\KB$) and ranked all the DIs in $\D^*$, we are ready to query the KB, \ie~checking whether a DI $\{B_1,\ldots,B_n\}\dsubs A$ is in the RC of $\KB^*$.

\begin{example}[$\mathtt{RationalClosure.Horn}$, Example 
\ref{ex_guiding_4} cont'd]\label{ex_guiding_5}

Eventually, we go through the $\mathtt{RationalClosure.Horn}(\KB,\alpha)$ procedure, where $\KB=\tuple{\T,\D}$ is as in Example \ref{ex_guiding_1}. Let the query be $\alpha=\{Du, H\}\dsubs Y$, that is, we wonder if in the RC of $\KB$ we can conclude that happy drug users are presumably young. 

We go through the procedure $\mathtt{RationalClosure.Horn}(\KB, \{Du, H\}\dsubs Y)$.  In line 1 we check whether the inclusion $\{Du, H\}\subs Y$ is a consequence of the TBox $\T$. Since $\T=\{Du\subs \exists use.D, CDu\subs Du\}$, that is not the case.

In line 4 we introduce the outputs of $\mathtt{ComputeRanking.Horn}(\KB)$, that, as explained in Example \ref{ex_guiding_4}, consist of $\T^*=\T$, $\D^*=\D$, and $\RR=\{\D_0,\D_1,\D_2,\D_3\}$, with $\D_0=\{Y\dsubs H\}$, $\D_1=\{Du\dsubs \neg H, Du\dsubs Y\}$,  $\D_2=\{CDu\dsubs  H\}$, and $\D_3=\{\{CDu, Tp\}\dsubs\neg H\}$.

Since $\T^*=\T$, lines 5-7 are redundant, as they repeat the check done in line 1. From line 9 we start calculating the rank of $\{Du, H\}$. We set $i=0$ and $\D_{\R}=\D$, and check whether $\{Du, H\}$ has rank $0$ by checking whether it can be populated by an object in layer $0$. 

\begin{itemize}
    \item Line 9: we set $i=0$ (we are considering layer $0$), and initialise $\D_\R=\D$. $\D_R$ represents the set of defeasible axioms we consider at each step. For rank $0$, we consider the entire set $\D$.
    \item Line 10: We apply the $\delta$-transformation to $\D_\R$:
    \[
    \D_{\delta_0}:=\{\{B'_1,\ldots,B'_m,\delta_0\}\subs C'\mid \{B'_1,\ldots,B'_m\}\dsubs C'\in\D_\R\} \ ,
    \] 
    
    \nd where  $\delta_0$ is a new concept name.

    \item Lines 11-13: we start the while loop by checking if $\T\cup \D_{\delta_0}$ implies the negation of the antecedent of the query, in particular, whether $\T\cup \D_{\delta_0}\entails \{Du, H,\delta_0\}\subs \bot$ (that is, whether $\T\cup \D_{\delta_0}$ entails an equivalent \dllitehornH~inclusion, like $\{Du, H\}\subs \neg \delta_0$). That is actually the case:
    \begin{itemize}
        \item $\T\cup \D_{\delta_0}\entails \{Du, H\}\subs \neg \delta_0$.
    \end{itemize}

    Hence we repeat the while-loop by setting:
    \begin{itemize}
        \item $\D_\R=\D_1\cup\D_2\cup\D_3$;
        \item $i=1$;
        \item $\D_{\delta_1}$ is the $\delta$-transformation of $\D_\R$.
    \end{itemize}

Again, for $i=1$ it holds that $\T\cup \D_{\delta_1}\entails \{Du, H\}\subs \neg \delta_1$, and we repeat the while-loop for $\D_\R=\D_2\cup\D_3$ and $i=2$. Eventually, we have $\T\cup \D_{\delta_2}\not\entails \{Du, H\}\subs \neg \delta_2$, and we can exit the while-loop.

\item Lines 15-16: Since $\D_\R\neq \emptyset$, we move to line 16 to check whether  $\{Du, H,\delta_2\}\subs Y$ is in the RC of $\KB$, and we do that by checking whether $\T^*\cup\D_{\delta_2}\entails \{Du, H,\delta_2\}\subs Y$. 
Now, it can be verified that
\begin{itemize}
    \item $Du\subs \exists use.D, CDu\subs Du, \{CDu,\delta_2\}\subs  H, \{CDu, Tp,\delta_2\}\subs\neg H\not\entails \{Du, H,\delta_2\}\subs Y$,
\end{itemize}
    
 and we can conclude that $\{Du, H\}\dsubs Y$ is not in the RC of our KB $\KB=\tuple{\T,\D}$.

\end{itemize}

\qed 
\end{example}

\begin{remark}
    We would like to point out an interesting property of \dllitehornH, regarding one of the main limits of the RC construction in \cite{BritzEtAl2021} or \cite{GiordanoEtAl2015}. As pointed out in \cite{PenselTurhan2018}\footnote{\cite{PenselTurhan2018} offers a possible solution for this problem for Defeasible $\mathcal{EL}_{\bot}$.}, RC for $\mathcal{ALC}$ does not allow us to extend presumptive reasoning beyond the role connections. The following example should clarify the problem. 
    
    Assume our KB contains only $B_1\subs \exists R.B_2$ and $B_2\dsubs A$. If not informed of the contrary, we would like to be able to apply the presumption of typicality also to the objects reached through a role, that is, we would like to be able to derive also $B_1\dsubs \exists R.A$. That cannot be done in the RC presented in \cite{BritzEtAl2021} and \cite{GiordanoEtAl2015}. 

    We argue that this is not the case for \dllitehornH, because of its particular expressivity. In Remark \ref{remexpdllitecore} we have recalled how an inclusion like $B\subs \exists R.A$, despite being not immediately expressible in \dllitehornH, can be represented by the set of inclusions $\{B\subs \exists P, P\subs R, \exists P^-\subs A\}$, where $P$ is a newly introduced role name. Hence, we may think of  expressing $B\dsubs \exists R.A$ as $\{B\dsubs \exists P, P\subs R, \exists P^-\dsubs A\}$. 
    Now, assume that we have a KB $\KB$ consisting of 
\begin{eqnarray*}
    \T & = &\{P\subs R\} \\
    \D & = & \{B\dsubs \exists P, \exists P^-\dsubs A, A\dsubs A'\}
\end{eqnarray*}
  We leave to the reader checking that we end up with a simple ranking in which all the DIs in $\D$ have rank $0$ and we can conclude that $\KB\minent \exists P^-\dsubs A'$ (that is, the call to the procedure $\mathtt{RationalClosure.Horn}(\KB,\exists P^-\dsubs A')$ returns $\mathtt{true}$). Hence, in the RC of $\KB$ we have $\{B\dsubs \exists P, P\subs R, \exists P^-\dsubs A'\}$, that is, $B\dsubs \exists R.A'$, as desired.

    Instead, assume that we add a potential conflict in the KB is such a way that we are forced to conclude that the objects reachable from $B$'s through $P$ must be atypical. In this case we do not to conclude that $B\dsubs \exists R.A'$. For example, consider a KB $\KB'$ consisting of 
\begin{eqnarray*}
\T' & = & \{P\subs R, B'\subs \neg A'\} \\
\D' & = & \{B\dsubs \exists P, \exists P^-\dsubs A, A\dsubs A', \exists P^-\dsubs B'\}
\end{eqnarray*}

 \nd   The addition of $\exists P^-\dsubs B'$ and $B'\subs \neg A'$, combined with $\exists P^-\dsubs A$ and $A\dsubs A'$ tells us that the object that is seen through the role $P$ is generally and $A$, but it must be an atypical $A$ since, while typical $A$'s are also in $A'$, the objects seen through $P$ should also be in $B'$, which is incompatible with $A'$.

    Hence, what we obtain with the ranking procedure applied to $\KB'$ is that $\exists P^-\dsubs A$ and $\exists P^-\dsubs A'$ have rank $1$, that is,.
\begin{eqnarray*}
\D_{0}' & = & \{B\dsubs \exists P, A\dsubs A'\} \\
\D_{1}' & = & \{\exists P^-\dsubs A, \exists P^-\dsubs B'\}
\end{eqnarray*}

\nd    with the result that we cannot derive $B\dsubs \exists R.A'$ any more, as desired.
    
    This behaviour extends also to defeasible \dllitecoreH,\footnote{Note that in our example we have actually considered only the expressivity of \dllitecoreH.} which we are going to consider in the next section. We plan to present a proper semantic characterisation of this behaviour in a future work.
\end{remark}


We  consider now the case of deciding whether a classic GCI $\{B_1,\ldots,B_n\}\subs C$ is in the RC of a defeasible $\dllitehornH$ KB $\KB=\tuple{\T,\D}$, that is whether $\KB\minent \{B_1,\ldots,B_n\}\subs C$. In order to do that is sufficient to check whether $\T^*\entails \{B_1,\ldots,B_n\}\subs C$, where $\T^*$ is the TBox obtained from the procedure $\mathtt{ComputeRanking.Horn}(\KB)$ \cite[p.22]{BritzEtAl2021}. Trivially, this holds also for role inclusions $R\subs S$, since they are cannot be affected by defeasible information, as the latter deals only with concept inclusions. In fact, the following proposition can be shown.

\begin{proposition}\label{prop_classic_inclusion}
    Let $\KB=\tuple{\T,\D}$ any \dllitehornH~defeasible KB, and $\T^*$ be the TBox obtained from $\mathtt{ComputeRanking.Horn}(\KB)$. Then for any  classical \dllitehornH~inclusion $\alpha$,    $\KB\minent \alpha$ iff $\T^*\entails\alpha$.
\end{proposition}





\nd To summarise, given a \dllitehornH~defeasible KB $\KB=\tuple{\T,\D}$:

\begin{itemize}
    \item we can run $\mathtt{ComputeRanking.Horn}(\KB)$ to obtain the KB $\KB^*=\tuple{\T^*,\D^*}$;
    \item to check whether a DI $\{B_1,\ldots,B_n\}\dsubs A$ is in the RC of $\KB$, \\ 
    we run $\mathtt{RationalClosure.Horn}(\KB,\{B_1,\ldots,B_n\}\dsubs A)$;
    \item to check whether a GCI $\{B_1,\ldots,B_n\}\subs C$ is in the RC of $\KB$, we check whether $\T^*\entails\{B_1,\ldots,B_n\}\subs C$;
    \item to check whether a role inclusion $R\subs S$ is in the RC of $\KB$, we check whether $\T^*\entails R\subs S$.
\end{itemize}

 \nd Moreover, we can check whether an antecedent $\{B_1,\ldots,B_n\}$ is exceptional w.r.t. $\KB$ (that is, whether $\T\cup\D^{\delta} \entails \{B_1,\ldots,B_n,\delta\}\subs\bot$) via Proposition \ref{entailmenthornH}. {\color{black}
 That is, in order to check whether  $\T\cup\D^{\delta} \entails \{B_1,\ldots,B_n,\delta\}\subs\bot$ it is sufficient to check whether $\T\cup\D^{\delta} \entails \{B_1,\ldots,B_n,\delta\}\subs\neg B_i$ ($1\leq i\leq n$) or $\T\cup\D^{\delta} \entails \{B_1,\ldots,B_n,\delta\}\subs\neg \delta$, that, by Proposition  \ref{entailmenthornH}, reduces to checking whether one of the following inclusions is in the NI-closure $\clnhH(\T\cup\D^{\delta})$:
 \begin{itemize}
     \item $CL\subs \neg B_i$, where $CL\subseteq \{B_1,\ldots,B_n,\delta\}$;
     \item $CL\subs \neg \delta$, where $CL\subseteq \{B_1,\ldots,B_n,\delta\}$.
 \end{itemize}

}



  


\subsubsection{Reasoning Procedures for \dllitecoreH~under RC}\label{sect_RC_core}

\nd In this section we are presenting the decision procedures optimised for \dllitecoreH. Let us note that in \dllitecoreH~we are  allowed to have only a single basic concept on the LHS, hence we are constrained to GCIs with the form $B\subs C$ and DIs with the form $B\dsubs A$.

Since a \dllitecoreH~KB is also a \dllitehornH~KB,  the procedures in the previous section are correct also for \dllitecoreH. However, since such procedures rely on the $\delta$-transformation, that in turn, due to the introduction of a conjunction on LHS of each DI, relies on the full expressivity of \dllitehornH, the computational costs are in line with \dllitehornH. The procedures we are going to presented here rely instead only on the \dllitecoreH~expressivity, resulting in better computational costs (see Section~\ref{compplex}). In order to prove the correctness and completeness of the following procedures, we are going to prove that, applied to \dllitecoreH~KBs, their outputs correspond to the ones obtained with the corresponding \dllitehornH~procedures.

Formally, to deal with defeasible \dllitecoreH, we introduce a new class of entities, called \emph{normal  concept names}. A normal concept name is a concept name of the form $\norm{B}$, where $B$ is a basic concept. Informally, $\norm{B}$ is thought as indicating the class of \emph{normal} $B$'s.\footnote{This resembles, \eg, the normal concepts occurring in~\cite{BonattiEtAl2015}.} That is, for each basic concept $B$ we introduce a new concept name $\norm{B}$ denoting the normal $B$'s. Given a defeasible \dllitecoreH~KB $\KB=\tuple{\T,\D}$, instead of the $\delta$-transformation we apply a different transformation, the \emph{n-transformation}, that uses the normal concept names. The \emph{n-transformation} of a KB $\KB=\tuple{\T,\D}$ is defined as  $\norm{\KB} = \norm{\T} \cup \norm{\D}$, where
\begin{equation} \label{ntrans}
\begin{array}{ll}
     \norm{\T} \defined&\T~\cup  \\
     & \{\norm{B_1} \subs \norm{B_2} \mid B_1\subs B_2\in\T\}~\cup\\
     & \{\norm{B_1}\subs \neg \norm{B_2} \mid B_1\subs \neg B_2\in\T\} \\ 
     & \\
    \norm{\D} \defined&\{\norm{B} \subs \norm{A} \mid B \dsubs A \in\D\} \ . 
\end{array}
\end{equation}

\nd We remind that, by Remark~\ref{roleinc}, $\norm{\T}$ contains also the set
\[
\{\norm{\exists R_1}\subs \norm{\exists R_2}, \norm{\exists {R_1^-}}\subs \norm{\exists {R_2^-}}\mid R_1\subs R_2\in\T\} \ .
\]



\nd Please note that a concept `$\norm{\exists R}$' is actually a novel concept name, not a basic concept involving an existential role, and $\norm{\KB}$ is indeed a \dllitecoreH~TBox. 

We also adjust the negative closure for this kind of KBs. The \emph{normal negative closure} of a defeasible  \dllitecoreH~KB $\KB=\tuple{\T,\D}$, denoted $\clnn(\norm{\KB})$, is inductively defined as for the NI-closure $\clncH(\norm{\KB})$ of $\norm{\KB}$ (see Section~\ref{nicor}) with the addition of the following rule:


\begin{description}

    





    \item[($\nSup$)] {\bf IF}  $B_1\subs \neg B_2 \in \clnn(\norm{\KB})$ {\bf THEN} 
    $\norm{B_1} \subs \neg \norm{B_2} $ is in $\clnn(\norm{\KB})$.
   
\end{description}


\nd That is, $\clnn(\norm{\KB})$ is computed as the closure of the rules ($\mathbf{\cHRef}$), ($\mathbf{\cHContr}$), ($\mathbf{\cHTrans}$), ($\mathbf{\RcH_i}$) ($1\leq i \leq 3$) and ($\mathbf{\norm{Sup}}$) applied to the n-transformation $\norm{\KB}$ of $\KB$.

Now, we have shown in the previous section that we can decide whether a basic concept $B$ is exceptional by checking in Procedures $\mathtt{Exceptional.Horn}$  if $\KB^\delta\entails \{B,\delta\}\subs\neg B$.  

Our first goal is to prove that working with \dllitecoreH~KBs such a check corresponds to checking whether $\norm{\KB}\entails\norm{B}\subs\neg\norm{B}$, that in turn corresponds to just checking whether the inclusion $\norm{B} \subs \neg \norm{B}$ is in $\clnn(\norm{\KB})$ (Proposition \ref{prop_dllitecoreexcept} below). That is, starting from $\KB$, we consider it's $n$-transformation $\norm{\KB}$, compute the NI-closure $\clncH(\norm{\KB})$ of $\norm{\KB}$, by considering also the ($\mathbf{\norm{Sup}}$) rule, and eventually check whether $\norm{B} \subs \neg \norm{B} \in \clnn(\norm{\KB})$. 

\begin{proposition}\label{prop_dllitecoreexcept}
Given any defeasible  \dllitecoreH~KB $\KB=\tuple{\T,\D}$ and any basic concept $B$, the following conditions are equivalent:
\begin{enumerate}
    \item $\KB^\delta\entails \{B,\delta\}\subs\neg B$;
    \item $\norm{\KB} \entails \norm{B} \subs \neg \norm{B}$;
    \item   $\norm{B} \impc \notc \norm{B} \in \clnn(\norm{\KB})$.
\end{enumerate}

\end{proposition}


\nd The procedure $\mathtt{Exceptional.Core}$, which uses Proposition~\ref{prop_dllitecoreexcept} for the exceptionality test in line 4, shows how to determine the set of exceptional axioms \wrt~a defeasible \dllitecoreH~KB $(\T,\E)$.

Whether a defeasible inclusion $B\dsubs A$ is in the RC of a defeasible $\dllitecoreH$ KB $\KB=\tuple{\T,\D}$ is then decided analogously to the case of \dllitehornH~by using the procedure $\mathtt{RationalClosure.Core}$, which is based on the procedures $\mathtt{ComputeRanking.Core}$ and $\mathtt{Exceptional.Core}$.


In fact, Proposition \ref{prop_dllitecoreexcept} allows one to prove that, once we restrict our expressivity to \dllitecoreH, the procedures $\mathtt{Exceptional.Core}$, $\mathtt{ComputeRanking.Core}$ and \\ $\mathtt{RationalClosure.Core}$ are equivalent to the procedures $\mathtt{Exceptional.Horn}$, \\ 
$\mathtt{ComputeRanking.Horn}$ and 
$\mathtt{RationalClosure.Horn}$, respectively (see  \ref{defdlliteproofs}). In the end, we can prove the following main theorem for defeasible \dllitecoreH.

\begin{theorem}\label{prop_dllitecoreRC}
Given a defeasible \dllitecoreH~KB  $\KB=\tuple{\T,\D}$, a DI $B\dsubs A$ is in the RC of $\KB$ iff 
$\mathtt{RationalClosure.Core}(\KB,B\dsubs A)$ returns $\mathtt{true}$.
\end{theorem}

\nd We refer the reader to~\ref{defdlliteproofs} for the detailed proofs of Proposition \ref{prop_dllitecoreexcept} and Theorem \ref{prop_dllitecoreRC}. As in Section \ref{sect_RC_horn}, here we are going to describe the procedures $\mathtt{Exceptional.Core}$, $\mathtt{ComputeRanking.Core}$ and $\mathtt{RationalClosure.Core}$ through our guiding example.

\begin{procedure}[h]
\caption{Exceptional.Core($\T,\E$)}\label{algexC}
{
 \KwIn{A  \dllitecoreH~TBox $\T$ and a set of DIs $\E\subseteq{\D}$}
 \KwOut{$\E'\subseteq\E$ such that $\E'$ is a set  of exceptional axioms \wrt~$(\T,\E)$}
$\E':=\emptyset$\;
$\norm{\KB} =  \norm{\T} \cup \norm{\E}$\; 
 \ForEach{$B \dsubs A \in\E$}{	
   	\If{$\norm{\KB} \entails \norm{B}\subs \neg \norm{B}$}{$\mathcal{E'}$ := $\mathcal{E'}\cup\{B \dsubs A\}$\;}
}
\Return{$\mathcal{E'}$}
}
\end{procedure}

\begin{example}[$\mathtt{Exceptional.Core}$]\label{ex_guiding_7}
\nd We continue using our guiding example to go through the application of the procedure $\mathtt{Exceptional.Core}$, but we simplify the knowledge base into $\KB'=\tuple{\T',\D'}$, with 
        \begin{itemize}
        \item $\T'=\{CDu\subs Du, Du\subs \exists P, \exists P^-\subs D,P\subs uses\}$\footnote{We remind (see Remark \ref{remexpdllitecore}) that the axioms $Du\subs \exists P$, $\exists P^-\subs D$, $P\subs uses$ (where $P$ is a newly introduced role name) represent the correct \dllite~formulation of the axiom $Du\subs \exists uses.D$  in the TBox. While in the \dllitehornH~case we used the succinct formulation,  we use the extended form in the \dllitecoreH~case since it is relevant for the $n$-transformation.},
        \item $\D'=\{Y\dsubs H, Du\dsubs\neg H,Du\dsubs Y, CDu\dsubs H\}$,
    \end{itemize}

    in order to use only \dllitecoreH~expressivity.
    
    As in Example \ref{ex_guiding_3}, the output $\E'$ of $\mathtt{Exceptional.Core}(\T',\D')$ must contain the exceptional axioms in $\D'$ \wrt~$\KB'$.

In line 1 $\E'$ is initialised as the empty set. In line 2, following \ref{ntrans}, the set $\norm{\KB'}=\tuple{\norm{\T'},\norm{\D'}}$ is defined as:
\begin{itemize}
    \item $\norm{\T'}=\T'\cup\{\norm{Du}\subs \norm{\exists P}, \norm{\exists P^-}\subs \norm{D}, \norm{CDu}\subs \norm{Du} \}$
    \item $\norm{\D'}=\{\norm{Y}\subs \norm{H}, \norm{Du}\subs\neg \norm{H}$,
    $\norm{Du}\subs \norm{Y}$, $\norm{CDu}\subs \norm{H}\}$.
\end{itemize}

    We execute lines 3-5 for each axiom in $\norm{\D'}$. We can check that
    \begin{eqnarray*}
    \norm{\T'}\cup\norm{\D'} & \entails & \norm{Du}\subs\neg \norm{Du} \\
    \norm{\T'}\cup\norm{\D'} & \entails & \norm{CDu}\subs\neg \norm{CDu} \\
    \norm{\T'}\cup\norm{\D'} & \not\entails & \norm{Y}\subs\neg \norm{Y} \ .
    \end{eqnarray*}
   
   \nd Hence the procedure returns $\E'=\{Du\dsubs\neg H,Du\dsubs Y, CDu\dsubs H, \}$, that are exceptional DIs \wrt~$\KB'$.
\qed 
\end{example}

 \begin{procedure}[h]
\caption{ComputeRanking.Core($\KB$)\label{algrankC}}
{
\KwIn{Satisfiable defeasible \dllitecoreH~KB $\KB=\tuple{\T,\D}$}
\KwOut{Defeasible \dllitecoreH~KB $\tuple{\T^*,\D^*}$, partitioning (ranking) $\RR =\{\D_0,\ldots,\D_n\}$ of $\D^*$}
$\T^*$:=$\T$\;
$\D^*$:=$\D$\;
$\RR$:=$\emptyset$\;
\Repeat {$\D_\infty=\emptyset$}
{$i$ := $0$\;
$\mathcal{E}_{0}$ := $\D^*$\;
$\mathcal{E}_{1}$ := $\mathtt{Exceptional.Core}$($\T^*,\mathcal{E}_{0}$)\;
\While{$\E_{i+1}\neq\E_{i}$}{
$i$ := $i$ + 1\;
$\E_{i+1}$ := $\mathtt{Exceptional.Core}$($\T^*,\E_{i}$)\;
}
$\D_\infty$ := $\E_{i}$\;
$\D^*$ := $\D^*\setminus\D_\infty$\;
$\T^*$ := $\T^*\cup\{B \subs \neg B\mid B\dsubs A \in \D_\infty\}$\;}

\For{$j$ = $1$ to $i$}{
$\D_{j-1}$ := $\E_{j-1}\setminus\E_{j}$\;
$\RR$ := $\R\cup\{\D_{j-1}\}$\;
}
\Return{$\tuple{\tuple{\T^*,\D^*},\RR}$}
}
\end{procedure}


\begin{example}[$\mathtt{ComputeRanking.Core}$, Example \ref{ex_guiding_7} cont'd]\label{ex_guiding_8}
We now run the procedure $\mathtt{ComputeRanking.Core}(\KB')$.

As for $\mathtt{ComputeRanking.Horn}$, the outputs of the procedure will be two: 
\begin{enumerate}
    \item A KB $\KB^*=\tuple{\T^*,\D^*}$,  equivalent to $\KB'=\tuple{\T',\D'}$, but in which all the DIs with infinite rank have been transformed into equivalent classical inclusions.
    \item A ranking $\RR =\{\D_0,\ldots,\D_n\}$ of $\D^*$ that partitions all the DIs according to their rank.
\end{enumerate}

\nd $\T^*$ and $\D^*$ are initialised as $\T'$ and $\D'$, respectively. The algorithm calls iteratively the procedure $\mathtt{Exceptional.Core}$, determining a sequence of DIs $\E_0,\ldots,\E_i$: namely,

\begin{itemize}
    \item $\E_0=\D'$.
    \item $\E_1$ is the set containing all the DIs that are exceptional \wrt~$\tuple{\T^*,\D^*}$, that is, $\tuple{\T',\E_0}$. By Example~\ref{ex_guiding_7}, $\E_1=\{Du\dsubs\neg H,Du\dsubs Y, CDu\dsubs H\}$.
    \item $\E_2$ is the set containing all the DIs that are exceptional \wrt~$\tuple{\T',\E_1}$. 
    \\ $\mathtt{Exceptional.Core}(\T',\E_1)$ returns $\E_2=\{CDu\dsubs H\}$.
    \item $\E_3$ is the set containing all the DIs that are exceptional \wrt~$\tuple{\T',\E_2}$. It turns out that $\E_3=\emptyset$, and consequently  $\E_4=\emptyset$ too.
    \item Since $\E_3=\E_4=\emptyset$, we exit the {\bf while} loop in lines 8-10, and we have:
    \begin{itemize}
        \item $\D_\infty=\E_3=\emptyset$;
        \item Since $\D_\infty=\emptyset$, $\D^*$ does not change, that is, $\D^*=\D'$;
        \item The same for $\T^*$, that remains equal to $\T'$.      
    \end{itemize}
\end{itemize}

\nd Since $\D_\infty=\emptyset$, we exit the {\bf repeat-until} loop in lines 4-14, and we end up with the ranking $\RR=\{\D_0,\D_1,\D_2\}$, where
\begin{itemize}
    \item $\D_0=\E_0\setminus\E_1=\{Y\dsubs H\}$;
    \item $\D_1=\E_1\setminus\E_2=\{Du\dsubs \neg H, Du\dsubs Y\}$;
    \item $\D_2=\E_2\setminus\E_3=\{CDu\dsubs  H\}$.
\end{itemize}

\nd Eventually, the procedure outputs the KB $\tuple{\T',\D'}$ and the ranking $\RR=\{\D_0,\D_1,\D_2\}$.
\qed 
\end{example}

\begin{procedure}[h]
\caption{RationalClosure.Core($\KB,\alpha$)} \label{RCalgC}
{
\KwIn{Defeasible \dllitecoreH~KB $\KB$ and DI $\alpha$ of the form  $B \dsubs A$}
\KwOut{$\mathtt{true}$ iff $B \dsubs A$ is in the Rational Closure of $\KB$}
$CL : = \T \entails B\subs A$ //Check if $\alpha$ holds classically\;
  \If{CL}{\Return{$CL$}}
  $\tuple{\tuple{\T^*,\D^*},\{\D_0,\ldots,\D_n\}}$ := $\mathtt{ComputeRanking.Core}$($\KB$)\;
  $CL : = \T^* \entails B\subs A$ //Check if $\alpha$ holds classically, after finding strict knowledge in $\D$\;
  \If{CL}{\Return{$CL$}}
  //Compute $B$'s rank $i$\;
  $i$ :=  $0$; $\D_\R$ :=  $\D^*$\;
  Let $\norm{\T^*}$ be the $n$-transformation of $\T^*$\;
  Let $\norm{\D_\R}$ bet the $n$-transformation of $\D_0$\;
  \While{${\norm{\T^*}} \cup \norm{\D_\R} \entails \norm{B}\subs \neg \norm{B}$ {\bf and} $\D_\R\neq\emptyset$}{
    $\D_\R$ := $\D_\R\backslash\D_{i}$; $i$ := $i + 1$\;
    Let $\norm{\D_\R}$ be the $n$-transformation of $\D_\R$\;
  }
  // Check now if $\alpha$ holds under RC\;
  \eIf{$\D_\R \neq \emptyset$}{\Return{${\norm{\T^*}} \cup \norm{\D_\R} \entails \norm{B}\subs \norm{A}$}}{\Return{$CL$}}
  }
\end{procedure}

\begin{example}[$\mathtt{RationalClosure.Core}$, Example 
\ref{ex_guiding_8} cont'd]\label{ex_guiding_9}

Finally, we go through the $\mathtt{RationalClosure.Core}$ procedure. We run $\mathtt{RationalClosure.Core}(\KB',CDu\dsubs Y)$. That is, we wonder if in the rational closure of $\KB'$ we can conclude that controlled drug users are presumably young. 

We go through the procedure $\mathtt{RationalClosure.Core}(\KB', CDu\dsubs Y)$.  In line 1 we check whether the inclusion $CDu\subs Y$ is a consequence of the TBox $\T'$. Since $\T'=\{Du\subs \exists use.D, CDu\subs Du\}$, that is not the case.

In line 4 we enter the outputs of $\mathtt{ComputeRanking.Core}(\KB')$, that, as explained in Example \ref{ex_guiding_8}, consist of $\T^*=\T'$, $\D^*=\D'$, and $\RR=\{\D_0,\D_1,\D_2\}$, with $\D_0=\{Y\dsubs H\}$, $\D_1=\{Du\dsubs \neg H, Du\dsubs Y\}$, and $\D_2=\{CDu\dsubs  H\}$.

Since $\T^*=\T'$, lines 5-7 are redundant, as the same check has been done in lines 1-3. From line 9 we start calculating the rank of $CDu$. We set $i=0$ and $\D_{\R}=\D$, and check whether $CDu$ has rank $0$ by checking whether it can be populated by an object in layer $0$. The  procedure is the following:

\begin{itemize}
    \item Line 9: we set $i=0$ (we are considering layer $0$). $\D_R$ represents the set of defeasible axioms we consider at each step, hence we initialise $\D_\R=\D'$. 
    
    \item Line 10-11: We apply the n-transformation to $\tuple{\T^*\D_\R}$:

    \begin{itemize}
        \item $\norm{\T^*}=\T'\cup\{\norm{Du}\subs \norm{\exists P}, \norm{\exists P^-}\subs \norm{D}, \norm{CDu}\subs \norm{Du} \}$.
        \item $\norm{\D_\R}=\{\norm{Y}\subs \norm{H}, \norm{Du}\subs\neg \norm{H}$,
    $\norm{Du}\subs \norm{Y}$, $\norm{CDu}\subs \norm{H}\}$.
    \end{itemize}

    \item Lines 12-14: we start the while loop by checking if $\norm{\T^*}\cup \norm{\D_\R}$ implies the negation of the antecedent of the query, in particular, whether $\norm{\T^*}\cup \norm{\D_\R}\entails  \norm{CDu}\subs \neg \norm{CDu}$. That is actually the case.

    Hence we repeat the while-loop by setting:
    \begin{itemize}
        \item $\D_\R=\D_1\cup\D_2$;
        \item $i=1$;
        \item $\norm{\D_\R}$ is the $n$-transformation of $\D_\R$.
    \end{itemize}

Again, for $i=1$ it holds that $\norm{\T^*}\cup \norm{\D_\R}\entails  \norm{CDu}\subs \neg \norm{CDu}$, and we repeat the while-loop for $\D_\R=\D_2$ and $i=2$. Eventually, we have $\norm{\T^*}\cup \norm{\D_\R}\not \entails  \norm{CDu}\subs \neg \norm{CDu}$, and we can exit the while-loop.

\item Lines 16-19: Since $\D_\R\neq \emptyset$, we move to line 17 to check whether  $CDu\dsubs Y$ is in the RC of $\KB'$, and we do that by checking whether $\norm{\T^*}\cup \norm{\D_\R} \entails  \norm{CDu}\subs  \norm{Y}$. 
Now, it can be verified that
\begin{itemize}
    \item $\T'\cup\{ \norm{Du}\subs \norm{\exists P}, \norm{\exists P^-}\subs \norm{D}, \norm{CDu}\subs \norm{Du}, \norm{CDu}\subs \norm{H}\}\not\entails \norm{CDu}\subs  \norm{Y}$,
\end{itemize}
    
 and we can conclude that $CDu\dsubs Y$ is not in the RC of our KB $\KB'=\tuple{\T',\D'}$.

\end{itemize}

\qed 
\end{example}

\nd To conclude this section, please note that also for the \dllitecoreH~case we can decide whether a classical \dllitecoreH~inclusion $B\subs C$ or $R\subs S$ is in the RC of a KB by simply checking whether it is derivable from the TBox. Since \dllitecoreH~is just a sublanguage of \dllitehornH~and we know that $\mathtt{ComputeRanking.Core}$ corresponds to $\mathtt{ComputeRanking.Horn}$ (see  \ref{defdlliteproofs}),  Corollary \ref{prop_classic_inclusion_core} below follows immediately from Proposition \ref{prop_classic_inclusion}.

\begin{corollary}\label{prop_classic_inclusion_core}
    Let $\KB=\tuple{\T,\D}$ any \dllitecoreH~defeasible KB, and $\T^*$ be the TBox obtained from $\mathtt{ComputeRanking.Core}(\KB)$. Given any  be  classical \dllitecoreH~inclusion $\alpha$, 
    $\KB\minent \alpha$ iff $\T^*\entails\alpha$.
\end{corollary}

\nd To summarise, given a \dllitecoreH~defeasible KB $\KB=\tuple{\T,\D}$:

\begin{itemize}
    \item we can run $\mathtt{ComputeRanking.Core}(\KB)$ to obtain the KB $\KB^*=\tuple{\T^*,\D^*}$;
    \item to check whether a DI $B\dsubs A$ is in the RC of $\KB$, it suffices to run \\
    $\mathtt{RationalClosure.Core}(\KB,B\dsubs A)$;
    \item to check whether a GCI $B\subs C$ is in the RC of $\KB$, it suffices to check whether $\T^*\entails B\subs C$;
    \item to check whether a role inclusion $R\subs S$ is in the RC of $\KB$, check whether $\T^*\entails R\subs S$.
\end{itemize}

\nd Moreover, by Proposition \ref{prop_dllitecoreexcept}, we can check whether a concept $B$ is exceptional w.r.t. $\KB$ (that is, whether $\norm{\KB} \entails \norm{B} \subs \neg \norm{B}$) by checking the \emph{normal negative closure} of $\KB$ (that is, whether $\norm{B} \impc \notc \norm{B} \in \clnn(\norm{\KB})$).

\subsection{Rational Closure with ABox}\label{sect_RC_ABox}

\nd We now consider also the case of  RC in the presence of an ABox. Our guiding rationale is the application of the \emph{Presumption of Typicality} also to the individuals occurring in the ABox: that is, we assume that an individual $a$ behaves according to our expectations as much as possible, compatibly with the information we have about it. 

\begin{remark}\label{rem_ranking}
    In what follows we assume that we work with KBs $\KB=\tuple{\T,\D,\A}$ s.t. the pair $\tuple{\T,\D}$ has already been processed using  the procedure $\mathtt{ComputeRanking.Horn}$ (or the equivalent procedure $\mathtt{ComputeRanking.Core}$, depending on the expressivity of the KB). That is, we have the ranking 
    $\RR= \{\D_0,\ldots,\D_n\}$ associated to $\tuple{\T,\D}$, and each defeasible inclusion $\{B_1,\ldots,B_n\}\dsubs A$ of rank $\infty$ has been eliminated from $\D$ while the corresponding inclusion $\{B_1,\ldots,B_n\}\subs \bot$ has been added to $\T$.
\end{remark}

\nd As a first step, we show how to check whether a defeasible KB $\KB=\tuple{\T,\D,\A}$ is satisfiable. We have seen that every \dllitehornH~KB with an empty ABox is always satisfiable, both in the case of a classical TBox $\T$ (Proposition \ref{tsat}) and of a defeasible KB $\tuple{\T,\D}$ (Proposition \ref{dsat}). However, as pointed out in Section~\ref{reassect}, the presence of an ABox, without a DBox, may  lead to unsatisfiability and we have seen how to check it for a classical KB $\KB =\tuple{\T,\A}$. We next show how to check the satisfiability in case all three components are present, \ie~how to check the satisfiability of a defeasible KB $\KB = \tuple{\T,\D,\A}$. We do so by showing that it suffices to check the satisfiability of its classical part $\tuple{\T,\A}$, if $\tuple{\T,\D}$ has already been treated as pointed out in Remark \ref{rem_ranking}. The following propositions are proved for \dllitehornH~KBs, and consequently hold also for \dllitecoreH~KB's, since the latter is just a special case of the former.


\begin{proposition}\label{prop_defcons}
    Let $\KB=\tuple{\T,\D,\A}$ be any defeasible \dllitehornH~KB, processed as indicated in Remark \ref{rem_ranking}. Then $\KB$ is satisfiable iff $\tuple{\T,\A}$ is satisfiable. 
\end{proposition}

\nd As a consequence, any of the methods illustrated in Section~\ref{kbsatsect} (see Propositions~\ref{propCalvanese05}, \ref{propCalvanese0607}, \ref{propCalvanese0607B} and~\ref{nominalABoxBis}) can be applied to check the satisfiability of a defeasible KB once the rank of the DIs has been computed.

If $\KB=\tuple{\T,\D,\A}$ is satisfiable, in order to model the \emph{Presumption of Typicality} we  consider the minimal models also for the ABox, that is, the models in which the individuals in the ABox are interpreted in the lowest possible layers: given the available information, we assume that each individual is interpreted in the lowest possible rank, so that we can apply to it as much defeasible information as possible. For example, if in the ABox the only information we have about 
\eg~Tweety is that it is a bird ($\cass{tweety}{Bird}$), and we know that birds usually fly ($Bird\dsubs Fly$), while atypical birds like penguins and ostriches do not fly, we would like to be able to conclude that, presumably, Tweety flies ($\dass{tweety}{Fly}$), since we have no information hinting that Tweety is an atypical bird.
%
%

From the semantics point of view, we are going to consider the \emph{minimal configurations} also for the content of the ABox. We have seen in Section \ref{sect_RCsemantics} that in case a KB is made of a TBox and a DBox only, we have a unique minimal configuration, represented by the \emph{Big Ranked Model}.  However, in general, for defeasible DL KBs that include an ABox, the typical problem we have to address is to deal with the presence of multiple minimal interpretations of the ABox (see, \eg, Example 7 in \cite{CasiniStraccia2010} and Example 11 in \cite{Casini19}). Here is an illustrative example. 

\begin{example}\label{example_multiple}
    Consider the following defeasible (non \dllitehornH) 
    KB, $\KB=\tuple{\T,\D,\A}$, with
    \begin{eqnarray*}
        \T & =& \{        C\subs \neg D\} \\
        \D & = & \{A\dsubs C, B\dsubs C, \{A,\exists P.C\}\dsubs D \} \\
        \A & =  & \{\cass{a}{A},\cass{b}{B},\rass{a}{b}{P}\} \ ,
    \end{eqnarray*}

    \nd where $A,B,C$ and $D$ are concept names and $P$ is a role name. 
    Consider the \emph{Big Ranked Model} $\OI$ of $\tuple{\T,\D}$. $\D$ will be partitioned by a ranking procedure such as     $\mathtt{ \ref{algrankH}}$  into $\D_0=\{A\dsubs C, B\dsubs C\}$ and $\D_1=\{\{A,\exists P.C\}\dsubs D\}$. We want to show that it is impossible to have a single interpretations in which both $a$ and $b$ are interpreted in the lowest possible rank.

In fact, assume we interpret $a$ into some object $x$ in the bottom layer $L_{0}$. Then $a$ must satisfy all the DIs in $\D$: with that we mean that for all $E\dsubs F \in \D$, if $x \in \min_{\pref^{\OI}}(E^{\OI})$ then $x \in F^\OI$. Since $\cass{a}{A}\in \A$, we have $x\in A^{\OI}$, and more precisely, being $a\in L_0$, $x\in \min_{\pref^{\OI}}(A^{\OI})$. Therefore, as $\OI$ satisfies $A\dsubs C$, we also have  $x\in C^\OI$.


Assume now that we also interpret  $b$ into some object $y$ in the bottom layer $L_{0}$. Then also  for $y$ we have that, for all $E\dsubs F \in \D$, if $y \in \min_{\pref^{\OI}}(E^{\OI})$ then  $y \in F^\OI$ and, thus, as $\OI$ satisfies $B\dsubs C$ and  $y\in \min_{\pref^{\OI}}(B^{\OI})$, it follows that $y\in C^\OI$ too.
 %
But $\OI$ has to satisfy also $\{A, \exists P.C\}\dsubs D$ and, thus,
$x\in D^\OI$. But, since $x\in C^\OI$ and $C \subs \neg D\in\T$,  $x\notin D^\OI$, which is impossible. Consequently, if $x$ and $y$ are both in the minimal layer, we have a contradiction.
%
%
Therefore, we have two possible solutions:
\begin{itemize}
    \item we interpret $a$ into some object in layer $L_{0}$ and $b$ into some object in layer $L_{1}$. In this way we are  forced to conclude only $\cass{a}{C}$ and $\cass{a}{\neg D}$, but neither $\cass{b}{C}$ nor $\cass{a}{D}$;
    \item we interpret $b$ into some object in layer $L_{0}$ and $a$ into some object in layer $L_{1}$. In this way we are  forced to conclude  $\cass{b}{C}$, but neither $\cass{a}{C}$ nor $\cass{a}{D}$.
\end{itemize}

\nd In summary, we have two possible minimal interpretations for  $\KB$'s ABox, which are depicted in Figure \ref{figalternative}.
\qed
\end{example}

\begin{figure}
    \centering
    \includegraphics[width=1\linewidth]{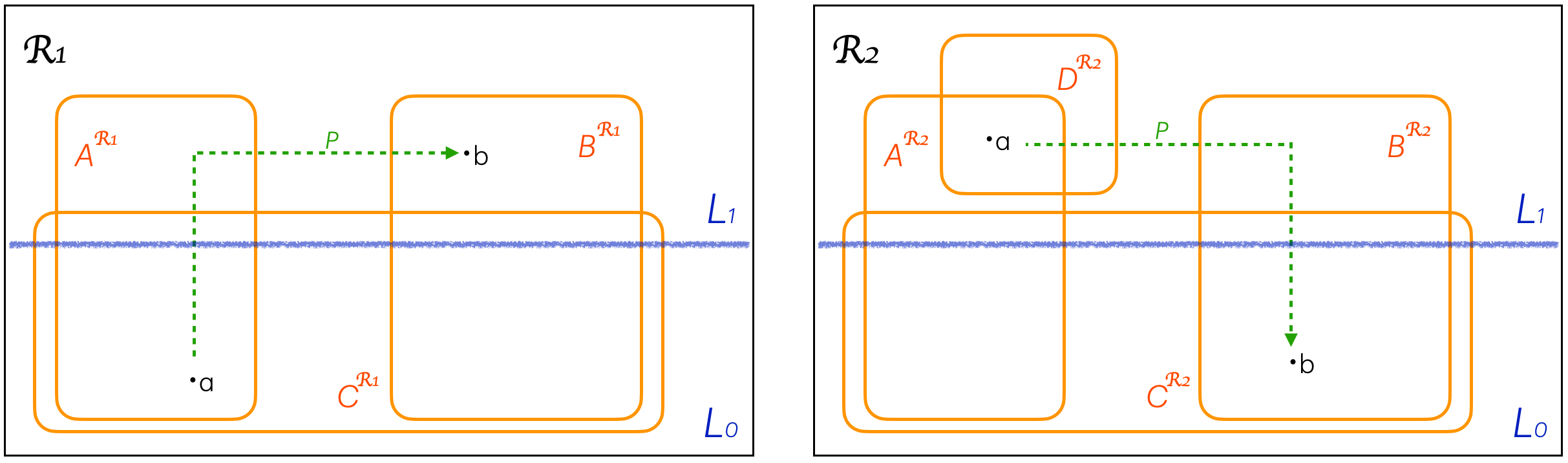}
    \caption{ two alternative minimal configurations for the ABox in Example~\ref{example_multiple}.}
    \label{figalternative}
\end{figure}


\nd The  example above shows a problem that is already present using some low-expressivity DL such as $\mathcal{EL}_\bot$~\cite{Casini19}, but please note that the KB above, in particular the DI $\{A,\exists P.C\}\dsubs D$,  is not expressible in $\dllitehornH$.  

What we are going to show next is that the occurrence of multiple minimal interpretations of an ABox, while it is a possibility in most of the DLs, is not possible in $\dllitehornH$ (and, consequently, in $\dllitecoreH$).

\begin{definition}[Minimal Height]\label{Def:minheight}
Let $\KB=\tuple{\T,\D,\A}$ be a satisfiable defeasible \\ \dllitehornH~KB and $\mathfrak{R}_{\KB}$ be the set of all the ranked models of $\KB$. We denote with $h_{\KB}(a)$  the \emph{minimal height} of $a$ \wrt~all the models in $\mathfrak{R}_{\KB}$, that is,
\[
h_{\KB}(a)=\min\{h_{\RI}(a^\RI)\mid \RI\in\mathfrak{R}_{\KB}\}.
\]
\end{definition}

\nd To decide the minimal height of an individual \wrt~$\KB$, it is actually sufficient to consider only the models in $\Moddelta{\KB}$ (see Section \ref{sect_RCsemantics}). This will be useful in the following sections.

\begin{proposition}\label{prop_minheight}
Consider any satisfiable defeasible \dllitehornH~KB $\KB=\tuple{\T,\D,\A}$ and any individual $a$. The minimal height of $a$ depends on the models in $\Moddelta{\KB}$ only, that is,
\[
h_{\KB}(a)=\min\{h_{\RI}(a^\RI)\mid \RI\in \Moddelta{\KB}\}.
\]
\end{proposition}

\nd Given a model $\RI$ of a satisfiable \dllitehornH~KB $\KB=\tuple{\T,\D,\A}$, we say that an individual $a$ is \emph{minimally interpreted} in $\RI$ \wrt~$\KB$ if $h_{\RI}(a^{\RI})=h_{\KB}(a)$. Let's go back to Example \ref{example_multiple}, where we consider one model in which the height of $a$ is $0$ and the height of $b$ is $1$, and one model in which it is the other way around. Hence we can conclude that $h_\KB(a)=h_{\KB}(b)=0$, but, as we have seen, we cannot have a single model of $\KB$ in which both $a$ and $b$ are minimally interpreted. Instead, the following Proposition \ref{prop_unique_abox} shows that, once we restrict our expressivity to \dllitehornH, the problem of the multiple minimal configurations of the ABox disappears.



\begin{proposition}\label{prop_unique_abox}
  Let $\KB=\tuple{\T,\D,\A}$ be a satisfiable defeasible \dllitehornH~KB. Then there is a model $\RI$ of $\KB$ in which every individual 
  is minimally interpreted \wrt~$\KB$, \ie~$h_{\RI}(a^{\RI})=h_{\KB}(a)$, for all individuals $a$. 
\end{proposition}

\nd Now that we know that in \dllitehornH, given a satisfiable defeasible KB $\KB=\tuple{\T,\D,\A}$, we have only one minimal configuration, we can define an entailment relation in the presence of an ABox. That is, given $\KB=\tuple{\T,\D,\A}$, let $\OI$ be the big ranked model for $\tuple{\T,\D}$. We then consider the extensions of $\OI$ that are also models for $\A$, that is, the models of $\KB$ obtained by interpreting in the domain of $\OI$ also the individuals. We indicate with $\mathfrak{O}^{\KB}=\{\OI^{\KB}_1, \OI^{\KB}_2, \ldots\}$ the set containing all the models of $\KB=\tuple{\T,\D,\A}$ obtained extending $\OI$ with the interpretation of the individuals in $\N$.

We define a preference preorder $\leq_{\A}$ over the elements of $\mathfrak{O}^{\KB}$ \wrt~the height of the individual names in each of them. That is, for any pair of models $\OI^{\KB}_i, \OI^{\KB}_j\in\mathfrak{O}^{\KB}$,

\begin{equation} \label{prefo}
\OI^{\KB}_i\leq_{\A} \OI^{\KB}_j\text{ iff, for every }a\in \N,\ h_{\OI^{\KB}_i}(a)\leq h_{\OI^{\KB}_j}(a) \ ,   
\end{equation}

\nd where $h_{\OI^{\KB}_i}(a)$ is the height of an individual $a$ in the model $\OI^{\KB}_i$ (Def. \ref{Def:height}). $\OI^{\KB}_i<_{\A} \OI^{\KB}_j$ is defined as usual as an abbreviation for $\OI^{\KB}_i\leq_{\A} \OI^{\KB}_j$ and $\OI^{\KB}_j\not\leq_{\A} \OI^{\KB}_i$, and $\OI^{\KB}_i\equiv_{\A} \OI^{\KB}_j$ as $\OI^{\KB}_i\leq_{\A} \OI^{\KB}_j$ and $\OI^{\KB}_j\leq_{\A} \OI^{\KB}_i$. Now, we take under consideration the models in $\mathfrak{O}^{\KB}$ that are minimal \wrt~$<_{\A}$, namely
\[
\min_{<_{\A}}(\mathfrak{O}^{\KB})=\{\OI^{\KB}_i\in \mathfrak{O}^{\KB}\mid \not\exists~ \OI^{\KB}_j\in \mathfrak{O}^{\KB}\text{ s.t. }\OI^{\KB}_j<_{\A}\OI^{\KB}_i\} \ .
\]

\nd That is, we consider all the models in $\mathfrak{O}^{\KB}$ that interpret each individual in the lowest possible rank, that is, as typical as possible. Proposition \ref{prop_unique_abox} guarantees that in all the models $\OI^{\KB}_i$ in $\min_{<_{\A}}(\mathfrak{O}^{\KB})$ and each individual $a\in\N$, $h_{\OI^{\KB}_i}(a)$ is the same. That is, for every pair  $\OI^{\KB}_i,\OI^{\KB}_j \in\min_{<_{\A}}(\mathfrak{O}^{\KB})$, $\OI^{\KB}_i\equiv_{\A}\OI^{\KB}_j$. We can now extend the notion of \emph{minimal entailment relation} $\minent$ in the presence of an  ABox by referring to the models in $\min_{<_{\A}}(\mathfrak{O}^{\KB})$.
\begin{definition}[$\minent$ in the presence of an ABox]\label{def_RC_abox}
    Let $\alpha$ be any \dllitehornH~assertion axiom of form $\dass{a}{A}$ or $\rass{a}{b}{P}$. A \dllitehornH~KB $\KB=\tuple{\T,\D,\A}$ \emph{minimally entails} $\alpha$ (written $\KB\minent\alpha$) iff 
    \[
    \OI^{\KB}_i\sat \alpha\text{ for every }\OI^{\KB}_i\in \min_{<_{\A}}(\mathfrak{O}^{\KB}) \ .
    \]
\end{definition}

\nd Note that we use the same symbol $\minent$ to indicate two non-monotonic entailment relations: the former for DIs and based on the big ranked model $\OI$ (Proposition \ref{prop_minent}), the latter for ABox assertions and based on the set of models $\min_{<_{\A}}(\mathfrak{O}^{\KB})$ (Definition \ref{def_RC_abox}). We use the same symbol since, when considering the DIs, referring either to $\OI$ or to the models in $\min_{<_{\A}}(\mathfrak{O}^{\KB})$ is equivalent.
\begin{proposition}\label{prop_entailment_equiv}
    Let $\KB=\tuple{\T,\D,\A}$ be a satisfiable defeasible \dllitehornH~KB and let $B\dsubs A$ be a DI.
    Then
    \[
    \OI\sat B\dsubs A\text{ iff }\OI^{\KB}_i\sat B\dsubs A\text{ for every }\OI^{\KB}_i\in \min_{<_{\A}}(\mathfrak{O}^{\KB}) \ .
    \]
\end{proposition}


\nd The entailment relation $\minent$ is of course non-monotonic \wrt~assertions of the form $\dass{a}{B}$ as the following example shows.

\begin{example}
    Consider a KB $\KB=\tuple{\T,\D,\A}$ with
    \begin{itemize}
        \item $\T=\emptyset$;
        \item $\D=\{A\dsubs C,A'\dsubs \neg C\}$;
        \item $\A=\{\cass{a}{A}\}$
    \end{itemize}

\nd   In $\KB$ there is no exceptionality in $\D$; that is, $\D_0=\D$. $a$ can be interpreted in $L_0$ and consequently, because of the DI $A\dsubs C$, it turns out that in all the models $\OI^{\KB}_i\in \min_{<_{\A}}(\mathfrak{O}^{\KB})$, $\OI^{\KB}_i\sat\cass{a}{C}$. However, if we add to the ABox the assertion $\cass{a}{A'}$, then we are forced to interpret $a$ in  $L_1$ or higher layers, since now, if we interpret $a$ in $L_0$, we would be forced to conclude both $\cass{a}{C}$ and $\cass{a}{\neg C}$. But once in every minimal model $a$ is in  $L_1$, we are not forced anymore to conclude neither $\cass{a}{C}$ nor $\cass{a}{\neg C}$.

Similarly, we could be forced to drop a conclusion even if we change the content of $\tuple{\T,\D}$. Consider a KB $\KB'=\tuple{\T',\D',\A'}$ with
    \begin{itemize}
        \item $\T'=\emptyset$;
        \item $\D'=\{A\dsubs C\}$;
        \item $\A'=\{\cass{a}{A}, \cass{a}{A'}\}$
    \end{itemize}

\nd    Again, we can conclude that $\KB\minent \cass{a}{C}$. But if we add $A'\dsubs \neg C$ to $\D'$ we have to drop such a conclusion, since, as in the previous case, $a$ must be interpreted into an exceptional object.
\qed
\end{example}

\nd Since here we do not consider defeasibility involving role inclusions, non-monotonicity does not hold for assertions of the kind $\rass{a}{b}{P}$. In such a case the relation $\minent$ corresponds to the classical \dllitehornH~entailment $\entails$. For the latter see Section~\ref{icssect}. 

\begin{proposition}\label{prop_entailment_equiv_roles}
    Consider any satisfiable defeasible \dllitehornH~KB $\KB=\tuple{\T,\D,\A}$,  and any assertion $\rass{a}{b}{P}$.
    Then
    \[
    \KB\minent \rass{a}{b}{P}\text{ iff }\tuple{\T,\A}\entails \rass{a}{b}{P} \ .
    \]
\end{proposition}

\nd We next focus on the decision whether $\KB\minent\dass{a}{A}$. 
Let us recap that from Proposition \ref{nominalABox} we know that in classical \dllitehornH, given a KB $\KB=\tuple{\T,\A}$, we can decide whether $\KB\entails \cass{a}{A}$ by checking \eg~whether $\T_\A\entails A_a\subs A$, where 
for each individual $a\in\N$, $A_a$ is a  new concept name and $\T_\A$ is defined by Eq.~\ref{ta}.

We want to apply the result of Proposition \ref{nominalABox} also to the defeasible reasoning. That is, we want to prove that, given a satisfiable defeasible KB $\KB=\tuple{\T,\D,\A}$, $\KB\minent \dass{a}{A}$ iff $\tuple{\T_\A,\D}\minent A_a\dsubs A$.  To do that, we first need to prove Proposition \ref{lemma_defeasible_nominals_2} below, proving that the rank of any individual $a$ \wrt~$\KB$ is the same as the corresponding concept $A_a$ \wrt~the KB $\tuple{\T_\A,\D}$. 

Specifically, we indicate by $\KB_\A$ the KB $\tuple{\T_\A,\D}$ obtained from the transformation of $\KB=\tuple{\T,\D,\A}$ into $\tuple{\T_\A,\D}$.
As a first step, we extend to individual name the notion of exceptionality that has been defined for concepts in Definition \ref{Def:Exceptionality}.

\begin{definition}\label{def_exceptional_indiv}
    Let $\KB=\tuple{\T,\D,\A}$ be a defeasible \dllitehornH~KB and $a$ any individual in $\N$. $a$ is \emph{exceptional} \wrt~$\KB$ iff there is no ranked model $\R$ of $\KB$ s.t. $h_\R(a^\R)=0$ (that is, $h_{\KB}(a)>0$). 
\end{definition}

\nd From the notion of exceptionality, we can define procedures for determining the ranking of any individual $a$ \wrt~a defeasible KB $\KB$, that corresponds to the minimal height $h_{\KB}(a)$. The first relevant result is Proposition \ref{lemma_defeasible_nominals_2}, showing that calculating the ranking of an individual $a$ \wrt~$\KB$ is equivalent to calculating the ranking of the corresponding concept $A_a$ \wrt~$\KB_\A$.


\begin{proposition}\label{lemma_defeasible_nominals_2}
    Let $\KB=\tuple{\T,\D,\A}$ be a satisfiable defeasible \dllitehornH~KB, $a$ any individual,  $\KB_\A$ the KB $\tuple{\T_\A,\D}$, and $A_a$ the concept corresponding to $a$ in $\KB_\A$. Then $h_{\KB}(a)=\rank_{\KB_\A}(A_a)$.
\end{proposition}

\nd From Proposition \ref{lemma_defeasible_nominals_2} we can prove now that $\KB\minent \dass{a}{A}$ can be decided by simply checking whether $A_a\dsubs A$ is in the rational closure of the KB $\KB_\A$. That is, we can reduce defeasible instance checking to the subsumption problem in RC, with the procedures defined in Sections \ref{sect_RC_horn} and \ref{sect_RC_core}.

\begin{proposition}\label{prop_defeasible_nominals_2}
    Let $\KB=\tuple{\T,\D,\A}$ be a satisfiable defeasible \dllitehornH~KB, $a$ any individual, and $A$ any concept name. Then,
    \[
    \KB\minent \dass{a}{A}\text{ iff }\KB_\A\minent A_a\dsubs A \ .
    \]
\end{proposition}

\nd Propositions \ref{lemma_defeasible_nominals_2} and \ref{prop_defeasible_nominals_2} are proved for \dllitehornH. Since any \dllitecoreH~KB is just a special case of a \dllitehornH~KB, Proposition~\ref{prop_defeasible_nominals_2} holds for \dllitecoreH~as well. Hence the instance checking problem for defeasible \dllitehornH~(resp.~\dllitecoreH) can be decided via procedure $\mathtt{RationalClosure.Horn}$ (resp.~$\mathtt{RationalClosure.Core}$).

We conclude by anticipating that, like for the classical case (see Remark~\ref{qansinstcheck}), 
the instance checking problem under RC can also be seen as a special case of query answering (see Remark~\ref{qansinstcheckDIs} later on), which we will adress next.

\section{Query Answering under RC}\label{sect_QRW}

\noindent\fbox{\begin{minipage}{0.975\textwidth}
\textsc{Overview:} 
 This section addresses the problem of answering conjunctive queries over defeasible KBs. The main finding is that one may define query reformulation procedures for \dllitecoreH~and \dllitehornH which are similar to the ones for their classical counterpart, and without compromising the computational cost (see Section~\ref{compplex}).

\end{minipage}}
\vspace*{5mm}

\nd In this section, we address the problem of answering \emph{Conjunctive Queries} (CQs) posed to defeasible KBs. Let us recap that  for classical DLs of  \dllite \ family we discussed here, this is done by relying on the so-called  \emph{query rewriting} technique~\cite{Calvanese05,Calvanese06a,Calvanese07}. 
We recall that a \emph{Conjunctive Query} (CQ) is a rule-like expression of the form 
\begin{equation} \label{eqdlcq}
q(\vec{x}) \leftarrow \exists \vec{y}.\varphi(\vec{x},\vec{y})
\end{equation}

\nd where  the rule body $\varphi(\vec{x},\vec{y})$  is a  conjunction $P_{1}(\vec{z}_{1})\land\ldots\land P_{n}(\vec{z}_{n})$\footnote{We may use the symbols `$\land$' or `$,$' to denote a conjunction in the rule body.} of unary or binary predicates  $P_{i}(\vec{z}_{i})$ ($1 \leq i \leq n$), where $P_{i}$ is either a concept name $A$ or a role name $P$. The variables in $\vec{x}$ are called the \emph{distinguished} variables of $q$, while the variables in $\vec{y}$ are called the \emph{non-distinguished} variables of $q$. $\vec{z}_{i}$ is a vector of \emph{terms}. A term is either an individual appearing in a KB $\KB$, a distinguished or a non-distinguished variable.  
If $P_{i}$ is a concept name (resp., a role name)  then $\vec{z}_{i}$ is a unary (resp., binary) tuple. For convenience, we may omit writing $\exists \vec{y}$ if clear from context.  We say that a variable is \emph{bound} if it corresponds to either a distinguished variable,  a shared variable (\ie~a variable occurring at least twice in the query body), or an individual, while we say that a variable is \emph{unbound} if it corresponds to a non-distinguished non-shared variable.
For instance,
\[
q(x) \leftarrow  Person(x), hasFriend(x,y), hasCitizenship(y,italy)
\]

\nd is a conjunctive query with only bound variables (either distinguished, or shared, or individuals), whose intended meaning is to retrieve all people having an Italian friend. 

A \emph{Boolean query} is a query without distinguished variables, while an \emph{Union of CQs} (UCQs) is a set $q$ of CQs having the same  \emph{head} $q(\vec{x})$. 

Given an interpretation~$\I$, with $q^\I$ we denote the set of tuples $\vec{t}$ of individuals such that 
the First-Order Formula  $\exists \vec{y}.\varphi(\vec{t},\vec{y})$ is satisfied by $\I$, denoted
\begin{eqnarray} 
    q^\I & = & \{\vec{t} \mid \I \sat q(\vec{t}) \} \ .   \label{qI}
\end{eqnarray}

\nd The \emph{answer} to $q$ over $\KB$  is the set $ans(q,\KB)$ of  tuples $\vec{t}$ of individuals appearing in $\KB$ such that $\vec{t} \in q^\I$ for every model $\I$ of $\KB$, denoted

\begin{equation} \label{anskbmain}
ans(\KB,q) = \{\vec{t} \mid \KB \models q(\vec{t}) \} \ .    
\end{equation}



\nd For UCQs $q$ we define:
\begin{equation} \label{unionanskb}
ans(\KB,q) = \{\vec{t} \mid \text{ there is } q_i \in q \text{ s.t. }\KB \models q_i(\vec{t}) \} \ .    
\end{equation}

\nd Now, conceptually, in order to find all answers to a CQ $q$ over a  KB 

\begin{enumerate}
\item first, we check if the KB is satisfiable, as querying a non-satisfiable KB does not make sense in our case;  

\item by considering the TBox $\T$ only, the  query $q$ is then \emph{reformulated} into a set of CQs, \ie~a UCQ, $r(q,\T)$. Informally, the basic idea is that the reformulation procedure closely resembles a top-down resolution procedure for logic programming, where each inclusion axiom $A_{1} \impc A_{2}$ is seen as a logic programming rule of the form $A_{2}(x) \leftarrow A_{1}(x)$. For instance, given the query $q(x) \leftarrow A(x)$ and suppose that $\T$ contains the inclusion axioms $A_{1} \impc A$ and $A_{2} \impc A$, then we can reformulate the query into two queries $q(x) \leftarrow A_{1}(x)$ and $q(x) \leftarrow A_{2}(x)$, exactly as it happens for top-down resolution methods in logic programming;

\item the reformulated queries in $r(q,\KB)$ are then \emph{evaluated} over the database $\mathtt{DB}(\A)$  only (which is $\A$ stored as a database, see Definition~\ref{def_DB(A)}), producing the requested answer set $ans(\KB,q)$. 
\end{enumerate}

\nd We recap that the main property of the query reformulation procedure is that (see~\cite{Calvanese05,Calvanese06a,Calvanese07,Calvanese13}) 
\begin{equation}\label{ansrel}
ans(\KB,q) =  \{\vec{t}\mid \text{ there is } q_{i} \in r(q,\KB) \text{ s.t. } \ \vec{t} \in q_{i}^{\mathtt{DB}(\A)} \} \ .    
\end{equation}


\nd The benefit of this approach is that the queries in $r(q,\KB)$ can  directly be evaluated by an SQL engine over the ABox stored as the database $\mathtt{DB}(\A)$, thus taking advantage of well-established query optimization strategies supported by current industrial-strength relational DB technology.



 In the following, for the sake of completeness, we succinctly recap  the query answering procedure for \dllitehornH~KBs without DIs of which the one for \dllitecoreH~is just a special case~\cite{Calvanese05,Calvanese06a,Calvanese07,Calvanese13}. Then, we show how to modify it to deal with DIs.

\subsection{Preliminaries: Query Rewriting for \dllitehornH}\label{sect_query_rewite}

\nd We recap here the query rewriting procedure used to answer CQs for \dllitehornH~KBs \cite{Calvanese05,Calvanese06a,Calvanese07,Calvanese13}. 
 So, given a CQ $q$  and a \dllitehornH~KB $\KB=\tuple{\T,\A}$, we recall that the answer set to $q$ over $\KB$  is the set $ans(q,\KB)$ of  tuples $\vec{t}$ of individuals appearing in $\KB$ as defined by Eq.~\ref{anskbmain}, while for the union of CQs the answer set is defined by Eq.~\ref{unionanskb}.

Notice that by definition $ans(\KB,q)$ is finite because $\KB$ is finite and the number of individuals appearing in $\KB$  is finite.
Notice also that the tuple $\vec{t}$ can be the empty tuple in the case in which $q$ is a Boolean conjunctive query. 
More precisely, in this case the set $ans(\KB,q)$ consists of the only  empty tuple if and only if the formula $q$ is satisfied in every model of $\KB$. Eventually, observe that, if $\KB$ is unsatisfiable, then $ans(\KB,q)$ is trivially the set of all possible tuples of individuals in $\KB$ whose arity is the one of the query, which is the reason while we omit this case here.

As anticipated, the query answering procedure over a  satisfiable \dllitehornH~KB is composed of two steps~\cite{Calvanese05,Calvanese06a,Calvanese07,Calvanese13}:
\begin{enumerate}
    \item by considering the TBox $\T$ only, a  query $q$ is reformulated into a set of CQs $r(q,\T)$; and then
    \item by leveraging Eq.~\ref{ansrel}, the answer set is computed by submitting all $q_i \in r(q,\KB)$ as SQL queries over 
    $\mathtt{DB}(\A)$.
\end{enumerate}

\nd As the second step is straightforward, below, we recap the query reformulation step only.
%
%
Specifically, we use the symbol ``$\_$'' to represent unbound variables. Note that an atom of the form $(\csome P)(z)$ (resp. $(\csome P^{-})(z)$) has the same meaning as $P(x,\_)$ (resp. $P(\_,z)$). For ease of exposition, in the following we will use the latter form only.  

So, given a satisfiable \dllitehornH~KB $\KB = \tuple{\T,\A}$, let $\spec \in \T_{PI}$ be a PI in $\T$. We say that $\spec$ is \emph{applicable} to a predicate $g$ occurring in the body of a query $q$ if:
\begin{enumerate}
\item $g$ is  of the form $A(z)$ and $\spec$ has $A$ in its right-hand side;
\item $g$ is of the form $P(z_{1}, z_{2})$ and one of the following conditions holds:
\begin{enumerate}
\item $z_{2} = \_$ and the right-hand side of $\spec$ is $\csome P$;
\item $z_{1} = \_$ and the right-hand side of $\spec$ is $\csome P^{-}$; 
\item $\tau$ is a role inclusion and its right-hand side is either $P$ or $P^-$.
\end{enumerate}  
\end{enumerate}

\nd We indicate with $gr(g, \spec)$ the predicate or the conjunction of predicates  obtained from the predicate
$g$ by applying the inclusion axiom $\spec$. Specifically,
\begin{enumerate}

    \item {\bf IF} $g=A(z)$ (resp.~$g=P(z,\_)$ or $g=P(\_,z)$) \textbf{AND} $\spec = \{B_1, \ldots, B_{n} \} \subs A$
    (resp. $\spec = \{B_1, \ldots, B_{n}\} \subs \csome P$ or $\spec = \{B_1, \ldots, B_{n} \} \subs \csome P^-$) 
    \textbf{THEN} $gr(g, \spec) = C_1(z) \land \ldots \land C_n(z)$, where for each $i \in \{1, \ldots, n\}$,
    \begin{itemize}
        \item $C_i(z) =A_i(z)$ if $B_i = A_i$, or
        \item $C_i(z) = P_i(z,\_)$ if $B_i = \csome P_i$, or
        \item $C_i(z) = P_i(\_,z)$ if $B_i = \csome P_i^-$; 
    \end{itemize}

    \item {\bf IF} $g = P(z_{1}, z_{2})$ {\bf AND} either
        $\tau =P_1 \subs P$ {\bf OR}
        $\tau = P_1^- \subs P^-$
        {\bf THEN} $gr(g, \spec) = P_1(z_{1}, z_{2})$;

   \item {\bf IF} $g = P(z_{1}, z_{2})$ {\bf AND} either
        $\tau =P_1 \subs P^-$ {\bf OR}
        $\tau = P_1^- \subs P$
        {\bf THEN} $gr(g, \spec) = P_1(z_{2}, z_{1})$.
\end{enumerate}

\nd With the above definitions at hand, the query reformulation algorithm is given by the $\mathtt{QueryRef}$ procedure.
In the procedure, $q_i[g/g']$ denotes the query obtained from $q_i$ by replacing the predicate $g$ with a new predicate or conjunction of predicates $g'$. At step 8, for each pair of predicates $g_{k}, g_{l}$ that unify, the procedure computes the query $q' = reduce(q_i,g_{k}, g_{l})$, by applying to $q_i$ the most general unifier between $g_{k}$ and $g_{l}$.\footnote{We say that two predicates $g_{k}= P(x_{1}, \ldots, x_{n})$ and $g_{l}= P(y_{1}, \ldots, y_{n})$ \emph{unify}, if for all $i$, either $x_{i} = y_{i}$ or  $x_{i} = \_$ or $y_{i} = \_ $. If $g_{k}$ and $g_{l}$ unify, then the unification of $g_{k}$ and $g_{l}$ is the predicate $P(z_{1}, \ldots, z_{n})$, where $z_{i}= x_{i}$ if $x_{i} = y_{i}$ or $y_{i} = \_$, otherwise $z_{i} = y_{i}$~\cite{Cali03}.} Due to the unification, variables that were bound in $q$ may become unbound in $q'$. Hence, inclusion axioms that were not applicable to predicates in $q_i$, may become applicable to predicates in $q'$ (in the next executions of step (5)). Function $\kappa$ applied to $q'$ replaces with $\_$ each unbound variable in $q'$.

It can be shown that $\mathtt{QueryRef}$ terminates, Eq.~\ref{ansrel} holds, but the size of $r(q,\KB)$ may be exponential \wrt~$\T$~\cite{Calvanese05,Calvanese06a,Calvanese07,Calvanese13}.

\begin{procedure}[t]
\caption{QueryRef($q,\KB$)\label{QueryRef}}
{
\KwIn{Union of CQs  $q = \{q_1, \ldots, q_m\}$ and satisfiable \dllitehornH~KB 
$\KB = \tuple{\T,\A}$.}
\KwOut{Set of reformulated conjunctive queries $r(q,\KB)$}
$r(q,\KB):= q$\;

\Repeat{$S=r(q,\KB) $}{
$S:= r(q,\KB)$\;
\ForEach{$q_i \in S$}{
\ForEach{$g \in q_i$}{
\If{$\spec \in \T_{PI}$ is applicable to $g$}{
$r(q,\KB) := r(q,\KB) \cup \{ q_i[g/gr(g,\spec)]\}$\;
}
}
\ForEach{$g_{k}, g_{l} \in q_i$}{
\If{$g_{k}$ and $g_{l}$ unify}{
$r(q,\KB) := r(q,\KB) \cup \{ \kappa(reduce(q_i,g_{k}, g_{l}))\}$\;
}
}
}
}
\Return{$r(q,\KB)$}\;
}
\end{procedure}

\subsection{Query Answering with DIs}\label{sect_query_rewiteDIs}

\nd In this section, we  show how to address the case of KBs with DIs.  To start with, consider a CQ $q$.  We define now the analogue notion of answer in the presence of DIs, \viz~the analogues of Eqs.~\ref{qI} - \ref{unionanskb} in the presence of DIs.
To do so, we first reformulate  Definition~\ref{def_RC_abox} for CQs.


\begin{definition}[$\minent$ for CQ answering]\label{def_RC_CQ}
    Consider a \dllitehornH~KB $\KB=\tuple{\T,\D,\A}$ and a CQ $q$. 
    We say that for a tuple $\vec{t}$ of individuals appearing in $\KB$, $\KB$ \emph{minimally entails} $q(\vec{t})$, denoted $\KB\entabox q(\vec{t})$, if $\exists \vec{y}.\varphi(\vec{t},\vec{y})$ is \emph{satisfied} by all  the models in $\mathfrak{O}^{\KB}$ in which all individuals are minimally interpreted. That is, $\KB\entabox q(\vec{t})$ iff 
    
    \[
    \OI^{\KB}_i\sat \exists \vec{y}.\varphi(\vec{t},\vec{y})\text{ for every }\OI^{\KB}_i\in \min_{<_{\A}}(\mathfrak{O}^{\KB}) \ .\footnote{The notion of satisfaction is defined similarly as for the FOL case in the obvious way.}
    \]


\end{definition}


\begin{definition}[RC-Answer] \label{rcanswer}
    Consider a satisfiable defeasible (\dllitehornH~or \dllitecoreH) KB $\KB=\tuple{\T,\D,\A}$. The \emph{RC-answer} to $q$ over $\KB$ is the set $ans_{RC}(q,\KB)$ of  tuples $\vec{t}$ of individuals appearing in $\KB$ such that $q(\vec{t})$ is minimally entailed by 
    $\KB$, \ie
        \begin{equation} \label{anskbmainDIs}
            ans_{RC}(\KB,q) = \{\vec{t} \mid \KB \entabox q(\vec{t}) \} \ .    
        \end{equation}


\nd Similarly to the non-defeasible case (\cf~with Eq.~\ref{qI}),  given a ranked interpretation~$\R$, with $q^\R$ we denote the set of tuples $\vec{t}$ of individuals such that the First-Order Formula  $\exists \vec{y}.\varphi(\vec{t},\vec{y})$ is satisfied by $\R$, denoted 
\begin{eqnarray}
    q^\R & = & \{\vec{t} \mid \R \sat q(\vec{t}) \} \ .   \label{qR}
\end{eqnarray}

\nd Then, obviously 
\begin{eqnarray}
\vec{t} \in ans_{RC}(\KB,q) & \text{iff} & 
\vec{t} \in q^{\OI^{\KB}_i} \text{ for every } \OI^{\KB}_i\in \min_{<_{\A}}(\mathfrak{O}^{\KB}) \ . \label{qReq}
\end{eqnarray}

\nd Finally, for UCQs $q$ we define:
\begin{equation} \label{unionanskbDIs}
ans_{RC}(\KB,q) = \{\vec{t} \mid \text{ there is } q_i \in q \text{ s.t. }\KB \entabox q_i(\vec{t}) \} \ .    
\end{equation}
   
\end{definition}

\nd We now want to define a query reformulation procedure similar to the classical reformulation procedure such that it has a DIs analogue of  Eq.~\ref{ansrel}. 

Here we present the  defeasible \dllitecoreH~case, while in~\ref{qahornh} we show the \dllitehornH~analogue.
%
%

Consider a satisfiable defeasible \dllitecoreH~KB $\KB=\tuple{\T,\D,\A}$. 
%
%
We  assume to have computed the rank of the LHS of PIs and DIs occurring in $\KB$ via procedure~$\mathtt{\ref{RankalgC}}$ below. Please note that from the ranking $r$ of a DI $B \dsubs A \in \D$, we also know that $\rank_{\KB}(B) = r$. For a PI $B \subs A \in \T$, if not already known from the DIs' rank, the rank of $B$ can be computed by executing lines 8 - 14 of the $\mathtt{RationalClosure.Core}$ procedure, which is the backbone of the  $\mathtt{Ranking.Core}$ procedure in this section.
\begin{procedure}[h]
\caption{Ranking.Core($\KB,B$)} \label{RankalgC}
{
\KwIn{Satisfiable defeasible \dllitecoreH~KB $\KB$ and a basic concept  $B$}
\KwOut{Rank value of $B$}
  $\tuple{\tuple{\T^*,\D^*},\{\D_0,\ldots,\D_n\}}$ := $\mathtt{ComputeRanking.Core}$($\KB$)\;
  //Compute $B$'s rank $i$\;
  $i$ :=  $0$; $\D_\R$ :=  $\D^*$\;
  Let $\norm{\T^*}$ be the $n$-transformation of $\T^*$\;
  Let $\norm{\D_0}$ bet the $n$-transformation of $\D_0$\;
  \While{${\norm{\T^*}} \cup \norm{\D_i} \entails \norm{B}\subs \neg \norm{B}$ {\bf and} $\D_\R\neq\emptyset$}{
    $\D_\R$ := $\D_\R\backslash\D_{i}$; $i$ := $i + 1$\;
    Let $\norm{\D_i}$ be the $n$-transformation of $\D_i$\;
  }
  \If{$\D_\R \neq \emptyset$}{\Return{$i$}}
  \eIf{$\D_\R = \emptyset$ {\bf and } ${\norm{\T^*}} \not \entails \norm{B} \subs \neg \norm{B}$}{\Return{$n+1$}}{\Return{$\infty$}}
  }
\end{procedure}
%
%

We also assume that the minimal height of each individual $a$ occurring in $\KB$, \ie~$a\in \N_\A$, has been computed.
This can be done either by using Proposition~\ref{lemma_defeasible_nominals_2} and, for each $a\in \N_\A$, apply the $\mathtt{Ranking.Core}$ procedure to $A_a$, or, alternatively, using the $\mathtt{ComputeRankIndividuals}$ procedure described in ~\ref{ranindqref} (we will also refer to the minimal height of an individual  as its \emph{rank}).
\footnote{We conjecture that the $\mathtt{ComputeRankIndividuals}$ procedure is a more efficient procedure in practice than applying the $\mathtt{Ranking.Core}$ procedure to $A_a$, for each $a\in \N_\A$.}

We also extend the database $\mathtt{DB}(\A)$ (see Definition \ref{def_DB(A)}) with a relational table  $tab_{rankInd}$ of arity 2, such that $\tuple{a,r} \in tab_{rankInd}$ if the rank of individual $a \in \N_\A$ is $r$.  
To Summarise, we  rank all the information contained in the KB: the antecedents of all the inclusions, both the defeasible and the strict ones, and of all individuals occurring in the ABox.

Consider now a CQ $q$ of the form ($\ref{eqdlcq}$), where, unlike Section~\ref{sect_query_rewite},
an unbound variable is not represented as $``\_"$, but with a variable name.
%
Now, each term $z$ occurring in $\varphi(\vec{x},\vec{y})$ may optionally have associated a rank $r \geq 0$, denoted $z^r$. We call such terms \emph{rank annotated} terms. If a variable is  rank annotated then it is also a  \emph{bound} variable.
The intuition about a rank annotated variable $z^r$ is that variable $z$ has to be  bound to some individual $a \in \N_\A$ of rank $r$. 
%
%
At the beginning, constants in $q$ are annotated with their rank and variables in $q$ are not rank annotated.
We will use the rank $r$, if present, to constrain the rewriting process as illustrated conceptually in the following. 
Assume we have a query $q$ of the form 
\[
q(x) \leftarrow \ldots, A(z), \ldots \ ,
\]

\nd where the term $z$ is either rank annotated or not.
Then 

\begin{description}
    \item[A) case $z$ is not rank annotated:] if
\begin{enumerate}
    \item    there is a PI $A_1 \subs A \in \T$ then rewrite $q$ as 
    \[
        q(x) \leftarrow \ldots, A_1(z), \ldots \ ;
    \] 

    \item  there is a DI $A_1 \dsubs A \in \D$ with $\rank_{\KB}(A_1) = r$ 
    then   rewrite $q$ as 
    \[
        q(x) \leftarrow \ldots, A_1(z^{r}), \ldots \ ,
    \]
    \nd and label all occurrences of $z$ in $q$ with rank $r$;
\end{enumerate}

    \item[B) case $z$ is annotated with rank $r$:] if
    there is  a  PI $A_1 \subs A \in \D$ with $\rank_{\KB}(A_1) \leq r$ or there is a DI $A_1 \dsubs A \in \T$ with $\rank_{\KB}(A_1) = r$ then rewrite $q$ as 
    \[
        q(x) \leftarrow \ldots, A_1(z^r), \ldots \ .
    \]


\end{description}


    


\nd So, for instance, given the query 
\[
q(x) \leftarrow A(x)
\]
\nd and DI $B \dsubs A \in \D$ with $\rank_{\KB}(B) = 1$,  we fall in the case {\bf A.2}, and we reformulate $q$ as 
\[
q(x) \leftarrow B(x^1) \ .
\]
\nd For the latter, as we will see later on, we may find answers via the query
\[
q(x) \leftarrow tab_B(x), tab_{rankInd}(x,1) \ 
\]

\nd submitted to the database $\mathtt{DB}(\A)$ that includes the rank tables, \ie~we are retrieving all known instances of $B$ of rank $1$.

Note that if there are no DIs in $\KB$, then all terms $z$ in a query do not have any rank superscript. This is covered by point 1 above: if there are no DIs in $\KB$, then the rewriting is as for the classical case 
(see Section~\ref{sect_query_rewite}).
If a rank annotated term $z$ occurs in an atom of the form $A(z^r)$, then it is obtained by rewriting it via a DI (point {\bf A.2} above) and the rank $r$ remains constant in any successive rewriting of $A(z^r)$ (case {\bf B} above). Essentially, point {\bf A.2} encodes the fact that if we are looking \eg~for all instances of $A$ and want to use  $A_1 \dsubs A$,  then we may look at those instances of $A_1$ that are typical, \ie~the instances of $A_1$ at rank $r = \rank_{\KB}(A_1)$. Once the rank has been fixed for a term $z$ , as the rank for individuals is unique, this rank can not change via rewriting. This is covered by case {\bf B} above. Eventually, if two concept name $A_1, A_2$ occur in $q$ and they share the same term $z$, then either they occur in the form $A_1(z), A_2(z)$ (no rank label for $z$), or both occurrences of $z$ are labelled with the same rank, \ie~they occur in $q$ in the form $A_1(z^r), A_2(z^r)$ (see point {\bf A.2} above). This is motivated by the fact that the rank of an individual is unique and, thus, \eg~there can not be any answer to a query $q(x) \leftarrow A_1(x^{r_1}), A_2(x^{r_2})$ if $r_1 \neq r_2$.

We now proceed to the formalisation of the above idea.
To do so, it suffices to adapt the notion of applicable axiom $\spec$  to a predicate $g$ described in Section~\ref{sect_query_rewite} to the case of DIs and/or rank annotated terms. Specifically, let $\spec \in \T_{PI} \cup \D$ be a PI or DI in $\KB$.

\begin{definition}[Applicability] \label{defapplicable}
We say that $\spec$ is \emph{applicable} to a predicate $g$ occurring in the body of a query $q$ if:
\begin{enumerate}
\item $g = A(z)$ and $\spec$ has $A$ in its RHS; 

\item $g = A(z^r)$ and $\spec$ has $A$ in its RHS and its LHS either \ii{i} has rank $r$ if $\spec \in \D$; or \ii{ii} has rank at most $r$ if $\spec \in \T_{PI}$;

\item $g$ is of the form $P(z_{1}, z_{2})$, $z_i$ may be rank annotated or not, $\spec \in \T_{PI}$  and one of the following conditions hold:
\begin{enumerate}


\item $z_{2}$ is an unbound variable,\footnote{Thus, it is a not rank annotated variable.} and the RHS of $\spec$ is $\csome P$. If $z_1 = z^r$ then the LHS of $\spec$ has rank at most $r$;


\item $z_{1}$ is an unbound variable, and the RHS of $\spec$ is $\csome P^{-}$. 
 If $z_2 = z^r$ then the LHS of $\spec$ has rank at most $r$;





\item $\tau$ is a role inclusion and its RHS is either $P$ or $P^-$.

\end{enumerate}  
\end{enumerate}

\nd If $\spec$ is applicable to a predicate $g$, we indicate with $gr(g, \spec)$ the predicate obtained from the predicate $g$ by applying the inclusion axiom $\spec$. Specifically,  

\begin{enumerate}

    \item {\bf IF} $g = P(z_{1}, z_{2})$,  where $z_i$ may be rank annotated or not  {\bf AND} either
        $\tau =P_1 \subs P$ {\bf OR}
        $\tau = P_1^- \subs P^-$
        {\bf THEN} $gr(g, \spec) = P_1(z_{1}, z_{2})$;

   \item {\bf IF} $g = P(z_{1}, z_{2})$,  where $z_i$ may be rank annotated or not {\bf AND} either
        $\tau =P_1 \subs P^-$ {\bf OR}
        $\tau = P_1^- \subs P$
        {\bf THEN} $gr(g, \spec) = P_1(z_{2}, z_{1})$;

   \item {\bf IF} $g=A(z)$ (resp.~$g=P(z, y)$ or $g=P(y,z)$,  where $y$ is unbound), $z$ is not rank annotated \textbf{AND} $\spec = B \subs A$
    (resp. $\spec = B \subs \csome P$ or $\spec = B \subs \csome P^-$)
    \textbf{THEN} $gr(g, \spec) = C(z)$, where
    \begin{itemize}
        \item $C(z) =A_1(z)$ if $B = A_1$, or
        \item $C(z) = P_1(z,y')$ if $B = \csome P_1$, or
        \item $C(z) = P_1(y',z)$ if $B = \csome P_1^-$
    \end{itemize}
    \nd and  $y'$ is a new unbound variable.


    \item {\bf IF} $g=A(z)$,  where $z$ is not rank annotated   \textbf{AND} $\spec = B \dsubs A$
    \textbf{AND}  $\rank_{\KB}(B) = r$ \textbf{THEN} $gr(g, \spec) = C(z^r)$,  where
    \begin{itemize}
        \item $C(z^r) =A_1(z^r)$ if $B = A_1$, or
        \item $C(z^r) = P(z^r,y')$ if $B = \csome P$, or
        \item $C(z^r) = P(y',z^r)$ if $B = \csome P^-$
    \end{itemize}
    \nd and $y'$ is a new unbound variable.

     \item {\bf IF} $g=A(z^r)$ (resp.~$g=P(z^r, y)$ or $g=P(y,z^r)$,  where $y$ is unbound) \textbf{AND} $\spec = B \subs A$ (resp. $\spec = B \subs \csome P$ or $\spec = B \subs \csome P^-$) with $\rank_{\KB}(B) \leq r$  
    \textbf{THEN} $gr(g, \spec) = C(z^r)$, where 
    \begin{itemize}
        \item $C(z^r) =A_1(z^r)$ if $B = A_1$, or
        \item $C(z^r) = P_1(z^r,y')$ if $B = \csome P_1$, or
        \item $C(z^r) = P_1(y',z^r)$ if $B = \csome P_1^-$
    \end{itemize}
    \nd and $y'$ is a new unbound variable.

     \item {\bf IF} $g=A(z^r)$  \textbf{AND}  $\spec = B \dsubs A$ with $\rank_{\KB}(B) = r$
    \textbf{THEN} $gr(g, \spec) = C(z^r)$, where 
    \begin{itemize}
        \item $C(z^r) =A_1(z^r)$ if $B = A_1$, or
        \item $C(z^r) = P(z^r,y')$ if $B = \csome P$, or
        \item $C(z^r) = P(y',z^r)$ if $B = \csome P^-$
    \end{itemize}
    \nd and $y'$ is a new unbound variable.

    
\end{enumerate}
\end{definition}


\nd Given the above definition, given a query $q$ and a satisfiable \dllitecoreH~KB $\KB=\tuple{\T,\D,\A}$, we now  apply to  $q$ and $\KB$ the procedure $\mathtt{DefQueryRef}$, a modified version of the $\mathtt{QueryRef}$ procedure described in Section~\ref{sect_query_rewite}, to compute the reformulated queries. That is,  given a query $q$ and a \dllitecoreH~KB $\KB=\tuple{\T,\D,\A}$, $\mathtt{DefQueryRef}$ reformulates $q$ in terms of a set of CQs, \ie~UCQs, $r(q,\KB)$. 

The few changes we have to make to the  $\mathtt{QueryRef}$ procedure to obtain $\mathtt{DefQueryRef}$ procedure is that the input is now a defeasible \dllitecoreH~KB and that at
\begin{description}
    \item[Line 7:] we replace $q_i[g/g']$ with $\rho(q_i[g/g'])$, where $q[g/g']$ denotes the query obtained from $q_i$ by replacing the predicate $g$ with the new predicate  $g'$ obtained by the definition above. Additionally, the function $\rho$ implements the following rule: if variable $z$ has been replaced with $z^r$, then we replace all occurrences of $z$ in the query with $z^r$;

    \item[Line 9:] the notion of unification needs to be adapted: we say that two predicates 
    $g_{1}= P(x_{1}, \ldots, x_{n})$ and $g_{2}= P(y_{1}, \ldots, y_{n})$ \emph{unify}, if for all $i$, either $x_{i} = y_{i}$ or  at least one among $x_{i}$ and $y_i$ is an unbound variables.\footnote{Note that the identity relation has to consider also the rank annotation. For example, considering a term $x$, $x^2= x^2$ but $x^2\neq x^3$.} If $g_{1}$ and $g_{2}$ unify, then the unification of $g_{1}$ and $g_{2}$ is the predicate $P(z_{1}, \ldots, z_{n})$, where $z_{i}= x_{i}$ if $x_{i} = y_{i}$ or $y_{i}$ unbound, otherwise $z_{i} = y_{i}$~\cite{Cali03};

    \item[Line 10:]  we remove the call to $\kappa$ as we do not represent unbound variables with ` $\_$ '. 
    
\end{description}

\begin{procedure}[h]
\caption{DefQueryRef($q,\KB$)\label{DefQueryRef}}
{
\KwIn{Union of CQs  $q = \{q_1, \ldots, q_m\}$ and a satisfiable \dllitehornH~KB 
$\KB = \tuple{\T,\D,\A}$.}
\KwOut{Set of reformulated conjunctive queries $r(q,\KB)$}
$r(q,\KB):= q$\;

\Repeat{$S=r(q,\KB) $}{
$S:= r(q,\KB)$\;
\ForEach{$q_i \in S$}{
\ForEach{$g \in q_i$}{
\If{$\spec \in (\T_{PI}\cup\D)$ is applicable to $g$}{
$r(q,\KB) := r(q,\KB) \cup \{ \rho(q_i[g/gr(g,\spec))]\}$\;
}
}
\ForEach{$g_{k}, g_{l} \in q_i$}{
\If{$g_{k}$ and $g_{l}$ unify}{
$r(q,\KB) := r(q,\KB) \cup \{ reduce(q_i,g_{k}, g_{l})\}$\;
}
}
}
}
\Return{$r(q,\KB)$}\;
}
\end{procedure}

\begin{example} \label{exref}
Consider a KB $\KB$ consisting of the following axioms (read $F$ as `flies',  $\bar{F}$ as `does not fly', $B$ as `bird', $P$ as `penguin', $PE$ as `possibly exceptional penguin', and $PJ$ as `penguin with a jetpack'):
\begin{eqnarray}
    \bar{F} & \subs & \neg F \label{bex0} \\
    B & \dsubs & F \label{bex1} \\
    P & \subs & B  \label{bex2} \\
    P & \dsubs & \bar{F}   \label{bex3a} \\     
    PE & \subs & P   \label{bex3} \\
    PJ & \subs & PE   \label{bex4} \\
    PJ & \dsubs & F \ .  \label{bex5} 
\end{eqnarray}

\nd Note that we use axioms  $P \dsubs  \bar{F}$ and $\bar{F}  \subs \neg F$ because the RHS of a DI has to be atomic and, thus, we may not use directly $P \dsubs \neg F$. 


It can be verified that the ranks of the antecedents of (\ref{bex0}) - (\ref{bex5}) are as follows:
\begin{eqnarray*}
    \rank_{\KB}(\bar{F}) & = & 0  \\
    \rank_{\KB}(B) & = & 0  \\
    \rank_{\KB}(P) & = & 1  \\
    \rank_{\KB}(PE) & = & 1  \\
    \rank_{\KB}(PJ) & = & 2  \ . 
\end{eqnarray*}

\nd Consider now the query
\begin{eqnarray}
    q(x) & \leftarrow &  F(x), P(x) \ . \label{qex1} 
\end{eqnarray}
\nd That is, we are looking for flying penguins. It can be verified that the set of reformulated queries $r(q,\KB)$ is
\begin{eqnarray}
    q(x) & \leftarrow &  F(x), P(x) \ (\text{same as $(\ref{qex1})$}) \label{refqex1} \\
    q(x) & \leftarrow &  B(x^0), P(x^0) \ (\text{use $(\ref{refqex1})$ and $(\ref{bex1})$}) \label{refqex2} \\
    q(x) & \leftarrow &  PJ(x^2), P(x^2)  \ (\text{use $(\ref{refqex1})$ and $(\ref{bex5})$})  \label{refqex3} \\
    q(x) & \leftarrow &  F(x), PE(x) \ (\text{use $(\ref{refqex1})$ and $(\ref{bex3})$}) \label{refqex4} \\
    q(x) & \leftarrow &  PJ(x^2), PE(x^2) \ (\text{use $(\ref{refqex3})$ and $(\ref{bex5})$}) \label{refqex5} \\ 
    q(x) & \leftarrow &  B(x^0), PE(x^0) \ (\text{use $(\ref{refqex4})$ and $(\ref{bex1})$}) \label{refqex6} \\
    q(x) & \leftarrow &  F(x), PJ(x) \ (\text{use $(\ref{refqex4})$ and $(\ref{bex4})$}) \label{refqex7} \\
    q(x) & \leftarrow &  PJ(x^2) \ (\text{use $(\ref{refqex5})$ and $(\ref{bex4})$}) \label{refqex8} \\
    q(x) & \leftarrow &  B(x^0), PJ(x^0) \ (\text{use $(\ref{refqex7})$ and $(\ref{bex1})$}) \ . \label{refqex9}
\end{eqnarray}

\nd Note that queries 
$(\ref{refqex2})$, 
$(\ref{refqex6})$ and 
$(\ref{refqex9})$
need not to be executed as we already know from the rankings that their answer set is empty. For instance, concerning $(\ref{refqex2})$, there can not be an individual at rank $0$ that is at the same time a bird and a penguin, since the latter is a concept with rank $1$. Similarly for the other cases. Therefore, to find the answer set of $q$, it suffice to compute the union of the answers to the queries 
\begin{eqnarray}
    q(x) & \leftarrow &  tab_F(x), tab_P(x) \label{Frefqex1} \\
    q(x) & \leftarrow &  tab_{PJ}(x), tab_{P}(x), tab_{rankInd}(x,2)  \label{Frefqex3} \\
    q(x) & \leftarrow &  tab_F(x), tab_{PE}(x)\label{Frefqex4} \\
    q(x) & \leftarrow &  tab_{PJ}(x), tab_{PE}(x), tab_{rankInd}(x,2)  \label{Frefqex5} \\
    q(x) & \leftarrow &  tab_F(x), tab_{PJ}(x) \label{Frefqex7} \\
    q(x) & \leftarrow &  tab_{PJ}(x), tab_{rankInd}(x,2) \label{Frefqex8} \ .
\end{eqnarray}

\qed
\end{example}

\nd The example above suggests us also the following optimisation.

\begin{remark} \label{noexec}
    If $q' \in \mathtt{DefQueryRef}(q,\KB)$ such that $A(x^r)$ occurs in $q'$ with $\rank_{\KB}(A) > r$ then the answer set of $q'$ is empty. The reason is simply that the rank of a class indicates the minimal rank of its elements, hence it is not possible to have an element of that class which has a rank that is lower of the class' rank. In particular, the  rule 5 above is based on this constraint.
\end{remark}

\begin{remark} \label{noexec2}
   Please note that if $\KB$ contains the axioms $\cass{a}{A}, A\subs \exists P, \exists P^-\subs B$ then necessarily $a$ is an answer to the query
   \[
    q(x) \leftarrow P(x,y), B(y) \ .
   \]

    \nd On the other hand, if we replace $\exists P^-\subs B$ with $\exists P^-\dsubs B$, where $\rank_{\KB}(\exists P^-) = r$, then one may verify that $a$ is an answer to $q$ only if 
    \begin{enumerate}
        \item $a$ has as $P$ successor an individual $b$ occurring in $\KB$ that is an instance of $B$; or
        \item $a$ has a $P$ successor an individual $b$ occurring in $\KB$ of rank $r$.
    \end{enumerate}
\nd In fact, note that $\mathtt{DefQueryRef}(q,\KB) = \{q,q'\}$, where $q'$ is the query
  \[
    q(x) \leftarrow P(x,y^r) \ .
   \]
\nd Also note that $A\subs \exists P$ can not be applied to $P(x,y^r)$ as the axiom does not guarantee that the $P$-successor of $a$ has rank $r$.   

        
\end{remark}

\nd The following proposition shows that the procedure $\mathtt{DefQueryRef}$ terminates.


\begin{proposition}[\cf~Lemma 34 in~\cite{Calvanese07}] \label{termdllitecore}
    Consider a  satisfiable defeasible \dllitecoreH~KB 
    $\KB = \tuple{\T,\D, \A}$, where the rank of the antecedents of all axioms  in $\T_{PI} \cup \D$ is given, and consider a CQ $q$. Then the procedure $\mathtt{DefQueryRef}(q,\KB)$ terminates.
\end{proposition}

\nd Obviously, Proposition~\ref{termdllitecore} also holds if the input is the union of CQs  $q = \{q_1, \ldots, q_m\}$.


We next address the correctness of the above described query-answering technique. To do so, we  rely again on a canonical model built using a specific notion of chase, called \emph{defeasible chase}, that consists of the classical notion of chase enriched with rules dealing with defeasible inclusions and rankings.



At first, we define the analogue of the definition of $\mathtt{Int}(\calS)$ (Definition \ref{def_DB(A)}) for the defeasible case.

\begin{definition}[Interpretation $\mathtt{Int}(\calS,h_\KB)$, Database $\mathtt{DB}(\A,h_\KB)$]\label{def_DB(A,h)}
Let $\KB = \tuple{\T,\D, \A}$ be a satisfiable defeasible KB, let $\RR =\{\D_0,\ldots,\D_n\}$ be the ranking of $\D$, and let $h_{\KB}$ be the minimal height function for the individuals in $\A$ \wrt~the defeasible KB $\KB$ (cf. Definition \ref{Def:minheight}).
Consider a set of assertion axioms $\calS$, possibly involving anonymous individuals, with $\A \subseteq \calS$. 
Let $\N_\calS$ be the set of individuals occurring in $\calS$.

The \emph{interpretation} $\mathtt{Int}(\calS,h_\KB) =(\Delta^{\mathtt{Int}(\calS,h_\KB)},\cdot^{\mathtt{Int}(\calS,h_\KB)},\pref^{\mathtt{Int}(\calS,h_\KB)})$ is defined as follows:

\begin{itemize}
    \item for each rank $0 \leq j \leq n$, let  $w_j$ be a new individual in $\Nanon \setminus \N_\calS$, called the \emph{witness} of rank $j$.  Let $\N_W$ be $\{w_0, \ldots, w_n\}$;
    \item $\Delta^{\mathtt{Int}(\calS,h_\KB)} = \N_\calS \cup  \N_W$;
    \item $a^{\mathtt{Int}(\calS,h_\KB)} = a$, for each individual $a \in \N_\calS \cup  \N_W$;
    \item $A^{\mathtt{Int}(\calS,h_\KB)}=\{a\mid \cass{a}{A}\in \calS\}$, for each concept name $A$;
    \item $P^{\mathtt{Int}(\calS,h_\KB)}=\{(a,b)\mid \rass{a}{b}{P}\in \calS \}$, for each role name $P$;
    \item $h_{\mathtt{Int}(\calS,h_\KB)}(a) = h_{\KB}(a)$, for each individual in $\N_\A$;
    \item $h_{\mathtt{Int}(\calS,h_\KB)}(w_j) = j$ for each witness $w_j \in \N_W$;
    \item $h_{\mathtt{Int}(\calS,h_\KB)}(b) = n+1$, for each anonymous individual $b \in \N_\calS \setminus \N_\A$;
    \item $\pref^{\mathtt{Int}(\calS,h_\KB)}$ is defined by referring to the height function $h_{\mathtt{Int}(\calS,h_\KB)}$.
\end{itemize}

\nd The \emph{database} $\mathtt{DB}(\A,h_\KB)$ is defined as follows:

\begin{enumerate}
        \item  for each  concept name $A$ occurring in $\A$, we define a relational table $tab_{A}$ of arity 1, such that $\tuple{a} \in tab_{A}$ iff $\cass{a}{A} \in \A$;
    
        \item for each role $P$ occurring in $\A$, we define a relational table $tab_{P}$ of arity 2, such that $\tuple{a,b} \in tab_{P}$ iff $\rass{a}{b}{P} \in \A$.

        \item for each individual name $a$ occurring in $\A$, we define a  relational table  $tab_{rankInd}$ of arity 2, such that $\tuple{a,h_\KB(a)} \in tab_{rankInd}$.
    \end{enumerate}

\nd {\color{black}Again, in what follows we will use $\mathtt{DB}(\A,h_\KB)$ to refer both to the database and to the corresponding interpretation ($\mathtt{Int}(\A,h_\KB)$).}

\end{definition}

\begin{remark} \label{remdbaranked}
Let us note the following points:
\begin{enumerate}
\item the definition of the database $\mathtt{DB}(\A,h_\KB)$ differs \wrt~the database $\mathtt{DB}(\A)$ (see Definition \ref{def_DB(A)}) as we add a table containing the minimal ranks of the individuals occurring in $\A$, which we have computed beforehand;

\item the addition of the witnesses $w_j$ is just a technicality to guarantee that every layer is populated, as imposed by the convexity property of ranked orders (Definition \ref{Def:Ranked_order}). The presence of witnesses does not have any impact on the properties satisfied by the individuals named in the ABox;

\item for all other anonymous individuals $b \in \N_\calS \setminus \N_\A$, we set their height to the top layer $n+1$. In this way only the constraints in $\T$ apply to them, not the defeasible information in $\D$, as expected.
\end{enumerate}

\end{remark}

\nd Now, the main property of the query reformulation procedure we are going to show is a defeasible analogue of Eq.~\ref{ansrel}: namely\footnote{See Eq.~\ref{qR} for the definition of $q_{i}^{\mathtt{DB}(\A, h_\KB)}$.}

\begin{equation}\label{ansreldef}
ans_{RC}(\KB,q)  =  \{\vec{t}\mid \text{ there is } q_{i} \in r(q,\KB) \text{ s.t. } \vec{t} \in q_{i}^{\mathtt{DB}(\A, h_\KB)}\} \ ,  
\end{equation}

\nd where any rank annotated variable of the form $z^r$ is mapped to an individual occurring in $\A$ of rank $r$ via the interpretation $\mathtt{DB}(\A, h_\KB)$. The above property allows us, similarly to non-defeasible \dllitecoreH, to find the answers to $q$, by evaluating the reformulated queries $q_{i} \in  r(q,\KB)$ over the relational database $\mathtt{DB}(\A, h_\KB)$.



    


Like for classical \dllitecoreH (see Section~\ref{sect_chase_niclos}), also for the defeasible case we may build a canonical model via chasing.

%
The \emph{defeasible chase}, denoted $\dchasecH$ and $\dchasecH_i$, is constructed as for the  (non-defeasible) \dllitecoreH, starting with $\A$, except that now we have to add two additional chase rules to deal with the defeasible inclusions. 
Specifically, consider the chase rules for (non-defeasible) \dllitecoreH, namely ($\mathbf{\CcH_1}$), ($\mathbf{\CcH_2}$) and ($\mathbf{\CcH_3}$) (see Section \ref{chasecore}). To deal with the defeasible axioms, we consider the following additional chase rules:

{
\begin{description}
    \item[($\mathbf{\CcH_4}$)] if 
    $\cass{a}{B}$ occurs in $\dchasecH_i(\KB)$,
    $h_\KB(a) = j$,\footnote{To compute the rank of an individual, use Proposition~\ref{lemma_defeasible_nominals_2} or procedure $\mathtt{ComputeRankIndividuals}$ in~\ref{ranindqref}.}
    $B \dsubs A \in \D$ with $\rank_{\KB}(B)=j$,\footnote{That is, $B \dsubs A \in \D_j$.}   
    and $\cass{a}{A}$ does not occur in $\dchasecH_i(\KB)$ (condition $f$), then let 
    \begin{equation} \label{dchasecdefi}
    \dchasecH_{i+1}(\KB) = \dchasecH_i(\KB)\cup\{\cass{a}{A}\} 
    \end{equation}
    \nd (condition $f_{new}$);

    \item[($\mathbf{\CcH_5}$)] if $b$ is the new anonymous individual created by the application of the ($\mathbf{\CcH_2}$) rule to 
    $B \subs \exists R \in \T$, then update $h_\KB$ with  $h_\KB(b) = n+1$, where $n$ is the highest rank among the DIs in $\D$.

\end{description}
}

\nd Rule ($\mathbf{\CcH_4}$) allows us to apply a defeasible inclusion to an individual only if its rank corresponds to the minimal rank of such individual,
while ($\mathbf{\CcH_5}$) defines the height of a newly introduced anonymous individual compliant with Definition~\ref{def_DB(A,h)}
(see also Remark~\ref{remdbaranked}).


Now, the defeasible chase of $\KB$, denoted $\dchasecH(\KB)$,  is defined similarly to Eq.~\ref{chasec}, \ie
\begin{equation} \label{dchasecdef}
\dchasecH(\KB) = \bigcup_{i} \dchasecH_i(\KB) \ .    
\end{equation}

\nd From it, using Definition~\ref{def_DB(A,h)}, we can build the ranked interpretation 
\begin{equation} \label{eqRIcan}
\RI_{\KB} = (\Delta^{\RI_{\KB}},\cdot^{\RI_{\KB}},\pref^{\RI_{\KB}}) 
\end{equation}

\nd defined as
\begin{equation} \label{canranked}
\RI_{\KB} = \mathtt{Int}(\dchasecH(\KB), h_\KB) \ .
\end{equation}

\nd Moreover, similarly as for  \dllitecoreH~(see Eq.~\ref{canmod}),  from every $\dchasecH_i(\KB)$, we may define  the ranked interpretation 

\begin{equation} \label{dcanmod}
   \RI_{\KB}^i =  (\Delta^{\RI_{\KB}^i},\cdot^{\RI_{\KB}^i},\pref^{\RI_{\KB}^i}) 
\end{equation}

\nd defined as
\begin{equation} \label{dcanmodB}
   \RI_{\KB}^i =  \mathtt{Int}(\dchasecH_i(\KB), h_\KB) \ . 
\end{equation}

\begin{remark} \label{canDIs}
     Note that by construction, the domain of $\RI_{\KB}$ and $\RI_{\KB}^i$ consists of the individuals occurring $\A$, the witnesses $w_j$ and the new anonymous individuals created by the ($\mathbf{\CcH_2}$) rule. For the latter, we  impose that their rank is  $n+1$ via the ($\mathbf{\CcH_5}$) rule, while all individuals in $\A$ are minimally ranked instead. That is, for all $a \in \A$, the rank of $a$ \wrt~$\RI_{\KB}$ (so as \wrt~$\RI_{\KB}^i$) is $h_{\KB}(a)$. Also note that $a^{\RI_{\KB}} = a^{\RI_{\KB}^i} = a$.
\end{remark}

\nd Concerning reformulated CQs, we impose that any rank annotated variable of the form $z^r$ is interpreted by $\RI_{\KB}$ (as well as by $\RI_{\KB}^i$) as an individual occurring in $\A$ of rank $r \leq n$.






%
\nd Now, the following proposition can be shown for the defeasible case (see in comparison Proposition~\ref{chasecor}).

\begin{proposition} \label{chasecordef}
  Consider a defeasible \dllitecoreH~$\KB=(\T,\D,\A)$.  Then
\begin{enumerate}
    \item $\RI_{\KB}$ is a ranked model of  $(\T_{PI}, \D, \A)$;
    

    \item $\RI_{\KB}$ is a ranked model of $\KB$ iff $\mathtt{DB}(\A, h_\KB)$ is a model of $(\clncH(\T), \A)$;
    
    
    \item $\KB$ is satisfiable iff $\RI_{\KB}$ is a ranked model of $\KB$.
    
\end{enumerate}

      

      
\end{proposition}

\nd Using Propositions~\ref{prop_defcons}, \ref{chasecor} and~\ref{chasecorA}, the following also holds.

\begin{proposition}\label{satdedllitecore}
  Consider a defeasible \dllitecoreH~$\KB=(\T,\D,\A)$.  Then  $\KB$ is satisfiable iff $\mathtt{DB}(\A)$ is a model of $(\clncH(\T), \A)$.
\end{proposition}

\nd Therefore, checking the satisfiability of a defeasible \dllitecoreH~KB can be done as for \dllitecoreH: namely, Proposition~\ref{satdedllitecore} tells us that in order to check the satisfiability of a defeasible \dllitecoreH~KB, 
it is sufficient (and necessary) to look at the database $\mathtt{DB}(\A)$, provided we have computed $\clncH(\T)$ (see also procedure $\mathtt{Consistent}$ in~\cite{Calvanese07}). Specifically, $\KB$ is not satisfiable iff there is a NI $B_1 \subs \neg B_2 \in \clncH(\T)$ such that the  \dllitecoreH~query 
\[
q(x) \leftarrow \bar{\gamma}(B_1)(x), \bar{\gamma}(B_2)(x)
\]
\nd has a non-empty answer set over $\mathtt{DB}(\A)$, where $\bar{\gamma}(B)(x)$ translates a basic concept $B$ into a query atom in the obvious way:
\begin{equation}\label{bargamma}
\bar{\gamma}(B)(x) = 
\begin{cases}
    A(x) & \text{if $B=A$} \\
    P(x,\_) & \text{if $B=\exists P$} \\
    P(\_,x) & \text{if $B=\exists P^-$} \ .
\end{cases}
\end{equation}

\nd Now, the following propositions can be shown (the proofs are in~\ref{sect_QRWproofs} and are similar to the non-defeasible analogues in~\cite{Calvanese07}).

\begin{proposition}[\cf~Theorem 29, \cite{Calvanese07}] \label{calvthm29}
     Consider a satisfiable defeasible \dllitecoreH~KB  $\KB = \tuple{\T,\D, \A}$, and let
     $q = \{q_1, \ldots, q_n\}$ be a UCQs. Assume that  $q_i$ is of the form $q_i(\vec{x}) \leftarrow \exists \vec{y}_i . \varphi(\vec{x}_i,\vec{y}_i)$.  Then  
     \[
     ans_{RC}(\KB,q) = \{\vec{t} \mid \text{ there is } q_{i} \in q \text{ s.t. } \RI_{\KB}  \text{ satisfies } \exists \vec{y}_i .\varphi(\vec{t},\vec{y}_i) \} \ .
     \]
\end{proposition} 

\begin{proposition}[\cf~Theorem 31, \cite{Calvanese07}] \label{calvthm31}
     Consider a satisfiable defeasible \dllitecoreH~KB  $\KB = \tuple{\T,\D, \A}$, and let
     $q = \{q_1, \ldots, q_n\}$ be a union of CQs. Then $ans_{RC}(\KB,q) = \bigcup_{q_i \in q} ans_{RC}(\KB,q_i)$.
\end{proposition} 




\begin{proposition}[\cf~Lemma 39,~\cite{Calvanese07}]\label{mainqadllite}
  Consider a satisfiable defeasible \dllitecoreH~KB 
    $\KB = \tuple{\T,\D, \A}$, where the rank of the antecedents of all axioms  in $\T_{PI} \cup \D$ is given, and $h_\KB$ is the ranking of individuals as per Proposition~\ref{prop_minheight}. 
    Consider a CQ $q$ and let $r(q,\KB) = \mathtt{DefQueryRef}(q,\KB)$ be the set of reformulated queries. Then Eq.~\ref{ansreldef} holds, \ie
\begin{equation*}
ans_{RC}(\KB,q)  =  \{\vec{t}\mid \text{ there is } q_{i} \in r(q,\KB) \text{ s.t. } \vec{t} \in q_{i}^{\mathtt{DB}(\A, h_\KB)}\} \ .    
\end{equation*}


\end{proposition}

\nd Now, by Proposition~\ref{calvthm31}, the following proposition follows immediately, which extends Proposition~\ref{mainqadllite} to UCQs.

\begin{proposition}[\cf~Theorem 40, \cite{Calvanese07}] \label{calvthm40}
Proposition~\ref{mainqadllite} holds also for the union of CQs.
\end{proposition} 

\nd Please note that, as for the no-DI case (see Remark~\ref{qansinstcheck}), we may use the query answering procedure also to determine instance checking under RC.

\begin{remark}\label{qansinstcheckDIs}
    The instance checking problem under RC can be seen as a special case of query answering: that is,
    an individual $a$ is an instance of concept $B$ \wrt~a defeasible KB $\KB = \tuple{\T,\D,\A}$ iﬀ 
    $a$ is in the answer set of the CQ $q(x) \leftarrow A(x)$ \wrt~ $\KB'$,
    where $\KB' = \tuple{\T',\D, \A}$ and $\T' = \T \cup \{B \subs A\}$, with $A$ a new concept name.
    On the other hand, an individual $a$ is an instance of concept $\neg B$ \wrt~a defeasible KB $\KB = \tuple{\T,\D,\A}$ iﬀ 
    $a$ is in the answer set of the UCQs  $\vec{q} = \{q(x) \leftarrow B'(x) \mid B' \subs \neg B \text{ in the NI-closure of } \T \}$ \wrt~ $\KB$. 
\end{remark}

\section{Computational Complexity}\label{compplex}
\noindent\fbox{\begin{minipage}{0.975\textwidth}
\textsc{Overview:} 
 The main contribution of this section is that  the computational complexity of the main reasoning and CQ answering procedures in the presence of DIs is essentially the same as for classical \dllitecoreH~and \dllitehornH. In particular, 
 answering unions of conjunctive queries in defeasible \dllitecoreH~and \dllitehornH~is in $\textsf{P}$ \wrt~the size of the union of the TBox and the DBox,  in $\textsf{AC}^0$ in the size of the ABox (data complexity), and is  $\textsf{NP-complete}$ in combined complexity (see Proposition~\ref{complexityqa}).


\end{minipage}}
\vspace*{5mm}

\nd We provide here also the computational complexity analysis of our reasoning procedures. To start with, we recap, for ease of presentation, those inherited from classical \dllite~\cite{ArtaleEtAl2009,Calvanese05,Calvanese06a,Calvanese07}. 
We recall that, in analysing the computational complexity of a reasoning problem
in the context of \dllite~languages, it is usual to distinguish among two types of computational complexities: data complexity and combined complexity. \emph{Data complexity} is the complexity with respect to the size of the ABox only; \emph{combined complexity} is the complexity with respect to the size of all inputs to the problem.

In Table~\ref{tabcomplexdllite} we recap the computational complexity results as from
~\cite{ArtaleEtAl2009} (recall that subsumption can be reduced to the satisfiability problem).\footnote{$\textsf{AC}^0 \subset \textsf{LOGSPACE} \subseteq \textsf{NLOGSPACE} \subseteq \textsf{P}$~\cite{Papadimitriou94}.}

\begin{table}
\centering
\begin{tabular}{|c|c|c|c|} \hline
\multirow{3}{*}{\textbf{Languages}} &   
\multicolumn{3}{c|}{\textbf{Complexity}} \\
\cline{2-4}
 &  \textbf{Combined complexity} & 
 \multicolumn{2}{c|}{\textbf{Data complexity}} \\
\cline{2-4}
 &  Satisfiability & Instance checking & Query answering \\
\hline
\dllitecoreH &   \textsf{NLOGSPACE} & in $\textsf{AC}^0$ & in $\textsf{AC}^0$ \\
\dllitehornH &  \textsf{P} & in $\textsf{AC}^0$ & in  $\textsf{AC}^0$ \\ \hline
\end{tabular}
\caption{Complexity of reasoning in classical \dllitecoreH~and $\dllitehornH$ as from
~\cite{ArtaleEtAl2009}.} \label{tabcomplexdllite}
\end{table}


\subsection{The Case of RC without ABox}

\nd We address here the computational complexity of reasoning for defeasible KBs without ABox as per Sections~\ref{sect_RC_horn} and \ref{sect_RC_core}.
The analysis parallels the one presented in~\cite{Casini19}.

Let us focus on defeasible \dllitecoreH~as the analysis for \dllitehornH~parallels the one of the former. 
Specifically, it is easily verified that   $|\norm{\KB}|$ is in $O(|\KB|)$ and that  $\mathtt{Exceptional.Core}(\T,\E)$ performs at most  $|\E| \in O(|\D|)$ subsumption tests. 

Now, let us analyse $\mathtt{ComputeRanking.Core}(\KB)$. 
Line 7 requires $O(|\D|)$ subsumption test. Lines 8 - 10 require at most $O(|\D|^2)$ subsumption tests as at each round $|\E_{i+1}|$ is $|\E_{i}| - 1$ in the worst case. At each \texttt{repeat} round $|\D^*|$ decreases in size (at line 12), and thus the \texttt{repeat} loop is iterated at most $O(|\D|)$ times. Therefore, $\mathtt{ComputeRanking.Core}$ requires at most $O(|\D|^3)$ subsumption tests. Now, it is easily verified that, by exactly the same analysis, we may obtain the same complexity bounds for the procedures  
$\mathtt{Exceptional.Horn}$ and $\mathtt{ComputeRanking.Horn}$.
Therefore, by the results in Table~\ref{tabcomplexdllite} the following holds.

\begin{proposition}[\cf~Proposition 14 in \cite{Casini19}]\label{complexity1a}
The $\mathtt{ComputeRanking.\{Core, Horn\}}$ procedures run in polynomial time \wrt~the input size.
\end{proposition}

\nd Now consider  $\mathtt{RationalClosure.Core}(\KB,\alpha)$.
Lines 1 - 3 require one subsumption test.
In line 4, the value of $n$ is bounded by $|\D|$ and line 4 requires at most 
$O(|\D|^3)$ subsumption tests.
Lines 5 - 7 require one subsumption test.
The loop in lines 12 - 14 is executed at most $|\D|$ times (as at each loop $|\D_{\R}|$ decreases), at each iteration we execute one subsumption test only, and there is at most one subsumption test between lines 16 - 19. 
Hence, $\mathtt{RationalClosure.Core}(\KB,\alpha)$ requires  at most $O(|\D|^3 + |\D|)$ subsumption tests. Now,  by exactly the same analysis, we may obtain the same complexity bound for the procedure $\mathtt{RationalClosure.Horn}$.
Therefore, by the results in Table~\ref{tabcomplexdllite} the following follows.

\begin{proposition}[\cf~Proposition 15 in \cite{Casini19}]\label{complexity1}
The  $\mathtt{RationalClosure.\{Core, Horn\}}$ procedures run in polynomial time \wrt~the input size.
\end{proposition}

\subsection{The Case of RC with ABox}

\nd Let us now consider the case in which there is an ABox. At  first, by Proposition~\ref{prop_defcons}, if the ranking of the DIs has been computed, the satisfiability of a  KB $\tuple{\T,\D,\A}$ is reduced to the satisfiability of $\tuple{\T,\A}$ and, thus, by Proposition~\ref{complexity1a} and Table~\ref{tabcomplexdllite} the following proposition holds.

\begin{proposition}\label{complexsatDIs}
The satisfiability problem of a defeasible KB can be solved in polynomial time.
\end{proposition}

\nd Consider now  the  $\mathtt{ABoxRationalClosure.\{Core, Horn\}}$ procedures.
Let us recall that these procedures are analogous to the 
$\mathtt{RationalClosure.\{Core, Horn\}}$  procedures.
Therefore, it is easily verified that the computational complexity analysis is the same as that for the case without ABox. Consequently, 

\begin{proposition}\label{complexity2}
The $\mathtt{ABoxRationalClosure.\{Core, Horn\}}$ procedures run in polynomial time \wrt~the input size.
\end{proposition}

\nd Please note that, however, by Remark~\ref{qansinstcheckDIs}, like for the non-DI case, the computational complexity of instance retrieval is in fact the same as for query answering, which  we address next (\cf~\cite[Theorem 8.3]{ArtaleEtAl2009}).

\subsection{The Case of Query Answering under RC}

\nd To start with, we address the computational complexity of  computing the rank of the LHS of PIs and DIs occurring in $\KB$ via procedure~$\mathtt{Ranking.Core}$. In fact, these ranks can be computed in polynomial time \wrt~the size of $\KB$, which is an immediate consequence of Corollary~\ref{rankcomplexity1} that follows from Proposition~\ref{complexity1}.   

\begin{corollary}\label{rankcomplexity1}
The  $\mathtt{Ranking.Core}$ procedure runs in polynomial time \wrt~the input size.
\end{corollary}

\nd Concerning the computational complexity of ranking all individuals occurring in the ABox,
one option consists of using Proposition~\ref{lemma_defeasible_nominals_2} and, for each $a\in \N_\A$ apply procedure $\mathtt{Ranking.Core}$ to $A_a$ and, thus, also the ranking all individuals occurring in the ABox can be computed in
polynomial time \wrt~the size of $\KB$.

An alternative consists of using $\mathtt{ComputeRankIndividuals}$ procedure (see~\ref{ranindqref}), whose computational complexity we address next. Please note that this procedure is computed once for all, unless the KB changes.
Now, it is easily verified that the number of loops of steps 4 - 11 is linearly bounded by the size of the DBox, the NI-closure can be computed in polynomial time \wrt~the size of $\T \cup \D$,
while the loop 7-9 may be executed at most $O(|\T \cup \D|^2)$ times. The query $q$ is an SQL query and, thus, its complexity is in $\textsf{AC}^0$. Therefore, 

\begin{proposition} \label{complexityrankinds}
The $\mathtt{ComputeRankIndividuals}$ procedure runs in polynomial time \wrt~its input size and
in $\textsf{AC}^0$ in the size of the ABox (data complexity).
\end{proposition}

\nd We now establish the computational complexity of the  $\mathtt{QueryRef}$ procedure, which is inherited directly from \cite[Lemma 42]{Calvanese07} and the proof of Proposition~\ref{termdllitecore} showing that the number of distinct conjunctive queries generated by the procedure is less than or equal to $(m \cdot (n+1)^2)^n$, where $n$ is the query size, \ie~$n$ is proportional to the number of atoms and the number of terms occurring in the query,  and $m$ is the number of antecedents of axioms occurring in $\T_{PI} \cup \D$. Since $m$ is linearly bound by the number of inclusion axioms, while $n$ does not depend on the size of $\T\cup \D$, it follows that $\mathtt{QueryRef}$ runs in time polynomial in the size of the union of the TBox and the DBox. 

\begin{proposition} \label{complexityqueryref}
Consider a  satisfiable defeasible \dllitecoreH~KB 
$\KB = \tuple{\T,\D, \A}$, where the rank of the antecedents of all axioms  in $\T_{PI} \cup \D$  is given and consider a union of CQs $\vec{q}$. Then the  $\mathtt{QueryRef}(\vec{q},\KB)$ procedure runs in time polynomial in the size of $\T\cup \D$.    
\end{proposition}

\nd It is not difficult to see that the analogue of Proposition~\ref{complexityqueryref} above holds for \dllitehornH~as well (the procedure is described in~\ref{qahornh}).

\begin{proposition} \label{complexityqueryrefHorn}
Proposition~\ref{complexityqueryref} holds also for \dllitehornH.
\end{proposition}

\nd Now, following \eg~the proofs of~\cite[Theorem 43]{Calvanese07}, \cite[Theorem 44]{Calvanese07} (see also~\cite[Theorem 6]{Calvanese05}), we get immediately the following main feature of defeasible \dllite: the complexity of query answering in the presence of DIs is the same as for \dllitecoreH~and $\dllitehornH$ (see~\cite{ArtaleEtAl2009,Calvanese05,Calvanese06a,Calvanese07}), which follows from the fact that the computational complexity of computing SQL relational queries is in $\textsf{AC}^0$~\cite{Papadimitriou99}, and that a reformulated query can be computed in non-deterministic polynomial time by the $\mathtt{QueryRef}$ procedure (\wrt~combined complexity).

\begin{proposition}\label{complexityqa}
The problem of answering unions of CQs in defeasible \dllitecoreH~and \dllitehornH~is in $\textsf{P}$ \wrt~the size of the union of the TBox and the DBox,  in $\textsf{AC}^0$ in the size of the ABox (data complexity), and is 
$\textsf{NP-complete}$ in combined complexity.
\end{proposition}

\section{Related Work}\label{related}


\nd While significant research has been devoted to non-monotonic extensions of Description Logics (DLs) — see, for instance, \cite{Alviano24a,Alviano24b,Alviano24,Baader93,BaaderHollunder1995,Baader95e,Bonatti15,Bonatti2019,BonattiEtAl2015,Bonatti09a,Bonatti22,BritzEtAl2008,BritzEtAl2011c,BritzEtAl2021,Britz16,Britz18,Britz19,Britz19a,CasiniStraccia2010,CasiniEtAl2014,CasiniStraccia2013,Casini21,Donini97,DoniniEtAl2002,Giordano10,Giordano13,Giordano15,Giordano21a,Giordano20a,Giordano22,Giordano19,Giordano20,GiordanoEtAl2013,GiordanoEtAl2015,Giordano17,Grimm09,Hahn24,Lambrix98,Pozzato19,Quantz1992,Straccia93} — in this section we focus exclusively with non-monotonic DL extensions that retain low computational complexity. For a more in-depth treatment of non-monotonic DLs and their general characteristics, the reader is referred to, \eg, \cite{Bonatti15,Giordano13}.
Moreover,  there exists a substantial body of work exploring the integration of non-monotonic logic programming with DLs, such as \eg, \cite{Bozzato14,Costa15a,deBruijn10,Eiter11,Eiter08,Ivanov13,Ivanov15,Kaminski15,Knorr11a,Knorr08,Knorr11,Knorr12,Lukasiewicz08,Lukasiewicz08d,Lukasiewicz07c,Lukasiewicz07i,Lukasiewicz08c,Lukasiewicz10a,Lukasiewicz07,Lukasiewicz08b,Lukasiewicz09,Lukasiewicz10,Motik10}. However, these approaches are somewhat less related, so we will not address them here, except for those cases involving tractable computational complexity.


Despite the extensive body of work in the field, we believe that our contributions are of notable significance. Specifically, it is worth emphasizing that, among the existing non-monotonic extensions of low-complexity DLs, see, \eg,~\cite{Bonatti15b,Bonatti11b,Bonatti11a,Bonatti21,Bonatti10,BonattiEtAl2011,Bonatti23,BritzEtAl2015a,Cadoli90,CasiniEtAl2014,CasiniStraccia2010,CasiniStraccia14,Casini19,DoniniEtAl2002,Giordano09,Giordano11,Giordano11a,Giordano12a,Giordano16a,Giordano09c,Giordano09d,PenselEtAl2017a,PenselEtAl2017}, the preservation of tractability of the underlying DL is achieved only in a very limited number of cases. Exceptions to these cases are some results reported in~\cite{Bonatti15,Bonatti11b,Bonatti11a,Bonatti10,Bonatti23,Casini19,Giordano16,Giordano16b}.

Specifically, \cite{Bonatti11a,Bonatti23} considers \emph{circumscription}~\cite{Lifschitz94} in low-complexity DLs such as DL-Lite~\cite{Artale09} and $\EL$ and identifies some non-monotone DL fragments whose computational complexity of the decisions problem at hand (deciding subsumption, instance checking and concept satisfiability) are in {\sc PTime}. 
%
%

The work presented in~\cite{Bonatti10} aligns with the approach of~\cite{Bonatti11a}, focusing on the logic $\mathcal{EL}$ enriched with \emph{default attributes}, that is, DIs of the form $A \usually \exists R.B$, interpreted under circumscription with  \emph{fixed} atomic concepts. However, the analysis in~\cite{Bonatti10} is restricted to \emph{conflict-safe} knowledge bases (KBs) that do not include assertions. Within this setting, it is shown that the set of so-called `normalised' concepts subsuming a given concept can be computed in polynomial time. These results were subsequently extended in~\cite{Bonatti11b}, where $\mathcal{EL}^{++}$~\cite{Baader05a} is considered in place of $\mathcal{EL}_\bot$, necessitating a refined notion of conflict-safety to accommodate the richer expressiveness of the logic.
%

%

In~\cite{Bonatti15}, a non-monotonic extension of DLs  is introduced, grounded in a mechanism of overriding and the use of \emph{normality concepts} of the form $\mathsf{N}C$, which denote prototypical (\ie, normal) instances of a concept $C$. Defeasible GCIs are expressed as $C \usually D$, governed by a priority relation $\prec$ among such GCIs. The intended interpretation is that ``instances of $C$ are normally instances of $D$, unless overridden by higher-priority defeasible inclusions''. While the resulting set of defeasible GCIs does not generally satisfy the rationality postulates, it has been shown in~\cite{Bonatti17} that the entailment relation exhibits some desirable properties, such as \emph{cumulativity} (\cf~\cite{Makinson1994a}). A notable aspect of this approach is that reasoning (deciding subsumption and instance checking) remains tractable in the case where the underlying DL is \dllitehornH~or $\mathcal{EL}^{++}$, as the algorithm relies on a sequence of standard DL subsumption tests.

Eventually and closest to our work are \cite{Casini19,Giordano16,Giordano16b}.
In \cite{Casini19} illustrates a polynomial time subsumption procedure for nominal safe \ELObot~under RC that relies entirely on a series of classical, monotonic \ELbot~subsumption tests.
%
%
In~\cite{Giordano16,Giordano16b}, the logic $\mathcal{SROEL}(\sqcap, \times)^\mathbf{R}\mathbf{T}$ is introduced, extending the DL $\mathcal{SROEL}(\sqcap, \times)$ \cite{Kroetzsch10} with \emph{typicality concepts} of the form $\mathbf{T}(C)$, intended to represent the typical instances of the concept $C$. Within this framework, defeasible GCIs of the form $C \usually D$ are expressed as $\mathbf{T}(C) \sqsubseteq D$, conveying that ``typical $C$-elements are $D$-elements''. The logic also supports more expressive constructs such as $D \sqsubseteq \mathbf{T}(C)$, interpreted as ``all $D$-elements are typical instances of $C$''. A key result is that instance checking under \emph{rational entailment}  is in $\mathsf{PTime}$ for $\mathcal{SROEL}(\sqcap, \times)^\mathbf{R}\mathbf{T}$. To achieve this, the authors extend the Datalog rule set originally defined for $\mathcal{SROEL}(\sqcap, \times)$~\cite{Kroetzsch10} with additional rules to handle the semantics of the typicality operator. Furthermore, for so-called \emph{simple} knowledge bases, where the typicality operator $\mathbf{T}$ appears only on the left-hand side of GCIs, it is shown that both concept \emph{exceptionality} and \emph{rank determination} can be computed in polynomial time, by means of appropriate Datalog queries over the extended rule set defined for rational entailment in $\mathcal{SROEL}(\sqcap, \times)^\mathbf{R}\mathbf{T}$.
However, as already pointed out in~\cite{Casini19},  while we agree on the fact that one may determine exceptional concepts and their rank in $\mathcal{SROEL}(\dlAnd, \times)^\mathbf{R}\mathbf{T}$, how to encode in Datalog the final step to decide entailment under RC in polynomial time has been left unspecified.

With regard to the class of \emph{hybrid} approaches—namely, those that integrate non-monotonic logic programming with low-complexity DLs, some results manage to retain tractability, such as~\cite{Costa15a,Eiter11,Ivanov13,Ivanov15,Kaminski15,Knorr11a,Knorr08,Knorr11,Knorr12,Lukasiewicz07i,Lukasiewicz08c,Lukasiewicz10a,Lukasiewicz07,Lukasiewicz08b,Lukasiewicz09,Lukasiewicz10}. Broadly speaking, these results show that data complexity remains polynomial for certain classes of hybrid normal logic programs under well-founded semantics, assuming that the underlying DL supports tractable instance checking.

\section{Conclusion and Future Work}\label{concl}


\paragraph{Contribution} A concise summary of our contributions is as follows:
\begin{itemize}
\item we have introduced a defeasible extension of \dllitecoreH~and \dllitehornH, notable representatives of the \dllite~family of DLs, 
with defeasible class inclusions of the form $\{B_1, \ldots, B_n\}\dsubs A$ ($n=1$ for \dllitecoreH) read as ``usually, an object in the conjunction of the  $B_i$'s is also in $A$'' under RC semantics;

\item we have developed reasoning algorithms for KB satisfiability, subsumption and instance checking under RC;

\item we have proved these logics possess the \emph{unique extension property}, unlike other defeasible DLs under RC, which is a fundamental property for efficient CQ answering;

\item we have developed a CQ answering procedure via query reformulation aligned with the non-defeasible versions;

\item we have demonstrated that these defeasible logics retain the low computational complexity of their classical counterparts: namely,
\begin{itemize}
    \item the satisfiability problem is in $\mathsf{P}$ \wrt~combined complexity;
    \item instance checking and query answering are in $\mathsf{AC}^0$ \wrt~ data complexity.
\end{itemize}

\end{itemize}

\paragraph{Future Work} 
%
As future work, we aim to investigate whether the decision problems explored in this paper can be extended to the Horn-DL family~\cite{Kroetzsch13}, which underpins the OWL 2 profile OWL RL~\cite{OWL2RL}, while preserving the computational complexity of their classical counterparts.
Another promising direction is the implementation and extension of our approach to the OWL 2 profile OWL QL, with an empirical evaluation in an OBDA setting, following the methodology of~\cite{BritzEtAl2015a}, using systems such as \textsf{Ontop} or \textsf{Mastro}.




 \bibliographystyle{plain} 
 \bibliography{bib,References}


\newpage

\begin{appendix}

\newtheorem{appdefinition}{Definition}[section]
\newtheorem{appproposition}{Proposition}[section]
\newtheorem{applemma}{Lemma}[section]
\newtheorem{appcorollary}{Corollary}[section]
\newtheorem{appremark}{Remark}[section]
\newtheorem{apptheorem}{Theorem}[section]

\renewcommand{\theappdefinition}{\Alph{section}.\arabic{appdefinition}}
\renewcommand{\theappproposition}{\Alph{section}.\arabic{appproposition}}
\renewcommand{\theapplemma}{\Alph{section}.\arabic{applemma}}
\renewcommand{\theappcorollary}{\Alph{section}.\arabic{appcorollary}}
\renewcommand{\theappremark}{\Alph{section}.\arabic{appremark}}
\renewcommand{\theapptheorem}{\Alph{section}.\arabic{apptheorem}}

\section*{Appendixes} \label{app_all}
\nd{\bf Note:} in what follows, we assume that the reader is familiar with  DLs such as $\mathcal{ALC}$ and $\mathcal{ALCHI}$~\cite{BaaderEtAl2007}. 
\section{Proofs of Section~\ref{backgr}: Background} \label{backgrproofs}


\setcounter{proposition}{\getrefnumber{nominalABox}-1}
\begin{proposition}
Consider any \dllitehornH~KB $\KB=\tuple{\T,\A}$ and any concept assertion $\cass{a}{A}$. 
Then
\begin{enumerate}
    \item $\KB$ is unsatisfiable iff $\T_\A \models A_a \impc \neg A_a$ for some individual $a$ occurring in $\KB$;
    \item If $\KB$ is satisfiable, $\KB \models \cass{a}{A}$ iff $\T_\A \models A_a \impc A$.
    
\end{enumerate} 
\end{proposition}

\nd We need to introduce some intermediate results in order to prove Proposition \ref{nominalABox}. We recall from Equation \ref{ta} that $\T_\A = \T \cup \{A_a \impc A' \mid \cass{a}{A'} \in \A \}  \cup  \{A_a \impc \csome P.A_b \mid \rass{a}{b}{P} \in \A \}  \cup  \{A_b \impc \csome P^-.A_a \mid \rass{a}{b}{P} \in \A \} $.

We start by proving the following lemma.

\begin{applemma}\label{lemma_nominals_2}
    Let $\KB=\tuple{\T,\A}$ be a \dllitehornH~KB that does not contain any concept of the form $A_a$. Then, for any complex concept $C$ and any individual $a$,
    \begin{itemize}
        \item If $\T_\A\entails A_a\subs C$ then in every model $\I$ of $\KB$, $a^\I\in C^\I$.
    \end{itemize}
\end{applemma}

\begin{proof}
    We prove the lemma by contradiction, assuming that $\T_\A\entails A_a\subs C$ and there is a model $\I$ of $\KB$ s.t. $a^\I\notin C^\I$. We define now an interpretation $\I'$ for $\T_\A$ that corresponds exactly to $\I$, apart from the interpretation of the concepts of form $A_a$, by setting $A_a^{\I'}=\{a^\I\}$ for every individual $a$. Since $\I$ is a model of $\KB$, it is easy to check that $\I'$ is a model of $\T_\A$, and we have that $A_a^{\I'}\not\subseteq C^{\I'}$, that is, $\I'\not\sat A_a\subs C$, against the hypothesis that $\T_\A\entails A_a\subs C$.
\end{proof}

\nd As a  corollary of Lemma \ref{lemma_nominals_2}, the following result follows (see similarly, Corollary 2 in~\cite{Kazakov14}).

\begin{appcorollary}\label{corollary_nominals_3}
  Let $\KB=\tuple{\T,\A}$ be a \dllitehornH~KB that does not contain any concept of the form $A_a$. Then, for any individual $a$ and its corresponding concept $A_a$,  

  \begin{itemize}
      \item If $\T_\A\entails A_a\subs \neg A_a$ then $\KB$ is unsatisfiable.
  \end{itemize}
\end{appcorollary}

\begin{proof}
    Assume  that is not the case, that is, $\T_\A\entails A_a\subs \neg A_a$ and $\KB$ has a model $\I$. 
    
    From $\T_\A\entails A_a\subs \neg A_a$ we have that, for any concept name $A$, $\T_\A\entails A_a\subs A$ and $\T_\A\entails A_a\subs \neg A$, that is, by Lemma \ref{lemma_nominals_2}, in every model $\J$ of $\T_\A$, $a^\J\in A^\J$ and $a^\J\in \neg A^\J$, that is equivalent to $a^\J\notin A^\J$. The fact that for every model $\J$ of $\T_\A$ we have both $a^\J\in A^\J$ and $a^\J\notin A^\J$ implies that $\T_\A$ has no model, which is in contradiction with the consistency of $\KB$.
\end{proof}

\nd Now we can prove Proposition~\ref{nominalABox}.

\begin{proof}[Proof of Proposition~\ref{nominalABox}]

Before proving Points 1 and 2, we introduce the following semantic construction. Assume $\T_\A\not\entails A_a\subs \neg A_a$ for every concept of form $A_a$. This implies that there is a model $\I$ of $\T_\A$ such that ${A_a}^\I\neq \emptyset$ for every concept of form $A_a$. It is easy to prove the existence of such a model: $\T_\A\not\entails A_a\subs \neg A_a$ for every concept of form $A_a$ implies that for each concept $A_a$ there is a model of $\T_\A$ in which $A_a$ is non-empty. Making the disjoint union of all such models, we obtain the desired model of $\T_\A$.

Now, we build an interpretation $\I_{\KB}=\tuple{\Delta_{\I_{\KB}},\cdot_{\I_{\KB}}}$ for the vocabulary of $\KB$, that is, excluding all the concepts of form $A_a$. We define $\I_{\KB}$ through the following procedure.

We initialise $\I_{\KB}$ as follows:
   \begin{itemize}
       \item $\Delta_{\I_{\KB}}=\Delta_{\I}$;
       \item for every concept name $A$ that does not have the form $A_a$, $A^{\I_{\KB}}=A^{\I}$;
       \item for every role name $P$, $P^{\I_{\KB}}=P^{\I}$;
       \item for each individual $a$, $a^{\I_{\KB}}=o$, where $o$ is any object in $A_a^{\I}$.
   \end{itemize}

\nd We then complete the interpretation $\I_{\KB}$ as follows:
\begin{itemize}
    \item for every assertion of form $\rass{b}{c}{P}$ in $\A$, add $(b^{\I_{\KB}},c^{\I_{\KB}})$ to $P^{\I_{\KB}}$;
    \item close $\I_{\KB}$ under all the axioms of the form $R\subs R'\in \T$.
\end{itemize}

\nd Now we want to prove that $\I_{\KB}$ is a model of $\KB=\tuple{\T,\A}$. It is easy to prove that $\I_{\KB}$ is a model of $\T$, starting from the fact that, since $\T\subseteq \T_{\A}$, $\I$ being a model of $\T_{\A}$ implies that $\I$ is a model of $\T$.
Specifically, 
\begin{itemize}
    \item to prove that $\I_{\KB}\sat \{B_1,\ldots,B_n\}\subs C$ for any $\{B_1,\ldots,B_n\}\subs C\in \T$, it is sufficient to prove that, for every basic concept $B$, $B^{\I}=B^{\I_{\KB}}$. That is,
    \begin{itemize}
        \item if $B$ is an atomic concept $A$, then $A^{\I}=A^{\I_{\KB}}$ by definition of $\I_{\KB}$;
        \item if $B$ is a concept $\exists R$, we have immediately from the definition of $\I_{\KB}$ that for every role name $P$, $P^{\I}\subseteq P^{\I_{\KB}}$. Hence, for any object $o$ if $o\in (\exists P)^{\I}$, then $o\in (\exists P)^{\I_{\KB}}$ (and the same for the inverse role $P^-$); 
        For the opposite direction, assume that $o\in (\exists P)^{\I_{\KB}}$. Then either there is a $o'$ s.t. $(o,o')\in P^\I$, or the pair $(o,o')$ has been added to $P^{\I_{\KB}}$ by closure under the role inclusions in $\T$ and the fact that $(o,o')$ has been added to ${P'}^{\I_{\KB}}$ for some role name $P'$ under the following condition:
        \begin{itemize}
            \item $o=a^{\I_{\KB}}$ for some individual $a$, $o'=b^{\I_{\KB}}$ for some individual $b$, and $\rass{a}{b}{P'}\in \A$. 
        \end{itemize}

        \nd In such a case, we would have  in $\I$ the following conditions: $A_a\subs \exists P'.A_b\in\T_{\A}$, $o\in A_a$ and $o'\in A_b$. That is, we would have  in $\I$ that $o\in (\exists P')^{\I}$ and, since $\T_\A$ contains all the role inclusions in $\T$ and $\I$ is a model of $\T_\A$, $o\in (\exists P)^{\I}$. Hence $(\exists P)^{\I}=(\exists P)^{\I_{\KB}}$ for any atomic role $P$, that implies that, for any role $R$, $(\exists R)^{\I}=(\exists R)^{\I_{\KB}}$;
    \end{itemize}
    \item to prove that, for every $R\subs R'\in \T$, $\I_{\KB}\sat R\subs R'$ is immediate, since we have enforced the closure of $\I_{\KB}$ under all such axioms.
\end{itemize}

\nd Therefore, $\I_{\KB}$ is a model of $\T$. Now we have to prove that $\I_{\KB}$ is also a model of $\A$:

\begin{itemize}
    \item for every assertion of the form $\cass{a}{A}\in\A$, we have $A^{\I_{\KB}}=A^{\I}$, $a^{\I_{\KB}}\in A_a^{\I}$, and $A_a^{\I}\subseteq A^{\I}$, hence we can conclude  $a^{\I_{\KB}}\in A^{\I_{\KB}}$, as desired;
    \item for every assertion of the form $\rass{a}{b}{P}\in\A$, we have enforced in the construciton of $\I_{\KB}$ that $\I_{\KB}\sat\rass{a}{b}{P}\in\A$.
\end{itemize}

\nd As a consequence, $\I_{\KB}$ is a model of $\KB$.

We are ready now to prove Point 1. 
In Corollary \ref{corollary_nominals_3} we have proved one half. It remains to prove that if $\KB$ is unsatisfiable, then $\T_\A\entails A_a\subs \neg A_a$ for some individual $a$ occurring in $\KB$. We prove it by contraposition. Assume that $\T_\A\not\entails A_a\subs \neg A_a$ for any individual $a$ occurring in $\KB$. But then we can define the model $\I_{\KB}$ of $\KB$ as described above, and we have proved that $\KB$ is satisfiable.

Now we address to Point 2. 
From right to left it is an immediate consequence of Lemma \ref{lemma_nominals_2}. 

From left to right, we need to prove that, given that $\KB$ is satisfiable, if $\KB\entails \cass{a}{A}$, then $\T_\A\entails A_a\subs A$.
We prove this by contradiction.   We assume that $\KB$ is satisfiable and,   for some assertion $\cass{a^*}{A^*}$, $\KB\entails \cass{a^*}{A^*}$ but $\T_\A \not\entails A_{a^*} \impc A^*$. Let $\I=\tuple{\Delta_\I,\cdot_\I}$ be a model of $\T_\A$ s.t. $\I\not\sat A_{a^*} \impc A^*$. Hence, there is an object $o'$ s.t. $o'\in A_{a^*}$ and $o'\notin A^*$. Now we build an interpretation $\I_{\KB}$ for $\KB$ as described above, with the only extra-condition that for the specific individual $a^*$, associated to the concept $A_{a^*}$, we impose ${a^*}^{\I_{\KB}}=o'$. Following the argument presented above, it is immediate to see that this extra condition does not affect the conclusion that $\I_{\KB}$ is a model of $\KB$.

Clearly $\I_{\KB}\not\sat \cass{a^*}{A^*}$, since ${a^*}^{\I_{\KB}}=o'$, $o\notin {A^*}^\I$, and ${A^*}^{\I_{\KB}}={A^*}^\I$. That is against our assumption that $\KB\entails \cass{a^*}{A^*}$, which concludes the proof.
\end{proof}



\setcounter{proposition}{\getrefnumber{entailmentcor}-1}
\begin{proposition} 
  Consider a \dllitecoreH~TBox $\T$ and any pair of basic concepts $B_1,B_2$. Then $\T \models B_{1} \impc \notc B_{2}$ iff one of the following holds:
  \begin{enumerate}
        \item $B_{1} \impc \notc B_{2} \in \clncH(\T)$;
        \item $B_{1} \impc \notc B_{1} \in \clncH(\T)$;
        \item $B_{2} \impc \notc B_{2} \in \clncH(\T)$.
  \end{enumerate}
\end{proposition}
\begin{proof}
    All the rules defining $\clncH$ are correct \wrt~DL entailment relation $\entails$, that is, if an NI $\alpha$ is in $\clncH(\T)$, then $\T\entails\alpha$. Hence if one of the points 1.~- 3. above hold, then $\T \models B_{1} \impc \notc B_{2}$.

    Vice-versa, assume $\T \models B_{1} \impc \notc B_{2}$, but none of the cases 1.~- 3. above hold. Note that if 
    $B_{1} \impc \notc B_{2} \not \in \clncH(\T)$ then also 
    $B_{2} \impc \notc B_{1} \not \in \clncH(\T)$ has to be the case, as NI-closure is closed under the rule ($\mathbf{\cHContr}$).
    Therefore,
    \begin{equation}\label{eqcap}
    \{B_{1} \impc \notc B_{2}, B_{2} \impc \notc B_{1}, B_{1} \impc \notc B_{1},  B_{2} \impc \notc B_{2}\} \cap cln(\T) = \emptyset 
    \end{equation}

    \nd  holds.

    Now, we show that from (\ref{eqcap}) we can construct a model of $\T$ that does not satisfy $B_{1} \impc \notc B_{2}$, thus obtaining a contradiction.

    Let us assume that $B_{1} \impc \notc B_{2}$ is of the form 
    $\alpha = A_1 \subs \neg A_2$, where $A_1$ and $A_2$ are concept names, and consider the 
    \dllitecore~$\KB= (\T,\A)$, where $\A=\{\cass{a}{A_1}, \cass{a}{A_2}\}$.  We show $\cancH(\KB)$ is the model we are looking for. That is, $\cancH(\KB) \sat \T$ but $\cancH(\KB)\not \sat \alpha$. The last property follows trivially by the form of $\A$. Hence, in the following we  focus on proving  $\cancH(\KB) \sat \T$.

    By Proposition~\ref{chasecorA}, it suffices to show that 
    $\mathtt{DB}(\A)$ is a model of $(cln(\T),\A)$, 
    \ie~we need to show that $\mathtt{DB}(\A)$ is a model of $cln(\T)$, as obviously $\mathtt{DB}(\A) \models \A$. 
    
    We recall that $A_1^{\mathtt{DB}(\A)}= A_2^{\mathtt{DB}(\A)} = 
    \{a\}$, while all other concept names and all role names have empty extension. Therefore, the only NI's in $\clncH(\T)$ that may be violated are of the form $A_1 \subs \neg A_2, A_2 \subs \neg A_1, A_1 \subs \neg A_1$ and $A_2 \subs \neg A_1$. However, we have assumed that (\ref{eqcap}) holds, and thus $\mathtt{DB}(\A) \models \clncH(\T)$. 

    Proceeding analogously as done above, we can easily prove the claim in those cases in which 
    $\alpha \in\{ A \subs \neg \exists R, \exists R \subs \neg A, \exists R_1 \subs \neg \exists R_2  \}$, which concludes.
\end{proof}

\setcounter{proposition}{\getrefnumber{Theo15calv07}-1}
\begin{proposition}[\dllitehornH~analogue of Theorem 15 in~\cite{Calvanese07}] 
     Let $\KB=(\T,\A)$ be a \dllitehornH~KB. Then $\KB$ is satisfiable iff the interpretation $\mathtt{DB}(\A)$ is a model of $(\clnhH(\T),\A)$.
\end{proposition}

\nd To prove Proposition \ref{Theo15calv07}, we need to prove some additional properties.

\begin{appproposition}[\dllitehornH~analogue of Proposition 6 in~\cite{Calvanese07}]\label{prop_prop6calv07}
    Let $\KB=(\T,\A)$ be a \dllitehornH~KB, and let $\alpha$ be a PI in $\T$. Then, if there is a step $i$ in the chasing s.t., for some individual $a$, the condition $f$ for $\alpha$ is satisfied but the corresponding $f_{new}$ is not in $\chasehH_i(\KB)$, then there is a step $j$, $j>i$, s.t. $f_{new}$ is added to $\chasehH_j(\KB)$.
\end{appproposition}
\begin{proof}
    The proof is exactly the same as for Proposition 6 in \cite{Calvanese07}.
\end{proof}

\begin{applemma}[\dllitehornH~analogue of Lemma 7 in~\cite{Calvanese07}]\label{lemma_lemma7calv07}
    Let $\KB=(\T,\A)$ be a \dllitehornH~KB, and let $\T_{PI}$ be the set of the positive inclusions in $\T$. Then $\canhH(\KB)$ is a model of $(\T_{PI},\A)$
\end{applemma}
\begin{proof}
        We prove that $\canhH(\KB)$ satisfies \ii{i} all the assertions in $\A$; and \ii{ii} all the inclusions in $\T_{PI}$. 
        
Concerning \ii{i}, since $\A\subseteq \chasehH(\KB)$, $\canhH(\KB)\sat\A$.
Concerning \ii{ii},  we proceed by contradiction,  assuming that there is some PI in $\T_{PI}$ that is not satisfied by $\canhH(\KB)$. The following cases may occur.

\begin{enumerate}
    \item The axiom has the form $R_1\subs R_2$. As it is not satisfied by $\canhH(\KB)$, there is a pair $a_1,a_2$ s.t. $\rass{a_1}{a_2}{R_1}\in \chasehH(\KB)$, but $\rass{a_1}{a_2}{R_2}\notin \chasehH(\KB)$. However, by 
    Proposition~\ref{prop_prop6calv07}, such a situation would trigger the rule $(\mathbf{\ChH_3})$ at a certain point in the chase, forcing the presence of $\rass{a_1}{a_2}{R_2}$ in $\chasehH(\KB)$, and consequently the axiom $R_1\subs R_2$ would be eventually satisfied, contrary to our assumption. 

    
    \item The axiom has the form $\{B_1,\ldots, B_n\}\subs A$. As it is not satisfied by $\canhH(\KB)$, there is some individual $a$ s.t.~all $\cass{a}{B_i}$ occur in $\chasehH(\KB)$ ($1\leq i\leq n$),
%
 %
    but $\cass{a}{A}\notin \chasehH(\KB)$. However, by Proposition~\ref{prop_prop6calv07}, such a situation would trigger the rule $(\mathbf{\ChH_1})$ at a certain point in the chase, forcing the presence of $\cass{a}{A}$ in $\chasehH(\KB)$, and consequently the axiom $B_1\dland\ldots\dland B_n\subs A$ would be eventually satisfied, contrary to our assumption. 

    \item The axiom has the form $\{B_1,\ldots, B_n\}\subs \exists R$. As it is not satisfied by $\canhH(\KB)$, there is some individual $a$ s.t.~all $\cass{a}{B_i}$ occur in $\chasehH(\KB)$ ($1\leq i\leq n$),
    %
    %
    but $ra(R,a,b)\notin \chasehH(\KB)$ for any individual $b$. 
    However, by Proposition \ref{prop_prop6calv07}, such a situation would trigger the rule $(\mathbf{\ChH_2})$ at a certain point in the chase, forcing the presence of $ra(R,a,b)$ for some new individual $b$ in $\chasehH(\KB)$, and consequently the axiom $\{B_1,\ldots, B_n\}\subs \exists R$ would be eventually satisfied, contrary to our assumption.
    
\end{enumerate}

\nd Therefore, it can not be the case that there is some positive inclusion axiom in $\T_{PI}$ that is not satisfied by $\canhH(\KB)$ and, thus, $\canhH(\KB)$ is a model of $(\T_{PI},\A)$, which concludes.
   
\end{proof}


\begin{applemma}[\dllitehornH~analogue of Lemma 10 in~\cite{Calvanese07}]~\label{lemma_lemma10calv07}
    Let $\T$ be a \dllitehornH~TBox and $\alpha$ a NI. If $\clnhH(\T)\entails \alpha$ then $\T\entails \alpha$.
\end{applemma}
\begin{proof}
    The proof is immediate from the fact that all the inclusions in $\clnhH(\T)$ are entailed by $\T$, and classical DL entailment $\entail$ satisfies the CUT property: namely, if $\T\cup\T'\entails \alpha$ and $\T\entails\beta$ for every $\beta\in\T'$, then $\T\entails \alpha$.
\end{proof}


\begin{applemma}[\dllitehornH~analogue Lemma 12 in~\cite{Calvanese07}]\label{lemma_lemma12calv07}
    Let $\KB=(\T,\A)$ be a \dllitehornH \\ KB. Then $\canhH(\KB)$ is a model of $\KB$ iff the interpretation $\mathtt{DB}(\A)$ is a model of $(\clnhH(\T),\A)$.
\end{applemma}
\begin{proof} 
\ \vspace*{1em}
    \begin{description}
        
        \item[(Left to right)] 
        Let us show that if $\canhH(\KB)$ is a model of $\KB$, then $\mathtt{DB}(\A)$ is a model of $(\clnhH(\T),\A)$.

       The proof is  exactly as the proof of Lemma 12 in~\cite{Calvanese07}. The only differences are that, when they refer to Lemma 10 in~\cite{Calvanese07}, we have to refer here to Lemma~\ref{lemma_lemma10calv07}, and we have to consider that the NIs are now of the form
       \[
       \{B_1,\ldots, B_n\}\subs \neg B_{n+1} \ ,
       \]
        \nd Notice that $A^{\mathtt{DB}(\A)} = A^{\canhH_0(\KB)} \subseteq \canhH(\KB)$ for every concept name  $A$ occurring in $\KB$, and $P^{\mathtt{DB}(\A)} = P^{\canhH_0(\KB)} \subseteq \canhH(\KB)$ for every  role name $P$ occurring in $\KB$. Now, by considering that the structure of a NI is such it cannot be contradicted by restricting the extension of concept names and role names, we can conclude that $\mathtt{DB}(\A)$ is a model of $(\clnhH(\T),\A)$.
        

        \item[(Right to left)] Let us show that if $\mathtt{DB}(\A)$ is a model of $(\clnhH(\T),\A)$, then $\canhH(\KB)$ is a model of $\KB$.

        From Lemma \ref{lemma_lemma7calv07} we know that $\canhH(\KB)$ is a model of $(\T_{PI},\A)$, and consequently it remains to prove that $\canhH(\KB)$ is also a model of $\T\setminus\T_{PI}=\T_{NI}$. We prove this by proving that 
        $\canhH(\KB)$ is a model of  $(\clnhH(\T),\A)$ (notice that $\T_{NI} \subseteq \clnhH(\T)$).
        
        We prove by induction on the construction of $\chasehH(\KB)$ that $\canhH(\KB)$ is a model of $(\clnhH(\T),\A)$.

        \begin{description}
            \item[Base step] $\chasehH_{0}(\KB)=\A$. By assumption $\mathtt{DB}(\A)\models (\clnhH(\T),\A)$, therefore  $\canhH_{0}(\KB)$ is a model of $(\clnhH(\T),\A)$.

            \item[Inductive step] we proceed by contradiction. Let us assume that $\canhH_{i}(\KB)$ is a model of $(\clnhH(\T),\A)$, while $\canhH_{i+1}(\KB)$ is not. 
            
            We consider all the possible ways in which $\chasehH_{i+1}(\KB)$ can be obtained from $\chasehH_{i}(\KB)$ by applying rules $(\mathbf{\ChH_i})$.



\begin{description}
    \item[Case 1] $\chasehH_{i+1}(\KB)$ has been obtained from $\chasehH_{i}(\KB)$ using rule $(\mathbf{\ChH_1})$ applied to $\{B_1, \ldots, B_{n}\} \subs A \in \T$ and individual $a$.
                In this case, we have that all $\cass{a}{B_i}$ are satisfied by $\canhH_{i}(\KB)$ ($1 \leq i \leq n$) and 
                $\cass{a}{A}$ does not occur in $\chasehH_{i}(\KB)$. Therefore, $\cass{a}{A}$ does occur in $\chasehH_{i+1}(\KB)$, and for $\canhH_{i+1}(\KB)$ not to be a model of $(\clnhH(\T),\A)$, while $\canhH_{i}(\KB)$ is, it must be the case that there is a NI $\{B'_1,\ldots,B'_m\}\subs \neg A \in \clnhH(\T)$ s.t.~all $\cass{a}{B'_k}$ are satisfied by $\canhH_{i}(\KB)$ ($1\leq k\leq m$). 
                Let $CL =  \{B'_1,\ldots,B'_m\} = CL'\cup\{B'_j\}$ for some $j$, $1\leq j\leq m$. By applying rule  ($\mathbf{\hHContr}$), we have also $CL'\cup\{A\}\subs \neg B'_j\in \clnhH(\T)$. Now, using rule ($\mathbf{\hHTrans}$) and $\{B_1, \ldots, B_{n}\} \subs A$ we get $CL'\cup\{B_1, \ldots, B_{n}\}\subs \neg B'_j\in \clnhH(\T)$. But this is incompatible with the assumption that  all $\cass{a}{B_i}$ and all $\cass{a}{B'_j}$ are satisfied by  $\canhH_i(\KB)$ and that at the same time $\canhH_i(\KB)$ is a model of $(\clnhH(\T),\A)$.

\item[Case 2] $\chasehH_{i+1}(\KB)$ has been obtained from $\chasehH_{i}(\KB)$ using rule $(\mathbf{\ChH_2})$, applied to some inclusion $\{B_1, \ldots, B_{n}\} \subs \csome R\in \T$, where all $\cass{a}{B_k}$ are satisfied by $\canhH_i(\KB)$ ($1\leq k\leq n$). By definition, $\chasehH_{i+1}(\KB)$ is obtained from $\chasehH_{i}(\KB)$ by adding $ra(R,a,b)$ to $\chasehH_i(\KB)$.
For $\canhH_{i+1}(\KB)$ not to be a model of $(\clnhH(\T),\A)$, it must be the case that there is $\{B'_1,\ldots,B'_m\}\subs \neg \csome R\in\clnhH(\T)$, where all $\cass{a}{B'_k}$ are satisfied by $\canhH_{i}(\KB)$ ($1\leq k\leq m$). From this point we proceed analogously as the previous case.

\item[Case 3] $\chasehH_{i+1}(\KB)$ has been obtained from $\chasehH_{i}(\KB)$ using rule $(\mathbf{\ChH_3})$, applied to some inclusion $R_1\subs R_2$, where $\rass{a}{b}{R_1}$ is satisfied by $\canhH_i(\KB)$, and by definition $\chasehH_{i+1}(\KB)$ is obtained from $\chasehH_{i}(\KB)$ by adding $ra(R_2,a,b)$ to $\canhH_i(\KB)$. W.l.o.g.~let's assume $R_2$ is a role name and $ra(R_2,a,b)=\rass{a}{b}{R_2}$. There are four possible ways in which adding $(a,b)$ to the interpretations of a relation $R_2$ would affect the satisfaction of  $\clnhH(\T)$:

\begin{description}
    \item[Case 3.1] there is a NI $\{B'_1,\ldots,B'_m\}\subs  \neg \exists R_2 \in\clnhH(\T)$ where all $\cass{a}{B_k}$ are satisfied by $\canhH_i(\KB)$ ($1\leq j\leq m$).
    \item[Case 3.2] there is a NI $\{B'_1,\ldots,B'_m\}\subs  \neg \exists R_2^- \in\clnhH(\T)$ where all $\cass{b}{B_k}$ are satisfied by $\canhH_i(\KB)$ ($1\leq j\leq m$).
    \item[Case 3.3] there is  NI $\exists R_2\subs  \neg \exists R_2 \in\clnhH(\T)$.
    \item[Case 3.4] there is  NI $\exists R_2^-\subs  \neg \exists R_2^- \in\clnhH(\T)$.
\end{description}

For Case 3.1, by rule ($\mathbf{\mathbf{\hHContr}}$) we have $\{B'_1,\ldots,B'_{m-1},\exists R_2\}\subs  \neg B'_m \in\clnhH(\T)$, and, by rule 
 ($\mathbf{\mathbf{\RhH_1}}$), $\clnhH(\T)$ contains also $\{B'_1,\ldots,B'_{m-1},\exists R_1\}\subs  \neg B'_m $. But the NI $\{B'_1,\ldots,B'_{m-1},\exists R_1\}\subs  \neg B'_m$ would be falsified by $\canhH_i(\KB)$, against our hypothesis that $\canhH_i(\KB)$ is a model of $\clnhH(\T)$.

We proceed analogously for Case 3.2, referring to rule 
 ($\mathbf{\mathbf{\RhH_2}}$) instead of rule 
 ($\mathbf{\mathbf{\RhH_1}}$).

 For Case 3.3, $\exists R_2\subs  \neg \exists R_2 \in\clnhH(\T)$ implies by the application of the rules ($\mathbf{\mathbf{\RhH_1}}$), ($\mathbf{\mathbf{\hHContr}}$), and ($\mathbf{\mathbf{\RhH_1}}$) that $\exists R_1\subs  \neg \exists R_1 \in\clnhH(\T)$, against the fact that $\rass{a}{b}{R_1}$ is satisfied by $\canhH_i(\KB)$.

We proceed analogously for Case 3.4, just using ($\mathbf{\mathbf{\RhH_2}}$) instead of ($\mathbf{\mathbf{\RhH_1}}$).

 \end{description}   
  \end{description}  
    \end{description}
\nd This concludes the proof.
\end{proof}


\begin{appcorollary}[\dllitehornH~analogue of Corollary 13 in~\cite{Calvanese07}]\label{corollary_corollary13calv07}
    Let $\T$ be a \dllitehornH \\ TBox, and $\alpha$ be a negative inclusion. If $\T\entails \alpha$, then $\clnhH(\T)\entails \alpha$.
\end{appcorollary}
\begin{proof}
    We prove the claim by contradiction: assume that $\T\entails \alpha$ but 
    $\clnhH(\T)\not\entails\alpha$. Now we show that from $\clnhH(\T)\not\entails\alpha$ one can construct a model of $\T$ that does not satisfy $\alpha$, thus obtaining a contradiction.

    So, assume $\clnhH(\T)\not\entails\alpha$. 
    Let us assume that $\alpha$ has the form  
    \[
    \{B_1,\ldots,B_n\}\subs \neg B \ ,
    \]
\nd where $B$ is a concept name $A$ (the other cases of $B$ being of the form $\exists P$  or $\exists P^-$ can be proven similarly). Now, consider $\KB=(\T,\A)$, where $\A$ contains $\cass{a}{A}$ and, for all $B_j$ ($1\leq j\leq n$):

    \begin{itemize}
        \item $\cass{a}{A_j}$, if $B_j=A_j$ for some concept name $A_j$;
        \item $\rass{a}{b}{P_j}$ for some new individual $b$, if $B_j=\exists P_j$ for some role name $P_j$;
        \item $\rass{b}{a}{P_j}$ for some new individual $b$, if $B_j=\exists P_j^-$ for some role name $P_j$.
    \end{itemize}

\nd We now show that $\canhH(\KB)$ is the model we are looking for, that is, $\canhH(\KB)  \models \T$, but 
$\canhH(\KB) \not\models \alpha$. Obviously, by construction of $\A$, $\canhH(\KB) \not\models \alpha$. Hence, in the following we concentrate on proving that $\canhH(\KB)  \models \T$.

By construction of $\mathtt{DB}(\A)$, the only NIs that can be violated by $\mathtt{DB}(\A)$ have the form
$CL \subs \neg B'$, with 
$CL\cup\{B'\} \subseteq  \{B_1,\ldots,B_n,A\}$.
But, by assumption, we have that $\clnhH(\T)\not\entails \{B_1,\ldots,B_n\}\subs \neg A$ and, thus, 
$\clnhH(\T)\not\entails CL \subs \neg B'$, \ie~ 
$CL \subs \neg B' \notin \clnhH(\T)$.
Therefore,
we can conclude that $\mathtt{DB}(\A)$ is a model of $\clnhH(\T)$ and, thus, a model of $(\clnhH(\T),\A)$.
Then, from Lemma~\ref{lemma_lemma12calv07} it follows that $\canhH(\KB)$ is a model of $\KB$, which concludes.
\end{proof}

\nd In order to prove Proposition \ref{Theo15calv07}, we need also the analogous of Lemma~14 in~\cite{Calvanese07}.

\begin{applemma}[\dllitehornH~analogue of Lemma 14, \cite{Calvanese07}] \label{Lemma14calv07}
     Let $\KB=(\T,\A)$ be a \dllitehornH \\ KB. Then $\KB$ is satisfiable iff $\canhH(\KB)$ is a model of $\KB$. 
\end{applemma}
\begin{proof}
    The proof follows the one of~\cite[Lemma 14]{Calvanese07}. 
    \begin{description}
        \item[(Right to left)]  If $\canhH(\KB)$ is a model of $\KB$ then obviously $\KB$ is satisfiable.

        \item[(Left to right)]  It suffices to prove that if $\canhH(\KB)$ is a not a model of $\KB$ then  $\KB$ is not satisfiable. By Lemma~\ref{lemma_lemma12calv07}, it follows that $\mathtt{DB}(\A)$ is not a model of $(\clnhH(\T),\A)$. That is, $\mathtt{DB}(\A)$ is not a model of $\clnhH(\T)$ (as $\mathtt{DB}(\A)$ is a model of $\A$). Therefore, there exists a NI $\alpha \in \clnhH(\T)$ such that $\mathtt{DB}(\A) \not \sat \alpha$.
        Let us assume that $\alpha$ has the form  
        \[
        \{B_1,\ldots,B_n\}\subs \neg B \ ,
        \]
        \nd where $B$ is a concept name $A$ (the other cases of $B$ being of the form $\exists P$  or $\exists P^-$ can be proven similarly). Then there is $a \in \Delta^{\mathtt{DB}(\A)}$ such that $a \in B_i^{\mathtt{DB}(\A)}$ ($1\leq i \leq n$) and $a \in A^{\mathtt{DB}(\A)}$. That is, all $\cass{a}{B_i}$ and $\cass{a}{A}$ occur in $\A$.    
        Assume now to the contrary that a model $\I$ of $\KB$ exists. Then, we may construct a homomorphism $\mu$ from $\Delta^{\mathtt{DB}(\A)}$ to $\Delta^\I$ such that $\mu(b) = b^\I$ for each individual $b$ occurring in $\A$. note that $\I$ assigns a distinct object to each such individual $b$ as $\I$ is a model of $\A$. Moreover, as  $\I$ is a model of $\A$, $\mu(a) \in B_i^\I$ ($1\leq i \leq n$) and $\mu(a) \in A^{\mathtt{DB}(\A)}$ holds. But then, 
        $\I$ does not satisfy $\alpha$, contrary to the assumption that $\I$ is a model of $\KB$. Therefore, $\KB$ is not satisfiable.
    \end{description}
\end{proof}

\nd Given the propositions above, Proposition \ref{Theo15calv07} follows immediately by combining Lemma~\ref{lemma_lemma12calv07} and Lemma~\ref{Lemma14calv07}.



\setcounter{proposition}{\getrefnumber{entailmenthornH}-1}
\begin{proposition}[\dllitehornH~analogue of Proposition~\ref{entailmentcor}] 
  Consider a \dllitehornH~TBox $\T$. Then $\T \models CL \impc \notc B$ iff one of the following holds:
  \begin{enumerate}
        \item $CL' \impc \notc B \in \clnhH(\T)$ with $CL' \subseteq CL$;

    \item $CL' \impc \neg B' \in \clnhH(\T)$ with $CL'\cup\{B'\} \subseteq CL$;
  
     \item $B \impc \notc B \in \clnhH(\T)$.

  \end{enumerate}
\end{proposition}


\begin{proof}
It is not difficult to see that all NI-closure rules are \emph{correct}, that is, if 
$CL \subs \neg B \in \clnhH(\T)$ then $\T \models CL \subs \neg B$. As a consequence, if one of the points 1.~- 3. above holds then we have, respectively, $\T\entails CL' \impc \notc B$, $\T\entails CL'\cup\{B'\} \impc \neg B'$, or $\T\entails B \impc \notc B$, and  we can conclude 
$\T \models CL \impc \notc B$, since the latter is a consequence of any of the three options.

    For the other direction, we proceed by contradiction. Assume that $\T \models CL \impc \notc B$ and none of the cases 1.~- 3. above hold.     We show that in such a case  we can construct a model of $\T$ that does not satisfy $CL \impc \notc B$, thus obtaining a contradiction.

Let $CL=\{B_1,\ldots,B_n\}$ and $B=B_{n+1}$, that it, the NI $CL \impc \notc B$ will have the form $\{B_1,\ldots,B_n\}\subs \neg B_{n+1}$, where each basic concept $B_j$ will be either a concept name or an existential role.

Let $\KB=(\T,\A)$ be the KB defined by $\T$ and an ABox $\A$ representing a counterexample of $\{B_1,\ldots,B_n\}\subs \neg B_{n+1}$. That is, for an individual $a$, $\A$ contains, for every $B_j$ ($1\leq j\leq (n+1)$):

    \begin{itemize}
        \item $\cass{a}{A_j}$, if $B_j=A_j$ for some  concept name $A_j$;
        \item $\rass{a}{b}{P_j}$ for some new individual $b$, if $B_j=\exists P_j$ for some  role name $P_j$;
        \item $\rass{b}{a}{P_j}$ for some new individual $b$, if $B_j=\exists P_j^-$ for some role name $P_j$.
    \end{itemize}

\nd We show now that $\canhH(\KB)$ 
    is the model we are looking for. That is, $\canhH(\KB) \sat \T$ but $\canhH(\KB)\not \sat \{B_1,\ldots,B_n\}\subs \neg B_{n+1}$. The latter property follows immediately by the form of $\A$. Hence, it remains to prove that $\canhH(\KB) \sat \T$.

    By Lemma \ref{lemma_lemma12calv07}, it suffices to show that 
    $\mathtt{DB}(\A)$ is a model of $(\clnhH(\T),\A)$, 
    \ie~$\mathtt{DB}(\A)$ is a model of $\clnhH(\T)$, since clearly $\mathtt{DB}(\A) \sat \A$. 

   By the construction of the interpretation $\mathtt{DB}(\A)$, the only NIs in $\clnhH(\T)$ that may be violated 
    by $\mathtt{DB}(\A)$ are of the form $CL'\subs \neg B'$, where $CL'\cup\{B'\}\subseteq \{B_1,\ldots,B_{n+1}\}$. We have the following options:

    \begin{description}
        \item[a] $B'=B$. Then we would have three options:
        \begin{description}
            \item[a.1] $B\notin CL'$. But then, we  have case 1., against the hypothesis that none of the cases 1.- 3. hold.
            
            \item[a.2] $\{B\}\subset CL'$. That is, the violated NI in $\clnhH(\T)$ has the form $CL^*\cup\{B\}\subs \neg B$, where $CL^*\neq\emptyset$. In such a case, by a double application of ($\mathbf{\hHContr}$), we  have  $CL^*\subs \neg B \in \clnhH(\T)$. That is,  again we  have case 1., against the hypothesis that none of the cases 1.- 3. holds.
            
            \item[a.3] $CL'=\{B\}$. But then, we are in case 3., again against the hypothesis that none of the cases 1.- 3. hold.
        \end{description}
        
        \item[b] $B'\neq B$. Then we would have two options:
        \begin{description}
            \item[b.1] $B\notin CL'$. But then, we  have case 2., against the hypothesis that none of the cases 1.- 3. hold.
            
            \item[b.2] $B \in CL'$. That is, the violated NI in $\clnhH(\T)$ has the form $CL^*\cup\{B\}\subs \neg B'$, where $CL^*$ may be empty. In such a case, by an application of ($\mathbf{\hHContr}$), we  have  $CL^*\cup\{B'\}\subs \neg B \in \clnhH(\T)$. That is,  again we  have case 1., against the hypothesis that none of the cases 1.- 3. holds.
        \end{description}
    \end{description}
\nd Therefore, none of the above cases can hold and, thus, $\mathtt{DB}(\A) \models \clnhH(\T)$, which concludes.
\end{proof}

\section{Proofs of Section~\ref{defdllite}: Defeasible \dllite} \label{defdlliteproofs}

\subsubsection*{Section~\ref{sect_RCsemantics}: Rational Closure - Semantics} \label{sect_RCsemanticsproofs}
 \vspace*{1em}


\setcounter{proposition}{\getrefnumber{dsat}-1}
\begin{proposition} 
    Any defeasible \dllitehornH~KB~$\KB=\tuple{\T,\D}$ is satisfiable.
\end{proposition}
\begin{proof}
    As in the proof of Proposition \ref{tsat}, just consider an interpretation mapping concept names and role names to the empty set.
\end{proof}








\subsubsection*{Section~\ref{sect_RC_horn}: Reasoning Procedures for \dllitehornH~under RC} \label{sect_RC_hornproofs}
\vspace*{1em}

\setcounter{theorem}{\getrefnumber{prop_dllitehornRC}-1}
\begin{theorem}
Given a defeasible \dllitehornH~KB  $\KB=\tuple{\T,\D}$, a DI $\{B_1,\ldots,B_n\}\dsubs A$ is in the RC of $\KB$ (that is, $\KB\minent \{B_1,\ldots,B_n\}\dsubs A$) iff 
\[
\mathtt{RationalClosure.Horn}(\KB, \{B_1,\ldots,B_n\}\dsubs A) \text{ returns } \mathtt{true} \ .
\]
\end{theorem}




\nd We prove Theorem \ref{prop_dllitehornRC} by referring to the results proved for $\mathcal{ALC}$ in \cite{BritzEtAl2021}. In \cite{BritzEtAl2021}, it is proved that the procedure $\mathtt{RationalClosure}(\KB,\alpha)$ (see below) is correct and complete to decide the subsumption problem for DIs \wrt~RC,  characterised semantically using the \emph{Big Ranked Model} as in Section \ref{sect_RCsemantics}. This completeness proof applies to all the languages that extend $\mathcal{ALC}$ and satisfy the \emph{disjoint union model property} (see Section \ref{sect_RCsemantics}), including $\mathcal{ALCHI}$, which extends $\mathcal{ALC}$ with inverted roles and role inclusions. 

Since \dllitehornH~(and, consequently, \dllitecoreH) is a sublanguage of $\mathcal{ALCHI}$, a \dllitehornH~KB $\KB=\tuple{\T,\D}$ is also an $\mathcal{ALCHI}$~KB. Thus the  $\mathtt{RationalClosure}(\KB,\alpha)$ procedure remains a correct and complete procedure for defeasible \dllitehornH~and \dllitecoreH. However, since the procedure $\mathtt{RationalClosure}(\KB,\alpha)$ was designed for $\mathcal{ALC}$ extensions, its computational cost is not optimised for \dllitehornH, as it employs logical operators, like negation and disjunction, which exceed \dllitehornH's expressivity. For this reason, we defined in Sections \ref{sect_RC_horn} and \ref{sect_RC_core} procedures that are tailored to the expressivity of \dllitehornH~and \dllitecoreH, respectively. Here we prove that 
$\mathtt{RationalClosure}(\KB,\alpha) = \mathtt{RationalClosure.Horn}(\KB,\alpha)$, when restricted to defeasible \dllitehornH.



A key step in the proof of Theorem \ref{prop_dllitehornRC} is the following lemma, for which we use the expressivity of the language $\mathcal{ALCHI}$ (see \cite{BaaderEtAl2007}). As can be seen in the procedure $\mathtt{Exceptional}$, line 3, for $\ALC$ the exceptionality check is done using a concept $\neg C\dlor D$ as a representative of a DI $C\dsubs D$, that is, the equivalent of the material implication corresponding to the subsumption relation.

\begin{applemma}\label{lemma_main}
Given a \dllitehornH~KB $\KB=\tuple{\T,\D}$  and a set of basic concepts  $\{B_1,\ldots,B_n\}$, the following two subsumption tests are equivalent:
{\scriptsize
\begin{eqnarray}
&& \T  \entails  \bigsqcap \{ \neg (B'_1\dland\ldots\dland B'_m) \dlOr A' \mid \{B'_1,\ldots,B'_m\}\} \usually A' \in \D\} \subs \neg (B_1\dland\ldots\dland B_n) \label{b1} \\
&& \T \cup \{\{B'_1,\ldots,B'_m, \delta_\D\} \subs A' \mid \{B'_1,\ldots,B'_m\} \usually A' \in\D\}  \entails   \{B_1,\ldots,B_n, \delta_\D\} \subs\bot 
 \label{b2}\end{eqnarray}
}
where $\delta_\D$ is a new atomic concept, the entailment (\ref{b1}) uses the full expressivity of $\mathcal{ALCHI}$, and the entailment (\ref{b2}) uses the expressivity of \dllitehornH.
\qed
\end{applemma}

\begin{proof}
We need to show that (\ref{b1}) and (\ref{b2}) are equivalent.

Assume (\ref{b1}) holds and assume to the contrary that (\ref{b2}) does not hold. Then there is a classical interpretation $\I_{\A}=\tuple{\Delta^{\I},\cdot^{\I}}$ such that
$\I \models \T$, $\I \models \{B'_1,\ldots,B'_m, \delta_\D\} \subs A'$ for all $\{B'_1,\ldots,B'_m\} \usually A' \in \D$, and $\I \not\models \{B_1,\ldots,B_n, \delta_\D\} \subs\bot$. 
Hence there must be some $o \in \Delta^\I$ s.t.  $o \in B_1^\I\cap\ldots\cap B_n^\I \cap \delta_\D^\I$. 

Since $\I \models \T$ and (\ref{b1}) holds, then 
{\small
\[
\I \models \bigsqcap \{ \neg (B'_1\dland\ldots\dland B'_m) \dlOr A' \mid \{B'_1,\ldots,B'_m\}\} \usually A' \in \D\} \subs \neg (B_1\dland\ldots\dland B_n) \ .
\]
}
\nd Now, since $\I \models \T$ and $\I \models \{B'_1,\ldots,B'_m, \delta_\D\} \subs A'$ for all $\{B'_1,\ldots,B'_m\} \usually A' \in \D$, it follows that $o\in (\neg (B'_1\dland\ldots\dland B'_m) \dlOr A')^\I$. That, by (\ref{b1}), implies $o\in (\neg (B_1\dland\ldots\dland B_n))^\I$, against our conclusion that $o \in B_1^\I\cap\ldots\cap B_n^\I \cap \delta_\D^\I$.
Hence (\ref{b1}) implies (\ref{b2}).

Conversely, assume that (\ref{b2}) holds, while (\ref{b1}) does not. Therefore, there is an interpretation $\I$ such that
$\I \models \T$ and, for some $o \in \Delta^\I$, 
\[
o\in (\bigsqcap\{ \neg (B'_1\dland\ldots\dland B'_m) \dlOr A' \mid \{B'_1,\ldots,B'_m\}\} \usually A' \in \D\})^\I
\]

\nd and $o\in   (B_1\dland\ldots\dland B_n)^\I$.
Now, we extend $\I$ to $\delta_\D$ by setting
\[
\delta^\I_\D =  (\bigsqcap_{\{B'_1,\ldots,B'_m\} \scriptsize{\usually} A' \in \D} ( \neg (B'_1\dland\ldots\dland B'_m) \dlOr A'  ))^\I \ .
\]
\nd Note that $o \in \delta^\I_\D$. Then, by construction, $\I \models \{B'_1,\ldots,B'_m,\delta_\D\} \subs A'$, for all $\{B'_1,\ldots,B'_m\}\} \usually A' \in \D$ holds. 
Therefore, $\I$ is a model of the antecedent in (\ref{b2}) with $o \in (\bigsqcap\{ \neg (B'_1\dland\ldots\dland B'_m) \dlOr A' \mid \{B'_1,\ldots,B'_m\}\} \usually A' \in \D\})^\I$ and $o \in   (B_1\dland\ldots\dland B_n\dland \delta_\D)^\I$, which cannot be the case as (\ref{b2}) holds by assumption.

Hence we conclude that (\ref{b2}) implies (\ref{b1}), and we have proved the equivalence of the two conditions.
\end{proof}

\nd  Using the results in \cite{BritzEtAl2021} and Lemma \ref{lemma_main}, it is quite straightforward to prove all we need. We do so by comparing the corresponding procedures.

 We start by comparing the procedure $\mathtt{Exceptional}(\T,\E)$ from \cite[p.17]{BritzEtAl2021} (that we report below for completeness) and $\mathtt{Exceptional.Horn}(\T,\E)$.

 \begin{procedure}[h]
\caption{Exceptional($\T,\E$)}\label{algexALC}
{
 \KwIn{A TBox $\T\text{ and a set of DIs } \E\subseteq{\D}$}
 \KwOut{$\E'\subseteq\E$ such that $\E'$ is a set  of exceptional axioms \wrt~$(\T,\E)$}
$\E':=\emptyset$\;
 \ForEach{$C \dsubs D \in\E$}{	
   	\If{$\T\entails\bigsqcap\{\neg C'\dlor D'\mid C' \dsubs D' \in\E\}\subs \neg C$}{$\mathcal{E'}$ := $\mathcal{E'}\cup\{C \dsubs D\}$\;}
}
\Return{$\mathcal{E'}$}
}
\end{procedure}

\nd In the light of Lemma~\ref{lemma_main}, the equivalence of the two procedures $\mathtt{Exceptional}$ and $\mathtt{Exceptional.Horn}$ 
\wrt~defeasible \dllitehornH~KBs is straightforward.

\begin{applemma}\label{prop_except}
    Let $\KB=\tuple{\T,\D}$ be any \dllitehornH~defeasible KB. Then 
    \[
    \mathtt{Exceptional}(\T,\D)=\mathtt{Exceptional.Horn}(\T,\D) \ .
    \]
\end{applemma}

\begin{proof}
    The only difference between the two procedures is between the line 3 of \\ 
    $\mathtt{Exceptional}(\T,\E)$ and line 4 of $\mathtt{Exceptional.Horn}(\T,\E)$. Once we restrict to defeasible \dllitehornH KBs, line 3 in $\mathtt{Exceptional}(\T,\E)$ becomes
    \[
    \T  \entails  \bigsqcap \{ \neg (B'_1\dland\ldots\dland B'_m) \dlOr A' \mid \{B'_1,\ldots,B'_m\}\} \usually A' \in \D\} \subs \neg (B_1\dland\ldots\dland B_n) \ ,
    \]
\nd  and the equivalence with line 4 of $\mathtt{Exceptional.Horn}(\T,\E)$ is immediately guaranteed by Lemma \ref{lemma_main}.
\end{proof}

\nd In the same way, we prove the equivalence between procedure $\mathtt{ComputeRanking.Horn}$ and $\mathtt{ComputeRanking}$ from \cite[p.18]{BritzEtAl2021} (that we replicate for completeness below).

 \begin{procedure}[h]
\caption{ComputeRanking($\KB$)\label{algrankALC}}
{
\KwIn{Defeasible KB $\KB=\tuple{\T,\D}$}
\KwOut{Defeasible KB $\tuple{\T^*,\D^*}$, partitioning (ranking) $\RR =\{\D_0,\ldots,\D_n\}$ of $\D^*$}
$\T^*$:=$\T$\;
$\D^*$:=$\D$\;
$\RR$:=$\emptyset$\;
\Repeat {$\D_\infty=\emptyset$}
{$i$ := $0$\;
$\mathcal{E}_{0}$ := $\D^*$\;
$\mathcal{E}_{1}$ := $\mathtt{Exceptional}$($\T^*,\mathcal{E}_{0}$)\;
\While{$\E_{i+1}\neq\E_{i}$}{
$i$ := $i$ + 1\;
$\E_{i+1}$ := $\mathtt{Exceptional}$($\T^*,\E_{i}$)\;
}
$\D_\infty$ := $\E_{i}$\;
$\D^*$ := $\D^*\setminus\D_\infty$\;
$\T^*$ := $\T^*\cup\{C \subs D\mid C\dsubs D \in \D_\infty\}$\;}

\For{$j$ = $1$ to $i$}{
$\D_{j-1}$ := $\E_{j-1}\setminus\E_{j}$\;
$\RR$ := $\RR\cup\{\D_{j-1}\}$\;
}
\Return{$\tuple{\tuple{\T^*,\D^*},\RR}$}
}
\end{procedure}
\nd Again, to show the equivalence between these two procedures, given Lemmas~\ref{lemma_main} and~\ref{prop_except},  is straightforward.

\begin{applemma}\label{prop_computerank}
    Let $\KB=\tuple{\T,\D}$ be any \dllitehornH~defeasible KB. Then 
    \[
    \mathtt{ComputeRanking}(\KB)=\mathtt{ComputeRanking.Horn}(\KB) \ .
    \]
\end{applemma}

\begin{proof}
    By Lemma \ref{prop_except}, we can replace every call to $\mathtt{Exceptional}$ with a call to $\mathtt{Exceptional.Horn}$. 
    Therefore, it remains to prove the equivalence between the lines 13 in the two procedures. 
    
    That is, we need to prove that 
    \[
    \T^*\cup \{(B_1\dland\ldots\dland B_n)\subs A\mid (B_1\dland\ldots\dland B_n)\dsubs A \in \D_\infty\}
    \]
    \nd that is, in our \dllitehornH~notation, 
    \begin{equation} \label{subs1}
    \T^*\cup \{\{B_1,\ldots, B_n\}\subs A\mid \{B_1,\ldots, B_n\}  \dsubs A \in \D_\infty\}
    \end{equation}
     and 
     \begin{equation} \label{subs2}
     \T^*\cup \{\{B_1,\ldots, B_n\}\subs \bot\mid \{B_1,\ldots, B_n\}\dsubs A \in \D_\infty\}
     \end{equation}
     \nd are logically equivalent, \ie, they have the same set of models.


    Obviously, if $\I$ is a model of (\ref{subs2}),  then it must be also a model of (\ref{subs1}) as $\{B_1,\ldots, B_n\}^\I = \emptyset$.
    
    On the other hand, assume now that $\I=\tuple{\Delta^\I,\cdot^\I}$ is a model of (\ref{subs1}).
    As $\D_\infty$ is a fixed point of the procedure $\mathtt{Exceptional}$, for every $\{B_1,\ldots, B_n\}\dsubs A\in \D_\infty$, 
    \begin{equation} \label{subs3}
    \T^*\entails \bigsqcap\{\neg (B_1\dland\ldots\dland B_n)\dlor A\mid B\dsubs A \in \D_\infty\}\subs \neg (B_1\dland\ldots\dland B_n) \ .    
    \end{equation}
    
    \nd As is a model of (\ref{subs1}), 
    \begin{equation} \label{subs4}
    (\bigsqcap\{\neg  (B_1\dland\ldots\dland B_n)\dlor A\mid \{B_1,\ldots, B_n\}\dsubs A \in \D_\infty\})^\I=\Delta^\I
    \end{equation}
    \nd follows. From (\ref{subs3}) and  (\ref{subs4}) it follows that for all $\{B_1,\ldots, B_n\}\subs A\in\D^\infty$,
    \begin{equation} \label{subs5}
        (B_1\dland\ldots\dland B_n)^\I = \emptyset.
    \end{equation}

    \nd Therefore, $\I$ is a model of (\ref{subs2})


          In summary, also lines 13  of the  $\mathtt{ComputeRanking}(\KB)$ and $\mathtt{ComputeRanking.Horn}(\KB)$ procedures are equivalent and, thus, we can now conclude that, for any \dllitehornH~defeasible KB $\KB$, 
          $\mathtt{ComputeRanking}(\KB)=\mathtt{ComputeRanking.Horn}(\KB)$, which concludes the proof.
\end{proof}

\nd Eventually, we  prove the equivalence of the procedures $\mathtt{RationalClosure.Horn}$ and  $\mathtt{RationalClosure}$ from \cite[p.21]{BritzEtAl2021} (that we replicate for completeness below).

\begin{procedure}[h]
\caption{RationalClosure($\KB,\alpha$)} \label{RCalg}
{
\KwIn{Defeasible KB $\KB=\tuple{\T,\D}$ and DI $\alpha$ of the form  $B \dsubs A$}
\KwOut{$\mathtt{true}$ iff $C \dsubs D$ is in the RC of $\KB$}
  $\tuple{\tuple{\T^*,\D^*},\{\D_0,\ldots,\D_n\}}$ := $\mathtt{ComputeRanking}$($\KB$)\;
  $i$ :=  $0$; $\D_\R$ :=  $\D^*$\;
  \While{$\T^*\entails \bigsqcap\{\neg C'\dlor D'\mid C'\dsubs D'\in\D_\R\}\dland C\subs \bot$ {\bf and} $\D_\R\neq\emptyset$}{
    $\D_\R$ := $\D_\R\backslash\D_{i}$; $i$ := $i + 1$\;
  }
  \eIf{$\D_\R\neq\emptyset$}{\Return{$\T^*\entails 
  \bigsqcap\{\neg C'\dlor D'\mid C'\dsubs D'\in\D_\R\}\dland C\subs D$}}{\Return{$CL$}}
  }
\end{procedure}

\begin{applemma}\label{prop_RC}
    Let $\KB=\tuple{\T,\D}$ be any \dllitehornH~defeasible KB and $\alpha$ any \dllitehornH~DI. Then
    \[
    \mathtt{RationalClosure}(\KB, \alpha)=\mathtt{RationalClosure.Horn}(\KB, \alpha) \ .
    \]
\end{applemma}

\begin{proof}
    Given Lemmas~\ref{prop_computerank} and~\ref{lemma_main}, the proof is quite straightforward by comparing the two procedures. The only point that needs to be addressed is the status of lines 1-3 and 5-7 in $\mathtt{RationalClosure.Horn}$, which do not appear in $\mathtt{RationalClosure}$. Actually, such lines have been added  only for the economy of the computation, since they cover  limit cases in which the result depends only on the classical entailment relation. 
    
    In such a case it is sufficient to check $\T$ (line 1) or $\T^*$ (line 5) to take a decision. The correctness of this is guaranteed by Proposition \ref{prop_classic_inclusion}, which concludes the proof.
%
%
\end{proof}





\nd Now, we recall the following theorem  (rephrased for the present notation) proved in \cite{BritzEtAl2021}, and valid for all the DLs enjoying the Finite Model and the Disjunctive Union properties, including $\mathcal{ALCHI}$ and its sublanguages, like \dllitecoreH and \dllitehornH.

\begin{apptheorem}[\cite{BritzEtAl2021}, Theorem 5]\label{th_britzcomplete}
    Given a defeasible KB  $\KB=\tuple{\T,\D}$, a DI $\alpha$ is in the RC of $\KB$ (that is, $\KB\minent \alpha$) iff 
$\mathtt{RationalClosure(\KB,\alpha)}$ returns $\mathtt{true}$.
\end{apptheorem}

\nd Now Theorem \ref{prop_dllitehornRC} follows immediately.

\begin{proof}[Proof of Theorem~\ref{prop_dllitehornRC}]
Theorem \ref{prop_dllitehornRC} follows from Lemma \ref{prop_RC} and Theorem \ref{th_britzcomplete}. 
\end{proof}

\setcounter{proposition}{\getrefnumber{prop_classic_inclusion}-1}
\begin{proposition}
    Let $\KB=\tuple{\T,\D}$ any \dllitehornH~defeasible KB, and $\T^*$ be the TBox obtained from $\mathtt{ComputeRanking.Horn}(\KB)$. Then for any  classical \dllitehornH~inclusion $\alpha$,    $\KB\minent \alpha$ iff $\T^*\entails\alpha$.
\end{proposition}

\begin{proof}
    In case $\alpha$ is a GCI of the form $\{B_1,\ldots,B_n\}\subs C$, the proof is quite straightforward:  in \cite[p.22]{BritzEtAl2021} it has been proved that, for any GCI $B_1\dland\ldots\dland B_n\subs C$, $\KB\minent B_1\dland\ldots\dland B_n\subs C$ iff $\T^*\entails\alpha$, where $\T^*$ is determined by $\mathtt{ComputeRanking}(\KB)$. Lemma \ref{prop_computerank} states that $\mathtt{ComputeRanking.Horn}(\KB)$ and $\mathtt{ComputeRanking}(\KB)$ determine the same $\T^*$, hence we conclude that $\KB\minent \{B_1,\ldots, B_n\}\subs C$ iff $\T^*\entails\alpha$, where $\T^*$ is determined by $\mathtt{ComputeRanking.Horn}(\KB)$.
    
    It remain to extend the proof to the role inclusions of the form $R\subs S$. That is, we need to prove that
    \[
    \tuple{\T,\D}\minent R\subs S\text{ iff }\T^*\entails R\subs S \ .
    \]

\nd From right to left it is straightforward: since $\OI$ is a model of $\T^*$ and $\T^*\entails R\subs S$, then $\OI\sat R\subs S$, that is, $\tuple{\T,\D}\minent R\subs S$.

From left to right, we prove it by contraposition. Assume $\T^*\not\entails R\subs S$. Since \dllitehornH~has the FMP \cite{Rosati08}, there is a finite model $\I=\tuple{\Delta^\I,\cdot^\I}$ of $\T^*$ s.t. $\I\not\sat R\subs S$. Let $\D^*$ be ranked into the sets $\{\D_0,\ldots,\D_n\}$ by $\mathtt{ComputeRanking.Horn}(\KB)$, and consider a model $\RI=\tuple{\Delta,\cdot^{\RI},\prec^{\RI}}$ in $\Moddelta{\KB}$ s.t. for every $C\dsubs D\in\D^*$, $h_{\RI}(C)=i$ iff $C\dsubs D\in\D_i$. We extend $\RI$ into a new model $\RI'=\tuple{\Delta^{\RI'},\cdot^{\RI'},\prec^{\RI'}}$ by adding $\I$ in the $n+1$ layer. That is:

\begin{itemize}
    \item $\Delta^{\RI'}=\Delta\uplus\Delta^{\I}$;
    \item $A^{\RI'}=A^{\RI}\cup A^{\I}$ for every  concept name $A$;
    \item $P^{\RI'}=P^{\RI}\cup P^{\I}$ for every role name $P$;
    \item $h_{\RI'}(x)=h_{\RI}(x)$ for every $x\in \Delta$; $h_{\RI'}(x)=n+1$ for every $x\in \Delta^{\I}$.
\end{itemize}

\nd It is easy to prove that $\RI'$ is a model of $\KB$ that does not satisfy $R\subs S$. Also, since $\Delta^{\RI'}$ has been obtained combining a countable domain with a finite one, it is a countable set too. That is, through a bijection we can define an equivalent model $\RI'_{\Delta}$ with $\Delta$ as domain. Consequently, $\RI'_{\Delta}\in\Moddelta{\KB}$ and takes part to the definition of the big ranked model $\OI$ for $\KB$. Since $\RI'_{\Delta}$ does not satisfy $R\subs S$, by the construction of $\OI$ we can conclude that $\OI\not\sat R\subs S$, that is, $\tuple{\T,\D}\not\minent R\subs S$, as desired.
\end{proof}

\subsubsection*{Section~\ref{sect_RC_core}: Rational Closure for \dllitecoreH}\label{sect_RC_coreproofs}
\vspace{1ex}



\setcounter{proposition}{\getrefnumber{prop_dllitecoreexcept}-1}
\begin{proposition} 
Given any defeasible  \dllitecoreH~KB $\KB=\tuple{\T,\D}$ and any basic concept $B$, the following conditions are equivalent:
\begin{enumerate}
    \item $\KB^\delta\entails \{B,\delta\}\subs\neg B$;
    \item $\norm{\KB} \entails \norm{B} \subs \neg \norm{B}$;
    \item   $\norm{B} \impc \notc \norm{B} \in \clnn(\norm{\KB})$.
\end{enumerate}

\end{proposition}

\nd In order to prove Proposition \ref{prop_dllitecoreexcept}, we need to prove some related facts.

\begin{applemma}\label{lemma_separation}
    Let $\KB=\tuple{\T,\D}$ be a defeasible \dllitecoreH~KB and $B_1,B_2$ be basic concepts. Then $B_1\subs \neg B_2\in \clnn(\norm{\KB})$ iff $B_1\subs \neg B_2\in \clncH(\T)$.
\end{applemma}

\begin{proof}
    The lemma is easily proved by induction, by checking that 
    \begin{itemize}
        \item in $\norm{\KB}$ every inclusion is composed either only by basic concepts or only by normal  concept names. Hence, for every inclusion $B_1\subs \neg B_2$ where $B_1$ and $B_2$ are not normal concept names, $B_1\subs \neg B_2\in\norm{\KB}$ iff $B_1\subs \neg B_2\in\T$.
        \item for every rule ($\mathbf{\cHRef}$), ($\mathbf{\cHContr}$), ($\mathbf{\cHTrans}$), ($\mathbf{\RcH_i}$) ($1\leq i \leq 3$) and ($\nSup$) that may be used to build   
        $\clnn(\norm{\KB})$, we may derive the NI $B_1\subs \neg B_2$ only from inclusions containing basic concepts. Hence, since $B_1$ and $B_2$ are not normal concept names, the rule ($\nSup$) cannot be used, and the only rules that can be used to add $B_1\subs \neg B_2$ to $\clnn(\norm{\KB})$ are the same rules that can be used to add  $B_1\subs \neg B_2$ to $\clncH(\T)$.
    \end{itemize}
\end{proof}

\begin{applemma}\label{lemma_negative_strict}
    Let $\KB=\tuple{\T,\D}$ be a defeasible \dllitecoreH~KB and  $B_1,B_2$ be basic concepts.
    Then $B_1\subs \neg B_2\in \clnn(\norm{\KB})$ iff $B_1\subs \neg B_2\in\clnd(\KB^\delta)$.
\end{applemma}
\begin{proof} \ \\
\begin{description}
    \item[(Right to left):]

    We prove that if $B_1\subs \neg B_2\in\clnd(\KB^\delta)$, then $B_1\subs \neg B_2\in\clnn(\norm{\KB})$. 
    We prove this by  induction on the application of the rules to build $\clnd(\KB^\delta)$. 
    So, let $\clnhH_i(\KB^\delta)$ be the set of NIs obtained from $\KB^\delta$ after the application of $i$ rules. Our induction hypothesis is that, for any basic concepts $B_1,B_2$, if $B_1\subs \neg B_2\in\clnhH_i(\KB^\delta)$, then $B_1\subs \neg B_2\in\clnn(\norm{\KB})$.

\begin{description}
        \item[Step 0:] $\clnhH_0(\KB^\delta)$ corresponds to all the NI's in $\KB^\delta$. Hence, if $B_1\subs \neg B_2\in\clnhH_0(\KB^\delta)$, then $B_1\subs \neg B_2\in\KB^\delta$. As $B_1,B_2$ are basic concepts, we have $B_1\subs \neg B_2\in\T$, and, thus, $B_1\subs \neg B_2\in \norm{\KB}$, and consequently $B_1\subs \neg B_2\in \clnn(\norm{\KB})$.

        \item[Step $i+1$:]  We need to prove that, whatever rule we use to obtain $\clnhH_{i+1}(\KB^\delta)$ from $\clnhH_{i}(\KB^\delta)$, if $B_1\subs \neg B_2\in\clnhH_{i+1}(\KB^\delta)$, then $B_1\subs \neg B_2\in\clnn(\norm{\KB})$, for any basic concepts $B_1,B_2$. We have to consider all the possible rules used to build  $\clnhH_{i+1}(\KB^\delta)$:

    \begin{description}
    \item[Case ($\mathbf{\hHContr}$):] {\bf IF} $CL \cup\{B_2\}  \subs \neg B_1 \in \clnhH_i(\KB^\delta)$ 
     {\bf THEN}  $CL \cup\{B_1\}  \subs \neg B_2 \in \clnhH_{i+1}(\KB^\delta)$. To obtain $B_1\subs \neg B_2$, $CL=\emptyset$ must be the case.
     As $B_2\subs \neg B_1 \in \clncH_i(\KB^\delta)$, by induction hypothesis $B_2\subs \neg B_1 \in \clnn(\norm{\KB})$ and, thus, 
     $B_1\subs \neg B_2\in\clnn(\norm{\KB})$ because of the closure of $\clnn(\norm{\KB})$ under the ($\mathbf{\cHContr}$) rule.

    \item[Case ($\mathbf{\hHTrans}$):] {\bf IF} $CL \subs B_3 \in \T$ {\bf AND} 
    $CL' \cup\{B_3\}  \subs \neg B_2 \in \clnhH_i(\KB^\delta)$ 
    {\bf THEN}  $CL \cup CL' \subs \neg B_2 \in \clnhH_{i+1}(\KB^\delta)$. To obtain $B_1\subs \neg B_2$, it must be the case that $CL \cup CL' = \{B_1\}$. So, $B_1 \subs B_3 \in \T$ and $CL' \subseteq  \{B_1\}$. As $B_1 \subs B_3 \in \T$, we have $B_1\neq\delta \neq B_3$. 
    If $CL' =  \{B_1\}$ then we would have $\{B_1, B_3\}  \subs \neg B_2 \in \clnhH_i(\KB^\delta)$, which is only possible if $\delta \in \{B_1, B_3\}$, as $\KB^\delta$ is the $\delta$-transformation of a \dllitecoreH~KB and, thus, exactly one of $B_1$ or $B_2$ has to be $\delta$.
    But, this is not possible as $B_1\neq\delta \neq B_3$ 
    and, thus, $CL'=\emptyset$ has to be the case. Therefore,
    $B_3  \subs \neg B_2 \in \clnhH_i(\KB^\delta)$. As a consequence, by induction hypothesis $B_3  \subs \neg B_2 \in \clnn(\norm{\KB})$
    and, thus, $B_1  \subs \neg B_2 \in \clnn(\norm{\KB})$  because of the closure of $\clnn(\norm{\KB})$ under the ($\mathbf{\cHTrans}$) rule.

    \item[Case ($\mathbf{\RhH_1}$):]  
    {\bf IF} $R_2 \subs R_1 \in \T$ {\bf AND} 
    $CL \cup \{\exists R_1\}  \subs \neg B_2 \in \clnhH_i(\norm{\KB})$ 
    {\bf THEN}  $CL \cup \{\exists R_2 \}  \subs \neg B_2  \in \clnhH_{i+1}(\norm{\KB})$. 
    To obtain $B_1\subs \neg B_2$, $CL=\emptyset$ must be the case and $B_1 = \exists R_2$. By induction hypothesis, 
    $\exists R_1  \subs \neg B_2  \in  \clnn(\norm{\KB})$. As $R_2 \subs R_1 \in \T$, $R_2 \subs R_1 \in \clnn(\norm{\KB})$ holds and, thus, 
    $\exists R_2\subs \neg B_2\in \clnn(\norm{\KB})$, \ie~$B_1\subs \neg B_2\in\clnn(\norm{\KB})$, because of the closure of $\clnn(\norm{\KB})$ under ($\mathbf{\RcH_1}$).
    

    \item[Case ($\mathbf{\RhH_i}$) with $i > 1$:].  The proof is similar to the case ($\mathbf{\RhH_1}$).  
    
\end{description}
\end{description}
\item[(Left to right):]
    We need to prove that if $B_1\subs \neg B_2\in\clnn(\norm{\KB})$, then $B_1\subs \neg B_2\in\clnd(\KB^\delta)$.  
By Lemma~\ref{lemma_separation}, if $B_1\subs \neg B_2\in\clnn(\norm{\KB})$, then $B_1\subs \neg B_2\in\clncH(\T)$. Since $\T\subseteq\KB^\delta$, $\clncH(\T)\subseteq \clnd(\KB^\delta)$ follows and, thus, $B_1\subs \neg B_2\in \clnd(\KB^\delta)$.
\end{description} 

\nd This completes the proof.
\end{proof}

\nd From Lemmas~\ref{lemma_separation} and \ref{lemma_negative_strict} we  immediately derive the Proposition \ref{corr_negative_strict}.

\begin{appproposition}\label{corr_negative_strict}
    Let $\KB=\tuple{\T,\D}$ be a defeasible \dllitecoreH~KB and  $B_1,B_2$ be basic concepts. Then the following conditions are equivalent:
    \begin{enumerate}
        \item $B_1\subs \neg B_2\in \clnn(\norm{\KB})$;
        \item $B_1\subs \neg B_2\in \clncH(\T)$;
        \item $B_1\subs \neg B_2\in\clnd(\KB^\delta)$.        
    \end{enumerate}
\end{appproposition}

\nd To prove Proposition~\ref{prop_dllitecoreexcept} we  need also Lemma \ref{lemma_exceptional_n_form}.

\begin{applemma}\label{lemma_exceptional_n_form}
    Let $\KB=\tuple{\T,\D}$ be a defeasible \dllitecoreH~KB and $B_1,B_2$ basic concepts. Then $\norm{B_1} \subs \neg \norm{B_2} \in \clnn(\norm{\KB})$ iff either $\{B_1, \delta\}\subs \neg B_2$ or $B_1\subs \neg B_2$ is in $\clnd(\KB^\delta)$.
\end{applemma}

\begin{proof} \ \\
\begin{description}

\item[(Left to right):]

    We need to prove that if $\norm{B_1}\subs \neg \norm{B_2}\in\clnn(\norm{\KB})$, then either  $\{B_1,\delta\}\subs \neg B_2$ or $B_1\subs \neg B_2$ is in $\clnd(\KB^\delta)$. We prove this by  induction on the application of the rules of the NI-closure  $\clnn$. Let $\clnn_i(\norm{\KB})$ be the set of negative inclusions obtained from $\norm{\KB}$ after the application of $i$ rules of $\clnn$.  
 Our induction hypothesis is that, for any basic concepts $B_1,B_2$, if $\norm{B_1}\subs \neg \norm{B_2}\in\clnn_i(\norm{\KB})$, then either $\{B_1,\delta\}\subs \neg B_2$ or $B_1\subs \neg B_2$ is in $\clnd(\KB^\delta)$. We need to prove that, whatever rule we use to obtain $\clnn_{i+1}(\norm{\KB})$, if $\norm{B_1}\subs \neg \norm{B_2}\in\clnn_{i+1}(\norm{\KB})$, then 
 either $\{B_1,\delta\}\subs \neg B_2$ or $B_1\subs \neg B_2$ is in $\clnd(\KB^\delta)$, for any basic concepts $B_1,B_2$.

    \begin{description}
        \item[Step 0:] If $\norm{B_1}\subs \neg \norm{B_2}\in\clnn_0(\norm{\KB})$ then $\norm{B_1}\subs \neg \norm{B_2}\in\norm{\KB}$. By the definition of the $n$-transformation, this means that either $B_1\dsubs \neg B_2$ or $B_1\subs \neg B_2$ is in $\KB$. In the former case, by $\delta$-transformation we have $\{B_1,\delta\}\subs \neg B_2\in\KB^\delta$, while in the latter case, we have $B_1\subs \neg B_2\in\KB^\delta$.

         \item[Step $i+1$:] We have to consider all the possible rules used in $\clnn$ to derive an inclusion of the form $\norm{B_1}\subs \norm{B_2}$. Note that ($\mathbf{\RcH_1}$)-($\mathbf{\RcH_3}$) do not need to be considered, as we cannot use them to derive a NI of the form $\norm{B_1}\subs \norm{B_2}$.

        \begin{description}
            \item[($\mathbf{\cHContr}$):] {\bf IF}  $B_{1} \impc \notc B_2\in \clnn_i(\norm{\KB})$ {\bf THEN} we have 
            $B_{2} \impc \notc B_{1} \in \clnn_{i+1}(\norm{\KB})$. For  $\norm{B_1}\subs \neg \norm{B_2}\in \clnn_{i+1}(\norm{\KB})$ via ($\mathbf{\cHContr}$), $\norm{B_2}\subs \neg \norm{B_1}$,  must be  in $\clnn_i(\norm{\KB})$. Hence, by induction hypothesis,  either $\{B_2,\delta\}\subs \neg B_1$ or $B_2\subs \neg B_1$ is in $\clnd(\KB^\delta)$, and consequently 
            either $\{B_1,\delta\}\subs \neg B_2$ or $B_1\subs \neg B_2$ is in $\clnd(\KB^\delta)$ because of the closure under ($\mathbf{\hHContr}$).
 
            \item[($\mathbf{\cHTrans}$):] {\bf IF}  $B_{1} \impc  B_{2} \in \norm{\KB}$ {\bf AND}  $B_{2} \impc \notc B_{3} \in \clnn_i(\norm{\KB})$ {\bf THEN}  $B_{1} \impc \notc B_{3} \in \clnn_{i+1}(\norm{\KB})$. For obtaining $\norm{B_1}\subs \neg \norm{B_2}\in \clnn_{i+1}(\norm{\KB})$ by the application of ($\mathbf{\cHTrans}$), we need $\norm{B_1}\subs  \norm{B_3}\in \norm{\KB}$ and $\norm{B_3}\subs \neg \norm{B_2}\in \clnn_{i}(\norm{\KB})$, for some $B_3$. By induction hypothesis,  $\norm{B_3}\subs \neg \norm{B_2}\in \clnn_{i}(\norm{\KB})$ implies either
            $\{B_3,\delta\}\subs \neg B_2$ or $B_3\subs \neg B_2$ in $\clnd(\KB^\delta)$. By $n$-transformation, $\norm{B_1}\subs  \norm{B_3}\in \norm{\KB}$ implies  $B_1\subs  B_3$ in the original $\T$ or $B_1\dsubs  B_3$ in the original $\D$. By $\delta$-transformation, the former implies $B_1\subs B_3\in \KB^\delta$, and the latter implies $\{B_1,\delta\}\subs  B_3\in \KB^\delta$. In any case, by ($\mathbf{\hHTrans}$) we have either $\{B_1,\delta\}\subs \neg B_2$ or  $B_1\subs \neg B_2$ in $\clnd(\KB^\delta)$.

            \item[($\nSup$):] {\bf IF}   $B_1\subs\neg B_2 \in \clnn_i(\norm{\KB})$ {\bf THEN}  $\norm{B_1}\subs\neg \norm{B_2}\in\clnn_{i+1}(\norm{\KB})$. If $\norm{B_1}\subs\neg \norm{B_2}\in\clnn_{i+1}(\norm{\KB})$ has been obtained by ($\nSup$), then $B_1\subs\neg B_2 \in \clnn_i(\norm{\KB})$ and, thus, $B_1\subs\neg B_2 \in \clnn(\norm{\KB})$. By Lemma~\ref{lemma_negative_strict} this implies $B_1\subs\neg B_2 \in \clnd(\KB^\delta)$, as desired.
        \end{description}
    \end{description}

\item[(Right to left):]


    We need to prove that if either $\{B_1,\delta\}\subs \neg B_2$ or $B_1\subs \neg B_2$ is in $\clnd(\KB^\delta)$, then $\norm{B_1}\subs \neg \norm{B_2}\in\clnn(\norm{\KB})$. Note that Lemma~\ref{lemma_negative_strict} and the rule ($\mathbf{\norm{Sup}}$) already guarantee  that if $B_1\subs \neg B_2\in\clnd(\KB^\delta)$, then $\norm{B_1}\subs \neg \norm{B_2}\in\clnn(\norm{\KB})$. Therefore, it remains to prove that if $\{B_1,\delta\}\subs \neg B_2\in\clnd(\KB^\delta)$, then $\norm{B_1}\subs \neg \norm{B_2}\in\clnn(\norm{\KB})$.

    We prove it by  induction on the application of the rules of the NI-closure $\clnd$. So, let $\clnd_i(\KB^\delta)$ be the set of negative inclusions obtained from $\KB^\delta$ after the application of $i$ rules of $\clnd$. Hence $\clnd_0(\KB^\delta)$ contains all the NI's in $\KB^\delta$. 
Our induction hypothesis is that,  for any basic concepts $B_1,B_2$,  if $\{B_1,\delta\}\subs \neg B_2\in\clnd_i(\KB^\delta)$, then $\norm{B_1}\subs \neg \norm{B_2}\in\clnn(\norm{\KB})$. 



    \begin{description}
        \item[Step 0:] $\{B_1,\delta\}\subs \neg B_2\in\clnd_0(\KB^\delta)$  means $\{B_1,\delta\}\subs \neg B_2\in\KB^\delta$. That is, 
        $B_1\dsubs \neg B_2 \in \D$. By $n$-transformation we have that $\norm{B_1}\subs \neg \norm{B_2}\in\norm{\KB}$, that implies $\norm{B_1}\subs \neg \norm{B_2}\in\clnn(\norm{\KB})$.

        \item[Step $i+1$:] 
        
        We need to prove that, whatever rule we use to obtain $\clnd_{i+1}(\KB^\delta)$, then, for any basic concept $B_1,B_2$, if $\{B_1,\delta\}\subs \neg B_2\in\clnd_{i+1}(\KB^\delta)$ then $\norm{B_1}\subs \neg \norm{B_2}\in\clnn(\norm{\KB})$. We have to consider all the possible rules used in $\clnd$:

        \begin{description}
            \item[($\mathbf{\hHContr}$):] {\bf IF} $CL \cup\{B_1\}  \subs \neg B_2 \in \clnd_i(\KB^\delta)$ 
     {\bf THEN}  $CL \cup\{B_2\}  \subs \neg B_1 \in \clnd_{i+1}(\KB^\delta)$.  If $\{B_1,\delta\}\subs \neg B_2$ has been obtained by ($\mathbf{\hHContr}$) then we have two possible cases.

     \begin{description}
         \item[Case 1:] $CL=\{\delta\}$, with $\{B_1,\delta\}\subs \neg B_2$ obtained from $\{B_2,\delta\}\subs \neg B_1$, that must be  in $\clnd_i(\KB^\delta)$ and then, by induction hypothesis,  $\norm{B_2}\subs \neg \norm{B_1}$ is in $\clnn(\norm{\KB})$. Hence $\norm{B_1}\subs \neg \norm{B_2}\in\clnn(\norm{\KB})$ because of the closure under ($\mathbf{\cHContr}$).

         \item[Case 2:] $CL=\{B_1\}$, with $\{B_1,\delta\}\subs \neg B_2$ obtained from $\{B_1,B_2\}\subs \neg \delta$, that must be  in $\clnd_i(\KB^\delta)$. Due to the $\delta$-transformation, $\{B_1,B_2\}\subs \neg \delta \not\in \KB^\delta$ and it must have been added to $\clnd_i(\KB^\delta)$ by the application of some rule. That is, $\{B_1,B_2\}\subs \neg \delta \in \clnd_j(\KB^\delta)$ for some $0< j\leq i$.
         Now, it can be proven, see {\bf Part 2} below, that if $\{B_1,B_2\}\subs \neg \delta$ is in $\clnd_j(\KB^\delta)$ ($0<j\leq i$) then $\norm{B_1}\subs \neg \norm{B_2}\in\clnn(\norm{\KB})$ has to be the case as well.
\end{description}

\item[($\mathbf{\hHTrans}$):] {\bf IF} $CL \subs B_1 \in \KB^\delta$ {\bf AND} 
    $CL' \cup\{B_1\}  \subs \neg B_2 \in \clnd_i(\KB^\delta)$ 
    {\bf THEN}  $CL \cup CL' \subs \neg B_2 \in \clnd_{i+1}(\KB^\delta)$. For obtaining $\{B_1,\delta\}\subs \neg B_2$ by the application of ($\mathbf{\hHTrans}$), we have the following possibilities: $\{B_1,\delta\}\subs B_3 \in \KB^\delta$ and $\{B_3,\delta\}\subs \neg B_2 \in \clnd_{i}(\KB^\delta)$; or $\{B_1,\delta\}\subs B_3 \in \KB^\delta$ and 
    $B_3\subs \neg B_2 \in \clnd_{i}(\KB^\delta)$; or $B_1\subs B_3 \in \KB^\delta$ and 
    $\{B_3,\delta\}\subs \neg B_2 \in \clnd_{i}(\KB^\delta)$. We can focus on the first case, as the other cases follow in a similar manner.

 So, assume $\{B_1,\delta\}\subs B_3 \in \KB^\delta$ and $\{B_3,\delta\}\subs \neg B_2 \in \clnd_{i}(\KB^\delta)$. By $\delta$-transformation and $n$-transformation we have that $\{B_1,\delta\}\subs B_3\in\KB^\delta$ implies $B_1\dsubs B_3\in \D$, that in turn implies $\norm{B_1}\subs \norm{B_3}\in \norm{\KB}$. By induction hypothesis, $\{B_3,\delta\}\subs \neg B_2\in\clnd_{i}(\KB^\delta)$ implies $\norm{B_3}\subs \neg \norm{B_2}\in\clnn(\norm{\KB})$ and, thus, by ($\mathbf{\cHTrans}$), 
  $\norm{B_1}\subs \norm{B_3} \in \norm{\KB}$ and $\norm{B_3} \subs \neg \norm{B_2}\in\clnn(\norm{\KB})$ imply $\norm{B_1}\subs \neg \norm{B_2}\in\clnn(\norm{\KB})$.

    \item[($\mathbf{\RhH_1}$):]  
    {\bf IF} $R_1 \subs R_2 \in \KB^\delta$ {\bf AND} 
    $CL \cup \{\exists R_2\}  \subs \neg B \in \clnd_i(\KB^\delta)$ 
    {\bf THEN}  $CL \cup \{\exists R_1 \}  \subs \neg B  \in \clnd_{i+1}(\KB^\delta)$. Assume that $\{B_1,\delta\}\subs \neg B_2$ has been obtained by ($\mathbf{\RhH_1}$): hence $B_1=\exists R_1$, $CL=\{\delta\}$ and $\{\exists R_1,\delta\}\subs \neg B_2$ has been obtained from $R_1\subs R_2$ and $\{\exists R_2,\delta\}\subs \neg B_2$, the former being in the original $\T$ and the latter in $\clnd_i(\KB^\delta)$. Then, $R_1\subs R_2\in\T$ implies $\norm{\exists R_1}\subs \norm{\exists R_2}\in\norm{\KB}$ by $n$-transformation, while, by induction hypothesis,  $\norm{\exists R_2}\subs \neg \norm{B_2}\in\clnn(\norm{\KB})$. Hence, by ($\mathbf{\cHTrans}$), $\norm{\exists R_1} \subs \neg \norm{B_2}\in\clnn(\norm{\KB})$.


\item[($\mathbf{\RhH_2}$):]   The reader can easily check that the proof for ($\mathbf{\RhH_2}$) is similar to the case $\mathbf{\RhH_1}$).



  \item[($\mathbf{\RhH_3}$):]   Eventually, ($\mathbf{\RhH_3}$) does not apply, since we cannot obtain a NI of form $\{B_1,\delta\}\subs \neg B_2$ from them.
            
\end{description}

\end{description}
\end{description}


\nd It remains to complete Case 2 of ($\mathbf{\hHContr}$) above, that we address next.

\vspace*{1em}
\nd {\bf Part 2.} Let us recall the induction hypothesis on $i$: if $\{B_1,\delta\}\subs \neg B_2\in\clnd_i(\KB^\delta)$, then $\norm{B_1}\subs \neg \norm{B_2}\in\clnn(\norm{\KB})$. We need to prove that if $\{B_1,B_2\}\subs \neg \delta \in \clnd_{j}(\KB^\delta)$ ($0< j\leq i$) then $\norm{B_1}\subs \neg \norm{B_2}\in\clnn(\norm{\KB})$ has to be the case as well. We prove the latter by induction on $0 < j \leq i$. 


 \begin{description}
        \item[Step $j=1$:] Assume $\{B_1,B_2\}\subs \neg \delta \in\clnd_1(\KB^\delta)$. The only possibility is that it has been obtained via ($\mathbf{\hHContr}$) from either $\{B_1,\delta\}\subs \neg B_2\in \clnhH_{0}(\KB^\delta)$ $\subseteq \clnhH_{i}(\KB^\delta)$ or $\{B_2,\delta\}\subs \neg B_1\in \clnhH_{0}(\KB^\delta) \subseteq \clnhH_{i}(\KB^\delta)$. 
        %
        %
        By induction hypothesis on $i$, either $\norm{B_1}\subs \neg \norm{B_2}$ or 
        $\norm{B_2}\subs \neg \norm{B_1}$ is in $\clnn(\norm{\KB})$. In the latter case, $\norm{B_1}\subs \neg \norm{B_2}\in\clnn(\norm{\KB})$ by ($\mathbf{\cHContr}$).

        \item[Step $j+1 \leq i$:]  Assume $\{B_1,B_2\}\subs \neg \delta \in\clnd_{j+1}(\KB^\delta)$.  We have to consider all the possible rules used to derive it (note that rule ($\mathbf{\RhH_3}$)  does not apply):

    \begin{description}
    \item[($\mathbf{\hHContr}$):] This case is like step $j=1$: if $\{B_1,B_2\}\subs \neg \delta$ has been obtained using ($\mathbf{\hHContr}$), the only two possibilities are that is has been obtained  from either $\{B_1,\delta\}\subs \neg B_2\in \clnhH_{j}(\KB^\delta) \subseteq \clnhH_{i}(\KB^\delta)$ or $\{B_2,\delta\}\subs \neg B_1\in \clnhH_{j}(\KB^\delta) \subseteq \clnhH_{i}(\KB^\delta)$. By induction hypothesis on $i$, $\norm{B_1}\subs \neg \norm{B_2}\in\clnn(\norm{\KB})$ or $\norm{B_2}\subs \neg \norm{B_1}\in\clnn(\norm{\KB})$. In the latter case, $\norm{B_1}\subs \neg \norm{B_2}\in\clnn(\norm{\KB})$ by ($\mathbf{\cHContr}$).

    \item[($\mathbf{\hHTrans}$):] Assume $\{B_1,B_2\}\subs \neg \delta \in \clnd_{j+1}(\KB^\delta)$ has been obtained from $B_1\subs B_3\in \KB^\delta$ and $\{B_3,B_2\}\subs \neg \delta \in \clnd_{j}(\KB^\delta) \subseteq \clnhH_{i}(\KB^\delta)$  (for some basic concept $B_3$). By $n$-transformation  $\norm{B_1}\subs  \norm{B_3}\in \norm{\KB}$. 
    By induction hypothesis on $j$, $\norm{B_3}\subs \neg \norm{B_2}\in\clnn(\norm{\KB})$  and, thus, by  ($\mathbf{\cHTrans}$), $\norm{B_1}\subs \neg \norm{B_2}\in\clnn(\norm{\KB})$.

    \item[($\mathbf{\RhH_1}$):] Assume that $\{\exists R_2,B_2\}\subs \neg \delta \in \clnd_{j+1}(\KB^\delta)$ has been obtained from $R_2\subs R_1\in \KB^\delta$ and $\{\exists R_1,B_2\}\subs \neg \delta \in \clnd_j(\KB^\delta)$  (for some role $R_1$). 
    By $n$-transformation  $\norm{\exists R_2} \subs  \norm{\exists R_1}\in \norm{\KB}$. By induction hypothesis on $j$, $\norm{\exists R_1}\subs \neg \norm{B_2}\in\clnn(\norm{\KB})$ and, thus, by  ($\mathbf{\cHTrans}$), $\norm{\exists R_2}\subs \neg \norm{B_2}\in\clnn(\norm{\KB})$. 

    \item[($\mathbf{\RhH_2}$):] Assume $\{\exists R_2^-,B_2\}\subs \neg \delta \in \clnd_{j+1}(\KB^\delta)$ has been obtained from $R_2\subs R_1\in \KB^\delta$ and $\{\exists R_1^-,B_2\}\subs \neg \delta \in \clnd_j(\KB^\delta)$  (for some role $R_1$). In this case the prove is the same as for ($\mathbf{\RhH_1}$).
\end{description}

\end{description}

\end{proof}

\nd Now the poof of Proposition \ref{prop_dllitecoreexcept} is straightforward. 

\begin{proof}[Proof of Proposition \ref{prop_dllitecoreexcept}]
It is an immediate consequence of Lemma \ref{lemma_exceptional_n_form} and   Proposition~\ref{entailmentcor}. 
Specifically, 
by Proposition~\ref{entailmentcor}, 
$\norm{B} \impc \notc \norm{B} \in \clnn(\norm{\KB})$ iff
$\norm{\KB} \entails \norm{B} \subs \neg \norm{B}$.
Now, assume $\KB^\delta\entails \{B,\delta\}\subs\neg B$, so by Lemma \ref{lemma_exceptional_n_form}, 
$\norm{B} \impc \notc \norm{B} \in \clnn(\norm{\KB})$. Vice-versa, assume 
$\norm{B} \impc \notc \norm{B} \in \clnn(\norm{\KB})$. Then, by Lemma~\ref{lemma_exceptional_n_form}
either  $\{B, \delta\}\subs \neg B$ or $B \subs \neg B$ is in $\clnd(\KB^\delta)$.
In the former case, $\KB^\delta\entails \{B,\delta\}\subs\neg B$ follows. In the latter case, 
$\KB^\delta\entails B\subs\neg B$, which implies $\KB^\delta\entails \{B,\delta\}\subs\neg B$.
This concludes the proof.
\end{proof}

\begin{appproposition}\label{correspondensP}
Given a defeasible \dllitecoreH~KB  $\KB=\tuple{\T,\D}$ and a DI  $B \dsubs A$, then  $\norm{\KB} \models \norm{B} \subs \norm{A}$ 
iff $\KB^\delta \models \{B,\delta\}\subs A$.
\end{appproposition}

\begin{proof} We proceed by proving both directions semantically, by contraposition.

\begin{description}

\item[(Left to right):]

We prove  that  if $\KB^\delta \not\models \{B,\delta\}\subs A$ then $\norm{\KB} \not\models \norm{B} \subs 
\norm{A}$.
So, 
$\KB^\delta \not\models \{B,\delta\}\subs A$ implies that there is some \dllitecoreH~interpretation $\I_{\delta}=\tuple{\Delta^{\I_{\delta}},\cdot^{\I_{\delta}}}$ that satisfies $\KB^\delta$ \sst~there is $x\in (B^{\I_{\delta}}\cap \delta^{\I_{\delta}})$ with $x\notin A^{\I_{\delta}}$. Now we define a \dllitecoreH~interpretation $\I_{\norm{}}=\tuple{\Delta^{\I_{\norm{}}},\cdot^{\I_{\norm{}}}}$ with $\Delta^{\I_{\norm{}}}=\Delta^{\I_{\delta}}$  and $\cdot^{\I_{\norm{}}}$ defined as follows:

\begin{itemize}
    \item $A^{\I_{\norm{}}}=A^{\I_{\delta}}$ for every  concept name $A$;
    \item $P^{\I_{\norm{}}}=P^{\I_{\delta}}$ for every  role name $P$;
    \item $\norm{B}^{\I_{\norm{}}}=B^{\I_{\delta}}\cap \delta^{\I_{\delta}}$ for every basic concept $B$.
\end{itemize}

\nd Now, it can easily be verified that by $\delta$-transformation and $n$-transformation we have that $\I_{\norm{}}$ is a model of $\norm{\KB}$ 
\sst~$\I_{\norm{}}\not\sat \norm{B} \subs \norm{A}$, as $x\in \norm{B}^{\I_{\norm{}}}$ but $x\notin \norm{A}^{\I_{\norm{}}}$.

\item[(Right to left):] We prove that if $\norm{\KB} \not\models \norm{B} \subs \norm{A}$ then $\KB^\delta \not\models \{B,\delta\}\subs A$.
So, $\norm{\KB} \not\models \norm{B} \subs \norm{A}$ implies that there is a 
model $\I_{\norm{}}=\tuple{\Delta^{\I_{\norm{}}},\cdot^{\I_{\norm{}}}}$ of the \ \dllitecoreH~KB $\norm{\KB}$ \sst~there is $\bar{x}\in \norm{B}^{\I_{\norm{}}}$ with $\bar{x}\notin \norm{A}^{\I_{\norm{}}}$.

Now we are going to define a  model for the \dllitehornH~KB $\KB^\delta$ that does not satisfy $\{B,\delta\}\subs A$. 
To start with, we define a model of $\T$.

Let $\Delta$ be any countable domain, and let $\M^\T$ be the set of all the models of $\T$ with domain $\Delta$. Let $\BM^\cup_\T=\tuple{\Delta^{\BM^\cup_\T},\cdot^{\BM^\cup_\T}}$ be the disjoint union of the $\dllitehornH$ interpretations in $\M^\T$. Specifically,  given $\Delta=\{x,y\ldots\}$, we indicate as $\Delta_\I=\{x_\I,y_\I,\ldots\}$ a copy of the domain $\Delta$ associated with an interpretation $\I\in\M^\T$, and we define 

\begin{itemize}
    \item $\Delta^{\BM^\cup_\T}=\bigcup_{\I\in\M^\T} \Delta_\I$;
    \item $\cdot^{\BM^\cup_\T}$ is defined as follows: for every $x_\I, y_{\I'}\in\Delta^{\BM^\cup_\T}$, every  concept name $A$ and every role name $P$,
\begin{itemize}
\item $x_\I\in A^{\BM^\cup_\T}$ iff $x\in A^{\I}$;
\item $(x_\I, y_{\I'})\in P^{\BM^\cup_\T}$ iff $\I=\I'$ and $(x,y)\in P^{\I}$.
\end{itemize}
\end{itemize}

Since $\dllitecoreH$ satisfies the \emph{Disjoint Union Property}, 
$\BM^\cup_\T$ is a model of $\T$. Moreover, since $\dllitecoreH$ satisfies the \emph{Finite Model Property} (Remark \ref{finitedllitehornH}), given  any concept $C$, if it is satisfied by some model of $\T$, then it is also satisfied by $\BM^\cup_\T$, as we are going to show in the next paragraph.

If $C$ is satisfied by some model of $\T$, then, by the \emph{Finite Model Property}, it is satisfied by some finite model $\I$ of $\T$. From $\I$ it is possible to define a model $\I'$ with a countable domain (e.g., by doing the disjoint union of a countable set of copies of $\I$). Given $\I'$, we can define a model $\I^\Delta\in\M^\T$ by defining a bijection $b$ from the domain of $\I'$ to $\Delta$ and defining over $\Delta$ a corresponding interpretation for concept names and role names. Specifically, for any $x\in\Delta$ and any atomic concept $A$, $x\in A^{\I^\Delta}$ iff $b^-(x)\in A^{\I'}$; for any pair $(x,y)\in\Delta\times\Delta$ and any role name $P$, $(x,y)\in P^{\I^\Delta}$ iff $(b^-(x),b^-(y))\in P^{\I'}$. $\I^\Delta$ satisfies $C$ and takes part in the definition of $\BM^\cup_\T$, and consequently the interpretation of $C$ will be non-empty also in $\BM^\cup_\T$. So, a concept $C$ is satisfiable \wrt~a TBox $\T$ iff $C$ is satisfied by $\BM^\cup_\T$, which concludes this part.

Now, from $\BM^\cup_\T$ we create a new interpretation $\BM'=\tuple{\Delta^{\BM'},\cdot^{\BM'}}$ that interprets also the concept $\delta$. We set $\Delta^{\BM'}=\Delta^{\BM^\cup_\T}\cup\{z\}$, where $z$ is some object not appearing in $\Delta^{\BM^\cup_\T}$, and $\cdot^{\BM'}$ is such that:

\begin{itemize}
    \item For every concept name $A$, apart from $\delta$, 
    \begin{itemize}
        \item  $A^{\BM'}$ := $A^{\BM^\cup_\T}\cup\{z\}$, if $\bar{x}\in \norm{A}^{\I_{\norm{}}}$;
        \item $A^{\BM'}$ := $A^{\BM^\cup_\T}$, otherwise.
    \end{itemize}

    \item For the new concept name $\delta$, 
    \begin{itemize}
        \item $\delta^{\BM'}$ := $\{z\}$.
    \end{itemize}

    \item For every role name $P$, we determine its interpretation following these three steps:
    \begin{enumerate}
        \item $P^{\BM'}$ := $P^{\BM^\cup_\T}$;
        \item if $\bar{x}\in \norm{\exists P}^{\I_{\norm{}}}$, $P^{\BM'}$ := $P^{\BM'}\cup \{(z,y)\mid y\in ({\exists P^-})^{\BM^\cup_\T}\}$;
        \item if  $\bar{x}\in {\norm{\exists P^-}}^{\I_{\norm{}}}$, $P^{\BM'}$ := $P^{\BM'}\cup \{(y,z)\mid y\in ({\exists P})^{\BM^\cup_\T}\}$;
    \end{enumerate}

\end{itemize}

where, we recall, $\bar{x}$ is an object in the model $\I_{\norm{}}$ s.t. 
$\bar{x}\in \norm{B}^{\I_{\norm{}}}$ and $\bar{x}\notin \norm{A}^{\I_{\norm{}}}$.

The idea is the following: $\BM^\cup_\T$ is a model of $\T$. From $\BM^\cup_\T$ we create a model $\BM'$ of $\KB^\delta$ by adding an extra object $z$ that plays in  $\BM'$ the same role that the object $\bar{x}$ plays in the interpretation $\I_{\norm{}}$, that is, $z\in (B\dland\delta)^{\BM'}$ with $z \notin A^{\BM'}$.

To this end, first of all, we have to prove that $\BM'$ is a model of $\KB^\delta$. We start by proving that it is a model of $\T$. We have proved that $\BM^\cup_\T$ is a model of $\T$, and we need to check whether the addition of $z$ to the domain may change this.

Please note the  conditions specified above for the definition of $P^{\BM'}$: given $z$ and some role $P$, if $\bar{x} \in \norm{\exists P}^{\I_{\norm{}}}$ (resp., $\bar{x} \in {\norm{\exists P^-}}^{\I_{\norm{}}}$), in $\BM'$ we create a connection through $P$ (resp., $P^-$) from $z$ to all the other objects $y$ \sst~$y$ was already in the extension of  $\exists P^-$ (resp., $\exists P$) \wrt~$\BM^\cup_\T$. 

We need to check whether, once we have $\bar{x}\in \norm{\exists P}^{\I_{\norm{}}}$ or $\bar{x}\in {\norm{\exists P^-}}^{\I_{\norm{}}}$, we are guaranteed that in $\BM^\cup_\T$ there are already pairs of objects that are connected through $P$. Indeed, it is easy to prove that this is actually the case: $\bar{x}\in \norm{\exists P}^{\I_{\norm{}}}$ implies that $\norm{\KB}\not\models \norm{\exists P}\subs\neg \norm{\exists P}$; from the latter, by Proposition \ref{prop_dllitecoreexcept}, we can derive that the concept $\exists P$ is not exceptional \wrt~$\KB=\tuple{\T,\D}$;  if $\exists P$ is not exceptional, then it cannot be the case that $\T\models \exists P\subs \neg \exists P$, since it would immediately enforce the exceptionality; but $\T\not\models \exists P\subs \neg \exists P$ implies that the concept $\exists P$ is satisfied by the model $\BM^\cup_\T$. Hence, if $\bar{x} \in \norm{\exists P}^{\I_{\norm{}}}$, then in $\BM^\cup_\T$ there are some pairs of objects connected through the role $P$. Analogously, this is also the case if $\bar{x} \in {\norm{\exists P^-}}^{\I_{\norm{}}}$.
Hence, for any concept $C$ and any $x$ in the domain of $\BM^\cup_\T$, we have that $x\in C^{\BM^\cup_\T}$ iff $x\in C^{\BM'}$.

It remains to check the status of the newly introduced $z$.
To do so, recall that $\I_{\norm{}}$ is a model of $\norm{\KB}$, hence for every $B_1\subs B_2\in \T$ (resp., $B_1\subs \neg B_2\in \T$) we have also $\norm{B_1}\subs \norm{B_2}\in \norm{\KB}$ (resp., $\norm{B_1}\subs \neg \norm{B_2}\in \norm{\KB}$). Consequently for every $B_1\subs B_2\in \T$ (resp., $B_1\subs \neg B_2\in \T$) we have that either  $\bar{x}\notin \norm{B_1}^{\I_{\norm{}}}$ or $\bar{x}\in \norm{B_1}^{\I_{\norm{}}}\cap \norm{B_2}^{\I_{\norm{}}}$ (resp., $\bar{x}\notin \norm{B_1}^{\I_{\norm{}}}$ or $\bar{x}\in \norm{B_1}^{\I_{\norm{}}}\cap {\neg \norm{B_2}}^{\I_{\norm{}}}$). 
Hence, by definition of $\BM'$, we have that for every $B\subs C\in \T$,  $z\notin B^{\BM'}$ or $z\in B^{\BM'}\cap C^{\BM'}$. 
The only inclusion axioms in $\T$ that still need to  be checked are the role inclusions $R_1\subs R_2$. 
By construction of $\norm{\KB}$, for every $R_1\subs R_2$ in $\T$, $\I_{\norm{}}$ satisfies also $\norm{\exists R_1}\subs \norm{\exists R_2}$ and $ \norm{\exists {R_1^-}}\subs \norm{\exists {R_2^-}}$, hence if $\bar{x}\in {\norm{\exists R_1}}^{\I_{\norm{}}}$, we have also $\bar{x}\in {\norm{\exists R_2}}^{\I_{\norm{}}}$. By construction of $\BM'$, then we have that $z$ is connected through $R_1$ to all the objects in   $({\exists R_1^-})^{\BM^\cup_\T}$  and is connected through $R_2$ to all the objects in  $({\exists R_2^-})^{\BM^\cup_\T}$ (analogously if $\bar{x}\in {\norm{\exists R_1^-}}^{\I_{\norm{}}}$). 
Since $\BM^\cup_\T$ is a model of $R_1\subs R_2\in \T$, every object in $\BM^\cup_\T$ that is in $\exists R_1^{\BM^\cup_\T}$ is also in 
$\exists R_2^{\BM^\cup_\T}$ (analogously, if it is in $({\exists R_1^-})^{\BM^\cup_\T}$ then it is also in $({\exists R_2^-})^{\BM^\cup_\T}$). 
Hence, in $\BM'$, if $z$ is connected to an object $y$ through $R_1$ (resp.~through $R_1^-$), it is also connected through $R_2$ (resp., through $R_2^-$). Therefore, $\BM'$ satisfies $R_1\subs R_2$.
So, we have proved that $\BM'$ is  a model of $\T$.

We have now to prove that $\BM'$ is also a model of the inclusions $\{B',\delta\}\subs A'\in \KB^\delta$. 
Let us recall that $\delta^{\BM'}$ := $\{z\}$ and, thus,  $(B' \andc\delta)^{\BM'} \subseteq \{z\}$. That is, we may restrict our attention to $z$.
If $\{B',\delta\}\subs A'\in \KB^\delta$, then $\norm{B'}\subs \norm{A'}\in \norm{\KB}$. 
Therefore, either  $\bar{x}\notin \norm{B'}^{\I_{\norm{}}}$ or $\bar{x}\in \norm{B'}^{\I_{\norm{}}}\cap \norm{A'}^{\I_{\norm{}}}$, and as a consequence either 
$z\notin {B'}^{\BM'}$ or $z\in {B'}^{\BM'}\cap\delta^{\BM'}\cap A^{\BM'}$. That is, $\BM'$ satisfies  $\{B',\delta\}\subs A'\in \KB^\delta$ and, thus,
$\BM'$ is  a model of $\KB^{\delta}$.

Finally, it remains to show that $\BM'$ does not satisfy $\{B,\delta\}\subs A$. We know that  $\bar{x}\in \norm{B}^{\I_{\norm{}}}$ and $\bar{x}\notin \norm{A}^{\I_{\norm{}}}$. As a consequence, by construction of $\BM'$, we have $z\in B^{\BM'}\cap\delta^{\BM'}$ but $z\notin A^{\BM'}$ and, thus, 
$\BM'$ is a model of $\KB^\delta$ not satisfying $\{B,\delta\}\subs A$, \ie~$\KB^\delta\not\models \{B,\delta\}\subs A$ holds, which concludes the proof.
\end{description} 
\end{proof}

\nd Now, we can eventually prove Theorem~\ref{prop_dllitecoreRC} from the main body.


\setcounter{theorem}{\getrefnumber{prop_dllitecoreRC}-1}
\begin{theorem}
Given a defeasible \dllitecoreH~KB  $\KB=\tuple{\T,\D}$, a DI $B\dsubs A$ is in the RC of $\KB$ iff 
$\mathtt{RationalClosure.Core}(\KB,B\dsubs A)$ returns $\mathtt{true}$.
\end{theorem}

\begin{proof}
By Theorem~\ref{prop_dllitehornRC}, it suffices to show the equivalence among the $\mathtt{\star_L.Horn}$ and the $\mathtt{\star_L.Core}$ procedures, with 
$\star_L \in \{\mathtt{Exceptional}, \mathtt{ComputeRanking}$, \\
$\mathtt{RationalClosure} \}$.

Concerning the $\mathtt{Exceptional}.\star_R$ procedures ($\star_R \in \{\mathtt{Horn}, \mathtt{Core}\}$), by Proposition~\ref{prop_dllitecoreexcept}, the entailment tests in the respective line 4 are equivalent and, thus, so are procedures $\mathtt{Exceptional.Core}$  and 
$\mathtt{Exceptional.Horn}$.

Concerning the $\mathtt{ComputeRanking}.\star_R$ procedures, they are equivalent as so are the $\mathtt{Exceptional}.\star_R$ procedures.

Concerning the $\mathtt{RationalClosure}.\star_R$ procedures, they are equivalent because so are the $\mathtt{ComputeRanking}.\star_R$ procedures, the entailment tests in line 1 and 5 are the same, by Proposition~\ref{prop_dllitecoreexcept}, the entailment test in line 12 of the $\mathtt{RationalClosure.Core}$ procedure is equivalent to the one in line 11 of the $\mathtt{RationalClosure.Horn}$ procedure; and finally, by Proposition~\ref{correspondensP},  the entailment test in line 17 of the $\mathtt{.Core}$ procedure is equivalent to the one in line 16  of the $\mathtt{.Horn}$ procedure, which concludes the proof.
\end{proof}

\subsection*{Section~\ref{sect_RC_ABox}: Rational Closure with ABox}\label{sect_RC_ABoxproofs}
 \vspace{1ex}


\setcounter{proposition}{\getrefnumber{prop_defcons}-1}
\begin{proposition}
    Let $\KB=\tuple{\T,\D,\A}$ be any defeasible \dllitehornH~KB, processed as indicated in Remark \ref{rem_ranking}. Then, $\KB$ is satisfiable iff $\tuple{\T,\A}$ is satisfiable. 
\end{proposition}

\begin{proof} \ 

   \begin{description}
       \item[(Left to right):] Proved by contraposition: if $\tuple{\T,\A}$ is not satisfiable, then it is clearly the case that $\tuple{\T,\D,\A}$ is also not satisfiable.

       \item[(Right to left):] Assume $\tuple{\T,\A}$ is satisfiable, and that $\D$ has been ranked into $\{\D_0,\ldots,$ $\D_n\}$. Let $\I_{\A}=\tuple{\Delta^{\I_{\A}},\cdot^{\I_{\A}}}$ be any model of $\tuple{\T,\A}$.

       By Proposition~\ref{dsat} and Lemma~\ref{Lemma:ClosureDisjointUnionMod},  there is a Big Ranked model \\ 
       $\OI=\tuple{\Delta^{\OI}, \cdot^{\OI},\prec^{\OI}}$ of $\tuple{\T,\D}$. 
       Now, define the ranked interpretation $\OI'$ by adding $\I_{\A}$ to the layer $L_{n+1}$ of $\OI$. That is, $\OI'=\tuple{\Delta^{\OI'}, \cdot^{\OI'},\prec^{\OI'}}$ is \sst:

       \begin{itemize}
           \item $\Delta^{\OI'} \defined \Delta^{\OI}\uplus\Delta^{\I_{\A}}$ is the disjoint union of the two domains;
           \item $\cdot^{\OI'} \defined \cdot^{\OI}\cup \ \cdot^{\I_{\A}}$, with $a^{\OI'} = a^{\I_{\A}}$, for all individuals $a$;
           \item $\prec^{\OI'}$ is determined by the height function $h_{\OI'}$ defined as follows:
           \begin{itemize}
                \item for every $x\in \Delta^{\OI}$, $h_{\OI'}(x) = h_{\OI}(x)$;
               \item for every $x\in \Delta^{\I_{\A}}$, $h_{\OI'}(x) = n+1$, where $\{\D_0,\ldots,\D_n\}$ is the ranking associated to $\tuple{\T,\D}$.
           \end{itemize}
           
       \end{itemize}
       
\nd Now, it remains to verify that indeed $\OI'$ is a model of   $\tuple{\T,\D,\A}$. 

\begin{description}
    \item[Satisfaction of $\T$:] both $\OI$ and $\I_{\A}$ are models of $\T$, hence their disjoint union $\OI'$  is a model of $\T$ too;
    
    \item[Satisfaction of $\D$:] we recall that $\{\D_0,\ldots,\D_n\}$ is the ranking associated to $\D$ obtained using the procedure $\mathtt{ComputeRanking.Horn}$ (or, equivalently, using the procedure $\mathtt{ComputeRanking.Core}$ depending on the expressivity of our KB). Given the correctness of the procedure  $\mathtt{ComputeRanking.Horn}$, 
    we know that if a DI $B\dsubs A$ is in $\D_i$ ($i\in \{0,\ldots,n\}$) then, considering all the models of $\KB$, the minimal height of the objects in the extension of $B$ is $i$. That implies, by the definition of $\OI$, that $\min_{\prec^{\OI}}(B^{\OI})\subseteq L_i$, where $L_i$ is the $i$-th layer in $\OI$ ($1 \leq i \leq n$) and, thus,  $\min_{\prec^{\OI}}(B^{\OI}) \subseteq A^{\OI}$, as $\OI$ satisfies $B\dsubs A$. But then also $\OI'$ satisfies $B\dsubs A$ as, by construction of $\OI'$, $\min_{\prec^{\OI'}}(B^{\OI'}) = \min_{\prec^{\OI}}(B^{\OI})$ (all objects of the domain of $\I_{\A}$ have height $n+1 > i$). As a consequence, $\OI'$ is a model of $\D$;

    
    \item[Satisfaction of $\A$:] by construction, for $\cass{a}{A} \in \A$ we have that $a^{\OI'} = a^{\I_{\A}} \in A^{\I_{\A}} \subseteq A^{\OI'}$. Hence $\OI'$ satisfies $\cass{a}{A}$. The proof for role assertion axioms is similar.   Hence the ABox $\A$ is satisfied also by $\OI'$.
    
    
\end{description}
   \end{description} 
   \nd This concludes the proof.
\end{proof}


\setcounter{proposition}{\getrefnumber{prop_minheight}-1}
\begin{proposition}
Consider any satisfiable defeasible \dllitehornH \ KB $\KB=\tuple{\T,\D,\A}$ and any individual $a$. The minimal height of $a$ depends on the models in $\Moddelta{\KB}$ only, that is,
\[
h_{\KB}(a)=\min\{h_{\RI}(a^\RI)\mid \RI\in \Moddelta{\KB}\}.
\]
\end{proposition}

\begin{proof}
Let $h_{\KB}(a)=n$. It is sufficient to prove that there is a model $\R$ of $\KB$ with a countably infinite domain s.t. $h_{\R}(a)=h_{\KB}(a)$. Then, from $\R$ we can then use a bijection with the domain $\Delta$ to define an equivalent model in $\Moddelta{\KB}$.
Hence, we have to prove that 
\begin{enumerate}
    \item[a)] if there is a finite model where $a$ is interpreted at its minimal height, there is a countably infinite model with the same property; and
    \item[b)] if there is an uncountable model where $a$ is interpreted at its minimal height, there is a countably infinite  model with the same property.
\end{enumerate}

\begin{description}
\item[Case a):]
Assume that there is a finite ranked model $\R'$ s.t. $h_{\R'}(a)=h_{\KB}(a)$. We can consider a countably infinite number of copies of $\R'$ and define a model $\R$ equivalent to their ranked union. The interpretation of the concept and  role names is the union of their respective extensions in each model, while the individual names are interpreted as in the original model $\R'$. $\R$ is a model of $\KB$ s.t. its domain in countably infinite and $h_{\R}(a)=h_{\R'}(a)=h_{\KB}(a)$.

\item[Case b):]
Assume that there is a ranked model $\R'=\tuple{\Delta^{\R'},\cdot^{\R'},\prec^{\R'}}$ of $\KB$ s.t. $h_{\R'}(a)=h_{\KB}(a)$ and s.t. $\Delta^{\R'}$ is uncountable. Let $n(a)$ indicate a set initialised with only $a^{\R'}$ and closed as follows:
\begin{enumerate}
    \item for every role name $P$ and every $x\in n(a)$, if $\R'\sat\cass{x}{\exists P}$  (resp., $\R'\sat\cass{x}{\exists P^-}$) and in $n(a)$ there is no object $y$ s.t. $\R'\sat\rass{x}{y}{P}$ (resp., $\R'\sat\rass{y}{x}{P}$), add one such object $y$ to $n(a)$;

    \item for every pair of individual names $b,c$ s.t. $\rass{b}{c}{P}\in\A$ (resp., $\rass{c}{b}{P}\in\A$), if $b^{\R'}\in n(a)$ and $c^{\R'}\notin n(a)$, add $c^{\R'}$ to $n(a)$.
\end{enumerate}

By construction, $n(a)$ is  a countable set.

Now, let the set $\D$ of DIs be ranked into $\{\D_0,\ldots,\D_m\}$, that is, $m$ is the highest  rank of the DIs in $\D$. Let $\R^*=\tuple{\Delta^{\R^*},\cdot^{\R^*},\prec^{\R^*}}$ be a model of $\KB$ with a countable domain  s.t., for every basic concept $B$, $h_{\R^*}(B)=\rank_{\KB}(B)$. We define a model $\R=\tuple{\Delta^{\R},\cdot^{\R},\prec^{\R}}$ as follows:

\begin{itemize}
    \item $\Delta^{\R}=\Delta^{\R^*}\uplus n(a)$ is their disjoint union;
    \item for every  concept name $A$, $A^{\R}=A^{\R^*}\cup (A^{\R'}\cap n(a))$;
    \item for every  role name $P$, $P^{\R}=P^{\R^*}\cup (P^{\R'}\cap (n(a)\times n(a)))$;
    \item for every individual $b$, if $b^{\R^*}\in n(a)$, then $b^{\R}=b^{\R^*}$;  $b^{\R}=b^{\R'}$ otherwise;
    \item for every object $x\in \Delta^{\R^*}$, $h_{\R}(x)=h_{\R^*}(x)$;
    \item for every object $x\in n(a)$, $h_{\R}(x)=h_{\R'}(x)$, if $h_{\R'}(x)\leq m$; $h_{\R}(x)=m+1$ otherwise.
\end{itemize}

\nd Now, it can easily be verified that $\R$ is a model of $\KB$ s.t. its domain is countable, and $h_{\R}(a)=h_{\KB}(a)$.
\end{description}
\end{proof}







\setcounter{proposition}{\getrefnumber{prop_unique_abox}-1}
\begin{proposition}
  Let $\KB=\tuple{\T,\D,\A}$ be a satisfiable  defeasible \dllitehornH~KB. Then $\KB$ has a model $\RI$ in which every individual 
  is minimally interpreted \wrt~$\KB$, \ie~$h_{\RI}(a^{\RI})=h_{\KB}(a)$, for all individuals $a$. 
\end{proposition}

\nd To prove Proposition \ref{prop_unique_abox} we need the following lemma.

\begin{applemma}\label{lemma_unique_abox}
  Let   $\KB=\tuple{\T,\D,\A}$ be a satisfiable defeasible \dllitehornH~KB  and $a,b$ be any two individuals. 
Then there is a ranked model $\RI$ of $\KB$ \sst~$h_{\RI}(a^{\RI})=h_{\KB}(a)$ and $h_{\RI}(b^{\RI})=h_{\KB}(b)$.
\end{applemma}


\begin{proof} Let us consider the case of \dllitehornH.  The proof for \dllitehornH~covers also the case for \dllitecoreH.
    So, given a satisfiable KB $\KB=\tuple{\T,\D,\A}$, let 
    \begin{itemize}
        \item $a,b$ be two individuals;
        \item $\RI_1=\tuple{\Delta^{\RI_1},\cdot^{\RI_1},\pref^{\RI_1}}$ be a model of $\KB$ s.t. $h_{\RI_1}(a^{\RI_1})=h_{\KB}(a)$;
        \item $\RI_2=\tuple{\Delta^{\RI_2},\cdot^{\RI_2},\pref^{\RI_1}}$ be a model of $\KB$ s.t. $h_{\RI_2}(b^{\RI_2})=h_{\KB}(b)$.
    \end{itemize}
    
\nd Now, we need to prove that there is a model $\RI$ of $\KB$ s.t. $h_{\RI}(a^{\RI})=h_{\KB}(a)$ and $h_{\RI}(b^{\RI})=h_{\KB}(b)$. We define the ranked interpretation $\RI=\tuple{\Delta^{\RI},\cdot^{\RI},\pref^{\RI}}$ from $\RI_1$ and $\RI_2$ as follows:

\begin{itemize}
    \item $\Delta^{\RI}=\Delta^{\RI_1}\uplus\Delta^{\RI_2}$  is the disjoint union of the domains of $\RI_1$ and $\RI_2$;
    \item for every  concept name $A$, $A^{\RI}=A^{\RI_1}\cup A^{\RI_2}$;
    
    \item for every individual $c$ s.t. $c\neq b$, $c^{\RI}=c^{\RI_1}$. In particular, $a^{\RI}=a^{\RI_1}$;
    
    \item for the individual $b$, $b^{\RI}=b^{\RI_2}$;
    
    \item for every role name $P$, for every individual $c$
    \begin{eqnarray*}
        P^{\RI} & = & P^{\RI_1}\cup P^{\RI_2} \cup \\
        && \{(c^{\RI},b^{\RI})\mid \tuple{\T,\A} \models\rass{c}{b}{P}\} \cup \\
        && \{(b^{\RI},c^{\RI})\mid \tuple{\T,\A}\models\rass{b}{c}{P}\}\ ;
    \end{eqnarray*}
    
    \item for every object $x\in \Delta^{\RI}$, $h_{\RI}(x)=h_{\RI_1}(x)$ if $x\in \Delta^{\RI_1}$; $h_{\RI}(x)=h_{\RI_2}(x)$ otherwise;
    
    \item for every pair of objects $x,y$ in $\Delta^{\RI}$, $x\pref^{\RI}y$ iff $h_{\RI}(x)<h_{\RI}(y)$.
\end{itemize}

\nd Essentially, $\RI$ has been obtained by merging $\RI_1$ and $\RI_2$ in a single model. Moreover, we have interpreted all the individuals 
into their original interpretation in $\RI_1$, apart form $b$ that has been interpreted into its original interpretation in $\RI_2$. Eventually, we have recreated all the role connections involving $b$.
%
Therefore, by construction, $h_{\RI}(a^{\RI})=h_{\KB}(a)$ and $h_{\RI}(b^{\RI})=h_{\KB}(b)$. 

It remains to prove that $\RI$ is a model of $\KB$. 
To do so, at first, we prove that for any basic concept, every object of the domain is in the extension of it iff it is in the extension of it also in the original model ($\RI_1$ or $\RI_2$). That is, for any $x\in \Delta^{\RI}$ and any  basic concept $B$, $x\in B^{\RI}$ iff $x\in B^{\RI_*}$, where $\RI_*$ is $\RI_1$ if $x$ was originally in $\Delta^{\RI_1}$ and $\RI_2$ otherwise. In fact,

\begin{itemize}
    \item for any object $x\in \Delta^{\RI}$ and any  concept name $A$, $x\in A^{\RI}$ iff $x\in A^{\RI_*}$, where $\RI_*$ is $\RI_1$ if $x$ was originally in $\Delta^{\RI_1}$ and $\RI_2$ otherwise. In fact, $\RI$ preserves exactly the extensions of the concept names in $\RI_1$ and $\RI_2$. Hence this statement holds by definition of $A^{\RI}$;
    
    \item for any object $x\in \Delta^{\RI}$ and any basic concept of the form $\exists P$, $x\in (\exists P)^{\RI}$ iff $x\in (\exists P)^{\RI_*}$, where $\RI_*$ is $\RI_1$ if $x$ was originally in $\Delta^{\RI_1}$ and $\RI_2$ otherwise. Indeed,  $\RI$ preserves exactly the extensions of the  concept names in $\RI_1$ and $\RI_2$, apart from the role connections involving $b$. We distinguish the following cases.
    \begin{itemize}
        \item $x$ is in $\Delta^{\RI_1}$. 
        \begin{itemize}
            \item $x=c^{\RI}$ for some $c$ 
            \sst~$\tuple{\T,\A}\models\rass{c}{b}{P}$. As $\tuple{\T,\A}\models\rass{c}{b}{P}$, $x$ is in both $(\exists P)^{\RI_1}$ and $(\exists P)^{\RI}$, hence the statement holds.
            \item Otherwise, the statement clearly holds from the definition of $P^{\RI}$.
        \end{itemize}
         
        \item $x$ is in $\Delta^{\RI_2}$. 
        \begin{itemize}
            \item  $x=b^{\RI}$. Then by the definition of $P^{\RI}$, if $x\in (\exists P)^{\RI_2}$, then $x\in (\exists P)^{\RI}$. On the other hand, if $x\in(\exists P)^{\RI}$, then either there was already a pair $(x,y)\in P^{\RI_2}$, or a new pair $(x,y)$, with $y\in \Delta^{\RI_1}$ has been added in $\RI$. But such a pair has been added because $y=c^{\RI}$ for some individual $c$ 
            \sst~$\tuple{\T,\A}\models\rass{b}{c}{P}$. Since $x$ is the interpretation of $b$ also in $\RI_2$ and $\RI_2$ is a model of $\tuple{\T,\A}$, there was already a pair $(x,y)\in P^{\RI_2}$ with $y= c^{\RI_2}$ that is preserved in $P^{\RI}$, \ie~$(x,y)\in P^{\RI}$. Therefore, $x\in(\exists P)^{\RI_2}$, as desired.
            
            \item If $x\neq b^{\RI}$, then by the definition of $P^{\RI}$ it is clear the $x\in (\exists P)^{\RI}$ iff $x\in (\exists P)^{\RI_2}$.
        \end{itemize}

    \end{itemize}

    \item for any object $x\in \Delta^{\RI}$ and any basic concept of the form $\exists P^-$, $x\in (\exists P^-)^{\RI}$ iff 
    $x\in (\exists P^-)^{\RI_*}$, where $\RI_*$ is $\RI_1$ if $x$ was originally in $\Delta^{\RI_1}$ and $\RI_2$ otherwise.  This statement can be proven analogously to the previous case about $\exists.P$.
    
\end{itemize}

\nd Therefore, for any object $x\in \Delta^{\RI}$ and any  basic concept $B$, $x\in B^{\RI}$ iff $x\in B^{\RI_*}$, where $\RI_*$ is $\RI_1$ if $x$ was originally in $\Delta^{\RI_1}$ and $\RI_2$ otherwise, as desired. This property is immediately extended to sets of basic concepts, that is, for any finite set of basic concepts $\{B_1, \ldots, B_n\}$, $x\in \{B_1, \ldots, B_n\}^{\RI}$  iff $x\in \{B_1, \ldots, B_n\}^{\RI_*}$, where $\RI_*$ is $\RI_1$ if $x$ was originally in $\Delta^{\RI_1}$ and $\RI_2$ otherwise.
From this property and the fact that $\RI_1$ and $\RI_2$ are models of $\KB$, it follows   that
for every inclusion $\{B_1, \ldots, B_n\}\subs C\in\T$, $\RI\sat \{B_1, \ldots, B_n\}\subs C$, \ie~$(\{B_1, \ldots, B_n\})^\RI \subseteq C^\RI$. 

Now, let us consider a role inclusion $R_1\subs R_2\in \T$. Let's assume, for ease of presentation, that $R_1$ and $R_2$ are role names; the proofs for the cases in which one or both of them are inverse roles are analogous. We have to prove that for every pair $(x,y)\in \Delta^{\RI}\times \Delta^{\RI}$, if $(x,y)\in R_1^{\RI}$, then $(x,y)\in R_2^{\RI}$. We have to consider four possible cases:

\begin{itemize}
    \item $x,y$ are both in $\Delta^{\RI_1}$. Then $(x,y)\in R_1^{\RI}$ implies $(x,y)\in R_1^{\RI_1}$. Since $\RI_1$ is a model of $\KB$, then $(x,y)\in R_1^{\RI_1}$ implies $(x,y)\in R_2^{\RI_1}$, and the latter implies $(x,y)\in R_2^{\RI}$, as desired;
    
    \item $x,y$ are both in $\Delta^{\RI_2}$. This case can be proven analogously to the previous one;
    
    \item $x\in\Delta^{\RI_1}$ and $y\in\Delta^{\RI_2}$. Then it must be the case that $y= b^{\RI}$ and $x= c^{\RI}$ for some $c$ 
    \sst~$\tuple{\T,\A}\models \rass{c}{b}{R_1}$. But, since $R_1\subs R_2\in \T$, we have also that $\tuple{\T,\A}\models \rass{c}{b}{R_2}$, that, by the definition of $\RI$, implies $(x,y)\in R_2^{\RI}$, as desired; 
    
    \item $x\in\Delta^{\RI_2}$ and $y\in\Delta^{\RI_1}$. This case can be proven similarly as the previous one, where $x= b^{\RI}$ and $y= c^{\RI}$ for some  $c$ 
    \sst~$\tuple{\T,\A}\models \rass{c}{b}{R_1}$ has to hold.
\end{itemize}

\nd Therefore, we also have that for every inclusion $R_1\subs R_2\in\T$, $\RI\sat R_1\subs R_2$ and, thus, $\RI$ is a model of $\T$.

As next, we consider a DI $\{B_1,\ldots,B_n\} \dsubs A \in \D$. We need to show that $\RI$ satisfies $\{B_1,\ldots,B_n\} \dsubs A$ as well. We know that for every object $x$, $h_{\RI}(x)=h_{\RI_1}(x)$ if $x\in\Delta^{\RI_1}$, $h_{\RI}(x)=h_{\RI_2}(x)$ otherwise. 
%
%
So, consider $x \in \min_{\pref^{\RI}}(B_1^{\RI}\cap\ldots\cap B_n^{\RI})$. We have two cases.
\begin{description}
    \item[$x\in \Delta^{\RI_1}$:] in that case,  $h_{\RI}(x)=h_{\RI_1}(x)$ and, thus, $x \in \min_{\pref^{\RI_1}}(B_1^{\RI}\cap\ldots\cap B_n^{\RI})$. But, $\RI_1$ satisfies $\{B_1,\ldots,B_n\} \dsubs A$ and, thus, $x \in A^{\RI_1} \subseteq A^{\RI}$;

    \item[$x\in \Delta^{\RI_2}$:] the proof is as the previous case by replacing $\RI_1$ with $\RI_2$.
    
\end{description}
 
\nd Therefore,  $\RI$ satisfies $\{B_1,\ldots,B_n\} \dsubs A$. That is, $\RI$ is a model of $\D$.

Eventually, we need to prove that $\RI$ is a model of $\A$. So, consider $\cass{c}{A} \in \A$.
We already have proved that for any $x\in \Delta^{\RI}$, $x\in A^{\RI}$ iff $x\in A^{\RI_*}$, where $\RI_*$ is $\RI_1$ if $x \in \Delta^{\RI_1}$ and $x \in \Delta^{\RI_2}$ otherwise. Like above, we have again two cases. If $c^\RI \in \Delta^{\RI_1}$ then, as $\RI_1$ satisfies $\cass{c}{A}$,  $c^\RI \in A^{\RI_1} \subseteq A^{\RI}$. The other case $c^\RI \in \Delta^{\RI_2}$, can be proven as before by replacing $\RI_1$ with $\RI_2$.
Therefore, $\RI$ satisfies also $\cass{c}{A} \in \A$. Now, consider $\rass{c}{d}{P} \in \A$. Like for the proof of role inclusions, we have four cases:
\begin{itemize}
    \item $c^\RI,d^\RI$ are both in $\Delta^{\RI_1}$. Therefore,  $(c^\RI,d^\RI) = (c^{\RI_1},d^{\RI_1})$ holds. As $\RI_1$ satisfies $\rass{c}{d}{P}$, we have  $(c^{\RI_1},d^{\RI_1})\in P^{\RI_1}$ and, thus, $(c^\RI,d^\RI)\in P^{\RI_1} \subseteq P^{\RI}$.

    \item $c^\RI,d^\RI$ are both in $\Delta^{\RI_2}$. This case can be proven analogously to the previous one;
    
    \item $c^\RI\in\Delta^{\RI_1}$ and $d^\RI\in\Delta^{\RI_2}$. Then it must be the case that $d = b$. But, in this case by construction of $P^\RI$,
    $(c^\RI,b^\RI) \in P^\RI$ as $\rass{c}{b}{P} \in \A$.

    \item $d^\RI\in\Delta^{\RI_1}$ and $c^\RI\in\Delta^{\RI_2}$. Like before,  it must be the case that $c= b$.
    But, in this case by construction of $P^\RI$,  $(b^\RI,d^\RI) \in P^\RI$ as $\rass{b}{d}{P} \in \A$. 
\end{itemize}

\nd Therefore, $\RI$ satisfies $\rass{c}{d}{P} \in \A$. That is, $\RI$ is a also model of $\A$ and, thus, we can conclude that $\RI$ is a model of $\KB=\tuple{\T,\D,\A}$, as desired.

\end{proof}

\nd Given Lemma \ref{lemma_unique_abox}, we may proceed to prove 
Proposition~\ref{prop_unique_abox}.


\begin{proof}[Proof of Proposition \ref{prop_unique_abox}] Let us consider \dllitehornH.  The proof for \dllitehornH~covers also the case for \dllitecoreH.

    Let $\tuple{a_1, a_2, \ldots, a_n}$ be a sequence of all the individuals. 
    Let $\tuple{\RI_1,\RI_2, \ldots, \RI_n}$ be a corresponding sequence of models of $\KB$, such that $a_i$ is minimally interpreted in $\RI_i$ \wrt~$\KB$. 
    
    Now, iteratively, using the semantic construction defined in the proof of Lemma~\ref{lemma_unique_abox}, we define new models of $\KB$ inductively as follows:

\begin{itemize}
    \item from $\RI_1$ and $\RI_2$ we define a model $\RI_{\tuple{1,2}}$, via Lemma~\ref{lemma_unique_abox}, in which $a_1$ and $a_2$ are both minimally interpreted \wrt~$\KB$;
    
    \item for $i\geq 2 $, if  $\RI_{\tuple{1,\ldots,i}}$ is the model obtained so far, then from it and $\RI_{i+1}$  we define a model $\RI_{\tuple{1,\ldots,i, i+1}}$, using the same semantic construction as in the proof of Lemma~\ref{lemma_unique_abox}, in which all the individuals $a_1,\ldots,a_i, a_{i+1}$ are  minimally interpreted \wrt~$\KB$.
\end{itemize}

\nd  Eventually, we end up with a model $\RI_{\tuple{1,\ldots,n}}$ of $\KB$ in which all  individuals 
are minimally interpreted \wrt~$\KB$, which concludes the proof.
\end{proof}

\setcounter{proposition}{\getrefnumber{prop_entailment_equiv}-1}
\begin{proposition}
    Let $\KB=\tuple{\T,\D,\A}$ be a satisfiable defeasible \dllitecoreH~KB and let $B\dsubs A$ be a DI.
    Then
    \[
    \OI\sat B\dsubs A\text{ iff }\OI^{\KB}_i\sat B\dsubs A\text{ for every }\OI^{\KB}_i\in \min_{<_{\A}}(\mathfrak{O}^{\KB}) \ .
    \]
\end{proposition}
\begin{proof}
    The proof is straightforward, since by definition the domain and the interpretation of concept and role names is exactly the same in $\OI$ and in every model in $\mathfrak{O}^{\KB}$.
\end{proof}

\setcounter{proposition}{\getrefnumber{prop_entailment_equiv_roles}-1}
\begin{proposition}
    Consider any satisfiable defeasible \dllitecoreH~KB $\KB=\tuple{\T,\D,\A}$,  and any assertion $\rass{a}{b}{P}$.
    Then
    \[
    \KB\minent \rass{a}{b}{P}\text{ iff }\tuple{\T,\A}\entails \rass{a}{b}{P} \ .
    \]
\end{proposition}

\begin{proof} \ \\
\begin{description}
    \item[(Right to left):]  Every model in $\min_{<_{\A}}(\mathfrak{O}^{\KB})$ is also a model of $\tuple{\T,\A}$. In particular, for each ranked model $\OI^{\KB}_i=\tuple{\Delta^{\OI^{\KB}_i},\cdot^{\OI^{\KB}_i},\prec^{\OI^{\KB}_i}}$ in $\min_{<_{\A}}(\mathfrak{O}^{\KB})$, the model $\I=\tuple{\OI^{\KB}_i,\cdot^{\OI^{\KB}_i}}$ is a classical \dllitehornH~model of $\tuple{\T,\A}$. Hence, since $\tuple{\T,\A}\entails \rass{a}{b}{P}$, $\I\sat \rass{a}{b}{P}$, that implies $\OI^{\KB}_i\sat \rass{a}{b}{P}$.

    \item[(Left to right):] We prove it by contraposition. Assume $\tuple{\T,\A}\not\entails \rass{a}{b}{P}$. Note: if there is some role $R$ s.t. $\tuple{\T,\A}\entails R\subs P$, of course it also must be the case that $\tuple{\T,\A}\not\entails \rass{a}{b}{R}$. Now consider some model $\OI^{\KB}_i\in \min_{<_{\A}}(\mathfrak{O}^{\KB})$. If $\OI^{\KB}_i\not\sat \rass{a}{b}{P}$, we are done. If $\OI^{\KB}_i\sat \rass{a}{b}{P}$, modify the model $\OI^{\KB}_i$ into a model $\OI'=\tuple{\Delta^{\OI'}, \cdot^{\OI'},\prec^{\OI'}}$ by adding to the domain two new objects, $o_a$ and $o_b$ in the following way:

    \begin{itemize}
        \item $\Delta^{\OI'}=\Delta^{\OI^{\KB}_i}\cup\{o_a, o_b\}$;
        \item $h_{\OI'}(x)=h_{\OI^{\KB}_i}(x)$ for every $x\in \Delta^{\OI^{\KB}_i}$; $h_{\OI'}(o_a)= h_{\OI^{\KB}_i}(a^{\OI^{\KB}_i})$ and $h_{\OI'}(o_b)= h_{\OI^{\KB}_i}(b^{\OI^{\KB}_i})$.
        \item for every concept name $A$, $A^{\OI'}=A^{\OI^{\KB}_i}\cup \{o_a,o_b\}$, if $a^{\OI^{\KB}_i},b^{\OI^{\KB}_i}\in A^{\OI^{\KB}_i}$; $A^{\OI'}=A^{\OI^{\KB}_i}\cup \{o_a\}$, if only $a^{\OI^{\KB}_i}\in A^{\OI^{\KB}_i}$; $A^{\OI'}=A^{\OI^{\KB}_i}\cup \{o_b\}$, if only $b^{\OI^{\KB}_i}\in A^{\OI^{\KB}_i}$;
        \item for every role name $R$, if  $(a,x)\in R^{\OI^{\KB}_i}$ or $(x,a)\in R^{\OI^{\KB}_i}$, add, respectively, $(o_a,x)\in R^{\OI^{\KB}_i}$ and $(x,o_a)\in R^{\OI^{\KB}_i}$. Analogously for $b$ and $o_b$;
        
        \item for the specific role name $P$, eliminate $(a,b)$ from $P^{\OI'}$. Analogously for every role $R$ s.t. $\tuple{\T,\A}\entails R\subs P$.
    \end{itemize}

\nd It is easy to prove now that $\OI'$ is still a model of $\KB$, since $a$ and $b$ still fall under the same basic concepts as in $\OI^{\KB}_i$, and $o_a$ and $o_b$ are just `copies' of $a$ and $b$, respectively. Also the role inclusions are still preserved, and, since $\tuple{\T,\A}\not\entails \rass{a}{b}{P}$, the role connections that we have eliminated cannot affect the satisfaction of the ABox. Also, we have defined $\OI'$ such that $\OI'\not\sat\rass{a}{b}{P}$. 

Since $\OI'$ has a countable domain, we can define through a bijection an equivalent model of $\KB$ that has $\Delta$ as domain and which, since $\OI^{\KB}_i$ was minimal, is minimal too (the heights of the interpretations of the individuals remain the same). Hence we end up with a model in $\min_{<_{\A}}(\mathfrak{O}^{\KB})$ that does not satisfy $\rass{a}{b}{P}$. The latter implies  $\KB\not\minent \rass{a}{b}{P}$, as desired.
\end{description}
\nd This concludes the proof.
\end{proof}

\begin{applemma}\label{lemma_defeasible_nominals}
    Let $\KB=\tuple{\T,\D,\A}$ be a satisfiable defeasible \dllitehornH~KB, $a$ any individual,  $\KB_\A$ the KB $\tuple{\T_\A,\D}$ and $A_a$ the concept corresponding to $a$. Then $a$ is exceptional \wrt~$\KB$ iff $A_a$ is exceptional \wrt~$\KB_\A$.
\end{applemma}

\begin{proof}
Assume $\KB=\tuple{\T,\D,\A}$ is  unsatisfiable. Then  $a$ is trivially exceptional and, by Proposition \ref{prop_defcons}, $\tuple{\T,\A}$ is  unsatisfiable too. By Corollary \ref{corollary_nominals_3}, $\tuple{\T,\A}$ is  unsatisfiable iff $\T_{\A}\entails A_a\subs\neg A_a$. As a consequence we  have that  $\R\sat A_a\subs\neg A_a$ for every model $\R$ of $\KB_{\A}$, that implies that $A_a$ is trivially exceptional \wrt~$\KB_\A$.

 Now, assume $\KB=\tuple{\T,\D,\A}$ is satisfiable. Hence also $\tuple{\T\cup\D_{\delta_0},\A}$ is satisfiable, since any model $\I$ of $\KB=\tuple{\T,\A}$ s.t. $\delta_0^\I=\emptyset$ is also a model of $\tuple{\T\cup\D_{\delta_0},\A}$. 
 
 Let $\A'=\A\cup\{\cass{a}{\delta_0}\}$. It is sufficient to prove that $a$ is exceptional \wrt~$\KB$ iff $\tuple{\T\cup\D_{\delta_0},\A'}$ is  unsatisfiable. We prove both directions semantically.

\begin{description}
    \item[(Left to right):] By contraposition, assume $\tuple{\T\cup\D_{\delta_0},\A'}$ is satisfiable. Now, Consider a model $\I=\tuple{\Delta^\I,\cdot^\I}$ of $\tuple{\T\cup\D_{\delta_0},\A'}$ and the big ranked model $\OI=\tuple{\Delta^\OI,\cdot^\OI,\prec^\OI}$ of $\tuple{\T,\D}$. We build now a ranked interpretation \\
    $\OI_\A=\tuple{\Delta^{\OI_\A},\cdot^{\OI_\A},\prec^{\OI_\A}}$ in the following way:

 \begin{itemize}
     \item $\Delta^{\OI_\A}=\Delta^{\OI}\uplus\Delta^\I$ (the disjoint union of the two domains);
     \item for every concept name $A$, $A^{\OI_\A}=A^{\OI}\cup A^{\I}$;
     \item for every role name $P$, $P^{\OI_\A}=P^{\OI}\cup P^{\I}$;
     \item for every individual name $a$, $a^{\OI_\A}=a^{\I}$;
     \item for every object $x$, 
        \begin{itemize}
            \item if $x\in \Delta^{\OI}$, $h_{\OI_\A}(x)=h_{\OI}(x)$;
            \item if $x\in \delta^{\I}$, $h_{\OI_\A}(x)=0$;
            \item if $x\in \Delta^{\I}\setminus\delta^{\I}$, $h_{\OI_\A}(x)=n+1$, where $n$ is the highest rank of the axioms in $\D$.
        \end{itemize}   
 \end{itemize}
    
\nd     Since $a^\I\in\delta^\I$, $h_{\OI_\A}(a)=0$, that is, $a$ is not exceptional in $\OI_\A$.
    
    Following a similar argument used in the proof of 
    Proposition~\ref{prop_defcons}, 
    we can prove that $\OI_\A$ is a model of $\KB$. Therefore, $a$ is not exceptional in $\KB$, proving by contraposition that if $a$ is exceptional \wrt~$\KB$, then $\tuple{\T\cup\D_{\delta_0},\A}$ is  unsatisfiable.

\item[(Right to left):] Again, we proceed by contraposition. Assume $a$ is not exceptional \wrt~$\KB$. Hence there is a ranked model $\R=\tuple{\Delta^\R,\cdot^\R,\prec^\R}$ of $\KB$ s.t. $a$ is in $L_0$. Now we define a \dllitehornH~interpretation $\I=\tuple{\Delta^\I,\cdot^\I}$ as follows:

\begin{itemize}
    \item $\Delta^\I=\Delta^\R$;
    \item $\cdot^\I=\cdot^\R$ for every symbol different from $\delta_0$;
    \item $\delta_0^\I=L_0$.
\end{itemize}

\nd    It is easy to prove that $\I$ is a model of $\tuple{\T\cup\D_{\delta_0},\A'}$. The only not immediate step is to prove that every axiom in $\D_{\delta_0}$ is satisfied. We can easily prove it by contradiction. In fact, assume there is an axiom  $\{B_1, ..., B_n,\delta_0\}\subs A\in\D_{\delta_0}$ s.t. $\I\not\sat \{B_1, ..., B_n,\delta_0\}\subs A$. That is, there is an $x\in (\{B_1, ..., B_n\}^\I\cap\delta_0^\I)$ s.t. $x\notin A^\I$. By the construction of $\I$ from $\R$, we would have that $x\in \min_{\prec^{\R}}(\{B_1, ..., B_n\}^\R)$ and $x\notin A^\R$, that is, $\R\not\sat \{B_1, ..., B_n\}\dsubs A$. But $\{B_1, ..., B_n,\delta_0\}\subs A\in\D_{\delta_0}$ implies $ \{B_1, ..., B_n\} \dsubs A\in \D$, against the hypothesis that $\R$ is a model of $\D$.

Therefore, we have proved that, for any individual $a$, $a$ is exceptional \wrt~$\KB$ iff $\tuple{\T\cup\D_{\delta_0},\A'}$ is  unsatisfiable.

    By Corollary \ref{corollary_nominals_3}, $\tuple{\T\cup\D_{\delta_0},\A'}$ is  unsatisfiable iff $(\T\cup\D_{\delta_0})_{\A'}\entails A_a\subs\neg A_a$. Since $(\T\cup\D_{\delta_0})_{\A'}=(\T\cup\D_{\delta_0})_{\A}\cup\{A_a\subs\delta_0\}$, the latter  corresponds to $(\T\cup\D_{\delta_0})_{\A}\cup\{A_a\subs\delta_0\}\entails A_a\subs\neg A_a$.

    The next step is proving that $(\T\cup\D_{\delta_0})_{\A}\cup\{A_a\subs\delta_0\}\entails A_a\subs\neg A_a$ iff $(\T\cup\D_{\delta_0})_{\A}\entails A_a\dland\delta_0\subs\neg A_a$.
    From right to left it is immediate. It remains to prove it from left to right.     
    We prove it by contraposition. Assume $(\T\cup\D_{\delta_0})_{\A}\not\entails A_a\dland\delta_0\subs\neg A_a$. Then the KB $\tuple{(\T\cup\D_{\delta_0})_{\A}, \A^*}$, with $\A^*=\{\cass{a}{A_a}, \cass{a}{\delta_0}\}$, is satisfiable. Now we apply the chase procedure (Section \ref{chasehorn}) to $\A^*$, with one extra condition: if there is an axiom of form $\{B_1,\ldots,B_n\}\subs \exists P.A_a$ (resp., $\{B_1,\ldots,B_n\}\subs \exists P^-.A_a$) that needs to be applied in the chase procedure, we create a connection through $P$ (resp., $P^-$) to the individual $a$ (we do not create a new object).\footnote{Actually, from Equation \ref{ta} we know that the only axioms with such a form can be only of the kind $A_b\subs \exists P. A_a$ or $A_b\subs \exists P^-. A_a$.} We end up with a model of $\tuple{(\T\cup\D_{\delta_0})_{\A}, \A^*}$ s.t. it satisfies also $A_a\subs \delta_0$. Hence $(\T\cup\D_{\delta_0})_{\A}\cup\{A_a\subs\delta_0\}$ is satisfiable too.

\end{description}

\nd    To summarise, we proved that:
    
    \begin{itemize}
        \item $a$ is exceptional \wrt~$\KB$ iff $\tuple{\T\cup\D_{\delta_0},\A'}$ is  unsatisfiable;
        \item $\tuple{\T\cup\D_{\delta_0},\A'}$ is  unsatisfiable iff $(\T\cup\D_{\delta_0})_{\A'}\entails A_a\subs\neg A_a$;
        \item $(\T\cup\D_{\delta_0})_{\A'}\entails A_a\subs\neg A_a$ iff $(\T\cup\D_{\delta_0})_{\A}\entails A_a\dland\delta_0\subs\neg A_a$.
    \end{itemize}

\nd Since the procedures $\mathtt{Exceptional.Horn}$ and $\mathtt{ComputeRanking.Horn}$ are correct, $(\T\cup\D_{\delta_0})_{\A}\entails A_a\dland\delta_0\subs\neg A_a$ iff $A_a$ is exceptional  \wrt~$\KB_\A$.
Hence we can conclude that $a$ is exceptional \wrt~$\KB$ iff $A_a$ is exceptional \wrt~$\KB_\A$, which concludes the proof.
\end{proof}


\setcounter{proposition}{\getrefnumber{lemma_defeasible_nominals_2}-1}
\begin{proposition}
    Let $\KB=\tuple{\T,\D,\A}$ be a satisfiable defeasible \dllitehornH~KB, $a$ any individual,  $\KB_\A$ the KB $\tuple{\T_\A,\D}$, and $A_a$ the concept corresponding to $a$ in $\KB_\A$. Then $h_{\KB}(a)=\rank_{\KB_\A}(A_a)$.
\end{proposition}
\begin{proof}
    \nd Proposition \ref{lemma_defeasible_nominals_2} is an immediate consequence of Lemma \ref{lemma_defeasible_nominals}.
\end{proof}

\setcounter{proposition}{\getrefnumber{prop_defeasible_nominals_2}-1}
\begin{proposition}
    Let $\KB=\tuple{\T,\D,\A}$ be a satisfiable defeasible \dllitehornH~KB, $a$ any individual, and $A$ any concept name. Then,
    \[
    \KB\minent \dass{a}{A}\text{ iff }\KB_\A\minent A_a\dsubs A \ .
    \]
\end{proposition}

\begin{proof}

According to Definition \ref{def_RC_abox}, $\KB\minent \dass{a}{A}$ iff $a^\I\in A^\I\text{ for every }\OI^{\KB}\in\min_{<_{\A}}(\mathfrak{O}^{\KB})$. According to Proposition  \ref{prop_minent} $\KB_\A\minent A_a\dsubs A$ iff $\min_{\prec^{\OI}}(A_a^\OI)\subseteq A^\OI$, where $\OI$ is the big ranked model for $\KB_A=\tuple{\T_A,\D}$.
Hence, we need to prove that 
\[
a^{\OI^{\KB}}\in A^{\OI^{\KB}}\text{ for every }\OI^{\KB}\in\min_{<_{\A}}(\mathfrak{O}^{\KB})\text{ iff }\min_{\prec^{\OI}}(A_a^\OI)\subseteq A^\OI \ .
\]

\nd From Proposition \ref{lemma_defeasible_nominals_2} we know that $h_{\KB}(a)=\rank_{\KB}(A_a)$.

\begin{description}
        \item[(Right to left):]  
We prove it by contraposition. Assume that there is a model $\OI'$ in $\min_{<_{\A}}(\mathfrak{O}^{\KB})$ s.t. $a^{\OI'}\notin A^{\OI'}$. From such a model $\OI'$ we are going to define a model $\R$ of $\KB$ s.t. $a^{\R}\notin A^{\R}$ and $\Delta^\R$ is countable.

We proceed as follow:
\begin{itemize}
    \item We recall from Section \ref{sect_RCsemantics} that the big ranked model $\OI$ (and consequently the interpretations in $\min_{<_{\A}}(\mathfrak{O}^{\KB})$) has been built making the ranked union of the interpretations in $\Moddelta{\KB}$ (see Definition \ref{Def:Bigrankedmodel}). We identify the set $S$ of models of $\KB$ in $\Moddelta{\KB}$ that are involved in the interpretation of $\A$ in $\OI'$. That is, 
    \[
    S=\{\R\in \Moddelta{\KB}\mid \exists~a\text{ occurring in }\A\text{ s.t. }a^{\OI'}=a^{\R}\} \ .
    \]


    \item Let $\R^{S}=\tuple{\Delta^{\R^{S}},\cdot^{\R^{S}},\prec^{\R^{S}}}$ be the ranked interpretation obtained by the ranked union of the interpretations in $S$, where all the individual names occurring in the ABox are interpreted as in $\OI'$. $\R^{S}$ is a model of $\KB$ s.t. $a^{\OI'}\notin A^{\OI'}$, because:

    \begin{itemize}
        \item $\R^{S}$ is a model of $\tuple{\T,\D}$ as it is the ranked union of models of $\tuple{\T,\D}$. Also, from the construction of $\R^{S}$, it is immediate to see that, for every concept $C$, $C^{\R^{S}}=C^{\OI'}\cap \Delta^{\R^{S}}$ and for every role $R$, $R^{\R^{S}}=R^{\OI'}\cap (\Delta^{\R^{S}}\times\Delta^{\R^{S}})$;
        
        \item $\R^{S}$ is a model of $\A$ as, for every individual $a$ occurring in $\A$, $a^{R^{S}}=a^{\OI'}$, and the interpretations of concepts and roles are as in $\OI'$, just constrained to the domain $\Delta^{\R^{S}}$;
    \end{itemize}

    \item Since $\R^{S}$ has been created from the ranked union of a finite number of models with a countable $\Delta$ as domain, also $\Delta^{\R^{S}}$ must be countable. Hence there must be a bijection that transforms $\R^{S}$ into an equivalent model $\R^{\Delta}=\tuple{\Delta, \cdot^{\R^{\Delta}}, \prec^{\R^{\Delta}}}$ having $\Delta$ as domain. 
    We define from $\R^{\Delta}$ a ranked model $\R^{\Delta}_{\A}$, just extending $\cdot^{\R^{\Delta}}$ to the concept names of kind $A_a$:
    \begin{itemize}
        \item for every individual name $a$, $A_a^{\R^{\Delta}_{\A}}=a^{R^{S}}=a^{\OI'}$ \ .
    \end{itemize}    
\end{itemize}

\nd  $\R^{\Delta}_{\A}$ is a model of $\KB_\A$ with $\Delta$ as domain, and s.t. $\min_{\prec^{\R^{\Delta}_{\A}}}(A_a^{\R^{\Delta}_{\A}})\not\subseteq A^{\R^{\Delta}_{\A}}$. Since $\R^{\Delta}_{\A}$ would take part to the ranked union defining the big ranked model $\OI$ of $\KB_\A$, we conclude that $\min_{\prec^{\OI}}(A_a^\OI)\not\subseteq A^\OI$, as desired.

\item[(Left to right):] Also in this case we proceed by contraposition. Hence,  assume that there is a big ranked model $\OI$ of $\KB_\A$, $\min_{\prec^{\OI}}(A_a^\OI)\not\subseteq A^\OI$. That is, there is a model $\R$ in $\Moddelta{\KB_\A}$ s.t. there is  $o \in\Delta$ with  $o\in A_a^{\R}$, $o \notin A^{\R}$, and $h_{\R}(o)=\rank_{\KB}(A_a)=h_{\KB}(a)$. Now we define a model $\R^*$ in the following way:
        \begin{itemize}
            \item for every individual $b$ occurring in the ABox, we choose a model $\R_b$ in $\Moddelta{\KB_\A}$ s.t. $\rank_{\KB_\A}(A_b)=h_{\KB}(b)$ (in particular, $\R_a=\R$);
            \item let $\R'$ be the ranked union of all these models, hence it is still a model of $\KB_\A$. Clearly, the domain of $\R'$ is countable;
            \item $\R^*$ is the model of $\KB_\A$ equivalent to $\R'$, obtained defining an appropriate bijection between the countable domain of $\R'$ and $\Delta$.
        \end{itemize}
        
\nd Now we define $\R_\A$ from $\R^*$ by extending $\cdot^\R$ in the following way:

        \begin{itemize}
            \item initialise $\cdot^{\R_\A}=\cdot^{\R^*}$;
            \item impose $a^{\R_\A}=o$;
            \item for any other individual $b$ occurring in $\A$, $b^{\R_\A}$ is some element in $\min_{\prec^{\R^*}}(A_b^{\R^*})$;
            \item for any role assertion $\rass{b}{c}{P}\in\A$, add $(b,c)$ to $P^{\R_\A}$, if not already present;
            \item close the interpretation of roles in order to ensure the satisfaction of every role inclusion $R\subs R'\in\T$.
        \end{itemize}

\nd Clearly $\R_\A$ is a model of $\KB$ with $\Delta$ as domain, and s.t. 

\begin{itemize}
    \item $\R_\A\not\sat\cass{a}{A}$;
    \item for every $b$ occurring in $\A$, $h_{\R_\A}(b)=h_{\KB}(b)$.
\end{itemize}

 \nd As a consequence, it is possible to built from $\R_\A$ a model $\OI^\KB$ that is in  $\min_{<_{\A}}(\mathfrak{O}^{\KB})$ with $a^{\OI^{\KB}}\not \in A^{\OI^{\KB}}$, as desired.

    \end{description}
    \nd This concludes the proof.
\end{proof}

\section{Proofs of Section~\ref{sect_QRW}: Query Answering under RC}\label{sect_QRWproofs}




\setcounter{proposition}{\getrefnumber{termdllitecore}-1}
\begin{proposition}[\cf~Lemma 34 in~\cite{Calvanese07}] 
    Consider a satisfiable defeasible \dllitecoreH~KB 
    $\KB = \tuple{\T,\D, \A}$, where the rank of the antecedents of all axioms  in $\T_{PI} \cup \D$ is given, and consider a CQ $q$. Then the procedure $\mathtt{DefQueryRef}(q,\KB)$ terminates.
\end{proposition}
\begin{proof}
    The proof is essentially the same as the one shown for \dllitecoreH~\cite[Lemma 34]{Calvanese07} without DIs. In particular, the number of distinct conjunctive queries generated by the procedure is less than or equal to  $(m \cdot (n+1)^2)^n$, where $n$ is the query size, 
    \ie~$n$ is proportional to the number of atoms and the number of terms occurring in the query,  and $m$ is the number of antecedents of axioms occurring in $\T_{PI} \cup \D$, which corresponds to the maximum number of executions of the repeat–until cycle of the procedure. In fact, there are at most $n$ atoms and at most $n+1$ terms that can occur in a query and $m \cdot (n+1)^2$ is the maximal number of ways an atom may be rewritten.
\end{proof}

\setcounter{proposition}{\getrefnumber{chasecordef}-1}
\begin{proposition} 
  Consider a defeasible \dllitecoreH~$\KB=(\T,\D,\A)$.  Then
\begin{enumerate}
    \item $\RI_{\KB}$ is a ranked model of $(\T_{PI}, \D, \A)$;
    

    \item $\RI_{\KB}$ is a ranked model of $\KB$ iff $\mathtt{DB}(\A, h_\KB)$ is a model of $(\clncH(\T), \A)$;
    
    
    \item $\KB$ is satisfiable iff $\RI_{\KB}$ is a ranked model of $\KB$.
    
\end{enumerate}

      

      
\end{proposition}
\begin{proof}
By construction (see Eq.~\ref{canranked}), $\RI_{\KB} = \mathtt{Int}(\dchasecH(\KB), h_\KB)$ is the interpretation built from $\dchasecH(\KB)$ according to Definition~\ref{def_DB(A,h)}. 
Now, it is not difficult to see that 
\begin{enumerate}
    \item[a.] similarly to Lemma~\ref{lemma_lemma7calv07} (\cf~Point 1. in Proposition~\ref{chasecor}), we can prove that $\RI_\KB$ is a model of $\tuple{\T_{PI}, \A}$;

    \item[b.] similarly to Lemma~\ref{lemma_lemma12calv07} (\cf~Proposition~\ref{chasecorA}),  we can prove that  $\RI_\KB$ is a model of $\tuple{\T,  \A}$ iff the interpretation $\mathtt{DB}(\A, h_\KB)$ is a model of $\tuple{\clncH(\T),  \A}$;

    \item[c.] similarly to Proposition~\ref{Theo15calv07} (\cf~ Point 2. in Proposition~\ref{chasecor}) we can prove that $\tuple{\T,  \A}$ is satisfiable iff  $\RI_\KB$ is a model of $\tuple{\T,\A}$.
\end{enumerate}




\nd Therefore, in order to prove all three points of the proposition, it remains to extend the points a. - c. also to $\D$.

\begin{enumerate}
    \item[i.] Let us show that $\RI_\KB$ is a model of $\tuple{\T_{PI}, \D, \A}$. By point a., it suffices to show that $\RI_\KB$ satisfies each DI in $\D$. Assume to the contrary that there is $B \dsubs A \in \D$ not satisfied by $\RI_\KB$. That is, there is $a \in \min_{\pref^{\RI_{\KB}}}(B^{\RI_{\KB}})$ such that $a \not \in A^{\RI_{\KB}}$ and $h_\KB(a) = \rank_{\KB}(B)$. But then, by the chase rule ($\mathbf{\CcH_4}$), we have $\cass{a}{A} \in \dchasecH(\KB)$ and, thus, $a \in A^{\RI_{\KB}}$, a contradiction. Therefore,  all $B \dsubs A \in \D$ are satisfied by $\RI_\KB$ and, thus, $\RI_\KB$ is a model of $\tuple{\T_{PI}, \D, \A}$.

    \item[ii.] By Proposition~\ref{prop_defcons}, $\KB$ is satisfiable iff $(\T,\D)$ is satisfiable. By combining point c. and b. we get:
    $\KB$ is satisfiable iff $\RI_\KB$ is a model of $\tuple{\T, \A}$ iff the interpretation $\mathtt{DB}(\A, h_\KB)$ is a model of $\tuple{\clncH(\T), \A}$.

     \item[iii.] If $\RI_{\KB}$ is a ranked model of $\KB$ then obviously $\KB$ is satisfiable. On the other hand, if $\KB$ is satisfiable then $\tuple{\T, \A}$ is satisfiable and, thus, by point c., $\RI_\KB$ is a model of $\tuple{\T, \A}$. But, in point i.~we have seen that $\RI_\KB$  also satisfies each DI in $\D$ and, thus, $\RI_{\KB}$ is a model of $\KB$.
     
\end{enumerate}
\nd This concludes the proof.
\end{proof}

\nd The following lemma is needed to prove Proposition \ref{calvthm29}. It extends Lemma 28 in \cite{Calvanese07}  to the defeasible case.


\begin{applemma}[cf. Lemma 28, \cite{Calvanese07}] 
\label{calvlem28}
    Consider a  satisfiable defeasible \dllitecoreH~KB  $\KB = \tuple{\T,\D, \A}$ and the ranked interpretation $\RI_{\KB}$. Let $\RI$ be a ranked model of $\KB$ in which all individuals in $\N_\A$ are minimally interpreted \wrt~$\KB$. Then, there is a function $\mu$ from 
    $\Delta^{\RI_{\KB}}$ to $\Delta^{\RI}$ such that
    \begin{enumerate}
        \item for each atomic concept $A$ in $\KB$ and each object $o \in \Delta^{\RI_{\KB}}$, if $o \in A^{\RI_{\KB}}$ then 
        $\mu(o) \in A^{\RI}$;

        \item for each atomic role $P$ in $\KB$ and each pair of objects $o, o' \in \Delta^{\RI_{\KB}}$, 
        if $(o,o') \in P^{\RI_{\KB}}$ then  $(\mu(o), \mu(o')) \in P^{\RI}$;

        \item for each individual $a \in \N_\A$, $h_{\KB}(a) = h_{\RI_{\KB}}(a) =  h_{\RI}(\mu(a))$.


    \end{enumerate}

\end{applemma}
\begin{proof}
The proof extends the one in~\cite[Lemma 28]{Calvanese07}.
Specifically, consider the ranked interpretation $\RI_{\KB}$ and let $\RI$ be a ranked model of $\KB$ in which all individuals named in the ABox are minimally interpreted \wrt~$\KB$. 

By Remark~\ref{canDIs}, for each individual $a\in\N_\A$, $a^{\RI_{\KB}}=a$, $h_{\RI_{\KB}}(a^{\RI_{\KB}}) = h_{\KB}(a)$ and, thus,
$h_{\KB}(a) = h_{\RI_{\KB}}(a) = h_\RI(a^{\RI})$, as all individuals in $\N_\A$ are ranked minimally in $\RI$.  Eventually, note that for all $i\geq 0$, we always have by construction that $A^{\RI_{\KB}^i} \subseteq A^{\RI_{\KB}^{i+1}}$ and $P^{\RI_{\KB}^i} \subseteq P^{\RI_{\KB}^{i+1}}$, for all atomic concepts $A$ and role names $P$.

Now, we define the function $\mu$ from $\Delta^{\RI_{\KB}}$ to $\Delta^{\RI}$ by induction on the construction
of the defeasible chase $\dchasecH(\KB)$, and simultaneously show that properties 1, 2 and 3~hold.

\begin{description}
    
    \item[Base step:] For each individual $a\in\N_\A$, we set 
    $\mu(a) = a^{\RI}$. Recall that $\dchasecH_0(\KB)=\A$ and $\Delta^{\RI_{\KB}^0}=\N_\A$
    Therefore, property 3 is immediately satisfied as for all $a \in \Delta^{\RI_{\KB}^0}$, $h_{\KB}(a) = h_{\RI_\KB}(a) = h_{\RI}(\mu(a))$ holds.

    Let $A$ be a concept name. By the construction of $\RI_{\KB}$ and $\R$, it is easy to see that, for each $a \in \Delta^{\RI_{\KB}^0}$, if 
    $a \in A^{\RI_{\KB}^0}$ then $\cass{a}{A}\in\A$, that in turn implies that $\mu(a) \in A^{\RI}$ . 

    Similarly, for each pair $a,b\in \Delta^{\RI_{\KB}^0}$ and each role name $P$, $(a,b) \in P^{\RI_{\KB}^0}$ implies $(a^\RI,b^\RI) \in P^{\RI}$.

    This concludes the base step.
    
    \item[Induction step:]  Let us assume that $\dchasecH_{i+1}(\KB)$ has been obtained from the chase $\dchasecH_{i}(\KB)$, by applying  rule ($\mathbf{\CcH_2}$).
    Therefore, there is $B \subs \exists R \in \T_{PI}$ and $\cass{a}{B}$ that occurs in $\dchasecH_i(\KB)$ such that  $\cass{a}{\csome R}$ does not occur in $\chasecH_i(\KB)$, and $\dchasecH_{i+1}(\KB)=\dchasecH_i(\KB)\cup\{ra(R,a,b)\}$, where $b$ is a new individual. As a consequence $ra(R,a,b)$ is satisfied by $\RI_{\KB}^{i+1}$. Let us assume that $R$ is an atomic role $P$ (the case $R$ is an inverse can be handled similarly).   
    By induction hypothesis on $a$, there is $o \in \Delta^{\RI}$ such that $\mu(a)= o$ and $o \in B^\RI$.
%
 %
%
    As $B \subs \exists P \in \T_{PI}$ and as $\RI$ is a ranked model of $\KB$, $\mu(a) = o \in (\exists P)^\RI$ follows. Therefore there is an object $o' \in \Delta^{\RI}$ such that $(o,o') \in P^{\RI}$.  Then we set $\mu(b) = o'$ and we can conclude that $(\mu(a), \mu(b)) \in P^{\RI}$. 

With a similar argument we can prove the inductive step also in those cases in
which $\dchasecH_{i+1}(\KB)$ has been obtained from $\dchasecH_{i}(\KB)$, by applying  chase rules ($\mathbf{\CcH_1}$) or ($\mathbf{\CcH_3}$).

Let's now consider the case in which $\dchasecH_{i+1}(\KB)$ has been obtained from $\dchasecH_{i}(\KB)$, by applying  chase rule ($\mathbf{\CcH_4}$), hence referring to a DI $B\dsubs A\in \D$. ($\mathbf{\CcH_4}$) can be applied only to individuals $a\in \N_\A$, since all the anonymous individuals have rank $n+1$ (by rule ($\mathbf{\CcH_5}$), where $n$ is the highest rank in $\D$) and consequently no rule in $\D$ is applicable to them. Therefore, from the base step we have that $h_{\KB}(a) = h_{\RI_\KB}(a^{\RI_\KB}) = h_{\RI}(\mu(a))$, and, by induction step, that if $a\in B^{\RI_\KB}$ then $\mu(a)\in B^{\RI}$. Since $\RI$ is a model of $\KB$, we can conclude that if $a\in A^{\RI_\KB}$ then $\mu(a)\in A^{\RI}$, which concludes the proof.
\end{description}
\end{proof}


\nd The following Lemma will be useful in the proof of Proposition~\ref{calvthm29} below.


\begin{applemma}\label{lemma_abox_noind}
  Let   $\KB=\tuple{\T,\D,\A}$ be a satisfiable \dllitehornH~or \dllitecoreH~KB  and $a \in \N \setminus \N_\A$ be an  individual not occurring in $\A$. Then $h_{\KB}(a)=0$.
\end{applemma}
\begin{proof}
The proof is quite immediate. Let $\RI=\tuple{\Delta^{\RI},\cdot^{\RI},\pref^{\RI}}$ be any ranked model of $\KB=\tuple{\T,\D,\A}$ and let $x$ be any object not in $\Delta^{\RI}$. Let $\RI'=\tuple{\Delta^{\RI'},\cdot^{\RI'},\pref^{\RI'}}$ be a model obtained from $\RI$ as follows:

\begin{itemize}
    \item $\Delta^{\RI'}=\Delta^{\RI}\cup\{x\}$;
    \item $A^{\RI'}=A^{\RI}$ for any concept name $A$;
    \item $P^{\RI'}=P^{\RI}$ for any role name $P$;
    \item $b^{\RI'}=b^{\RI}$ for any individual name $b\neq a$;
    \item $a^{\RI'}=x$;
    \item we impose $h_{\RI'}(x)=0$, and $\pref^{\RI'}$ is obtained extending $\pref^{\RI}$ accordingly.
\end{itemize}

\nd $\RI'$ is still a model of $\KB=\tuple{\T,\D,\A}$, since the addition of $x$ to the domain cannot influence the satisfaction of $\T$, $\D$, and $\A$. Hence $h_{\KB}(a)=0$.

\end{proof}



\setcounter{proposition}{\getrefnumber{calvthm29}-1}
\begin{proposition}[cf. Theorem 29, \cite{Calvanese07}] 
     Consider a satisfiable defeasible \dllitecoreH~KB  $\KB = \tuple{\T,\D, \A}$, and let
     $q = \{q_1, \ldots, q_n\}$ be a UCQs. Assume that  $q_i$ is of the form $q_i(\vec{x}) \leftarrow \exists \vec{y}_i . \varphi(\vec{x}_i,\vec{y}_i)$.  Then  
     \[
     ans_{RC}(\KB,q) = \{\vec{t} \mid \text{ there is } q_{i} \in q \text{ s.t. } \RI_{\KB}  \text{ satisfies } \exists \vec{y}_i .\varphi(\vec{t},\vec{y}_i) \} \ .
     \]
\end{proposition} 
\begin{proof}

\nd The proof is similar to~\cite[Theorem 29]{Calvanese07}. $ans_{RC}(\KB,q)$ is defined in  (\ref{unionanskbDIs}) in Definition \ref{rcanswer}. 
By Remark~\ref{canDIs}, for each individual $a\in\N_\A$, $a^{\RI_{\KB}}=a$, and, thus, for a tuple $\vec{t}$ of individuals in $\KB$ we have that  $\vec{t}^{{\RI_{\KB}}} = \vec{t}$. 
Moreover,  let us consider as the domain $\Delta$ of all ranked models in $\Moddelta{\KB}$, 
$\Delta = \N \cup \Delta^{\RI_{\KB}}$. Recall that by Remark~\ref{canDIs}, $\Delta^{\RI_{\KB}}$ is the set of all individuals occurring in $\A$, the witnesses for each rank in $\RI_{\KB}$ and the anonymous individuals introduced by the chase and, thus,  $\Delta$ is countable and infinite. 
As $\KB$ is satisfiable, we also have that $\min_{<_{\A}}(\mathfrak{O}^{\KB}) \neq \emptyset$. 
\begin{description}

\item[($\subseteq$):] 
Suppose  $\bar{\vec{t}} \in ans_{RC}(\KB,q)$. By  $(\ref{unionanskbDIs})$ in Definition \ref{rcanswer}, 
there is $q_i \in q$ such that $\KB \entabox q_i(\bar{\vec{t}})$. 
Hence, for every $\OI^{\KB}_i\in \min_{<_{\A}}(\mathfrak{O}^{\KB})$, $\OI^{\KB}_i\sat q_i(\bar{\vec{t}})$. 
We need to prove that $\R_\KB\sat q_i(\bar{\vec{t}})$. 
%
To this end, we proceed as follows. Consider $\R_\KB$ and an $\OI^{\KB}_1 \in \min_{<_{\A}}(\mathfrak{O}^{\KB})$. Let 
$\OI^{\KB}_2$ be the ranked union of the two, where for each $a \in \N_\A$, we set $a^{\OI^{\KB}_2} = a$.
Note that $\OI^{\KB}_2$ is ranked model of $\KB$ in which all individuals in $\N_\A$ are minimally ranked in $\OI^{\KB}_2$ and, thus, by Lemma~\ref{lemma_abox_noind}  so are all individuals in $\N$ (all individuals in $\N\setminus \N_\A$ have rank $0$ and, thus, are minimally ranked as well). 

Moreover, it is easily verified that for every $a \in \N_\A$,  $a \in A^{\R_\KB}$ iff $a^{\OI^{\KB}_2} \in A^{\OI^{\KB}_2}$ (and similarly for roles). Eventually, since $\OI^{\KB}_2\in \min_{<_{\A}}(\mathfrak{O}^{\KB})$, according to (\ref{qReq}) we must conclude 
 $\OI^{\KB}_2\sat q_i(\bar{\vec{t}})$, which implies $\R_\KB\sat q_i(\bar{\vec{t}})$.

\item[($\supseteq$):] 
Suppose that $\bar{\vec{t}} \in \{\vec{t} \mid \text{ there is } q_{i} \in q \text{ s.t. } \RI_{\KB}  \text{ satisfies } \exists \vec{y}_i . \varphi(\vec{t},\vec{y}_i) \}$. Therefore, for some $q_i \in q$, $\RI_{\KB}$ satisfies  $\exists \vec{y}_i . \varphi(\vec{t},\vec{y}_i)$.
That is, there is an assignment $\nu\colon V \to \Delta^{\RI_{\KB}}$ that maps the variables $V$ occurring in $\varphi(\bar{\vec{t}},\vec{y}_i)$ to objects in $\Delta^{\RI_{\KB}}$, such that all atoms in $\varphi(\bar{\vec{t}},\vec{y}_i)$ under the assignment $\nu$ evaluate to true in $\RI_{\KB}$.\footnote{Recall that any occurrence of a rank annotated variable of the form $z^r$ is mapped to an object of rank $r$.}

Now let $\OI^{\KB}_i$ be any ranked model of $\KB$ in $\min_{<_{\A}}(\mathfrak{O}^{\KB})$. That is, $\OI^{\KB}_i$ is a model of $\KB$ in which all individuals in $\N$ are minimally interpreted \wrt~$\KB$. 
By Lemma~\ref{calvlem28}, there is a homomorphism $\mu$ from $\Delta^{\RI_{\KB}}$ to $\Delta^{\OI^{\KB}_i}$. 
Therefore, the function obtained by composing $\mu$ and $\nu$ is a function that maps the variables $V$  occurring in $\varphi(\bar{\vec{t}},\vec{y}_i)$ to objects of the domain of $\OI^{\KB}_i$, such that all atoms in  $\varphi(\bar{\vec{t}},\vec{y}_i)$ under the assignment $\mu$ evaluate to true in $\OI^{\KB}_i$. Therefore,  $\OI^{\KB}_i$ satisfies $\varphi(\bar{\vec{t}},\vec{y}_i)$ for all ranked models $\OI^{\KB}_i\in\min_{<_{\A}}(\mathfrak{O}^{\KB})$, since in all of them all the individuals occurring in the ABox are minimally interpreted \wrt~$\KB$ and
by Lemma~\ref{lemma_abox_noind} so are all individuals in $\N$.
That is, $\bar{\vec{t}} \in ans_{RC}(\KB,q)$, which concludes the proof.

\end{description}

%
\end{proof}

\setcounter{proposition}{\getrefnumber{calvthm31}-1}
\begin{proposition}[cf. Theorem 31, \cite{Calvanese07}] 
     Consider a satisfiable defeasible \dllitecoreH~KB  $\KB = \tuple{\T,\D, \A}$, and let
     $q = \{q_1, \ldots, q_n\}$ be a union of CQs. Then $ans_{RC}(\KB,q) = \bigcup_{q_i \in q} ans_{RC}(\KB,q_i)$.
\end{proposition} 
\begin{proof}
The proof is similar to~\cite[Theorem 31]{Calvanese07}. 
The proof of $\bigcup_{q_i \in q} ans_{RC}(\KB,q_i) \subseteq ans_{RC}(\KB,q)$ is immediate. It remains to prove that 
$ans_{RC}(\KB,q) \subseteq \bigcup_{q_i \in q} ans_{RC}(\KB,q_i)$.
So, suppose  $\vec{t} \in ans_{RC}(\KB,q)$ and that each $q_i$ is of the form $q_i(\vec{x}) \leftarrow \exists \vec{y}_i . \varphi(\vec{x}_i,\vec{y}_i)$.
Then, by Proposition~\ref{calvthm29} there is $q_i\in q$ such that $\RI_{\KB}$ satisfies $\varphi(\vec{t},\vec{y}_i)$ and, thus, again by Proposition~\ref{calvthm29}, $\vec{t} \in ans_{RC}(\KB,q_i)$, which concludes the proof.


\end{proof}


\setcounter{proposition}{\getrefnumber{mainqadllite}-1}
\begin{proposition}[cf. Lemma 39,~\cite{Calvanese07}]
    Consider a satisfiable defeasible \dllitecoreH~KB 
    $\KB = \tuple{\T,\D, \A}$, where the rank of the antecedents of all axioms  in $\T_{PI} \cup \D$ is given, and $h_\KB$ is the ranking of individuals as per Proposition~\ref{prop_minheight}. 
    Consider a CQ $q$ and let $r(q,\KB) = \mathtt{DefQueryRef}(q,\KB)$ be the set of reformulated queries. Then Eq.~\ref{ansreldef} holds, \ie
\begin{equation*}
ans_{RC}(\KB,q)  =  \{\vec{t}\mid \text{ there is } q_{i} \in r(q,\KB) \text{ s.t. } \vec{t} \in q_{i}^{\mathtt{DB}(\A, h_\KB)}\} \ .    
\end{equation*}


\end{proposition}


\begin{proof} 
The proof is an adaption of~\cite[Lemma 39]{Calvanese07} to our setting.
To start with, consider a defeasible \dllitecoreH~KB  $\KB = \tuple{\T,\D, \A}$, a CQ 
\[
q(\vec{x}) \leftarrow \exists \vec{y}.\varphi(\vec{x},\vec{y}) 
\]

\nd and a vector $\vec{t}$ of individuals  in $\N_\A$. Let $\calG$
be a set of assertion axioms s.t.~all the individual names appearing in them are from $\N_\A$. $\calG$ is a \emph{witness} of $\vec{t}$ \wrt~$q$, 
if there exists a substitution $\sigma$  from the variables $\vec{y}$
in $\varphi(\vec{x},\vec{y})$ to individuals occurring in $\calG$ such that the set
of atoms in $\varphi(\vec{t},\sigma(\vec{y}))$ is equal to $\calG$ and $\sigma$ maps any rank annotated variable of the form $z^r$ into an individual of rank $r$.\footnote{In the following, with $|\calG|$ we denote the cardinality of $\calG$, \ie~the number of assertions in $\calG$.}

What we are looking for is the witnesses of a tuple $\vec{t}$ \wrt~a query $q$ that are contained in $\dchasecH(\KB)$, that is, in the model $\R_{\KB}$: we have proved in Proposition \ref{calvthm29} that, given any query $q$, the model  $\R_\KB$ characterises the answer set \wrt~$\KB$. Now we are interested into the witnesses $\calG$ that are satisfied by $\R_{\KB}$, since, by Proposition \ref{calvthm29}, the set of such witnesses define the answer set $ans_{RC}(\KB,q)$. Clearly, each such witness $\calG$ corresponds to a subset of $\dchasecH(\KB)$ that is sufficient for the formula $\exists \vec{y}.\varphi(\vec{x},\vec{y})$ to be satisfied by $\R_{\KB}$. More precisely, we have that $\vec{t} \in ans_{RC}(\KB,q)$ iff there exists a witness $\calG \subseteq \dchasecH(\KB)$ of $\vec{t}$ \wrt~$q$.

Now, as $\KB$ is satisfiable, by Proposition~\ref{calvthm29}, 
\begin{eqnarray*}
    ans_{RC}(\KB,q) & = & \{\vec{t} \mid \text{ there is } q_{i} \in q \text{ s.t. } \RI_{\KB}  \text{ satisfies } \exists \vec{y}_i .\varphi(\vec{t},\vec{y}_i) \} \ . 
\end{eqnarray*}






\nd Let $q = q(\vec{x})\leftarrow \exists \vec{y} .\varphi(\vec{x},\vec{y})$ be a CQ, $q_R$ be the union of CQs obtained by rewriting  $q$ (that is $q_{R}=r(q,\KB)$), and $q_i \in q_R$ be a single CQ in $q_{R}$. Let us recall that:

\begin{equation}\label{qrew}
\begin{array}{rcl}
     q^{\RI_{\KB}} & = & \{\vec{t} \mid \RI_{\KB}  \text{ satisfies } q(\vec{t})\}  \\
     q^{\mathtt{DB}(\A, h_\KB)} &  = & \{\vec{t}\mid  \mathtt{DB}(\A, h_\KB) \text{ satisfies } q(\vec{t})\} \\
    q_{i}^{\mathtt{DB}(\A, h_\KB)} &  = & \{\vec{t}\mid  \mathtt{DB}(\A, h_\KB) \text{ satisfies } q_{i}(\vec{t})\} \\
\end{array}
\end{equation}

\nd and by Proposition~\ref{calvthm31} 

\begin{equation}\label{qrewB}
\begin{array}{rcl}
    q_{R}^{\mathtt{DB}(\A, h_\KB)} &  = & \bigcup_{q_i \in q_{R}}  q_{i}^{\mathtt{DB}(\A, h_\KB)} \\
    & = &  \{\vec{t}\mid \text{ there is } q_{i} \in q_{R} 
    \text{ s.t. } \ \mathtt{DB}(\A, h_\KB) \text{ satisfies } q_{i}(\vec{t})\} \ .
\end{array}
\end{equation}


\nd Consequently, to prove the claim, we need to show that $q_{R}^{\mathtt{DB}(\A, h_\KB)} = q^{\RI_{\KB}}$.

\begin{description}
    \item[($\subseteq$):]  We want to prove that $q_{R}^{\mathtt{DB}(\A, h_\KB)} \subseteq q^{\RI_{\KB}}$. For that, we are going to prove $q_{r}^{\mathtt{DB}(\A, h_\KB)} \subseteq q^{\RI_{\KB}}$, for each $q_{r} \in q_{R}$.

    
     We  consider two CQs $q_i$ and $q_{i+1}$, such that $q_{i+1}$ is obtained from $q_i$ by means of step 7 of the procedure $\mathtt{DefQueryRef}$. That is, $q_{i+1}=\rho(q_i[g/gr(g,\tau)])$ for some inclusion $\tau$ in $\T\cup\D$. We show that $q_{i+1}^{\RI_{\KB}} \subseteq q_{i}^{\RI_{\KB}}$. 
    
    Let $\vec{t}$ be a tuple of individuals occurring in $\KB$ such that $\vec{t}^{\RI_{\KB}} \in q_{i+1}^{\RI_{\KB}}$, that is 
    $\vec{t} \in q_{i+1}^{\RI_{\KB}}$ as $\vec{t}^{\RI_{\KB}} = \vec{t}$.
     Therefore, $\RI_{\KB}$ satisfies $q_{i+1}(\vec{t})$, and, we can conclude that there is some $\calG \subseteq \dchasecH(\KB)$ such that $\calG$ is a witness of $\vec{t}$ \wrt~$q_{i+1}$.  
    

    Let us assume that $q_{i+1}$ is obtained from 
    $q_i$ by applying a DI $\tau$ of the form $A' \dsubs A$ via Case 4. of Definition~\ref{defapplicable}. That is, $q_{i+1} = q_i[A(x)/A'(x^r)]$, where $r=\rank_{\KB}(A')$. Then,  there exists an assertion in $\calG$, involving an individual of rank $r$, to which the PI $A' \dsubs A$ is applicable via chase rule ($\mathbf{\CcH_4}$), which implies that there exists a witness of $\vec{t}$ \wrt~$q_i$ contained in $\dchasecH(\KB)$. Therefore, $\vec{t} \in q_{i}^{\RI_{\KB}}$.
    Assume now that $q_{i+1}$ is obtained from $q_i$ by means of step 10 of the procedure $\mathtt{QueryRef}$, \ie~$q_{i+1} = \kappa(reduce(q_i,g_{1}, g_{2}))$, where where $g_1$ and $g_2$ are two query atoms belonging to $q_i$ such that $g_1$ and $g_2$ unify. It is easy to see that in such a case $\calG$ is also a witness of $\vec{t}$ \wrt~$q_i$ and therefore $\vec{t}^{\RI_{\KB}} \in q_{i}^{\RI_{\KB}}$. 

    The proof for the other cases listed in Definition~\ref{defapplicable} is analogous.
    
    Each CQ of $q_R$ is either in $q$ or a query obtained from $q$ by repeatedly applying steps 7 and 10 of the procedure $\mathtt{QueryRef}$, it follows that for each  $q_{r} \in r(q,\KB)$, $q_{r}^{\RI_{\KB}} \subseteq q^{\RI_{\KB}}$, by repeatedly applying the property $q_{i+1}^{\RI_{\KB}} \subseteq q_{i}^{\RI_{\KB}}$. As $q_{r}^{\mathtt{DB}(\A, h_\KB)} \subseteq q_{r}^{\RI_{\KB}}$, for each  $q_{r} \in r(q,\KB)$,  $q_{r}^{\mathtt{DB}(\A, h_\KB)}\subseteq q^{\RI_{\KB}}$ holds and, thus, 
    $q_{R}^{\mathtt{DB}(\A, h_\KB)} \subseteq q^{\RI_{\KB}}$.

    \item[($\supseteq$):] Now we prove that $q^{\RI_{\KB}} \subseteq q_{R}^{\mathtt{DB}(\A, h_\KB)}$. We have to show that for each tuple 
    $\vec{t} \in q^{\RI_{\KB}}$, there exists $q_r \in q_R$ such that $\vec{t} \in  q_{r}^{\mathtt{DB}(\A, h_\KB)}$. 
    
    At first, as $\vec{t} \in q^{\RI_{\KB}}$, it follows that there is finite number $k$ such that there is a witness $\calG_k$ of $\vec{t}$ \wrt~$q$ contained in $\dchasecH_k(\KB)$
    %
    %
    Without loss of generality, we may assume that every chase rule ($\mathbf{\CcH_1}$) - ($\mathbf{\CcH_5}$) used in the construction of $\dchasecH_k(\KB)$ is necessary in order to generate such a witness $\calG_k$ and that $k$ is minimal. Note that $\dchasecH_k(\KB)$ can be seen as a forest 
    (set of trees) where: \ii{i} the roots correspond to the assertions of $\A$; \ii{ii} there are exactly $k$ edges, where each edge corresponds to an application of a rule; and \ii{iii} each leaf is either  a root  or an assertion in $\calG_k$.


    In the following, we say that an assertion $\alpha$ is an \emph{ancestor} of an  assertion $\beta$  (or, equivalently, $\beta$ is a \emph{successor} of $\alpha$) in a set of assertions $\calS$, if there exist $\alpha_1, \ldots, \alpha_n$ in $\calS$, where $\alpha_1=\alpha$, $\alpha_n=\beta$, such that, for each $2\leq i \leq n$ $\alpha_i$ can be generated by applying a chase rule to $\alpha_{i-1}$. 
    %
    %
    Furthermore, for each $i \in\{0, \ldots, k\}$, we denote with $\calG_i$ the \emph{pre-witness} of $\vec{t}$ \wrt~$q$ in $\dchasecH_k(\KB)$, defined as follows:
    
    

    \begin{eqnarray*}
    \calG_i & = & \{\alpha \in  \dchasecH_i(\KB) \mid \text{there exists } \beta \in \calG_k \text{ s.~t. } \alpha \text{ is an  ancestor of } \beta\\
    && \text{ in } \dchasecH_k(\KB) \text{ and there is no successor of } \alpha \text{ in } \dchasecH_i(\KB) \\
    && \text{ that is an ancestor of } \beta \text{ in } \dchasecH_k(\KB)\}  \ .   
    \end{eqnarray*}

\nd    Now we prove by induction on $i$ that, starting from $\calG_k$, we can `go back' through the chase rule applications and find a query $q_r \in q_R$ such that the pre-witness $\calG_{k-i}$ of $\vec{t}$ \wrt~$q$ in $\dchasecH_{k-i}(\KB)$ is also a witness of $\vec{t}$ \wrt~$q_r$. To this aim, we prove that there exists 
    $q_r \in q_R$ such that $\calG_{k-i}$ of $\vec{t}$ \wrt~$q_r$ and 
    $|q_r| = |\calG_{k-i}|$, where $|q_r|$indicates the number of atoms in the CQ $q_r$. Then the claim follows for $i=k$, as $\calG_0 \subseteq \dchasecH_{0}(\KB) = \A$.

    \begin{description}
    
    \item[Base step, $i=0$:] we know that $\calG_k$ is a witness of $\vec{t}$ \wrt~$q$. We prove now that this implies that $\calG_k$ is a witness of $\vec{t}$ \wrt~some $q_r\in q_R$ s.t. $|q_r| = |\calG_k|$, where $|q_r|$ indicates the number of atoms in the CQ $q_r$.

    We have three possible cases:

    \begin{enumerate}
        \item $|q| < |\calG_k|$. This cannot be the case, since $\calG_k$ is a witness of $\vec{t}$ \wrt~$q$.
        \item $|q| = |\calG_k|$. In such a case our condition is immediately satisfied.
        \item $|q| > |\calG_k|$. In such a case there must exist two atoms $g_1$ and $g_2$ in $q$ and a 
    membership assertion $\alpha$ in $\calG_k$ such that $\alpha$ and $g_1$ unify and $\alpha$ and $g_2$ unify, which implies that $g_1$ and $g_2$ unify.
    Therefore, by step 10 of the  $\mathtt{QueryRef}$ procedure, it follows that there  exists a query $q' \in q_R$ (with $q' = reduce(q, g_1, g_2))$ such that $\calG_k$ is a witness of $\vec{t}$ \wrt~$q'$ and $|q'| = |q| - 1$. 
    Now, if  $|\calG_k| < |q'|$, we can iterate the above argument. Thus we conclude that  there exists $q_r \in q_R$ such that $\calG_k$ is a witness of $\vec{t}$ \wrt~$q_r$ and $|\calG_k| = |q_r|$.
    \end{enumerate}


    \item[Induction step:] Suppose there is some $q_r \in q_R$ such that $\calG_{k-i+1}$ is a witness of $\vec{t}$ \wrt~$q_r$ and $|q_r| = |\calG_{k-i+1}|$. 
    Let us assume that $\dchasecH_{k-i+1}(\KB)$ is obtained by applying the chase rule
    ($\mathbf{\CcH_4}$) to $\dchasecH_{k-i}(\KB)$ (the proof is analogous for all other chase rules). Therefore, there is  $B \dsubs A \in \D$ with $\rank_{\KB}(B)=j$,
    $\cass{a}{B}$ occurs in $\dchasecH_{k-i}(\KB)$ and $\cass{a}{A}$ does not occur in $\dchasecH_{k-i+1}(\KB)$. Let us assume that $B$ is a concept name $A'$ (the other cases are proved similarly). Therefore, 
    $\dchasecH_{k-i+1}(\KB) =  \dchasecH_{k-i}(\KB) \cup \{\cass{a}{A} \}$. 
    As $\calG_{k-i+1}$ is a witness of $\vec{t}$ \wrt~$q_r$ and $|q_r| = |\calG_{k-i+1}|$ it follows that $A(x)$ occurs in $q_r$.  Therefore, by Step 7. of the  $\mathtt{QueryRef}$ procedure, it follows that there exists a query $q' \in q_R$, \ie~$q' = q_r[A(x)/A'(x)]$, such that $\calG_{k-i}$ is a witness of $\vec{t}$ \wrt~$q'$. 
    
    Now, there are two possible cases: either $|q'| = |\calG_{k-i}|$, and in this case the claim is immediate; or $|q'| = |\calG_{k-i}| + 1$. The latter case arises if and only if the membership assertion $\cass{a}{A'}$ to which the chase rule ($\mathbf{\CcH_4}$) is applied is both in $\calG_{k-i}$ and $\calG_{k-i+1}$.
    This implies that there exist two atoms $g_1$ and $g_2$ in $q'$ such that 
    $A(a)$ and $g_1$ unify and $A(a)$ and $g_2$ unify, and hence $g_1$ and $g_2$ unify.
    Therefore, by step 10 of the  $\mathtt{QueryRef}$ procedure, it follows that there 
    exists a query $q''\in q_R$ (with $q'' = reduce(q', g_1, g_2))$ such that 
    $\calG_{k-i}$ is  a witness of $\vec{t}$ \wrt~$q''$ and $|q''| = |q'| - 1 = |\calG_{k-i}|$, which proves the claim. 
    \end{description}
     
\end{description}

\end{proof}

\section{Proofs of Section~\ref{compplex}: Computational Complexity}\label{compplexproofs}

\setcounter{proposition}{\getrefnumber{complexityqa}-1}
\begin{proposition}
Answering unions of conjunctive queries in defeasible \dllitecoreH~and \dllitehornH~is in $\textsf{P}$ \wrt~the size of the union of the TBox and the DBox,  in $\textsf{AC}^0$ in the size of the ABox (data complexity), and is 
$\textsf{NP-complete}$ in combined complexity.
\end{proposition}
\begin{proof}
    The belonging to $\textsf{P}$ follows from Proposition~\ref{complexityqueryref}.
    The belonging to $\textsf{AC}^0$ follows from the facts that the number of reformulated queries does not depend on the ABox and that the computational complexity of computing SQL relational queries, \viz~a CQ, is in $\textsf{AC}^0$~\cite{Papadimitriou99}.
    The proof of $\textsf{NP-completeness}$ is the same as in~\cite[Theorem 44]{Calvanese07} (see also~\cite[Theorem 6]{Calvanese05}). In short, membership in $\textsf{NP}$ follows from the fact that, a 
    reformulated query can be computed in non-deterministic polynomial time by the $\mathtt{DefQueryRef}$ procedure \wrt~combined complexity, while $\textsf{NP-hardness}$ follows from $\textsf{NP-hardness}$ of CQ evaluation over relational databases.    
\end{proof}

\section{Query Answering for Defeasible \dllitehornH~under RC}\label{qahornh}

\nd We address here the query reformulation procedure for defeasible \dllitehornH.
The procedure follows essentially the same pattern as described for defeasible \dllitecoreH. To avoid tedious replications of definitions and properties described for defeasible \dllitecoreH, we provide here a sketch description only.

To start with, the \emph{chase} for a defeasible \dllitehornH~KB is built as for defeasible \dllitecoreH, except that now we use the rules $(\mathbf{\ChH_i})$ in place of the rules $(\mathbf{\CcH_i})$ ($1\leq i \leq 5$), with 


{
\begin{description}
    \item[($\mathbf{\ChH_4}$)] if 
    $\cass{a}{B_k}$ occurs in $\dchasehH_i(\KB)$ (for all $1\leq k \leq n$),
    $h_\KB(a) = j$,\footnote{To compute the rank of an individual, see~\ref{ranindqref}.}
    $\{B_1, \ldots, B_n\} \dsubs A \in \D$ with $\rank_{\KB}(\{B_1, \ldots, B_n\})=j$,\footnote{That is, $\{B_1, \ldots, B_n\} \dsubs A \in \D_j$.}  
    and $\cass{a}{A}$ does not occur in $\dchasehH_i(\KB)$ (condition $f$), then let 
    \begin{equation} \label{dchasehdefi}
    \dchasehH_{i+1}(\KB) = \dchasehH_i(\KB)\cup\{\cass{a}{A}\} 
    \end{equation}
    \nd (condition $f_{new}$);




     \item[($\mathbf{\ChH_5}$)] if $b$ is the new individual created by the application of the ($\mathbf{\ChH_2}$) rule to 
    $\{B_1, \ldots, B_n\} \subs \exists R \in \T$, then update $h_\KB$ with  $h_\KB(b) = n+1$, where $n$ is the highest rank of DIs in $\D$.

\end{description}
}

\nd The definition of chase and that of the ranked model $\RI_{\KB}$ is similar as for defeasible \dllitecoreH. 
Like for defeasible \dllitecoreH~(see procedure $\mathtt{RationalClosure.Core}$), we further need to take into account that we may require to compute the rank of a conjunction $\{B_1, \ldots, B_n\}$ of basic atoms $B_i$. But, this again can be done by executing lines 8 - 13 of the $\mathtt{RationalClosure.Horn}$ procedure. Procedure~$\mathtt{\ref{RankalgH}}$ illustrates how to compute the ranking of 
$\{B_1, \ldots, B_n\}$.


\begin{procedure}[h]
\caption{Ranking.Horn($\KB,\{B_1, \ldots, B_n\}$)} \label{RankalgH}
{
\KwIn{Defeasible \dllitehornH~KB $\KB$ and a set of basic concepts  $B_i$}
\KwOut{Rank value of $B_1 \andc \ldots \andc B_n$}
  $\tuple{\tuple{\T^*,\D^*},\{\D_0,\ldots,\D_n\}}$ := $\mathtt{ComputeRanking.Horn}$($\KB$)\;
  //Compute $B_1 \andc \ldots \andc B_n$'s rank $i$\;
  $i$ :=  $0$; $\D_\R$ :=  $\D^*$\;
    $\D_{\delta_0}:=\{\{B'_1,\ldots, B'_m,\delta_0\}\subs A'\mid \{B'_1,\ldots, B'_m\}\dsubs A'\in\D_\R\}$, where  $\delta_0$ is a new concept name\;
  \While{$\T^*\cup\D_{\delta_i}\entails \{B_1,\ldots, B_n,\delta_i\}\subs \bot$ {\bf and} $\D_\R\neq\emptyset$}{
    $\D_\R$ := $\D_\R\backslash\D_{i}$; $i$ := $i + 1$\;
    $\D_{\delta_i}:=\{\{B'_1,\ldots, B'_m,\delta_i\}\subs A'\mid \{B'_1,\ldots, B'_m\}\dsubs A'\in\D_\R\}$, where  $\delta_i$ is a new concept name\;
  }
  \If{$\D_\R \neq \emptyset$}{\Return{$i$}}
  \eIf{$\D_\R = \emptyset$ {\bf and } $\T^*\cup\D_{\delta_i} \not\entails \{B_1, \ldots, B_n,\delta_i\} \subs \bot$}{\Return{$n+1$}}{\Return{$\infty$}}
  }
\end{procedure}

\nd Equations~\ref{dchasecdef} - \ref{dcanmodB} are  adapted to defeasible \dllitehornH~in the obvious way.

Then,  it not difficult to see that 

\begin{appproposition}\label{propsumdefchase}
The defeasible \dllitehornH~analogues of Propositions~\ref{chasecordef}, \ref{satdedllitecore}, 
\ref{calvthm29} and \ref{calvthm31} hold as well.
\end{appproposition}

\nd We next address the query reformulation procedure for defeasible \dllitehornH.
Once more, Definition~\ref{defapplicable} can easily be adapted to defeasible \dllitehornH~by starting with Definition~\ref{defapplicable} and adapt it accordingly by considering the definition of applicability and reformulation illustrated in Section~\ref{sect_query_rewiteDIs}. So, for instance, Case 6. in Definition~\ref{defapplicable} becomes:

\begin{itemize}
    \item[5.]  {\bf IF} $g=A(z^r)$  \textbf{AND} $\spec = \{B_1, \ldots, B_{n} \}  \dsubs A$ with $\rank_{\KB}(\{B_1, \ldots, B_{n} \} ) = r$
    \textbf{THEN} 
    $gr(g, \spec) = C_1(z^r) \land \ldots \land C_n(z^r)$, where for each $i \in \{1, \ldots, n\}$,
    \begin{itemize}
        \item $C_i(z^r) =A_i(z^r)$ if $B_i = A_i$, or
        \item $C_i(z^r) = P_i(z^r,y')$ if $B_i = \csome P_i$, or
        \item $C_i(z^r) = P_i(y',z^r)$ if $B_i = \csome P_i^-$ 
    \end{itemize}
    
    \nd and $y'$ is a new unbound variable.
    

    
\end{itemize}

\nd Then,  similarly to defeasible \dllitecoreH, one may prove that

\begin{appproposition}\label{propsumdefchaseB}
The defeasible \dllitehornH~analogues of Propositions~\ref{termdllitecore}, \ref{mainqadllite} and  Proposition~\ref{calvthm40} hold as well.
\end{appproposition}


\section{Ranking of Individuals}\label{ranindqref}

\nd The $\mathtt{ComputeRankIndividuals}$ procedure below, computes the ranking of the individuals occurring in a defeasible \dllitecoreH~KB $\KB$, which we comment next.
In steps 1 ans 2, we initialise the set of all individuals occurring in the ABox of a KB. In the while loop 4-11, we determine the individuals' rank. To do so, we compute the exceptional individuals for each rank $i \leq n$. Specifically, in step 5, we initialise the exceptional individuals at rank $i$. Then, at step 6, we determine the TBox resulting from the original TBox by including those DIs that have rank greater than or equal to $i$ and compute the NI-closure of it. In the loop 7 - 9, we compute the set of exceptional individuals of rank $i$. To do so, note that an instance $a$ of both $B_1$ and $B_2$, for a NI $\norm{B_1}\subs\neg \norm{B_2}\in \norm{rank_i}$, is considered exceptional at rank $i$ and, thus, can not have rank $i$. Also note that $q$ is just an SQL query to the underlying database $\mathtt{DB}(\A)$, where $\bar{\gamma}$ is defined as in Eq.~\ref{bargamma}.

\begin{procedure}
\caption{ComputeRankIndividuals($\KB,\RR$)}\label{proc_rank_individual}
{
\KwIn{A defeasible satisfiable \dllitecoreH~KB  $\KB = \tuple{\T,\D,\A}$ and its associated ranking $\RR = 
\{\D_0,\ldots,\D_n\}$}
\KwOut{Ranking of the individuals occurring in $\KB$}
$indsA : = \{a \mid \text{ individual $a$ occurs in } \KB \}$\;
$inds : = indsA$\;
$i : = 0$\;
\ForEach{$i \leq n$}
{
$except_i : = \emptyset$\;
 $\norm{rank_i} = \clnn(\norm{\T}\cup\bigcup_{j=i}^n\norm{\D_j}$)\;
 \ForEach{$\tau = \norm{B_1}\subs\neg \norm{B_2}\in \norm{rank_i}$}{  
$q_\tau : = q_\tau(x)\leftarrow \bar{\gamma}(B_1)(x),\bar{\gamma}(B_2)(x)$\;
$except_i \leftarrow except_i \cup ans(\A,q_\tau)$\;
}

$rinds_i : = inds\setminus except_i$\;
$inds : = except_i$\;
}
$rinds_{n+1} : = inds$\;
\Return{$\{rinds_0,\ldots,rinds_{n}, rinds_{n+1}\}$} 
  }
\end{procedure}



The following example illustrates how the individuals' ranking is computed.

\begin{example}\label{exrankindcore}
    Consider the defeasible \dllitecoreH~KB $\KB  = \tuple{\T,\D,A}$, where
    \begin{eqnarray*}
    \T & =  & \{\bar{F} \subs \neg F, P \subs B\} \\ 
    \D & = & \{B \dsubs F, P \dsubs \bar{F} \}\\ 
    \A & = & \{\cass{a}{B}, \cass{b}{P}, \cass{c}{B}, \cass{c}{\bar{F}}, \cass{d}{P}, \cass{d}{F} \} \ .
\end{eqnarray*}

\nd It can be verified that the ranks of DIs, concept names and individuals are:
\begin{eqnarray*}
    \D_0 & = & \{B \dsubs F \} \\ 
    \D_1 & = & \{P \dsubs \bar{F}  \} \\ \\
    \rank_{\KB}(B) & = & 0 =  \rank_{\KB}(F) = \rank_{\KB}(\bar{F}) \\
    \rank_{\KB}(P) & = & 1 \\ \\
    h_{\KB}(a) & = & 0 \\
    h_{\KB}(b) & = & 1 =  h_{\KB}(c) \\
    h_{\KB}(d) & = & 2 \ .
\end{eqnarray*}

\nd It can be verified that the following table illustrates the content of the various sets during the iterations.

\[
\begin{array}{|c|c|c|c|} \hline
0\leq i \leq n = 1 & inds & except_i & rinds_i \\ \hline
0 & \{b,c,d\} & \{b,c,d \} & \{a\} \\ \hline    
1 & \{d\} & \{d \} & \{b,c\}  \\ \hline    
n+1=2 & \{d\} & - & \{d\} \\ \hline    
\end{array}
\]
    
\qed 
\end{example}

\nd We now prove that the above procedure is correct.


\begin{appproposition} \label{rankindsproc}
    Consider a satisfiable defeasible \dllitecoreH~KB  $\KB = \tuple{\T,\D,\A}$ and its associated ranking $\RR = 
\{\D_0,\ldots,\D_n\}$. Let $res= \{rinds_0$, \ldots, $rinds_{n},$ $rinds_{n+1}\}$ be the output of $\mathtt{ComputeRankIndividuals}(\KB,\RR)$.
Then, for all $0 \leq i \leq n+1$ and all individuals $a$ occurring in $\KB$, the following holds: $a \in rinds_{i}$ iff $h_\KB(a) = i$.
\end{appproposition}
\begin{proof}
    As $\KB$ is satisfiable, by Proposition~\ref{prop_unique_abox}, $\KB$ has a model $\RI$ in which every individual $a$ occurring in $\KB$ is minimally interpreted \wrt~$\KB$, \ie~$h_{\RI}(a^{\RI})=h_{\KB}(a)$. For a non-empty set $\mathtt{I}$ of individuals occurring in $\A$, with $\A_{\mathtt{I}}$ we denote the set of all assertions of $\A$ that involve individuals occurring in $\mathtt{I}$. Moreover, for a set of assertions $\A'$, let 
    $\norm{\A'} = \A' \cup \{\cass{a}{\norm{A}} \mid \cass{a}{A} \in \A' \}$.
    Now, let us note that for all $0 \leq i \leq n$ we have that by construction 
    \begin{enumerate}
    \item $\norm{rank_i} \supseteq \norm{rank_{i+1}}$;
    \item $res$ is a partition of the individuals occurring in $\KB$, \ie~$indsA = \bigcup_{i=0}^{n+1} rinds_{i}$ and for all $0\leq i<j\leq n+1$,
    $rinds_{i} \cap rinds_{j} = \emptyset$.
    \end{enumerate}

    \nd We now prove the proposition by induction on $0 \leq i \leq n+1$.

    \begin{description}
        \item[Case $i=0$] \mbox{ \ } \\
        $(\Rightarrow).$ Assume $a \in rinds_{0}$. Assume to the contrary that $h_\KB(a) = h_a > 0$. Therefore, 
        $h_{\RI}(a^{\RI}) = h_a > 0$. That is, by Definition~\ref{def_exceptional_indiv}, $a$ is exceptional \wrt~$\tuple{\norm{rank_{h_a}},\norm{\A_{rinds_0}}}$. As a consequence, 
        $\tuple{\norm{rank_{h_a}},\norm{\A_{rinds_0}}}$ has to be unsatisfiable. By Propositions~\ref{propCalvanese0607} and \ref{propCalvanese0607B}, this means that there must exist $\norm{B_1}\subs\neg \norm{B_2}\in \norm{rank_{h_a}} \subseteq \norm{rank_{0}}$ such that both $\cass{a}{\norm{B_1}}$ and $\cass{a}{\norm{B_2}}$ occur in $\norm{\A_{rinds_0}} \subseteq \norm{\A}$, \ie~$a\in ans(\A,q_\tau)$.        
        But then $a \in except_0$ and, thus, $a \not\in  rinds_{0}$, contrary to our assumption. 
        
        \vspace*{1ex}
        \nd $(\Leftarrow).$ Assume $h_\KB(a) = 0$. Assume to the contrary that $a \in rinds_{h_a}$ with $h_a > 0$. But then, there is 
        $\norm{B_1}\subs\neg \norm{B_2}\in \norm{rank_{h_a}} \subseteq \norm{rank_{0}}$ such that both $\cass{a}{\norm{B_1}}$ and $\cass{a}{\norm{B_2}}$ are in $\norm{\A}$. Therefore, $\tuple{\norm{rank_{0}},\norm{\A_{rinds_0}}}$ is unsatisfiable and, thus, it can not be the case that $\RI$ is a model of $\KB$ with $h_{\RI}(a^{\RI}) = h_\KB(a) = 0$.

        \item[Case $i+1 < n+1$] Let us assume by induction that for each $0 \leq j \leq i$ and each individual $a$ occurring in $\KB$, we have that 
        $a \in rinds_{j}$ iff $h_\KB(a) = j$. Let us prove that this also holds for $j=i+1$. The proof is similar to the  case $i=0$.

        \vspace*{1ex}
        \nd $(\Rightarrow).$ Assume $a \in rinds_{i+1}$. Therefore, it can not be the case that $a \in rinds_{j}$ for $0\leq j \leq i$ and, thus, by induction hypothesis it can not be the case $h_\KB(a) \leq i$, \ie~$h_\KB(a) \geq i+1$ has to be the case.    
        Assume to the contrary that $h_\KB(a) = h_a > i+1$.
        Therefore, $h_{\RI}(a^{\RI}) = h_a > i+1$. That is, by Definition~\ref{def_exceptional_indiv}, $a$ is exceptional \wrt~$\tuple{\norm{rank_{h_a}},\norm{\A_{rinds_{h_a}}}}$. As a consequence, $\tuple{\norm{rank_{h_a}},\norm{\A_{rinds_{h_a}}}}$ has to be unsatisfiable and by Propositions~\ref{propCalvanese0607} and ~\ref{propCalvanese0607B}, this means that there must exist $\norm{B_1}\subs\neg \norm{B_2}\in \norm{rank_{h_a}} \subseteq \norm{rank_{i+1}}$ such that both $\cass{a}{\norm{B_1}}$ and $\cass{a}{\norm{B_2}}$ are in $\norm{\A_{rinds_{h_a}}} \subseteq \norm{\A}$, \ie~$a\in ans(\A,q_\tau)$. But then $a \in except_{i+1}$ and, thus, $a \not\in  rinds_{i+1}$, contrary to our assumption. 

        \vspace*{1ex}
        \nd $(\Leftarrow).$ Assume $h_\KB(a) = i+1$. Therefore, by induction hypothesis, it can not be the case that $a \in rinds_{j}$ for $0\leq j \leq i$ (as otherwise $h_\KB(a) = j \leq i$) and, thus, $a \in rinds_{j}$ for $j \geq i+1$ has to be the case. Assume to the contrary that $a \in rinds_{h_a}$ with $h_a > i+1$. 
        But then, there is $\norm{A_1}\subs\neg \norm{A_2}\in \norm{rank_{h_a}} \subseteq \norm{rank_{i+1}}$ such that both $\cass{a}{\norm{A_1}}$ and $\cass{a}{\norm{A_2}}$ are in $\norm{\A_{rinds_{h_a}}} \subseteq \norm{\A}$. Therefore, $\tuple{\norm{rank_{i+1}},\A_{rinds_{h_a}}}$ is unsatisfiable and, thus, it can not be the case that $\RI$ is a model of $\KB$ with $h_{\RI}(a^{\RI}) = h_\KB(a) = i+1$.


    \item[Case $i+1 = n+1$] By induction on $i\leq n$,  $a \in rinds_{i}$  iff $h_\KB(a) = i$. Therefore,  $a \in rinds_{n+1}$  iff $h_\KB(a) = n+1$, which concludes the proof.
        
    \end{description}

\end{proof}

\nd To conclude this section, we illustrate how to adapt the $\mathtt{ComputeRankIndividuals}$ also to defeasible
\dllitehornH~KBs. To do so, the steps that we have to revise are steps 6, 7 and 8. Step 6 has to be replaced with step 6h
%
\begin{eqnarray*}
6h. && \norm{rank_i}  =  \clnhH(\T\cup\\
&& 
\bigcup_{j=i}^n \{\{B_1,\ldots, B_m,\delta\}\subs A \mid \{B_1,\ldots, B_m\}\dsubs A\in \D_j\})  \ ,
\end{eqnarray*}

\nd while steps 7-8 are replaced with steps 7h-8h (\cf~Propositions~\ref{propCalvanese0607}, \ref{propCalvanese0607B}):
\begin{eqnarray*}
7h. & & {\bf foreach \ } \tau = \{B_1, \ldots, B_n, \delta\} \subs\neg B_{n+1}\in \norm{rank_i}{\bf \ do}  \\
8h. & & \hspace*{0.7cm} q_\tau : = q_\tau(x)\leftarrow \bar{\gamma}(B_1)(x), \ldots, \bar{\gamma}(B_{n+1})(x) \ .
\end{eqnarray*}



\begin{appproposition} \label{rankindsprocH}
    Consider a satisfiable defeasible \dllitehornH~KB  $\KB = \tuple{\T,\D,\A}$ and its associated ranking $\RR = 
\{\D_0,\ldots,\D_n\}$. Let $res= \{rinds_0$, \ldots, $rinds_{n},$  $rinds_{n+1}\}$ be the output of $\mathtt{ComputeRankIndividuals}(\KB,\RR)$, revised according to the new steps 6h, 7h and 8h.
Then, for all $0 \leq i \leq n+1$ and all individuals $a$ occurring in $\KB$, the following holds: $a \in rinds_{i}$ iff $h_\KB(a) = i$.
\end{appproposition}
\begin{proof}
The proof is by induction following the same reasoning as for Proposition~\ref{rankindsproc}, by using Proposition~\ref{prop_unique_abox}, Definition~\ref{def_exceptional_indiv}, Propositions~\ref{propCalvanese0607}, \ref{propCalvanese0607B}, related to \dllitehornH~and the existence of inclusion axioms $\{B_1, \ldots, B_n, \delta\} \subs\neg B_{n+1}\in \norm{rank_i}$ in place of  
$\norm{B_1}\subs\neg \norm{B_2} \in \norm{rank_i}$.
\end{proof}


\nd This concludes this section.

\end{appendix}

\pagebreak 

\end{document}